\documentclass[a4paper, twoside]{report}
\pdfoutput=1

\usepackage[nottoc]{tocbibind}
\usepackage[round,authoryear,sort]{natbib}
\usepackage{nicefrac}
\usepackage{multirow}
\usepackage{graphicx}
\usepackage{subfig}
\usepackage{amsmath, amsfonts}
\usepackage{amsthm}
\usepackage{enumitem}
\usepackage{makeidx} \makeindex
\usepackage{bm}
\usepackage[a4paper, top=75.87401pt, bottom=4cm, hmarginratio=1:1]{geometry}
\usepackage{fancyhdr}
\usepackage{url}
\usepackage[hidelinks,linktoc=all]{hyperref}
\usepackage{booktabs}
\usepackage{framed}
\usepackage{tikz}
\usepackage{microtype}

\usetikzlibrary{arrows}
\usetikzlibrary{positioning}
\usetikzlibrary{calc}

\fancypagestyle{plain}{
  \fancyhf{} 
  
}

\numberwithin{equation}{chapter}
\numberwithin{figure}{chapter}
\numberwithin{table}{chapter}
\allowdisplaybreaks[1]

\newcommand{\vect}[1]{\mathbf{#1}}
\newcommand{\mat}[1]{\mathbf{#1}}
\newcommand{\br}[1]{\mathopen{}\left(#1\right)\mathclose{}}
\newcommand{\set}[1]{\left\{#1\right\}}
\newcommand{\pair}[2]{\br{#1,#2}}

\newcommand{\norm}[1]{\left\|#1\right\|}
\newcommand{\pderivnull}[1]{\frac{\partial}{\partial#1}}
\newcommand{\pderiv}[2]{\frac{\partial #1\hfill}{\partial #2\hfill}}
\newcommand{\spderiva}[2]{\frac{\partial^2 #1}{{\partial #2}^2}}
\newcommand{\spderivb}[3]{\frac{\partial^2 #1}{\partial #2\partial #3}}
\newcommand{\bigo}[1]{\mathcal{O}\br{#1}}
\newcommand{\prob}[1]{p\br{#1}}
\newcommand{\uprob}[1]{\bar{p}\br{#1}}
\newcommand{\Rop}[1]{\mathcal{R}\mathopen{}\left\{#1\right\}\mathclose{}}
\newcommand{\kl}[2]{D_{\mathrm{KL}}\br{#1\,\|\,#2}}
\newcommand{\avg}[2]{\left\langle#1\right\rangle_{#2}}
\newcommand{\avgnull}[1]{\left\langle#1\right\rangle}
\newcommand{\nth}[1]{#1^{\text{th}}}

\newcommand{\sigm}[1]{\sigma\br{#1}}
\newcommand{\gaussian}[3]{\mathcal{N}\br{#1\g #2,#3}}
\newcommand{\g}{\,|\,}
\newcommand{\deltaf}[1]{\delta\br{#1}}
\newcommand{\arrstretchvalue}{1.3}
\newcommand{\circled}[1]{\tikz[baseline=(char.base)]{\node[shape=circle,draw,inner sep=1.5pt] (char) {#1};}}

\newtheorem{proposition}{Proposition}[chapter]

\begin{document}

\begin{titlepage}

\newcommand{\HRule}{\rule{\linewidth}{0.5mm}}

\center
\textsc{\LARGE University of Edinburgh}\\[1.5cm] 
\textsc{\Large School of Informatics}\\[0.5cm] 
\textsc{\large MSc by Research in Data Science}\\[3.5cm]

\HRule \\[0.6cm]
{ \huge \bfseries Distilling Model Knowledge}\\[0.4cm]
\HRule \\[1.5cm]
 
{\Large by\\[0.4cm] \textsc{George Papamakarios}}\\[10cm] 

{\large August 2015}\\[3cm] 

\vfill 

\end{titlepage}

\newcounter{savepage}
\setcounter{savepage}{\thepage}
\begin{abstract}

Top-performing machine learning systems, such as deep neural networks, large ensembles and complex probabilistic graphical models, can be expensive to store, slow to evaluate and hard to integrate into larger systems. Ideally, we would like to replace such cumbersome models with simpler models that perform equally well.

In this thesis, we study knowledge distillation, the idea of extracting the knowledge contained in a complex model and injecting it into a more convenient model. We present a general framework for knowledge distillation, whereby a convenient model of our choosing learns how to mimic a complex model, by observing the latter's behaviour and being penalized whenever it fails to reproduce it.

We develop our framework within the context of three distinct machine learning applications: (a) model compression, where we compress large discriminative models, such as ensembles of neural networks, into models of much smaller size; (b) compact predictive distributions for Bayesian inference, where we distil large bags of MCMC samples into compact predictive distributions in closed form; (c) intractable generative models, where we distil unnormalizable models such as RBMs into tractable models such as NADEs.

We contribute to the state of the art with novel techniques and ideas. In model compression, we describe and implement derivative matching, which allows for better distillation when data is scarce. In compact predictive distributions, we introduce online distillation, which allows for significant savings in memory. Finally, in intractable generative models, we show how to use distilled models to robustly estimate intractable quantities of the original model, such as its intractable partition function.

\end{abstract}
\setcounter{page}{\thesavepage}
\stepcounter{page}

\renewcommand{\abstractname}{Acknowledgements}
\setcounter{savepage}{\thepage}
\begin{abstract}

This work was supervised by Dr Iain Murray. My long discussions with Iain have considerably expanded my knowledge of machine learning and his detailed feedback has greatly improved my writing. I would not have known how to recognize bad kerning had it not been for him. I consider myself fortunate to have been in the privileged group of individuals that have been supervised by Iain.

I am grateful to all fellow PhD students in the Institute for Adaptive and Neural Computation and the Centre for Doctoral Training in Data Science with whom I have had numerous discussions about machine learning (and beyond!). Special thanks go to Harri Edwards and Krzysztof Geras for regularly pointing me to good papers.

None of this would have existed had it not been for the directors of the Centre for Doctoral Training in Data Science, Prof.~Chris Williams and Dr Charles Sutton. Many thanks to them for our regular interactions and making everything possible.

This work was supported in part by the EPSRC Centre for Doctoral Training in Data Science, funded by the UK Engineering and Physical Sciences Research Council (grant EP/L016427/1) and the University of Edinburgh, and by Microsoft Research through its PhD Scholarship Programme.

\end{abstract}
\setcounter{page}{\thesavepage}
\stepcounter{page}

\tableofcontents

\chapter{Introduction}
\label{chapter:introduction}

Recent progress in machine learning\index{Machine learning} has been phenomenal. Record after record is being broken, as models are becoming increasingly more accurate and better performing \citep{Russakovsky:2015:imagenet, Bojar:2014:workshop_MT}. This has been made possible by a series of theoretic and algorithmic breakthroughs in the last decade, and further facilitated by significant technological advances in computer hardware and software, the explosion of available data, and a renewed commercial interest in the field.

The above improvements have been accompanied by a notable increase in model size and complexity. Large ensembles consisting of several individual models, deep neural architectures with complex multilayer structure, large-scale probabilistic graphical models with complicated interactions between variables, and sophisticated generative models of data are only some examples of this ongoing trend. Clearly, increased model size and complexity has played a major role in the success of machine learning models.

Nevertheless, a large and complex model can give rise to significant practical challenges. Deep neural networks with several hidden layers can be slow to evaluate and expensive to store. Inference in large, hierarchical probabilistic models can be computationally demanding, even when done approximately. Adding sophisticated structure to generative models may compromise their tractability, and as a consequence reduce their interpretability and make them harder to integrate into larger systems. Size and complexity may present models with power, but they can also make them cumbersome and inconvenient to use.

Let us suppose for a moment that we have access to an accurate and powerful model, which however is cumbersome and inconvenient to use. Imagine that we could replace it with a convenient model that does the same job and is just as good. From then on, for all practical purposes, we could just get rid of the cumbersome model and use the convenient model instead. The question then becomes: \emph{using the cumbersome model, can we actually construct a convenient model that does the same job and is just as good}?

In this thesis, we address the above question, and present a framework for constructing such convenient models from cumbersome ones. We call this procedure \emph{knowledge distillation}\index{Knowledge distillation}; as if we could extract the knowledge hidden inside the cumbersome model, distil it, and inject it into the convenient model. From a high-level perspective, our framework is based on training the convenient model to mimic the cumbersome model, by getting it to observe its behaviour, and punish it when it fails to replicate it successfully.

Our work brings together ideas often referred to in the literature as knowledge distillation \citep{Hinton:2015:distilling_knowledge}, model compression \citep{Bucila:2006:model_compression}, compact approximations \citep{Snelson:2005:compact_approximations}, mimicking models \citep{Ba:2014:deep_nets_need_to_be_deep}, and teacher-student models \citep{Romero:2014:fitnets}. As a technical term, ``knowledge distillation'' was introduced only recently by \citet{Hinton:2015:distilling_knowledge}, who used it to mean model compression. Here, we use it more broadly to refer to the act of transferring the knowledge of a cumbersome model into a convenient one. We believe that this usage of the term is more appropriate, since it permits a unifying description of tasks that are not necessarily limited to model compression.

The above description of knowledge distillation is very general; the words ``cumbersome'' and ``convenient'' can mean different things in different contexts. For instance, for a deep neural network, cumbersome might mean ``too large'', in which case a convenient model would be a small neural network. On the other hand, for a generative model, cumbersome might mean ``intractable'', in which case a convenient model would be a tractable generative model. 

In this work, we apply the aforementioned idea of knowledge distillation into three fairly general---and highly important---fields of machine learning: \emph{discriminative models}, \emph{Bayesian inference} and \emph{generative models}. For each one, we define the meaning of ``cumbersome'' and ``convenient'', exemplify our knowledge distillation framework, and apply it to solve an example problem within the field.

The rest of the thesis is organized into three chapters, one for each field in which we apply the knowledge distillation framework. Each chapter is standalone; that is, it can be read independently and does not require knowledge of the others. The following three sections provide an overview of each chapter and list our individual contributions in each one.

\section{Model compression}
\index{Model compression}

In chapter~\ref{chapter:model_compression}, we apply knowledge distillation on discriminative models. In this particular context, knowledge distillation becomes synonymous with model compression. That is, the cumbersome model is a large discriminative model, such as an ensemble or deep neural networks, which, due to its size, is expensive to store and slow to evaluate. The objective is to distil it into a smaller model that would be cheaper to store and evaluate, but would not seriously compromise accuracy and performance.
Our particular contributions in this chapter are listed below.
\begin{enumerate}[label=(\roman*)]
\item We present a generic and flexible framework for model compression, which is viewed as a special case of knowledge distillation. Our framework is general enough to be applied to a wide range of models, and it extends the work of \citet{Bucila:2006:model_compression}.
\item We successfully apply the above framework in compressing an ensemble of neural networks into a single small neural network.
\item We make a comprehensive review of the model compression literature and compare it to our approach.
\item We extend the model compression literature in two distinct ways. Firstly, we show how to efficiently use derivative information in the distillation procedure. Secondly, we make use of generative models, such as NADE \citep{Larochelle:2011:NADE}, for providing input data for distillation. We show that both approaches lead to better generalization performance when the input data is limited.
\end{enumerate}

\section{Compact predictive distributions}

In chapter~\ref{chapter:compact_predictive_distributions}, we apply knowledge distillation on approximate Bayesian inference with Markov Chain Monte Carlo \citep{Neal:1993:mcmc}. Here, the cumbersome model is a Markov chain that explores a posterior distribution and generates MCMC samples from it. Averaging over those samples approximates the true predictive distribution. If high accuracy is required, then a large number of MCMC samples needs to be generated and stored, which is an inefficient and wasteful way to represent the predictive distribution. The convenient model that is distilled from the Markov chain is a compact closed-form approximation to the true predictive distribution, which it is cheaper to store and evaluate compared to a bag of MCMC samples. Our contributions in this chapter are listed below.
\begin{enumerate}[label=(\roman*)]
\item We present a framework for learning compact predictive distributions from Markov chains, which extends that of \citet{Snelson:2005:compact_approximations}.
\item We apply our framework to the problems of Bayesian density estimation and Bayesian binary classification. We showcase it in the special cases of Bayesian Mixtures of Gaussians and Bayesian logistic regression respectively.
\item We show how distillation in our framework can be performed in an online fashion, that is, while the Markov chain is being simulated. This leads to a form of online  distillation that performs as well as batch distillation, but requires only a fraction of the memory.
\item We show how to efficiently include derivatives within the distillation procedure. This way, we manage to achieve better distillation, by conveying more information about the true predictive distribution to the compact one.
\end{enumerate}

\section{Intractable generative models}

In chapter~\ref{chapter:generative_models}, we focus on generative models. Here, the cumbersome model is a generative model that cannot be normalized, hence tasks such as calculating its partition function, evaluating its probabilities, marginalizing and conditioning are all intractable. The task is to distil it into a tractable model, for which doing all the above would be tractable. This way, the convenient tractable model can be used within a larger system in ways that the cumbersome intractable model could not. Our particular contributions in this chapter are listed below.
\begin{enumerate}[label=(\roman*)]
\item We present a framework for distilling intractable generative models into tractable ones. Our framework manages to extract the knowledge contained in the intractable model, bypassing the need to calculate its intractable partition function.
\item We apply our framework to successfully distil an intractable Restricted Boltzmann Machine \citep{Smolensky:1986:RBM} into a tractable NADE\@.
\item We show that the distilled NADE can be successfully used to produce robust estimates of the intractable partition function of the RBM, that are comparable to the state of the art estimates provided by \citet{Salakhutdinov:2008:quantitative_analysis_of_DBNs}.
\end{enumerate}

\chapter{Model Compression}
\label{chapter:model_compression}

Discriminative models\index{Discriminative model} play a dominant role in machine learning. From a high-level point of view, these are models that take as input some datapoint $\vect{x}$ and output some prediction $\vect{y}$ that is a function $t$ of $\vect{x}$, that is
\begin{equation}
\vect{y} = t\br{\vect{x}}.
\end{equation}
Variable $\vect{x}$ can represent any sort of data, such as speech, images or video. In the simplest case, variable $\vect{y}$ can be a numerical prediction (in which case the model is used for regression), or a categorical prediction (in which case the model is used for classification).

Increasing the capacity of a discriminative model typically makes it capable of learning more complex functions. Consider for example the case of \emph{neural networks}\index{Neural network}; in the last decade, neural networks have significantly increased in size and have become deeper. This has been made possible by having access to bigger datasets, more powerful computers (such as GPUs) and the invention of strong regularization methods, such as dropout\index{Dropout} \citep{Srivastava:2014:dropout}. As a result, deep neural networks have produced state of the art results in many applications, such as automatic speech recognition \citep{Graves:2013:ASR}, machine translation \citep{Sutskever:2014:MT, Bahdanau:2015:neural_machine_translation} and image classification \citep{Ciresan:2012:mnist, Szegedy:2015:googlenet}.

A different way of improving discriminative models is to combine several of them into a single model, which outputs in some sense the ``average'' or the ``consensus'' across models. This sort of model is known as an \emph{ensemble}\index{Ensemble}. The intuition why this works is that each separate model might make errors on different inputs, but by combining predictions from various models these errors tend to be overruled, hence the final predictions tend to be better. A prominent example of a model that owes its success to being an ensemble is the \emph{random forest}\index{Random forest} \citep{Breiman:2001:random_forests}. In a random forest, a large number of decision trees are combined together, each of which is a weak model but altogether make a strong model. In general, the individual models need not be of the same class, and forming ensembles can also be done by combining fundamentally different kinds of models \citep{Caruana:2004:ESL}.

Increasing the capacity of a single model or combining several models into an ensemble typically improves predictions; however it comes at a cost. The final model might be too large to store, or too slow to evaluate at test time. It is natural then to ask the question, is the extra size really necessary, or could the same function in principle be represented by a simpler, smaller model? If this is indeed the case, could we then distil the knowledge of the large model into a smaller model that does the same job?

Motivated by the above questions, the field of \emph{model compression}\index{Model compression} emerged, pioneered by \citet{Bucila:2006:model_compression}. Having a accurate but large model $t\br{\vect{x}}$, the objective of model compression is to train a smaller model $f\br{\vect{x}\g\bm{\theta}}$ to mimic $t\br{\vect{x}}$ as closely as possible. In this chapter, we study model compression as a special case of knowledge distillation. We describe a general framework for model compression, and we study the effect of various techniques within the framework. We contribute to the state of the art with two novel ideas: matching derivatives and guiding the distillation process with a generative model of the input data. Finally, we showcase our framework by compressing an ensemble of large neural networks into a single small neural network.

\section{A framework for model compression}
\label{sec:model_compression:framework}

Suppose we are given a discriminative model which computes the function $\vect{y} = t\br{\vect{x}}$. We shall refer to this model as the ``teacher''. We then specify a family of models $f\br{\vect{x}\g\bm{\theta}}$, parameterized by a set of parameters $\bm{\theta}$, which we shall refer to as the ``student''. Our goal is to train the student to be as similar as possible to the teacher; that is, we want to find the value of $\bm{\theta}$ for which the approximation
\begin{equation}
t\br{\vect{x}} \approx f\br{\vect{x}\g\bm{\theta}}
\end{equation}
becomes as close to an equality as possible.

In the context of model compression, $t\br{\vect{x}}$ is typically significantly more complex than $f\br{\vect{x}\g\bm{\theta}}$, therefore, in general, we cannot expect the student to match the teacher everywhere. The hope instead is that the particular function that is represented by $t\br{\vect{x}}$ is simple enough to be closely approximated by $f\br{\vect{x}\g\bm{\theta}}$ at least in the region of space we are interested in. In practical applications, the region of interest may be tiny or significantly lower-dimensional compared to the whole space. For instance, if $\vect{x}$ represents greyscale images, and $t\br{\vect{x}}$ is meant to be used to classify images of handwritten digits, then we are indeed interested only in a minuscule fraction of all greyscale images imaginable.

Our framework for model compression, besides the teacher $t\br{\vect{x}}$ and the student $f\br{\vect{x}\g\bm{\theta}}$, also involves a data generator $\prob{\vect{x}}$ and a loss $E\br{\vect{x},\bm{\theta}}$. The data generator is a probability distribution over the input space; in practice it can be any procedure which can generate datapoints on demand. The loss $E\br{\vect{x},\bm{\theta}}$ is a measure of discrepancy between the teacher and the student at input $\vect{x}$ and for parameters $\bm{\theta}$. The training objective is to minimize the average loss with respect to $\bm{\theta}$, which can be written as
\begin{equation}
\min_{\bm{\theta}}{\avg{E\br{\vect{x},\bm{\theta}}}{\prob{\vect{x}}}}.
\end{equation}

The choice of the data generator and of the loss is highly important. The data generator specifies the region of the input space we are interested in. Agreement between the teacher and the student on datapoints $\vect{x}$ for which $\prob{\vect{x}}$ is high is prioritized, since the loss is weighted by $\prob{\vect{x}}$ in the training objective. On the other hand, the loss specifies what it means for the student to agree with the teacher. In our case study in section~\ref{sec:model_compression:case_study}, we will study the effect of different choices of data generator and loss, when compressing an ensemble of neural networks. In the following two sections, we consider two broad families of losses, and describe a stochastic way of minimizing them.

\subsection{Matching function values}
\label{sec:model_compression:framework:matching_values}

The simplest choice for the loss is a function that directly measures the difference between the predictions of the teacher $t\br{\vect{x}}$ and the predictions of the student $f\br{\vect{x}\g\bm{\theta}}$ at input $\vect{x}$. For a regression task, such a loss can be the square error\index{Square error}, defined by
\begin{equation}
E_\mathrm{SE}\br{\vect{x},\bm{\theta}} = \frac{1}{2}\norm{f\br{\vect{x}\g\bm{\theta}} - t\br{\vect{x}}}^2.
\end{equation}
For a classification task, where the outputs of both models are discrete probability distributions with $I$ elements, a suitable loss is the cross entropy\index{Cross entropy}, defined by
\begin{equation}
E_\mathrm{CE}\br{\vect{x},\bm{\theta}} = -\sum_i{t_i\br{\vect{x}}\log{f_i\br{\vect{x}\g\bm{\theta}}}}.
\end{equation}
Such losses only involve function values $t\br{\vect{x}}$ and $f\br{\vect{x}\g\bm{\theta}}$, and they encourage those values to be as similar as possible.

We will now describe a stochastic gradient\index{Stochastic gradient} training procedure for minimizing the average loss $E\br{\bm{\theta}} = \avg{E\br{\vect{x},\bm{\theta}}}{\prob{\vect{x}}}$, when the loss is of the above type. The gradient of $E\br{\bm{\theta}}$ with respect to the parameters $\bm{\theta}$ is
\begin{equation}
\vect{g}\br{\bm{\theta}} = \avg{\pderivnull{\bm{\theta}} E\br{\vect{x},\bm{\theta}}}{\prob{\vect{x}}}
\end{equation}
The above gradient can be stochastically approximated by
\begin{equation}
\hat{\vect{g}}\br{\bm{\theta}} = \frac{1}{S}\sum_{s}{\pderivnull{\bm{\theta}} E\br{\vect{x}_s,\bm{\theta}}},
\end{equation}
where $\set{\vect{x}_s}$ are samples generated from $\prob{\vect{x}}$. It is easy to see that, in expectation, the stochastic gradient is equal to the original gradient. Using the above stochastic gradient, we can minimize $E\br{\bm{\theta}}$ using the following procedure, which is  shown pictorially in Figure~\ref{fig:model_compression:overview}.
\begin{framed}
\begin{enumerate}[label=(\roman*)]
\item Generate a minibatch $\set{\vect{x}_s}$ of size $S$ from $\prob{\vect{x}}$.
\item Calculate $t\br{\vect{x}_s}$ and $f\br{\vect{x}_s\g\bm{\theta}}$ for all datapoints in the minibatch.
\item\label{model_compression:matching_values:step3} Calculate the stochastic gradient $\hat{\vect{g}}\br{\bm{\theta}} = \frac{1}{S}\sum_{s}{\pderivnull{\bm{\theta}} E\br{\vect{x}_s,\bm{\theta}}}$.
\item Make an update on $\bm{\theta}$ using $\hat{\vect{g}}\br{\bm{\theta}}$.
\item Repeat until convergence.
\end{enumerate}
\end{framed}
\noindent The above algorithm falls into the category of stochastic optimization methods. Provided the samples generated from $\prob{\vect{x}}$ are exact and independent, and under certain regularity conditions, it is guaranteed to converge almost surely to a stationary point of $E\br{\bm{\theta}}$ \citep{Bottou:1999:online_learning}.

The above procedure is fairly general and assumes very little about $\prob{\vect{x}}$ and $f\br{\vect{x}\g\bm{\theta}}$. Notice that we do not need to be able to evaluate $\prob{\vect{x}}$, but only to be able to sample from it. Hence $\prob{\vect{x}}$ can be any data generating procedure, such as resampling a dataset. On the other hand, the only thing we assume for $f\br{\vect{x}\g\bm{\theta}}$ is that it must be differentiable with respect to $\bm{\theta}$, so as to be able to calculate the gradient in step~\ref{model_compression:matching_values:step3}. If $f\br{\vect{x}\g\bm{\theta}}$ is a neural network, this can be done by backward propagation\index{Backward propagation}, as described in more detail in appendix~\ref{chapter:derivatives_in_neural_networks}.

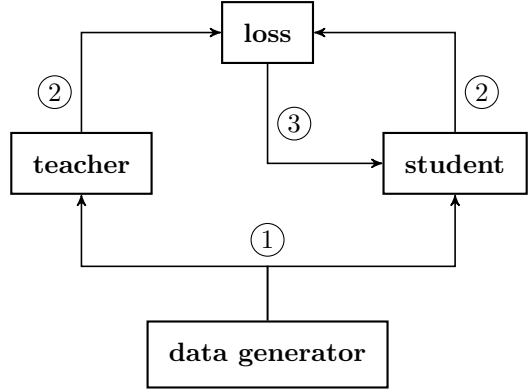
\begin{figure}[tbp]
\centering
\begin{tikzpicture}[>=stealth']

\tikzstyle{block} = [draw, rectangle, thick, inner sep=8pt, color=black]

\node[block] (loss) {\textbf{loss}};
\node[block, below left=0.9cm and 0.9cm of loss] (teacher) {\textbf{teacher}}; 
\node[block, below right=0.9cm and 0.9cm of loss] (student) {\textbf{student}}; 
\node[block, below=3.4cm of loss] (generator) {\textbf{data generator}};

\draw [->, semithick] (teacher) |- node [left,  pos=0.2] {\circled{2}} (loss);
\draw [->, semithick] (student) |- node [right,  pos=0.2] {\circled{2}} (loss);
\draw [->, semithick] (loss) |- node [right,  pos=0.3] {\circled{3}} (student);
\draw [->, semithick] (generator) [above,  pos=0] -- (0,-3.1cm) -| node {\circled{1}} (student);
\draw [->, semithick] (generator) -- (0,-3.1cm) -| node {} (teacher);

\end{tikzpicture}
\caption{Overview of stochastic gradient training for model compression. The generator produces datapoints \protect\circled{1} on which both the teacher and the student are evaluated \protect\circled{2}. The loss measures the discrepancy between the predictions of the teacher and the student, and its gradient with respect to the student's parameters provides training signal for the student \protect\circled{3}.}
\label{fig:model_compression:overview}
\end{figure}

\subsection{Matching function derivatives}
\label{sec:model_compression:framework:matching_derivatives}

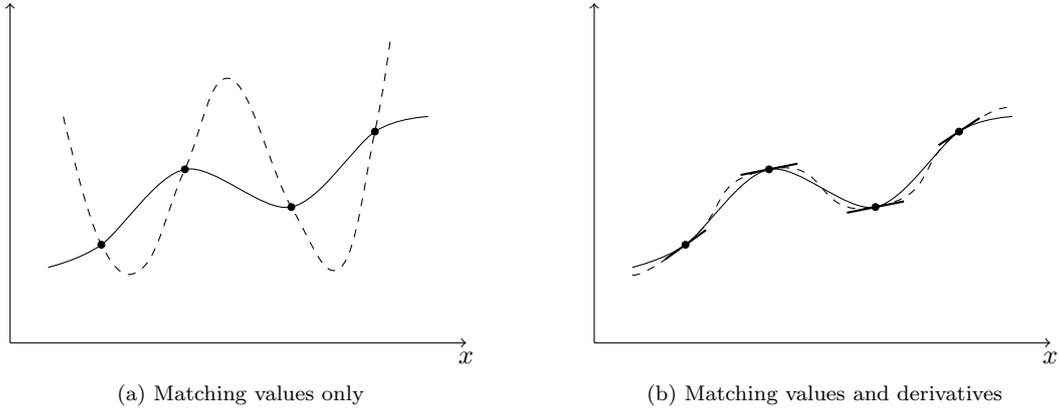
\begin{figure}[tbp]
\centering

\subfloat[Matching values only]{
\begin{tikzpicture}

\draw[->] (0,0) -- (6,0) node[anchor=north] {$x$};
\draw[->] (0,0) -- (0,4.5) node {};

\pgfmathsetmacro{\xa}{1.2}; \pgfmathsetmacro{\ya}{1.3};
\pgfmathsetmacro{\xb}{2.3}; \pgfmathsetmacro{\yb}{2.3};
\pgfmathsetmacro{\xc}{3.7}; \pgfmathsetmacro{\yc}{1.8};
\pgfmathsetmacro{\xd}{4.8}; \pgfmathsetmacro{\yd}{2.8};
\pgfmathsetmacro{\radius}{0.045};

\draw[fill] (\xa,\ya) circle (\radius);
\draw[fill] (\xb,\yb) circle (\radius);
\draw[fill] (\xc,\yc) circle (\radius);
\draw[fill] (\xd,\yd) circle (\radius);

\pgfmathsetmacro{\xbeg}{0.5}; \pgfmathsetmacro{\ybeg}{1};
\pgfmathsetmacro{\xend}{5.5}; \pgfmathsetmacro{\yend}{3};

\draw plot [smooth, tension=0.55] coordinates {(\xbeg,\ybeg) (\xa,\ya) (\xb,\yb) (\xc,\yc) (\xd,\yd) (\xend,\yend)};

\draw[dashed] plot [smooth, tension=0.7] coordinates {(0.7,3) (\xa,\ya) (1.75,1) (\xb,\yb) (2.9,3.5) (\xc,\yc) (4.35,1) (\xd,\yd) (5,4)};

\end{tikzpicture}
}
\hspace{1cm}
\subfloat[Matching values and derivatives]{
\begin{tikzpicture}

\draw[->] (0,0) -- (6,0) node[anchor=north] {$x$};
\draw[->] (0,0) -- (0,4.5) node {};

\pgfmathsetmacro{\xa}{1.2}; \pgfmathsetmacro{\ya}{1.3};
\pgfmathsetmacro{\xb}{2.3}; \pgfmathsetmacro{\yb}{2.3};
\pgfmathsetmacro{\xc}{3.7}; \pgfmathsetmacro{\yc}{1.8};
\pgfmathsetmacro{\xd}{4.8}; \pgfmathsetmacro{\yd}{2.8};
\pgfmathsetmacro{\radius}{0.045};

\draw[fill] (\xa,\ya) circle (\radius);
\draw[fill] (\xb,\yb) circle (\radius);
\draw[fill] (\xc,\yc) circle (\radius);
\draw[fill] (\xd,\yd) circle (\radius);

\pgfmathsetmacro{\xbeg}{0.5}; \pgfmathsetmacro{\ybeg}{1};
\pgfmathsetmacro{\xend}{5.5}; \pgfmathsetmacro{\yend}{3};

\draw plot [smooth, tension=0.55] coordinates {(\xbeg,\ybeg) (\xa,\ya) (\xb,\yb) (\xc,\yc) (\xd,\yd) (\xend,\yend)};

\pgfmathsetmacro{\dxa}{\xb-\xbeg}; \pgfmathsetmacro{\dya}{\yb-\ybeg};
\pgfmathsetmacro{\dxb}{\xc-\xa};   \pgfmathsetmacro{\dyb}{\yc-\ya};
\pgfmathsetmacro{\dxc}{\xd-\xb};   \pgfmathsetmacro{\dyc}{\yd-\yb};
\pgfmathsetmacro{\dxd}{\xend-\xc}; \pgfmathsetmacro{\dyd}{\yend-\yc};
\pgfmathsetmacro{\len}{0.15};

\draw[thick] (\xa-\len*\dxa,\ya-\len*\dya) -- (\xa+\len*\dxa,\ya+\len*\dya);
\draw[thick] (\xb-\len*\dxb,\yb-\len*\dyb) -- (\xb+\len*\dxb,\yb+\len*\dyb);
\draw[thick] (\xc-\len*\dxc,\yc-\len*\dyc) -- (\xc+\len*\dxc,\yc+\len*\dyc);
\draw[thick] (\xd-\len*\dxd,\yd-\len*\dyd) -- (\xd+\len*\dxd,\yd+\len*\dyd);

\draw[dashed] plot [smooth, tension=1.2] coordinates {(\xbeg,\ybeg-0.1) (\xa,\ya) (\xb,\yb) (\xc,\yc) (\xd,\yd) (\xend,\yend+0.1)};

\end{tikzpicture}
}

\caption{Motivational sketch for matching derivatives. Using the same number of input locations, matching derivatives further constrains the student (dashed line) to better match the teacher (solid line). Based on Figure~2 by \citet{Simard:1992:tangentprop}.}
\label{fig:model_compression:matching_derivatives}
\end{figure}

In the previous section, we considered losses that only involve function values $t\br{\vect{x}}$ and $f\br{\vect{x}\g\bm{\theta}}$. We then described a training procedure, whereby the function represented by the teacher is probed at various input locations, and the student is then adjusted to match the teacher's function values at those locations. In this sense, we reduced model compression to a regression problem, where train data were generated by $\prob{\vect{x}}$ and labelled by $t\br{\vect{x}}$.

However, since we have access to the full function represented by the teacher, every time we probe it at some location to calculate its value, we can go one step further and also calculate its derivative with respect to $\vect{x}$. In principle, calculating the derivative of any function is possible at a cost which is at most a constant factor greater than the cost of calculating a single function value \citep{Griewank:2012:differentiation}.\footnote{This result is commonly known as the \emph{cheap gradient principle}.} We can then encourage the student to not only match the teacher's function values at some input location, but also the whole tangent hyperplane specified by its derivative with respect to the input at the same location. Thus, we can further constrain the student using the same number of input locations, as illustrated in Figure~\ref{fig:model_compression:matching_derivatives}.

In order to train the student by matching derivatives, we need to include $\pderiv{t}{\vect{x}}$ and $\pderiv{f}{\vect{x}}$  in the loss, so as to minimize the discrepancy between them. For instance, assuming that both the teacher and the student have $I$ outputs, a loss that measures the discrepancy between derivatives is the derivative square error\index{Square error}, defined as follows
\begin{equation}
E_\mathrm{DSE}\br{\vect{x},\bm{\theta}} = \frac{1}{2I}\sum_i{\norm{\pderiv{f_i}{\vect{x}} - \pderiv{t_i}{\vect{x}}}^2}.
\end{equation}
In order to apply the same stochastic gradient training procedure with losses that involve derivatives with respect to $\vect{x}$ such as the above, we need to be able to calculate the derivatives of the loss with respect to $\bm{\theta}$. As we will show, this is always possible at a cost which is at most a constant factor greater than the cost of using losses without derivatives. This is surprising, since one might expect that the cost of including derivatives would scale badly with the input dimensionality. However, it turns out that this is not the case.

To see why it is not the case, let both $t\br{\vect{x}}$ and $f\br{\vect{x}\g\bm{\theta}}$ be functions of $I$ outputs, and $E\br{\vect{x},\bm{\theta}}$ be a function of only $\pderiv{t_i}{\vect{x}}$ and $\pderiv{f_i}{\vect{x}}$ for all $i$. Then, the derivative of the loss with respect to $\bm{\theta}$ becomes
\begin{equation}
\pderivnull{\bm{\theta}}E\br{\vect{x},\bm{\theta}} = \sum_i{\spderivb{f_i}{\bm{\theta}}{\vect{x}}\,\pderiv{E}{\vect{d}_i}},
\end{equation}
where $\vect{d}_i = \pderiv{f_i}{\vect{x}}$. In the above expression, $\spderivb{f_i}{\bm{\theta}}{\vect{x}}$ is part of the Hessian of $f_i$ and $\pderiv{E}{\vect{d}_i}$ is some vector that depends on $\pderiv{t}{\vect{x}}$ and $\pderiv{f}{\vect{x}}$. 
The key thing to notice here is that the above expression only involves the Hessian in the form of Hessian-vector products. Despite the full Hessian needing quadratic time to be calculated, Hessian-vector products can be directly calculated in only linear time, using a method known as the \emph{R technique}\index{R technique} \citep{Pearlmutter:1994:fast_mult_hess}. The R technique, as described in detail in appendix~\ref{chapter:R_technique}, avoids calculating the full Hessian by introducing the operator $\Rop{\cdot}$, which is defined as 
\begin{equation}
\Rop{\vect{g}} = \br{\pderiv{\vect{g}}{\vect{x}}}^T\vect{v}.
\end{equation}
Making use of the R operator and identifying $\vect{g} = \pderiv{f_i}{\bm{\theta}}$ and $\vect{v} = \pderiv{E}{\vect{d}_i}$, we can write the derivative of $E\br{\vect{x},\bm{\theta}}$ with respect to $\bm{\theta}$ as
\begin{equation}
\pderivnull{\bm{\theta}}E\br{\vect{x},\bm{\theta}} = \sum_i{\Rop{\pderiv{f_i}{\bm{\theta}}}}.
\end{equation}
For example, when the loss is the derivative square error, we have
\begin{equation}
\pderivnull{\bm{\theta}}E_\mathrm{DSE}\br{\vect{x},\bm{\theta}} = \frac{1}{I}\sum_i{\spderivb{f_i}{\bm{\theta}}{\vect{x}}\br{\pderiv{f_i}{\vect{x}} - \pderiv{t_i}{\vect{x}}}} = \frac{1}{I}\sum_i{\Rop{\pderiv{f_i}{\bm{\theta}}}}.
\end{equation}
In the above, the vector each Hessian $\spderivb{f_i}{\bm{\theta}}{\vect{x}}$ is multiplied by is
\begin{equation}
\vect{v}_i = \pderiv{f_i}{\vect{x}} - \pderiv{t_i}{\vect{x}}.
\end{equation}

Putting all the above together, the stochastic gradient training procedure for losses that involve derivatives with respect to the input is described below.
\begin{framed}
\begin{enumerate}[label=(\roman*)]
\item Generate a minibatch $\set{\vect{x}_s}$ of size $S$ from $\prob{\vect{x}}$.
\item Calculate $\pderiv{t}{\vect{x}}$ and $\pderiv{f}{\vect{x}}$ for all datapoints in the minibatch.
\item For each output $i$ and for each $\vect{x}_s$, calculate $\vect{v}_i = \pderiv{E}{\vect{d}_i}$ and make use of the R technique (appendix~\ref{chapter:R_technique}) to calculate the Hessian-vector product $\Rop{\pderiv{f_i}{\bm{\theta}}} = \spderivb{f_i}{\bm{\theta}}{\vect{x}}\,\vect{v}_i$.
\item Calculate the stochastic gradient $\hat{\vect{g}}\br{\bm{\theta}} = \frac{1}{S}\sum_{s}{\pderivnull{\bm{\theta}} E\br{\vect{x}_s,\bm{\theta}}}$.
\item Make an update on $\bm{\theta}$ using $\hat{\vect{g}}\br{\bm{\theta}}$.
\item Repeat until convergence.
\end{enumerate}
\end{framed}
\noindent Under certain regularity conditions, and provided the samples from $\prob{\vect{x}}$ are exact and independent, the above algorithm is guaranteed to converge almost surely to a stationary point of the average loss \citep{Bottou:1999:online_learning}.

Notice that the above algorithm is fairly general, since the only requirement it has on the student is that $f\br{\vect{x}\g\bm{\theta}}$ must be differentiable. If $f\br{\vect{x}\g\bm{\theta}}$ is a neural network, then each $\pderiv{f_i}{\bm{\theta}}$ can be calculated using backward propagation and each  $\Rop{\pderiv{f_i}{\bm{\theta}}}$ can be calculated using R\{backprop\}, as described in detail in appendix~\ref{chapter:derivatives_in_neural_networks}. 

Since backprop and R\{backprop\} have roughly the same cost, calculating $\pderiv{f_i}{\bm{\theta}}$ and $\Rop{\pderiv{f_i}{\bm{\theta}}}$ for all $i$ can be done at the cost of $2I$ backprops. In contrast, loss functions that only match function values require a singe backprop per datapoint. Hence, including derivatives of functions with $I$ outputs in the loss increases the computational cost roughly $2I$ times compared to losses that do not involve derivatives. As a consequence, even though matching derivatives does not scale badly with the input dimensionality, it does scale linearly with the output dimensionality, therefore it is better suited to models with relatively few outputs.

\section{Case study: compressing a neural ensemble}
\label{sec:model_compression:case_study}

In this section, we put our model compression framework into practice, in order to compress an ensemble of neural networks into a single small neural network. The framework asks for the implementation of (a) a loss for measuring the discrepancy between the teacher and the student, and (b) an appropriate data generator for the input data. In the following, we describe the neural ensemble to be compressed, the losses and data generators we employed, and we present and discuss experimental results.

\subsection{The neural ensemble}
\index{Ensemble}

We trained an ensemble of $30$ neural networks\index{Neural network} for handwritten digit classification on the MNIST dataset\index{MNIST dataset}. The MNIST dataset \citep{LeCun:MNIST_web} is a dataset of greyscale $28\times 28$ images of handwritten digits ($0$ to $9$), and it is commonly used as a benchmark for classification algorithms. The dataset is partitioned into a train set of $60{,}000$ images, which we used to train the ensemble, and a test set of $10{,}000$ images. The objective is to classify each image as being one of the $10$ digits. Figure~\ref{fig:model_compression:mnist:original} shows some random samples from the MNIST train set.

The ensemble is composed of $30$ neural networks of the same architecture. Each neural network has two hidden layers with $500$ and $300$ hidden units respectively, and an output layer of $10$ units. The hidden layers use ReLU\index{Layer!ReLU} activation functions (as described in section~\ref{sec:derivs_neural_nets:relu}) and the output layer uses softmax\index{Layer!Softmax} activation functions (as described in section~\ref{sec:derivs_neural_nets:softmax}). Each network is therefore of the following type
\begin{equation}
784\xrightarrow{\mathit{ReLU}} 500\xrightarrow{\mathit{ReLU}}300\xrightarrow{\mathit{softmax}}10.
\end{equation}

Each network in the ensemble was trained separately on a different dataset of $60{,}000$ images, that was created by resampling the MNIST train set with replacement. We trained each network by minimizing cross entropy\index{Cross entropy} with stochastic gradient descent\index{Stochastic gradient descent}, using minibatches of $20$ datapoints, and performing $20$ passes over the dataset in total. We used ADADELTA\index{ADADELTA} \citep{Zeiler:2012:adadelta} as a learning rate strategy, which uses a different learning rate for each network parameter and adapts them during training. ADADELTA is controlled by two hyperparameters, for which we used the values suggested in the original paper.

Given an image $\vect{x}$, each individual network $n$ outputs a discrete probability distribution $t^{\br{n}}\br{\vect{x}}$ with $10$ elements, corresponding to how likely the network thinks it is that $\vect{x}$ belongs to each class. The predictions from all networks are then averaged together, in order to give a final prediction as follows
\begin{equation}
t\br{\vect{x}} = \frac{1}{N}\sum_n{t^{\br{n}}\br{\vect{x}}}.
\end{equation}
On the test set, the above ensemble achieves an accuracy of $98.67\% \pm 0.23\%$ and an average log probability of $-0.053 \pm 0.011$ nats. The error bars correspond to two standard deviations. These results are comparable to those reported in the literature for models of similar architecture \citep{LeCun:MNIST_web}.

\begin{figure}[tbp]
\def\imwidth{0.48\textwidth}
\centering
\subfloat[Original\label{fig:model_compression:mnist:original}]{
\includegraphics[width=\imwidth]{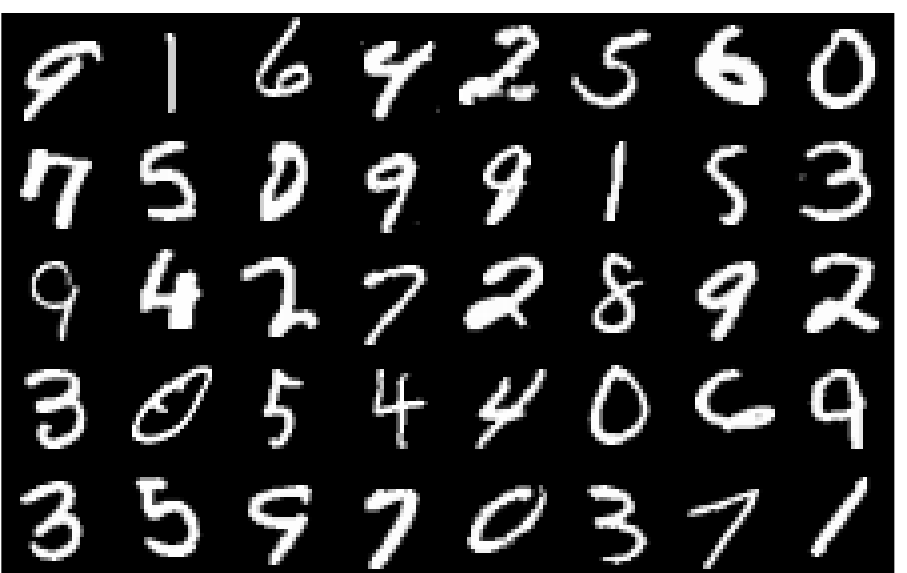}}
\hfill
\subfloat[Binarized\label{fig:model_compression:mnist:binarized}]{
\includegraphics[width=\imwidth]{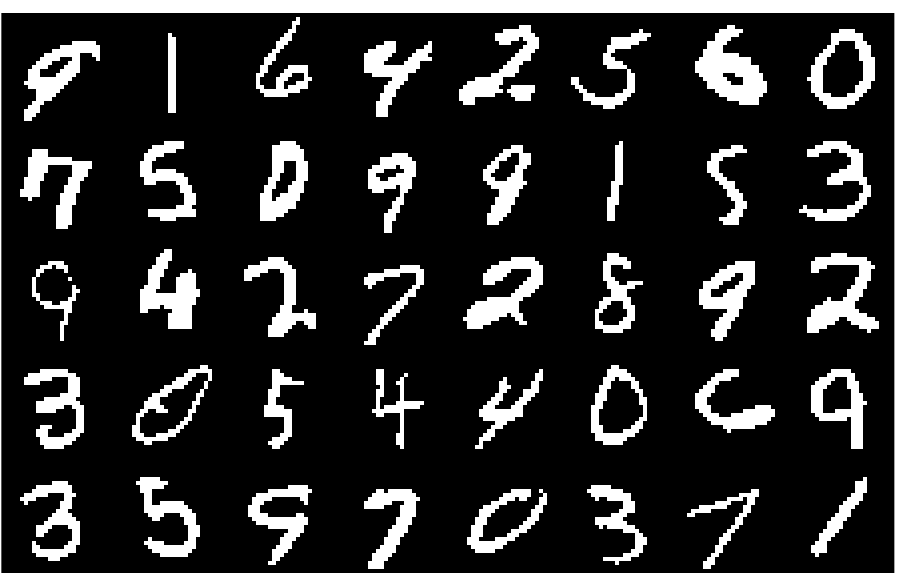}}
\caption{Samples from the MNIST dataset of handwritten digits. White corresponds to $1$ and black to $0$. The binary version was only used to train NADE\@.}
\label{fig:model_compression:mnist}
\end{figure}

\subsection{Loss functions}

Our model compression framework requires the specification of a suitable loss $E\br{\bm{\theta}} = \avg{E\br{\vect{x},\bm{\theta}}}{\prob{\vect{x}}}$ in order to measure the discrepancy between the neural ensemble $t\br{\vect{x}}$ and the student network $f\br{\vect{x}\g \bm{\theta}}$. As presented in section~\ref{sec:model_compression:framework}, our framework supports two kinds of losses; losses that (a) try to match function values, and losses that (b) try to match function derivatives. In our case study, we used two losses, namely \emph{cross entropy}, which matches function values, and \emph{derivative square error}, which matches function derivatives. The following two sections discuss these two losses in detail and establish their validity.

\subsubsection{Cross entropy}
\index{Cross entropy}

In our case study, both the teacher $t\br{\vect{x}}$ and the student $f\br{\vect{x}\g \bm{\theta}}$, for each input location $\vect{x}$, output a vector of probabilities, which corresponds to how likely the models think it is that $\vect{x}$ belongs to each class. A natural measure of discrepancy between $t\br{\vect{x}}$ and $f\br{\vect{x}\g \bm{\theta}}$ would then be the KL divergence\index{KL divergence} from the teacher to the student. For some input $\vect{x}$, this KL divergence can be written as
\begin{equation}
\kl{t\br{\vect{x}}}{f\br{\vect{x}\g \bm{\theta}}} = 
\sum_i {t_i\br{\vect{x}}\log{t_i\br{\vect{x}}}} - 
\sum_i {t_i\br{\vect{x}}\log{f_i\br{\vect{x}\g \bm{\theta}}}}.
\end{equation}
In the above expression, the term $\sum_i {t_i\br{\vect{x}}\log{t_i\br{\vect{x}}}}$ is the negative entropy of $t\br{\vect{x}}$ and is a constant with respect to $\bm{\theta}$. Therefore, minimizing the above KL divergence is equivalent to minimizing the quantity
\begin{equation}
E_\mathrm{CE}\br{\vect{x},\bm{\theta}} = -\sum_i{t_i\br{\vect{x}}\log{f_i\br{\vect{x}\g\bm{\theta}}}},
\end{equation}
which is known as \emph{cross entropy}.

Cross entropy is a loss commonly used in training for classification problems. It only involves function values, so it can be optimized as described in section~\ref{sec:model_compression:framework:matching_values}. It is well known that the KL divergence is always non-negative, and that it becomes zero if and only if $t\br{\vect{x}} = f\br{\vect{x}\g \bm{\theta}}$ \citep[section~2.6]{MacKay:2002:IT}. Therefore, assuming $f\br{\vect{x}\g \bm{\theta}}$ has enough capacity, $\avg{E_\mathrm{CE}\br{\vect{x},\bm{\theta}}}{\prob{\vect{x}}}$ will achieve its global minimum if and only if $t\br{\vect{x}} = f\br{\vect{x}\g \bm{\theta}}$ everywhere in the support of $\prob{\vect{x}}$.

\subsubsection{Derivative square error}
\index{Derivative square error}

A different approach for making $t\br{\vect{x}}$ and $f\br{\vect{x}\g \bm{\theta}}$ agree everywhere is by making sure that their derivatives with respect to $\vect{x}$ match. To measure the discrepancy between these derivatives, we herein propose the \emph{derivative square error} loss, which we define as follows
\begin{equation}
E_\mathrm{DSE}\br{\vect{x},\bm{\theta}} =
\frac{1}{2I}\sum_i{\norm{\pderivnull{\vect{x}}\log{f_i\br{\vect{x}\g \bm{\theta}}} - \pderivnull{\vect{x}}\log{t_i\br{\vect{x}}}}^2}.
\end{equation}
Note that since both $t\br{\vect{x}}$ and $f\br{\vect{x}\g \bm{\theta}}$ output probabilities, we choose to convert them into the log domain before taking their derivatives. This loss contains only derivatives, and therefore it can be optimized using the framework described in section~\ref{sec:model_compression:framework:matching_derivatives}.

In the general case, matching the derivatives of two arbitrary---but differentiable---functions makes them equal only up to an additive constant. Therefore, in general, only matching derivatives is not sufficient in itself to make the functions equal. However, in our case, where the functions are probability distributions, we can show that, under very mild conditions, matching derivatives is actually sufficient for making the functions equal. The following proposition proves this claim, and provides the conditions that need to apply.
\begin{proposition}
Assuming that $f\br{\vect{x}\g \bm{\theta}}$ has sufficient capacity, and that there exist points $\vect{x}_1, \vect{x}_2, \ldots, \vect{x}_I$ in the support of $\prob{\vect{x}}$ for which $t\br{\vect{x}_1}, t\br{\vect{x}_2}, \ldots, t\br{\vect{x}_I}$ are linearly independent, then the global minimum of $\avg{E_\mathrm{DSE}\br{\vect{x},\bm{\theta}}}{\prob{\vect{x}}}$ is achieved if and only if $f\br{\vect{x}\g \bm{\theta}} = t\br{\vect{x}}$ everywhere in the support of $\prob{\vect{x}}$.
\end{proposition}
\begin{proof}
Let $S_p$ be the support of $\prob{\vect{x}}$, that is, the set of points $\vect{x}$ for which $\prob{\vect{x}}>0$. Notice that $\avg{E_\mathrm{DSE}\br{\vect{x},\bm{\theta}}}{\prob{\vect{x}}} \ge 0$, as the expectation of a non-negative quantity. If $f\br{\vect{x}\g \bm{\theta}} = t\br{\vect{x}}$ in $S_p$, then obviously $\avg{E_\mathrm{DSE}\br{\vect{x},\bm{\theta}}}{\prob{\vect{x}}} = 0$, hence it is globally minimized. Conversely, if $\avg{E_\mathrm{DSE}\br{\vect{x},\bm{\theta}}}{\prob{\vect{x}}} = 0$, then for all $i$ and all $\vect{x}$ in $S_p$ we have
\begin{equation}
\pderivnull{\vect{x}}\log{f_i\br{\vect{x}\g \bm{\theta}}} = \pderivnull{\vect{x}}\log{t_i\br{\vect{x}}}
\quad\Rightarrow\quad
f_i\br{\vect{x}\g \bm{\theta}} = c_i\,t_i\br{\vect{x}},
\end{equation}
for some positive constants $c_i$. Since both $t\br{\vect{x}}$ and $f\br{\vect{x}\g \bm{\theta}}$ are probability distributions, we have that
\begin{equation}
\sum_i {c_i\,t_i\br{\vect{x}}} = \sum_i {f_i\br{\vect{x}\g \bm{\theta}}} = 1=\sum_i {t_i\br{\vect{x}}}
\quad\Rightarrow\quad
\sum_i{\br{c_i - 1}t_i\br{\vect{x}}} = 0.
\end{equation}
The above can be written as $\bm{\alpha}^Tt\br{\vect{x}} = 0$, where $\bm{\alpha}$ is defined to be a vector of length $I$ that has $\br{c_i - 1}$ as its elements. Assume that $\bm{\alpha} \neq \vect{0}$. Since $\vect{x}_1, \vect{x}_2, \ldots, \vect{x}_I$ are all in $S_p$, we have that $\bm{\alpha}$ is orthogonal to all $t\br{\vect{x}_i}$, hence the collection of all $t\br{\vect{x}_i}$ together with  $\bm{\alpha}$ forms a set of linearly independent vectors. This is a contradiction, since it implies the existence of $\br{I+1}$ linearly independent vectors in an $I$-dimensional space. We have thus shown that $\bm{\alpha} = \vect{0}$, hence $c_i = 1$ for all $i$, therefore $f\br{\vect{x}\g \bm{\theta}} = t\br{\vect{x}}$ for all $\vect{x}$ in $S_p$.
\end{proof}

The above proposition says that, as long as we can find $I$ input locations for which $t\br{\vect{x}}$ are linearly independent, then derivative square error is a valid loss. For most practical problems of interest, this is quite a reasonable assumption to make. For example, in our case study, the neural ensemble was trained on MNIST with data from all $10$ classes. Therefore, there must be $10$ regions in the input space for which a different one of the outputs is close to $1$, while all others are close to $0$. If this is not the case, then the model would be a poor classifier and we would probably not want to compress it in the first place.

\subsection{Data generators}

In our model compression framework, we need to specify a data generator $\prob{\vect{x}}$, whose job is to generate the input data used during training. The choice of data generator is highly important, as it determines the region of input space for which agreement between the teacher and the student is sought. In the following sections, we describe the three data generators we employed in our case study, namely \emph{resampling the dataset}, \emph{NADE} and \emph{random noise}.

\subsubsection{Resampling the dataset}

If we have access to a dataset of input locations, then an obvious thing to do is to use it as is. The way this works is that during training, in each iteration, we sample a minibatch of points from the dataset.

Sampling from the dataset can be done either with or without replacement. With replacement, each point in the minibatch is independently and uniformly at random selected from the whole dataset. In this case, the data generator takes the following form
\begin{equation}
\prob{\vect{x}} = \frac{1}{N} \sum_n\deltaf{\vect{x} - \vect{x}_n},
\end{equation}
where $\deltaf{\cdot}$ is the delta function and $\vect{x}_1, \vect{x}_2, \ldots, \vect{x}_N$ are the points in the dataset. Without replacement, each point is sampled uniformly at random, but then removed from the dataset. The sampling process resumes from the beginning when all the dataset has been sampled. This ensures that we see every other point in the dataset before seeing any point twice. In our experiments, we sampled the dataset without replacement.

In many real-life applications, labelled datasets are an expensive and potentially scarce resource, since it typically takes time and manual effort to do the labelling. However, unlabelled datasets may be abundant, especially in domains like vision and speech. Model compression is well-suited for such domains, since it does not require the dataset to be labelled; instead, all the labelling is done by the teacher model.
In our experiments, we used the MNIST train set (or part of it), on which the neural ensemble was trained in the first place.

\subsubsection{NADE}

Having access to a dataset is good, however it might not always be ideal. The dataset might be too small to cover the input space sufficiently, or if it is large, it might be expensive to store. And, no matter how large it is, it will always be finite. On the other hand, if we have access to a generative model\index{Generative model} from which we can generate data, we have essentially access to an infinite dataset, for limited storage cost. 

We can imagine having a dataset to begin with, which is not sufficiently large to fully cover the input space, but is large enough to train a generative model on it. Then, we can use this generative model for model compression. Essentially, this way we achieve massive dataset expansion, while at the same time never having to store more than a minibatch of data, since minibatches are generated on the fly and thrown away after being used.

Motivated by the above, in our experiments we used the Neural Autoregressive Distribution Estimator\index{NADE} (NADE for short), which was introduced by \citet{Larochelle:2011:NADE}. NADE is a generative model whose likelihood is fully tractable, therefore it can be trained on a dataset using maximum likelihood. Also, generating from NADE produces samples that are exact and independent, therefore it is well-suited for our framework. Furthermore, NADE is fairly flexible, so it can fit well datasets like MNIST\@. NADE is extensively used in chapter~\ref{chapter:generative_models}, so we refer the reader to section~\ref{sec:generative_models:NADE} for more details on it.

\begin{figure}[tbp]
\def\imwidth{0.48\textwidth}
\centering
\subfloat[Binary samples\label{fig:model_compression:samples_nade_60k:samples}]{
\includegraphics[width=\imwidth]{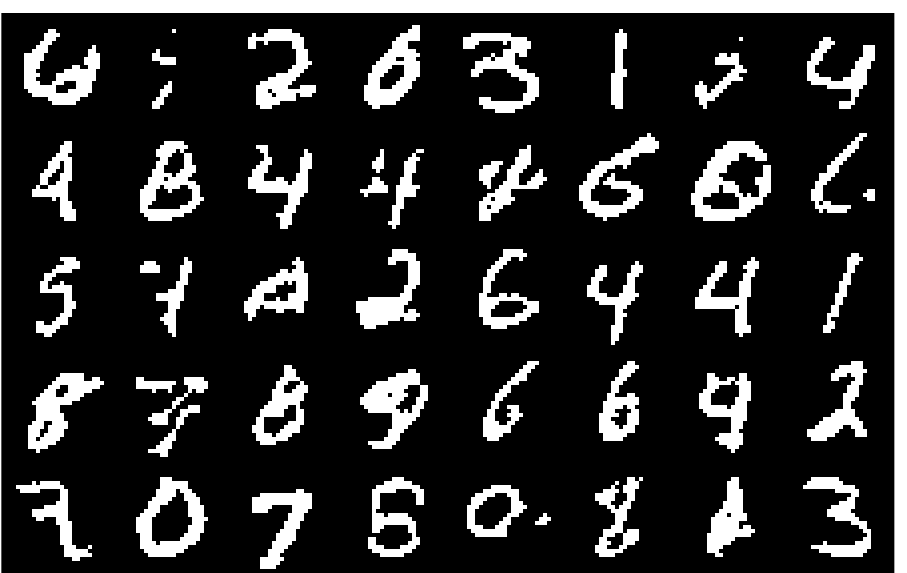}}
\hfill
\subfloat[Conditional probabilities\label{fig:model_compression:samples_nade_60k:prob}]{
\includegraphics[width=\imwidth]{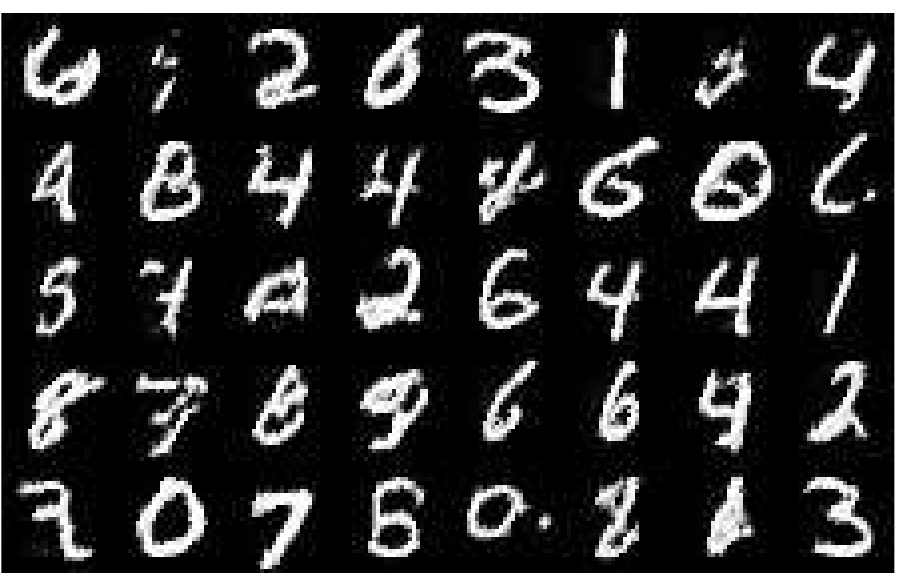}}
\caption{Samples from NADE trained on the full MNIST train set and the corresponding conditional probabilities. For model compression, the conditional probabilities were used instead of the actual samples.}
\label{fig:model_compression:samples_nade_60k}
\end{figure}

\begin{figure}[tbp]
\def\imwidth{0.48\textwidth}
\centering
\subfloat[Binary samples\label{fig:model_compression:samples_nade_6k:samples}]{
\includegraphics[width=\imwidth]{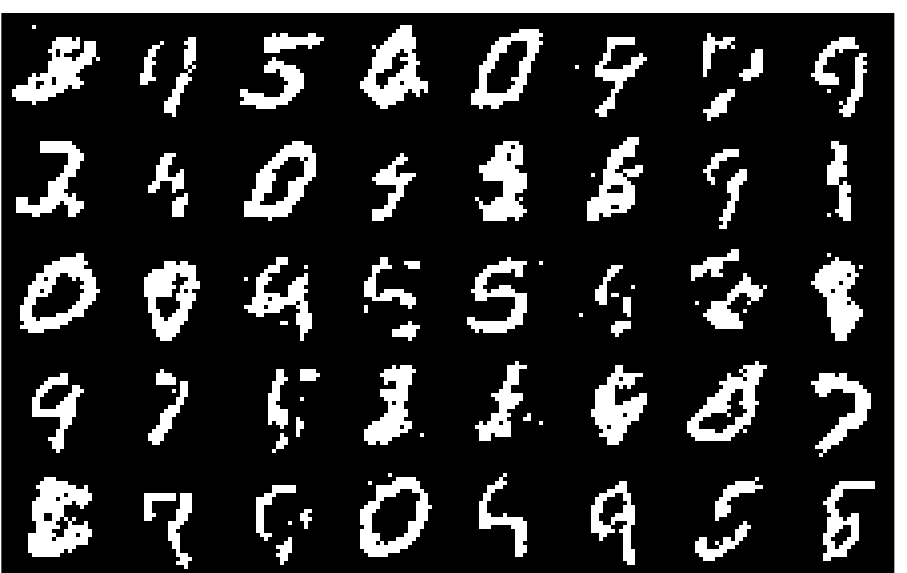}}
\hfill
\subfloat[Conditional probabilities\label{fig:model_compression:samples_nade_6k:prob}]{
\includegraphics[width=\imwidth]{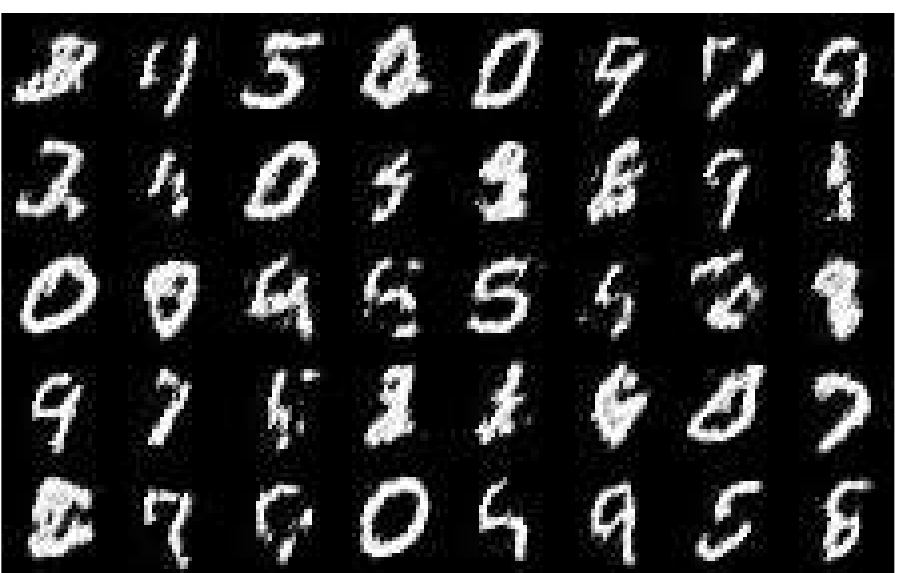}}
\caption{Samples from NADE trained on $10\%$ of the MNIST train set and the corresponding conditional probabilities. For model compression, the conditional probabilities were used instead of the actual samples.}
\label{fig:model_compression:samples_nade_6k}
\end{figure}

In our experiments, we trained a NADE with $500$ hidden units on the MNIST train set (or part of it). We used stochastic gradient descent\index{Stochastic gradient descent}, minibatches of size $20$, and performed $10$ passes over the train set. To set the learning rate strategy, we used ADADELTA\index{ADADELTA} \citep{Zeiler:2012:adadelta}, with its two hyperparameters set to the values suggested in the original paper. NADE is sensitive to the ordering of the input variables (which in this case are the pixel intensities). We used a raster columnwise ordering instead of a random ordering, which we found in practice to give a better generative model.

We used the binary version of NADE, which works only with binary inputs. To train it, we first binarized the MNIST train set by rounding its pixel intensities (which normally vary from $0$ to $1$). Figure~\ref{fig:model_compression:mnist:binarized} shows some binarized images from the MNIST train set. A problem with binary NADE is that the samples it generates are also binary. If these samples are used during model compression, this would limit severely the input space covered to the corners of a hypercube. However, together with samples, NADE also outputs the conditional probabilities of each pixel being turned on, given all generated pixels before it. These probabilities, varying from $0$ to $1$ typically resemble a ``smoothed'' version of the actual binary sample (see Figures~\ref{fig:model_compression:samples_nade_60k} and \ref{fig:model_compression:samples_nade_6k}). In all our experiments, we used the images formed by the conditional probabilities instead of the actual binary samples, since this gives a better coverage of the input space.\footnote{Another hacky option would have been to postprocess the samples with a low-pass filter to smoothen them.}

\subsubsection{Random noise}

In case we have neither a dataset nor a generative model of the input, we can always resort to using random noise as input. Nevertheless, generating inputs randomly is rarely a good idea. Typically, most randomly generated inputs will lie outside the regions of input space we are interested in, especially in high-dimensional spaces. In practice, it is rarely---if ever---the case that we are unable to build a data generator that produces inputs at least somewhat similar to those we are interested in, so we typically would not use random noise as input.

In our experiments, we used random Gaussian noise. That is, each image pixel was generated by a zero-mean, unit-variance Gaussian distribution. The main purpose of using random noise in our experiments is to demonstrate the importance of having a proper data generator when doing model compression.

\subsection{Results and discussion}
\label{sec:model_compression:case_study:results}

\begin{figure}[p]
\def\imwidth{0.48\textwidth}
\centering
\subfloat[Cross entropy]{
\includegraphics[width=\imwidth]{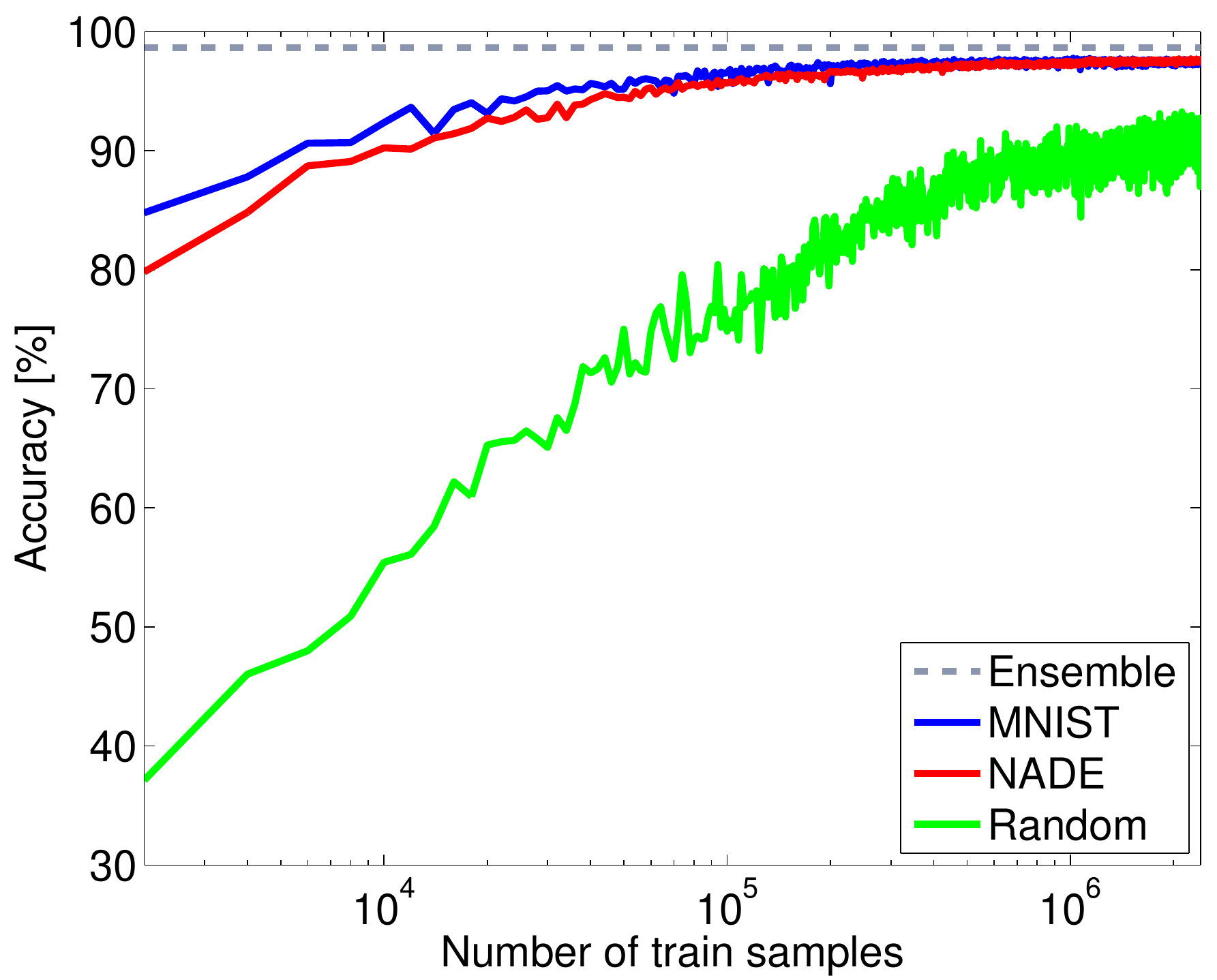}}
\hfill
\subfloat[Derivative square error]{
\includegraphics[width=\imwidth]{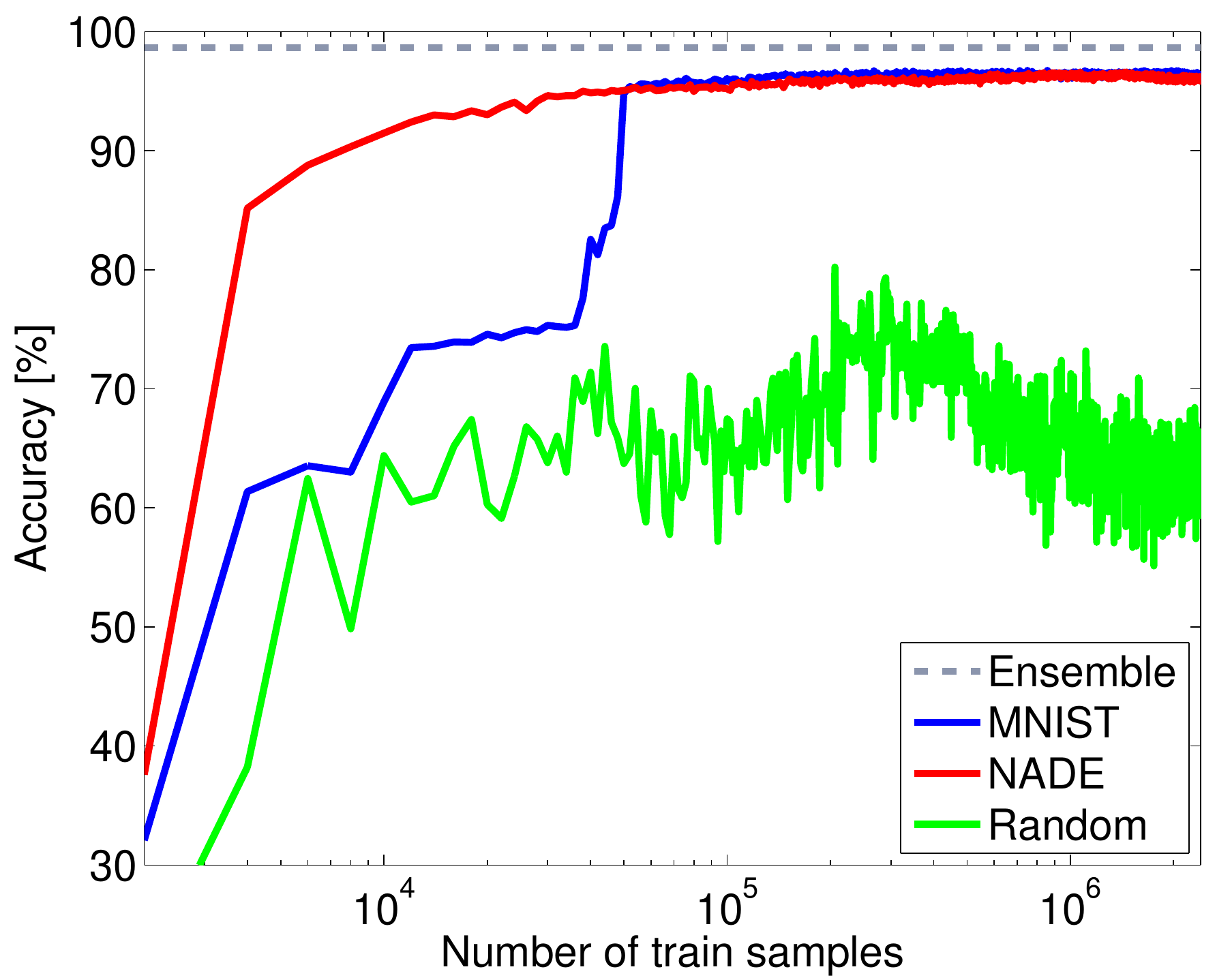}}
\caption{Training progress during compression of the neural ensemble, using the full MNIST train set. The plots show the accuracy of the student network on the MNIST test set for each loss and each data generator, as it progresses with the number of train samples. Note that the horizontal axis is in log scale.}
\label{fig:model_compression:results_accuracy_60k}
\end{figure}

\begin{figure}[p]
\def\imwidth{0.48\textwidth}
\centering
\subfloat[Cross entropy]{
\includegraphics[width=\imwidth]{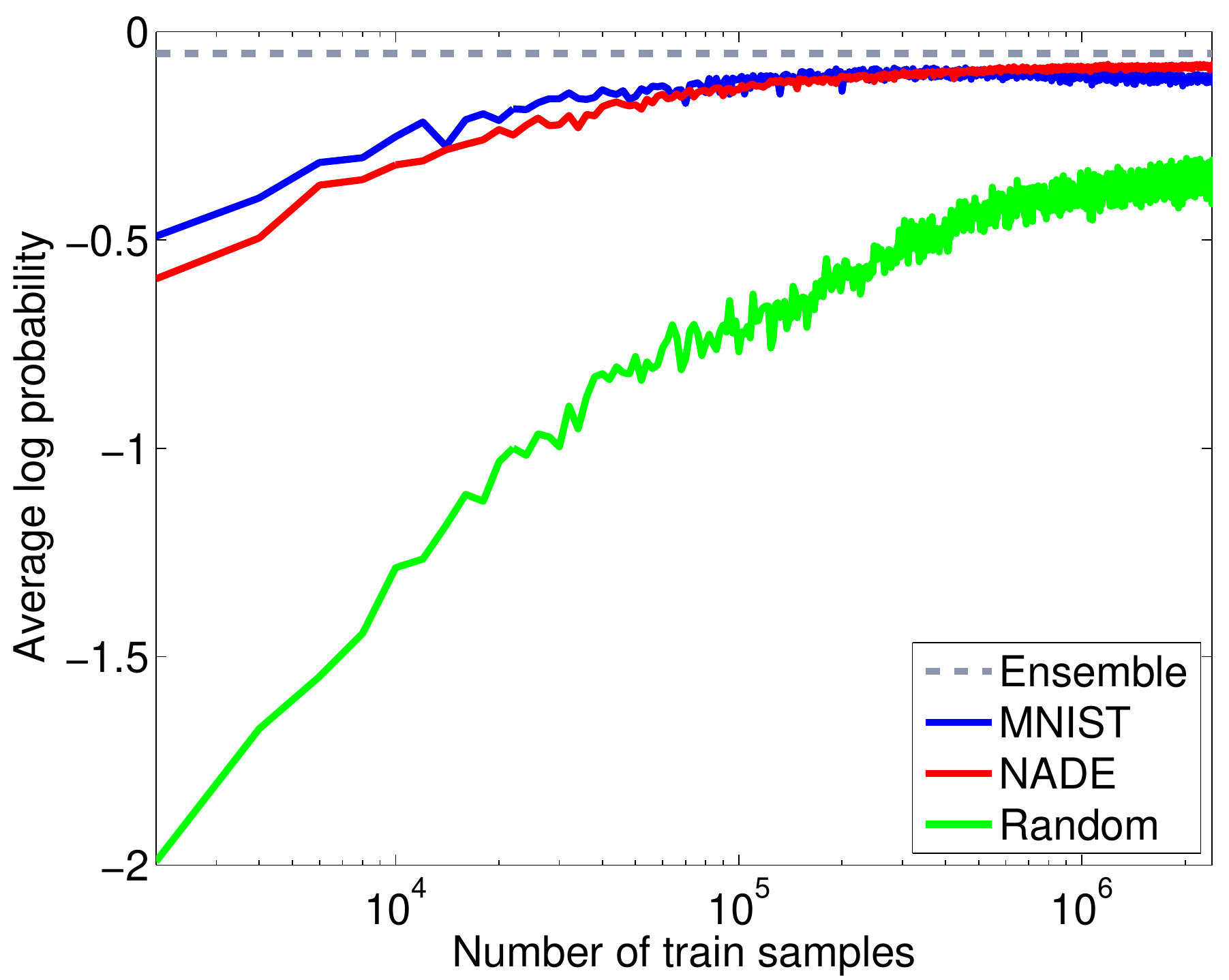}}
\hfill
\subfloat[Derivative square error]{
\includegraphics[width=\imwidth]{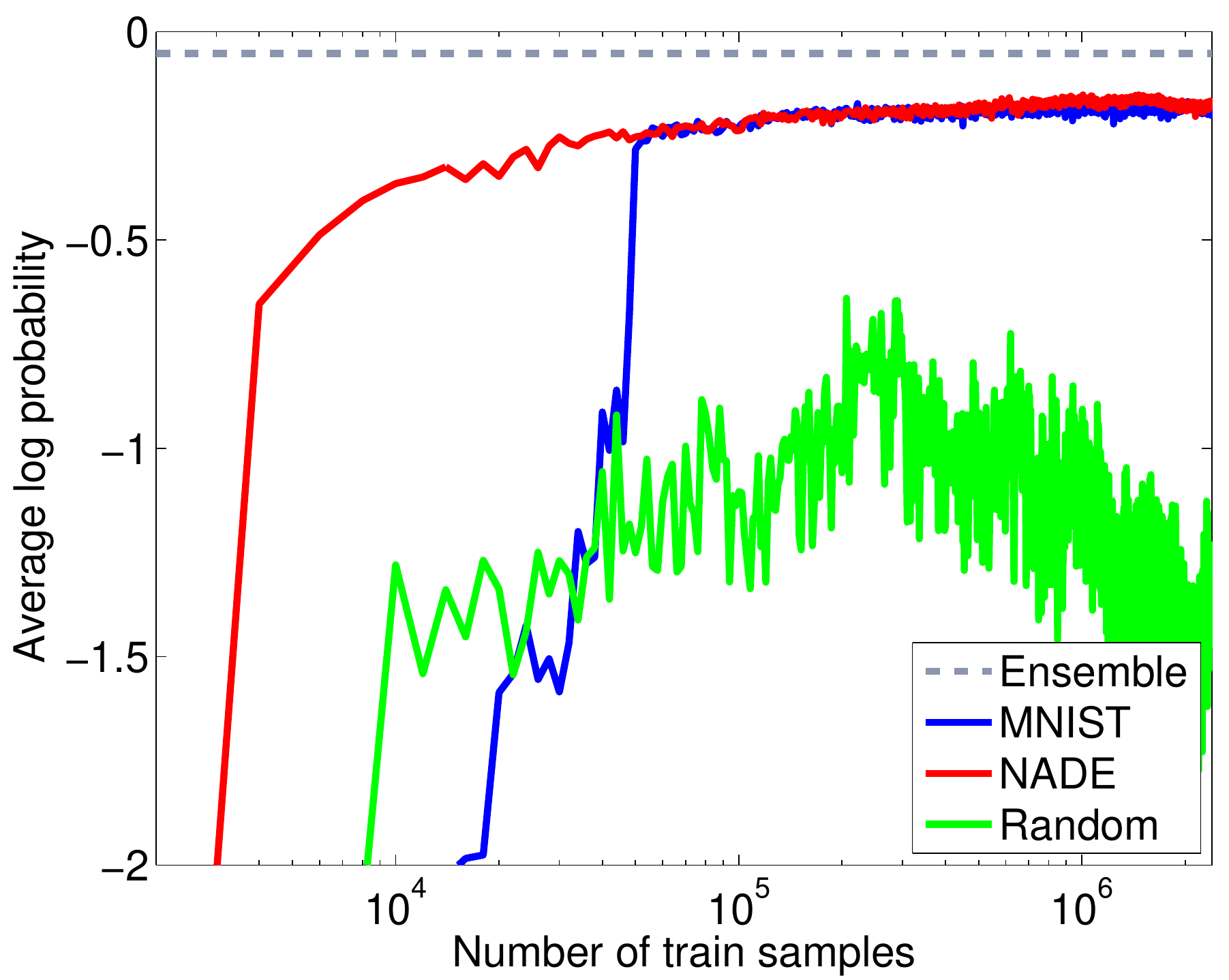}}
\caption{Training progress during compression of the neural ensemble, using the full MNIST train set. The plots show the average log probability of the student network on the MNIST test set for each loss and each data generator, as it progresses with the number of train samples. Note that the horizontal axis is in log scale.}
\label{fig:model_compression:results_logprob_60k}
\end{figure}

\begin{figure}[p]
\def\imwidth{0.48\textwidth}
\centering
\subfloat[Cross entropy]{
\includegraphics[width=\imwidth]{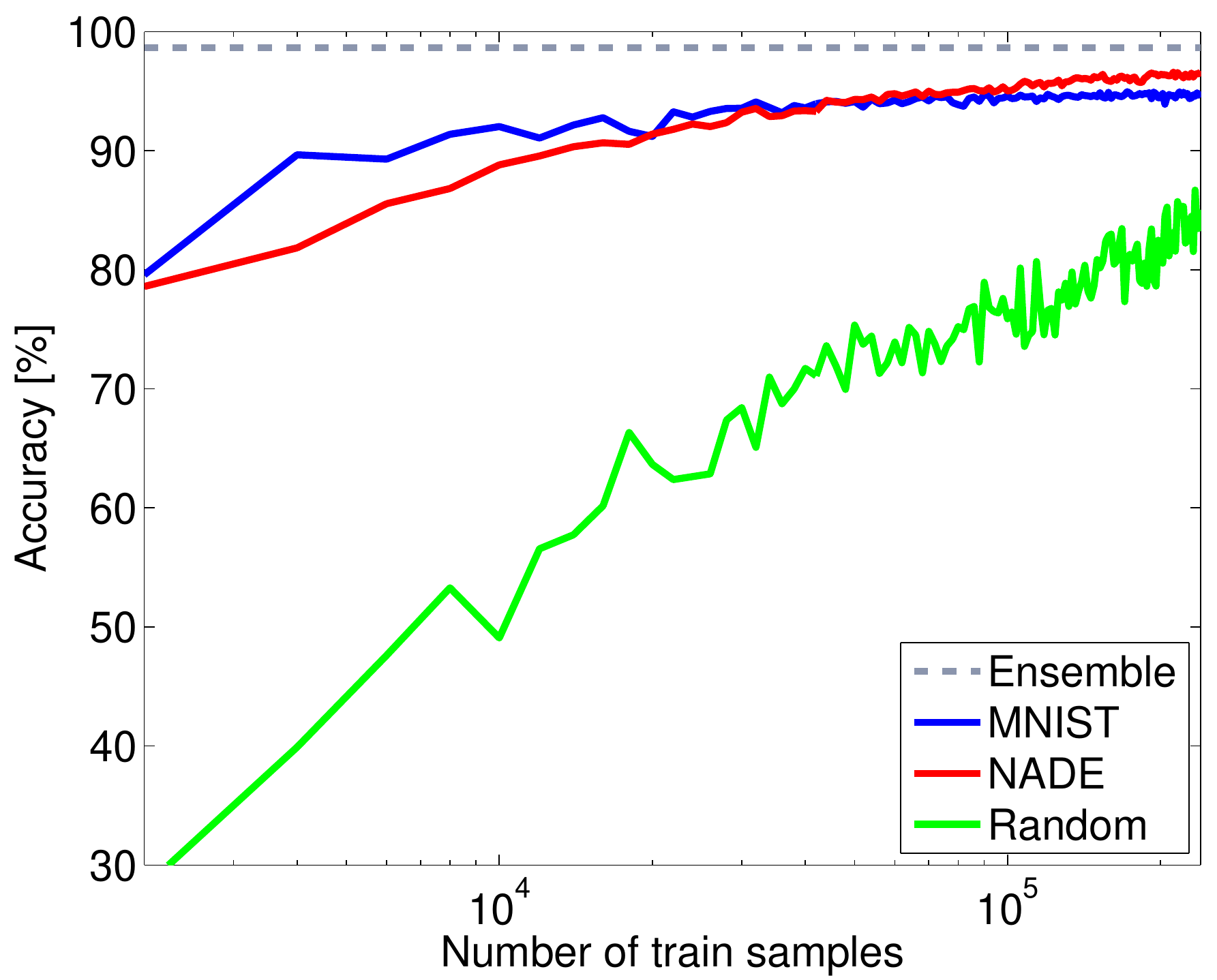}}
\hfill
\subfloat[Derivative square error]{
\includegraphics[width=\imwidth]{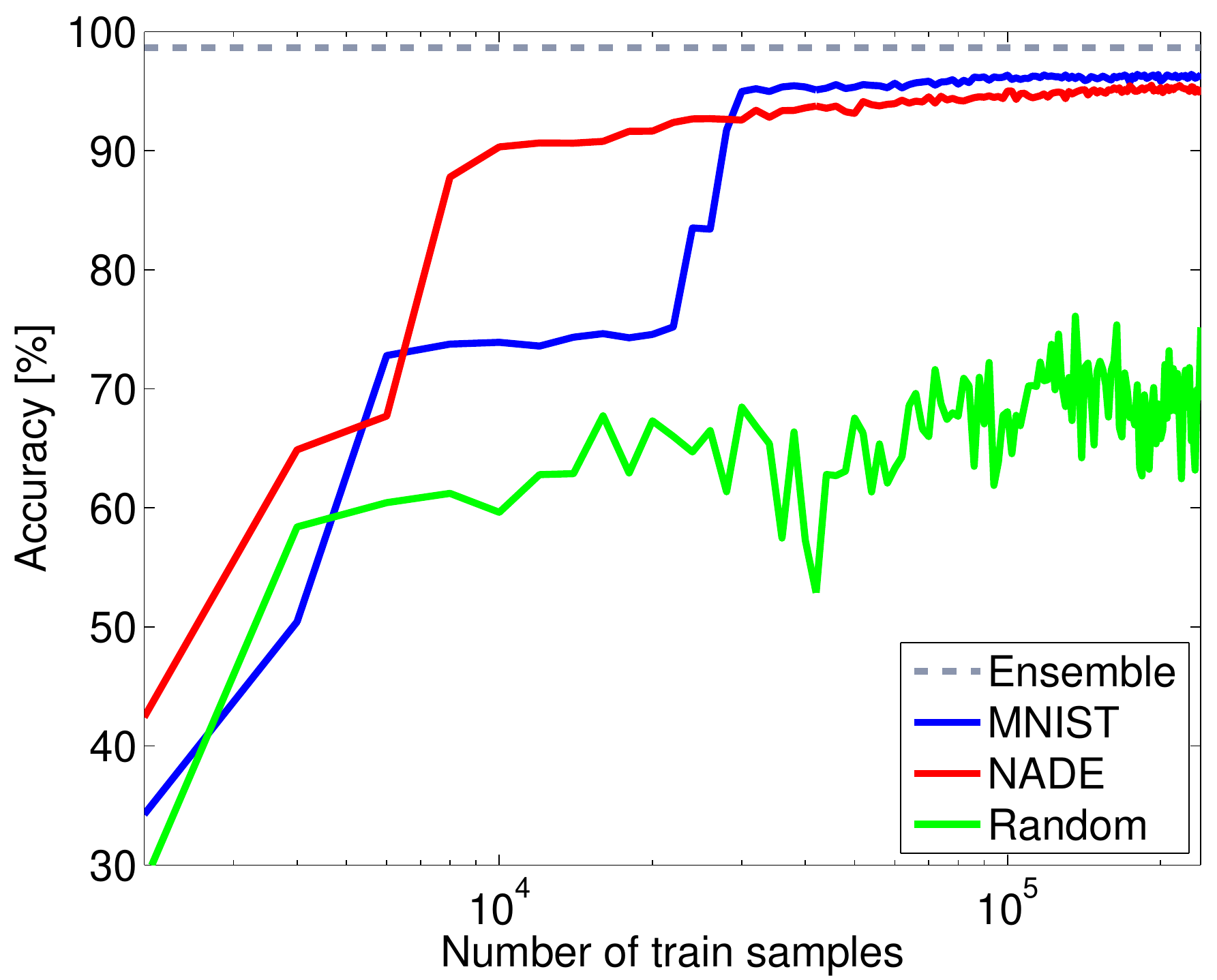}}
\caption{Training progress during compression of the neural ensemble, using $10\%$ of the MNIST train set. The plots show the accuracy of the student network on the MNIST test set for each loss and each data generator, as it progresses with the number of train samples. Note that the horizontal axis is in log scale.}
\label{fig:model_compression:results_accuracy_6k}
\end{figure}

\begin{figure}[p]
\def\imwidth{0.48\textwidth}
\centering
\subfloat[Cross entropy]{
\includegraphics[width=\imwidth]{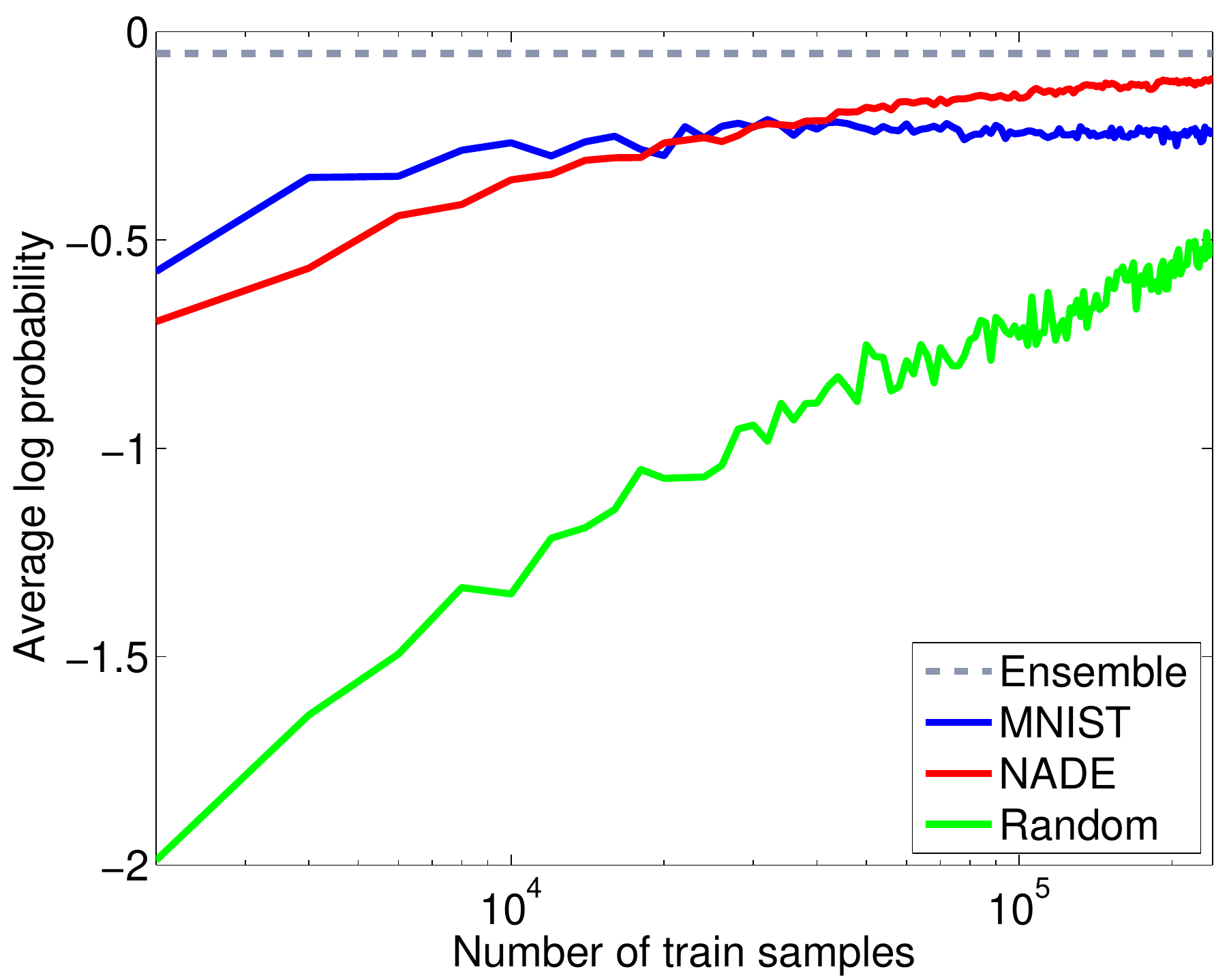}}
\hfill
\subfloat[Derivative square error]{
\includegraphics[width=\imwidth]{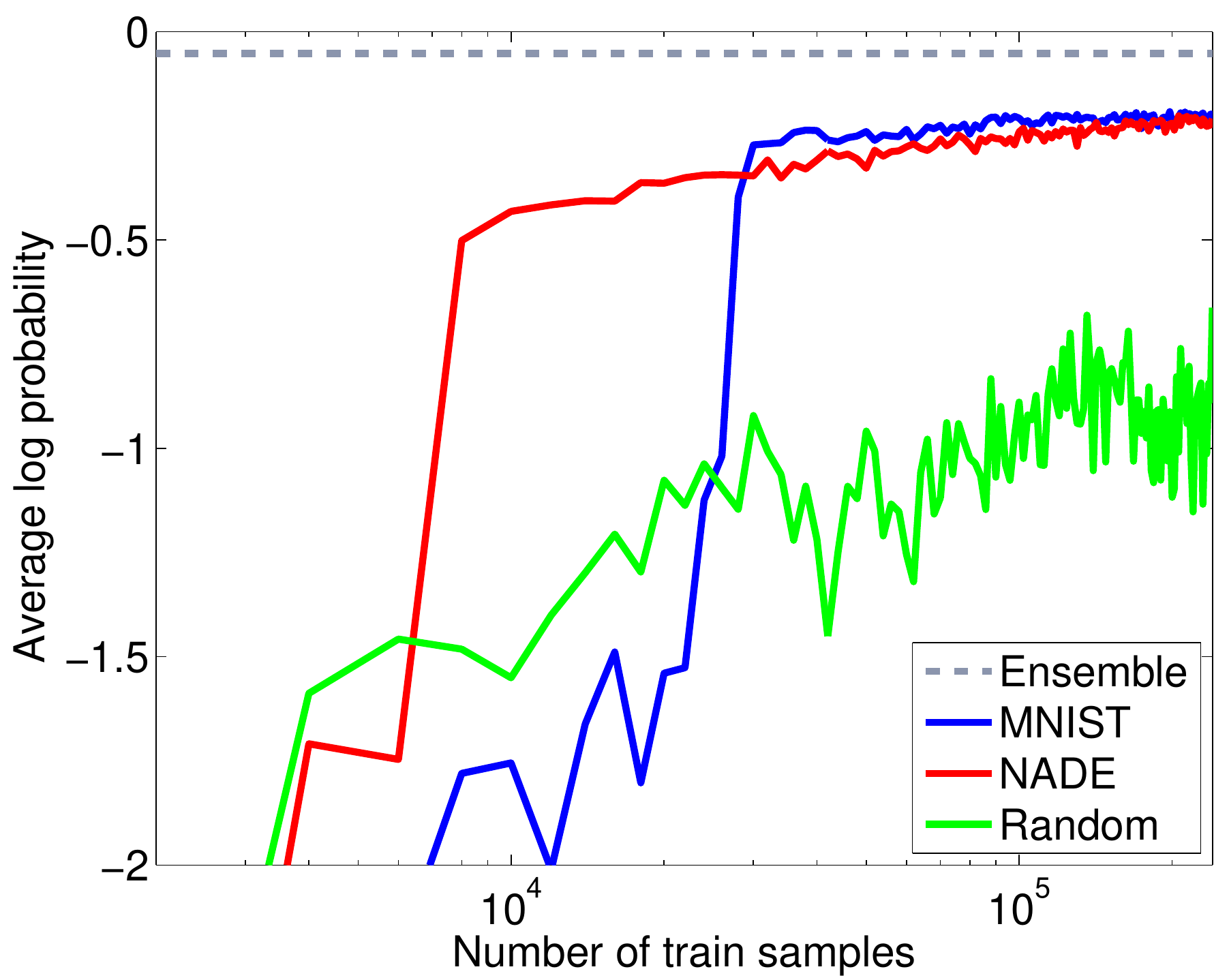}}
\caption{Training progress during compression of the neural ensemble, using $10\%$ of  the MNIST train set. The plots show the average log probability of the student network on the MNIST test set for each loss and each data generator, as it progresses with the number of train samples. Note that the horizontal axis is in log scale.}
\label{fig:model_compression:results_logprob_6k}
\end{figure}

We used our model compression framework to compress the neural ensemble into a single small neural network of the same type as the networks forming the ensemble, but with only $10\%$ the number of hidden units. In other words, the student model is a neural network of type
\begin{equation}
784\xrightarrow{\mathit{ReLU}} 50\xrightarrow{\mathit{ReLU}}30\xrightarrow{\mathit{softmax}}10.
\end{equation}
We used the $2$ losses (cross entropy and derivative square error) and the $3$ data generators (resampling the dataset, NADE and random noise) as described in the previous sections, resulting into $6$ different models. In all cases, we used minibatches of $20$ samples, and the equivalent of $40$ passes over the train data. We used ADADELTA\index{ADADELTA} \citep{Zeiler:2012:adadelta} to adaptively set the learning rate, with its two hyperparameters set to the values suggested in the original paper. Note that ADADELTA provides no convergence guarantees, however we found that it works well in practice, while it requires minimal tuning.

As already discussed, the MNIST train set was used both as a data generator and in order to train NADE\@. To measure the effect the size of the train set has on model compression, we performed the set of experiments twice. In the first set of experiments, we used the full MNIST train set, whereas in the second set, we only used $10\%$ of it. We measure how effectively the neural ensemble was compressed by evaluating the accuracy and the average log probability of each compressed neural network on the MNIST test set. Figures~\ref{fig:model_compression:results_accuracy_60k}, \ref{fig:model_compression:results_logprob_60k}, \ref{fig:model_compression:results_accuracy_6k} and \ref{fig:model_compression:results_logprob_6k} show the progress during training, and Tables~\ref{table:model_compression:results_acc} and \ref{table:model_compression:results_logprob} show the final results.

\begin{table}[t]
\renewcommand{\tabcolsep}{0.6cm}
\renewcommand{\arraystretch}{\arrstretchvalue}
\centering
\begin{tabular}{cccc}
\toprule
\textbf{Data generator} & \textbf{Train data} & \textbf{Cross entropy} & \textbf{Deriv.~sq.~error}\\
\midrule
\multirow{2}{*}{\textbf{MNIST}} & \textbf{All} & $97.46 \pm 0.31$& $96.50 \pm 0.37$ \\
& \textbf{$\mathbf{10}$\%} & $94.63 \pm 0.45$ & $96.36 \pm 0.37$ \\
\cmidrule{1-4}
\multirow{2}{*}{\textbf{NADE}}& \textbf{All} & $97.60 \pm 0.31$& $96.27 \pm 0.38$ \\
& \textbf{$\mathbf{10}$\%} & $96.64 \pm 0.36$& $95.33 \pm 0.42$  \\
\cmidrule{1-4}
\multirow{2}{*}{\textbf{Gaussian noise}} &  \textbf{All} & $91.98 \pm 0.54$ & $63.98 \pm 0.96$ \\
&  \textbf{$\mathbf{10}$\%} & $85.04 \pm 0.71$  & $75.17 \pm 0.86$  \\
\bottomrule
\end{tabular}
\caption{Performance of each compressed network, as measured by accuracy [\%] on the MNIST test set. The accuracy of the original ensemble is $98.67\% \pm 0.23\%$. Error bars correspond to $2$ standard deviations.}
\label{table:model_compression:results_acc}
\end{table}

\begin{table}[t]
\renewcommand{\tabcolsep}{0.6cm}
\renewcommand{\arraystretch}{\arrstretchvalue}
\centering
\begin{tabular}{cccc}
\toprule
\textbf{Data generator} & \textbf{Train data} & \textbf{Cross entropy} & \textbf{Deriv.~sq.~error} \\
\midrule
\multirow{2}{*}{\textbf{MNIST}} & \textbf{All} & $-0.110 \pm 0.017$ &  $-0.184 \pm 0.025$  \\
& \textbf{$\mathbf{10}$\%} & $-0.232 \pm 0.024$  &  $-0.198 \pm 0.026$   \\
\cmidrule{1-4}
\multirow{2}{*}{\textbf{NADE}}& \textbf{All} &$-0.083 \pm 0.011$  & $-0.177 \pm 0.022$ \\
& \textbf{$\mathbf{10}$\%}  &$-0.113 \pm 0.012$   & $-0.209 \pm 0.024$ \\
\cmidrule{1-4}
\multirow{2}{*}{\textbf{Gaussian noise}} &  \textbf{All}  &$-0.323 \pm 0.011$  & $-1.328 \pm 0.044$ \\
&  \textbf{$\mathbf{10}$\%} &$-0.507 \pm 0.014$  & $-0.662 \pm 0.025$ \\
\bottomrule
\end{tabular}
\caption{Performance of each compressed network, as measured by average log probability on the MNIST test set. The average log probability of the original ensemble is $-0.053 \pm 0.011$ nats. Error bars correspond to $2$ standard deviations.}
\label{table:model_compression:results_logprob}
\end{table}

\begin{table}[t]
\renewcommand{\tabcolsep}{0.6cm}
\renewcommand{\arraystretch}{\arrstretchvalue}
\centering
\begin{tabular}{cccc}
\toprule
\textbf{Data generator} & \textbf{Train data} & \textbf{Cross entropy} & \textbf{Deriv.~sq.~error}\\
\midrule
\multirow{2}{*}{\textbf{MNIST}} & \textbf{All} & $\hphantom{-}0.32 \pm 0.33$& $\hphantom{0}{-0.64} \pm 0.35$ \\
& \textbf{$\mathbf{10}$\%} & $\hphantom{-}0.20 \pm 0.39$  & $\hphantom{-0}1.93 \pm 0.40$ \\
\cmidrule{1-4}
\multirow{2}{*}{\textbf{NADE}}& \textbf{All} & $\hphantom{-}0.46 \pm 0.30$& $\hphantom{0}{-0.87} \pm 0.38$ \\
& \textbf{$\mathbf{10}$\%} & $\hphantom{-}2.21 \pm 0.40$& $\hphantom{-0}0.90 \pm 0.44$  \\
\cmidrule{1-4}
\multirow{2}{*}{\textbf{Gaussian noise}} &  \textbf{All} & $-5.16 \pm 0.53$ & $-33.16 \pm 0.98$ \\
&  \textbf{$\mathbf{10}$\%} & $-9.39 \pm 0.71$ & $-19.26 \pm 0.88$  \\
\bottomrule
\end{tabular}
\caption{Average difference in accuracy [\%] across examples in the MNIST test set, between compressed networks and networks trained directly on data. A positive number indicates that the compressed network has higher accuracy. Error bars correspond to $2$ standard deviations.}
\label{table:model_compression:results_diff_acc}
\end{table}

\begin{table}[t]
\renewcommand{\tabcolsep}{0.6cm}
\renewcommand{\arraystretch}{\arrstretchvalue}
\centering
\begin{tabular}{cccc}
\toprule
\textbf{Data generator} & \textbf{Train data} & \textbf{Cross entropy} & \textbf{Deriv.~sq.~error} \\
\midrule
\multirow{2}{*}{\textbf{MNIST}} & \textbf{All} & $\hphantom{-}0.072 \pm 0.023$ &  $-0.002 \pm 0.024$  \\
& \textbf{$\mathbf{10}$\%} & $\hphantom{-}0.089 \pm 0.024$   &  $\hphantom{-}0.124 \pm 0.027$   \\
\cmidrule{1-4}
\multirow{2}{*}{\textbf{NADE}}& \textbf{All} &$\hphantom{-}0.099 \pm 0.022$  & $\hphantom{-}0.005 \pm 0.024$ \\
& \textbf{$\mathbf{10}$\%}  &$\hphantom{-}0.209 \pm 0.029$  & $\hphantom{-}0.112 \pm 0.029$ \\
\cmidrule{1-4}
\multirow{2}{*}{\textbf{Gaussian noise}} &  \textbf{All}  &$-0.141 \pm 0.026$  & $-1.146 \pm 0.048$ \\
&  \textbf{$\mathbf{10}$\%} & $-0.185 \pm 0.031$  & $-0.341 \pm 0.034$\\
\bottomrule
\end{tabular}
\caption{Average difference in log probability across examples in the MNIST test set, between compressed networks and networks trained directly on data. A positive number indicates that the compressed network has higher log probability. Error bars correspond to $2$ standard deviations.}
\label{table:model_compression:results_diff_logprob}
\end{table}

\begin{table}[t]
\renewcommand{\tabcolsep}{0.7cm}
\renewcommand{\arraystretch}{\arrstretchvalue}
\centering
\begin{tabular}{ccc}
\toprule
\textbf{Train data} & \textbf{Accuracy [\%]} & \textbf{Av.~log probability}\\
\midrule
\textbf{All} & $97.14 \pm 0.33$   & $-0.182 \pm 0.028$ \\
\textbf{$\mathbf{10}$\%} & $94.43 \pm 0.46$  &  $-0.321 \pm 0.035$ \\
\bottomrule
\end{tabular}
\caption{Performance on the MNIST test set of two networks of the same type as the student network that were trained with cross entropy directly on data. Error bars correspond to $2$ standard deviations.}
\label{table:model_compression:results_nets_directly_on_data}
\end{table}

From the results we can see that, in all cases, using random Gaussian noise as data generator performs significantly worse than using MNIST or using NADE\@. This highlights the importance of the data generator for the model compression process. For the student model to be trained properly, it is vital that during training it sees input samples similar to those that it is going to be evaluated on. Still though, when cross entropy is used as a loss, even with random noise the student network classifies correctly about $9$ in $10$ images, and its performance appears to continue increasing with more training samples.

When using the MNIST train set or a NADE trained on it as data generators, the performance appears to be good. To confirm that this is indeed the case, we also trained two neural networks of the same type as the student network, but this time directly on the MNIST train set, using the true train labels. These networks were trained with the same hyperparameters we used for training the ensemble components. One was trained on the full MNIST train set and the other on $10\%$ of it. Their performance is shown in Table~\ref{table:model_compression:results_nets_directly_on_data}, and a paired comparison with the compressed networks is shown in Tables~\ref{table:model_compression:results_diff_acc} and \ref{table:model_compression:results_diff_logprob}. We can see that in most cases the compressed networks have similar or better performance than the networks trained directly on data, especially when only $10\%$ of the dataset was used. It is important to note that even when the accuracy appears to be similar, the log probability of the compressed networks is usually higher, suggesting that the networks trained directly on data are more prone to overfitting. Indeed, using the teacher model to perform the labelling provides soft instead of hard labels, effectively reducing the risk of overfitting.

Comparing cross entropy to derivative square error, we can see that when the full train set was used, cross entropy tends to perform better. However, when the amount of data was limited to $10\%$, the performance of cross entropy dropped significantly, whereas the performance of derivative square error dropped only slightly (as seen from Tables~\ref{table:model_compression:results_acc} and \ref{table:model_compression:results_logprob}, most prominently when MNIST was used as a data generator). This suggest that using derivatives improves the generalization potential of the compressed network. Indeed, by matching derivatives, the student network tries to match a whole tangent hyperplane of $784$ dimensions instead of a single scalar value. This means that every time the teacher model is probed, a lot more information is transferred to the student network about the shape of the function surface it represents. Furthermore, by passing the information about the derivatives, the teacher model effectively informs the student about how the output should change if the input is slightly perturbed in every possible direction. Due to this, the student network can generalize better without having to see more data. On the other hand, when the data is sufficient, it appears to be preferable to use function values instead.

It is interesting to compare NADE to MNIST as a data generators. When using MNIST, the student network sees the same input data over and over again. If the full dataset is used, then this is less of a problem, as shown by the fact that in this case the networks compressed with NADE perform similarly to the networks compressed with MNIST\@. However, when the size of the dataset is reduced, NADE appears to have a clear advantage. When using NADE, the student sees novel input data each time, and it does not recycle the same dataset over and over again. Note that NADE is also affected by the reduction in the train set, since it relies on it to be trained. This can be seen in Figures~\ref{fig:model_compression:samples_nade_60k} and \ref{fig:model_compression:samples_nade_6k}, which show random samples from the NADE trained on the full train set and the NADE trained on $10\%$ of it. Clearly, the quality of samples deteriorates with the reduction in the train set. However, it appears that this affects NADE less that it affects the compression process.

To summarize our conclusions, (a) the choice of data generator is highly important, as shown by the poor performance of random noise, (b) having the teacher provide the labels reduces the risk of overfitting, (c) cross entropy works better than derivative square error if a lot of data is available, (d) derivative square error generalizes better than cross entropy when data is limited, and (e) model compression using NADE has a clear advantage over using the dataset when the latter is small.

\section{Related work}

There is a fairly rich literature related to model compression. In this section, we present the most important approaches to model compression that have been proposed so far, and we compare them to the approach adopted in the present thesis. We also review past work that provided inspiration for the methods discussed in this chapter.

\subsection{Optimal brain damage}
\index{Optimal brain damage}

Optimal brain damage \citep{LeCun:1990:optimal_brain_damage} is one of the first methods proposed for reducing the size of a neural network. The main idea is to take a trained neural network and start deleting parameters from it, by setting them to zero. This is equivalent to removing connections and/or neurons from the network, hence the term ``brain damage''. Brain damage is done ``optimally'' in the sense that those parameters that matter less are deleted first. After deletion of a few parameters, the network is retrained and the process can be iterated as many times as desired.

In the context of optimal brain damage, how much a parameter matters is measured by the change the deletion of the parameter causes to the network's loss function. In general, this change would be costly to measure exactly, however, it can be efficiently approximated by using the second derivatives of the loss function with respect to the parameters. The second derivatives can be easily calculated by a procedure similar to backward propagation \citep{Becker:1989:second_derivs}.

Optimal brain damage is significantly different to knowledge distillation, and it is constrained by the fact that it can only compress neural networks to neural networks. In contrast, knowledge distillation can be used to compress any type of model to a model not necessarily of the same type. The only requirement of knowledge distillation is that the student model must be differentiable, which is a fairly weak one.

\subsection{Model compression}
\index{Model compression}

The term ``model compression'' in the sense used in this chapter was first introduced in a seminal paper by \citet{Bucila:2006:model_compression}. Their work focuses on compressing large ensembles of a multitude of different models into a single neural network that has the capacity to approximate well the function represented by the ensemble. The main idea is to collect or create a large amount of unlabelled data, label it using the ensemble, and then treat it as train data to train the neural network on. As the authors mention, with this setup the neural network hardly runs any risk of overfitting, a fact we also observed in our experiments in section~\ref{sec:model_compression:case_study:results}.

One of the main focuses of \citet{Bucila:2006:model_compression} is how to obtain the large unlabelled data set needed. They discuss both the importance and the difficulty of ensuring that this dataset comes from the correct manifold of the input space, as we also discussed in this chapter. To achieve this, \citet{Bucila:2006:model_compression} propose a method they call MUNGE, which, given a dataset, creates new input datapoints by combining and perturbing the original ones. They also experiment with generative models, but they use a weak Naive Bayes density estimator that is found to underperform.

Our model compression framework contains their method as a special case and it is largely inspired from it. Instead of creating synthetic data by modifying existing datasets however, in our framework we take a more principled approach and we propose to use flexible density estimators such as NADE trained on the original input data. Also, our implementation has low storage requirements, since labelling the input data by the teacher is done in minibatches, thanks to the stochastic way the student is trained.

\subsection{Mimicking neural networks}

In the field of automatic speech recognition, hidden Markov models that use large deep neural networks in calculating their emission probabilities have become considerably successful in recent years \citep{Hinton:2012:DNNs_in_ASR}. Such networks typically have a softmax output layer that produces a vector of probabilities, which correspond to the states of the Markov model. However, deploying large networks on devices like smart phones is impractical, due to limitations on both memory and processing power. 

Motivated by this problem, \citet{Li:2014:small_size_dnn} propose replacing the large deep neural networks found in existing automatic speech recognition systems with smaller ones that have been trained specifically to mimic them. Their method is based on training the small network by minimizing the KL divergence from the output of the large network to the output of the small one, averaged over an unlabelled dataset. Essentially, their method reduces to training the small network using cross entropy, with labels provided by the large network. Note that this method is not necessarily limited to neural networks, but can still only be applied to models that output discrete probability distributions. From this perspective, this work can been seen as a special case both of the model compression framework of \citet{Bucila:2006:model_compression} and of our model compression framework described in this chapter, specifically applied to the domain of automatic speech recognition.

Though not strictly on model compression, another work related to training models to mimic other models is that of \citet{Ba:2014:deep_nets_need_to_be_deep}. The focus there is to provide insights as to whether it is possible for shallow neural networks with a similar number of parameters to represent the functions learnt by deep, possibly convolutional, neural networks. In attempting to answer this question, they train shallow neural networks to mimic deep neural networks, by minimizing the discrepancy between them averaged on a large unlabelled dataset. They conclude that it is indeed possible to represent the same functions with shallow networks, and that the reason shallow networks are not typically used in practice is their tendency to overfit when trained directly on data.

In their work, they found that for networks with softmax outputs, matching the logits\index{Logit} by minimizing square error was more effective than matching the outputs by minimizing cross entropy. This is a particularly interesting insight, as it indicates new potential ways of performing knowledge distillation. However, this method is only limited to a specific type of networks, namely those with softmax outputs.

\subsection{Knowledge distillation}
\index{Knowledge distillation}

The term ``knowledge distillation'' originated from the recent work of \citet{Hinton:2015:distilling_knowledge}. Whereas in this thesis we are using it more broadly to mean ``transferring knowledge from a cumbersome model to a convenient model'', in their work they use it only in the context of model compression.
They focus on neural networks, or ensembles thereof, that have a softmax output. Their goal is to compress them to smaller neural networks with a softmax output. Their method is in principle similar to the standard one of \citet{Bucila:2006:model_compression}, however they come up with a novel insight that leads to a slightly different approach. 

Their main observation is that, in the output of the softmax layer, the relative magnitudes of the different components, even of those that are very small, contain a significant amount of knowledge. However, when training with cross entropy, the training signal is typically dominated by the largest component, which is often very close to $1$. In order to extract the knowledge hidden in the small differences between the tiny components, they propose scaling the logits\index{Logit} by a factor of $\nicefrac{1}{T}$. Parameter $T$ can be interpreted as a ``temperature'', and the higher it is set the smoother the softmax output becomes, thus making the small differences in the probability components become more pronounced. The drawback of their approach is that $T$, as a newly introduced free hyperparameter, needs to be tuned.

Increasing the temperature is not the only way of enhancing the differences between tiny probability components. The same effect can be achieved by transforming the probabilities to a different domain. For example, \citet{Ba:2014:deep_nets_need_to_be_deep} use the logit domain, and in our derivative square error loss we use the log domain, both of which succeed in enhancing the small differences. Indeed, \citet{Hinton:2015:distilling_knowledge} argue that in the limit of setting the temperature infinitely high, their method approximates a transformation to the logit domain.

\subsection{FitNets}
\index{FitNets}

FitNets is an enhancement to the techniques used for model compression, proposed in a recent work by \citet{Romero:2014:fitnets}. The authors argue that depth is highly important for the representation power of neural networks and therefore it should not be sacrificed when compressing them. Even more so, they argue that the compressed networks should become even deeper, so that they can compensate with depth for the loss in parameters. This leads to deep and thin compressed networks (playfully termed FitNets), on which, according to the authors, knowledge distillation approaches such as that of \citet{Hinton:2015:distilling_knowledge} tend to underperform.

In order to make the training of FitNets more successful, they propose having the teacher network provide supervision not only to the final layer of the student network, but also to some of its intermediate layers. That is, in addition to matching the outputs, they also attempt to match the intermediate representations of the teacher and the student. Unlike the output layers which are the same, the intermediate layers of the teacher and the student are of different size and typically learn different representations. Therefore, for making their matching possible, they propose having a regressor that tries to predict the teacher's intermediate representations from the student's intermediate representations. Training the student to make the regressor predict the teacher's intermediate representations well becomes then part of the training objective during compression.

This approach departs from using only the output of models, which has so far been the dominant approach in model compression, including our framework presented in this thesis. Therefore, it is interesting to see what new research directions it may give rise to in the near future. Its drawback however is that it is specific to neural networks and it is not obvious how it can be generalized to other models.

\subsection{Matching derivatives}

In this chapter, we argued that knowing the derivatives of the function represented by the teacher with respect to the input provides a significant amount of information for the student. However, in machine learning, training models to learn input-output mappings has traditionally been done on labelled datasets, from which it is not obvious how to extract derivative information. Model compression is a rare example of a setup where the whole function to be learnt is readily available, and therefore its derivatives can be computed. To the best of our knowledge, our work is the first in the literature that proposed using derivatives in the context of model compression. In the following, we review three works that have used derivatives in related contexts.

\subsubsection{Tangent prop}

The importance of knowing the derivatives of the function to be learnt was also highlighted by \citet{Simard:1992:tangentprop}. In their work, they made a keen observation that for certain transformations of input points, the network's predictions should not change. For example, a slightly rotated image of the digit ``$3$'' should still be classified as a ``$3$''. This implies that the directional derivative to the direction of the transformation should be zero. 

To introduce this invariance to the training process, they propose adding an extra term to the training objective that encourages these directional derivatives to be zero. This extra term is similar to our derivative square error loss, except that it does not use the log domain and it contains only certain directional derivatives. This extra term can  be differentiated with respect to the model parameters using a modified version of backprop, termed \emph{tangent prop}\index{Tangent prop}, which is similar to R\{backprop\} as used by us (see section~\ref{sec:derivatives_in_neural_nets:Rbackprop}). The main difference with our work is that we can evaluate the full derivative with respect to the input, and by matching this, we can essentially match the full tangent hyperplane and not just a few selected directional derivatives.

\subsubsection{Contractive auto-encoders}

Encouraging derivatives to be zero for introducing invariance was also proposed by \citet{Rifai:2011:contractive_autoencoders} in the context of auto-encoders\index{Auto-encoder}. Auto-encoders \citep[section~5.2]{Bengio:2009:deep_AI} are neural networks consisting of two cascaded phases: the encoding phase, which transforms the input to an intermediate representation, and the decoding phase, which reconstructs the input from the intermediate representation. Auto-encoders are typically trained by minimizing the discrepancy between the input and its reconstruction.

\citet{Rifai:2011:contractive_autoencoders} argue that the intermediate representation of an auto-encoder should be invariant to small changes in the input. In other words, small perturbations of the input should yield similar intermediate representations. This is equivalent to saying that the derivatives of the intermediate representation with respect to the input should be small. Motivated by this, they introduce \emph{contractive auto-encoders}\index{Auto-encoder!Contractive auto-encoder}, which are trained by adding an extra regularization term to their loss function, consisting of the sum of the squares of all partial derivatives of the intermediate representation with respect to the input. Viewed from our framework's perspective, this is equivalent to matching all derivatives to zero by derivative square error.

\subsubsection{Score matching}

The idea of matching derivatives was also used by \citet{Hyvarinen:2005:score_matching}, where it was termed \emph{score matching}\index{Score matching}. There, the focus is on learning unnormalizable generative models. Evaluating the exact probability distribution of an unnormalizable generative model is intractable, making maximum likelihood learning hard. However, the derivative of the log probability is tractable, because it does not involve the intractable normalization constant. The main idea of score matching is to train the intractable model by matching its tractable derivatives with the derivatives of the distribution that we want it to learn.

Score matching is similar to our derivative square error loss, the difference being that in score matching the derivatives are with respect to the distribution's random variable, whereas in derivative square error they are with respect to the input. Also, score matching is used in a different context to that of model compression. Nevertheless, score matching is used in our framework for distilling intractable models in chapter~\ref{chapter:generative_models}, and more details about it can be found in section~\ref{sec:generative_models:loss_functions}.

\section{Summary and conclusions}

In this chapter we presented a framework for model compression. We view model compression as a special case of knowledge distillation, where the knowledge of a large model is distilled into a small model. Our framework is fairly general, and it can in principle be applied to any sort of model, as long as the student model is differentiable.

The two key elements of our framework are the data generator and the loss. We commented extensively on the vital importance of the data generator, an importance which we also demonstrated in our experiments. We showed that a flexible generative model like NADE can be successfully used as a data generator, one which is capable of providing as much train data as desired. 

To the best of our knowledge, we are the first to have used derivative matching in the model compression literature. We argued that knowing the derivatives of the teacher model with respect to the input provides significant information for the student, and we showed that in some cases knowing only the derivatives is sufficient for training. In our experiments, we demonstrated that matching derivatives has a greater generalization potential than matching function values when the input data is limited.

The literature on model compression is quite rich, and we made a detailed review of it. The main ideas found in the literature are rather similar, and have been renamed over the years as ``model compression'', ``mimicking'' and ``knowledge distillation''. We hope our work has added new ideas to this literature, and a practical implementation of them.

\chapter{Compact Predictive Distributions}
\label{chapter:compact_predictive_distributions}

Bayesian inference\index{Bayesian inference} is a particular way of doing machine learning, based on using nothing but the rules of probability \citep{Barber:2012:BRML, MacKay:2002:IT}. In Bayesian inference, every quantity of interest, regardless of what it corresponds to, admits a probability distribution, which represents our uncertainty about its value. Then, both learning and making predictions are done by simply using the rules of probability throughout the process.

Bayesian inference is elegant, principled and conceptually simple, and it is based on the well-founded and well-understood theory of probability. Nevertheless, the rules of probability, despite their seemingly innocuous simplicity, lead to intractable computations in all but the most trivial cases. This computational difficulty of Bayesian inference is one of the main factors that have hindered its widespread adoption in real-life applications.

One way of tackling the computational difficulty of Bayesian inference is via Monte Carlo\index{Monte Carlo}\footnote{Named after the eponymous casino in the tiny principality of Monaco.} \citep[chapter 11]{Bishop:2006:PRML}. The idea behind Monte Carlo is to use a bag of samples generated from a probability distribution as a proxy for it. Having access to such a bag of samples, any expectation over the original probability distribution can be approximated by the arithmetic average over the samples. Moreover, this approximation can be made arbitrarily good by sufficiently increasing the number of samples. Hence, expectations that might be intractable to calculate analytically, can be effectively approximated via Monte Carlo.

However, the quality of the Monte Carlo approximation improves very slowly with the number of samples. Assuming samples are independent, the standard deviation of a Monte Carlo estimate decreases as $\bigo{\nicefrac{1}{\sqrt{S}}}$, where $S$ is the number of samples. If the samples are not independent, which is more often the case, the standard deviation of the estimate decreases even more slowly. For ballpark estimates this low precision might not be a problem, but for high accuracy computations a very large number of samples may be needed. Furthermore, if the distribution we wish to sample from is defined over e.g.~the half-million parameter vector of a large deep neural network, then the Monte Carlo approximation can become prohibitively expensive in both computation time and storage space.

It is natural then to ask the question, having been given a large bag of samples, is it possible to distil the knowledge about the original distribution contained in the samples to a more compact form? That is, to a form that takes little time to evaluate and limited space to store? If the answer is yes, then we can throw away the original bulky bag of samples and use the compact model instead.

In this chapter we describe a set of techniques for distilling bags of samples to a compact model. Our framework is based on that of \cite{Snelson:2005:compact_approximations}, but viewed from the perspective of knowledge distillation. In addition, we propose novel extensions to the framework, allowing us to avoid storing the full set of samples and to take advantage of derivative information about the target prediction surface. We validate our framework in two Bayesian inference problems: Bayesian density estimation and Bayesian binary classification.

\section{Distilling bags of MCMC samples}

A typical setup of a problem in Bayesian inference\index{Bayesian inference} is as follows. Suppose we have observed some data $D = \set{\vect{x}_1, \vect{x}_2, \ldots, \vect{x}_N}$, which we know were independently generated by some model $\prob{\vect{x}\g\vect{w}}$, parameterized by a set of unknown parameters $\vect{w}$. That is, we do not know what the value of $\vect{w}$ is for the model that actually generated $D$. The question we want to answer is, given that we observed $D$, how likely do we believe $\vect{x}$ to be? That is, what is our prediction for the next point to be generated by the model?

Essentially, what the question asks is to calculate the probability $\prob{\vect{x}\g D}$, which is known as the \emph{predictive distribution}\index{Predictive distribution}. In order to do this, we need to have some idea of what the parameters $\vect{w}$ were before we observed $D$. That is, we need to know---or assume---$\prob{\vect{w}}$, which is known as the \emph{prior distribution}\index{Prior distribution}. Finally, for simplicity, let us assume that if we actually knew the true value of $\vect{w}$, then our observations $D$ would not influence our predictions about $\vect{x}$. Based on the above assumptions, and using nothing but the rules of probability, it is easy to show that the predictive distribution is given by
\begin{equation}
\prob{\vect{x}\g D} = \avg{\prob{\vect{x}\g\vect{w}}}{\prob{\vect{w}\g D}},
\end{equation}
where
\begin{equation}
\prob{\vect{w}\g D} \propto \prob{\vect{w}}\prod_n{\prob{\vect{x}_n\g \vect{w}}}.
\end{equation}
In the above equations, $\prob{\vect{w}\g D}$ is known as the \emph{posterior distribution}\index{Posterior distribution} and $\prob{D\g\vect{w}} = \prod_n{\prob{\vect{x}_n\g\vect{w}}}$ as the \emph{likelihood}\index{Likelihood}. The notation $\avg{\prob{\vect{x}\g\vect{w}}}{\prob{\vect{w}\g D}}$ denotes expectation of $\prob{\vect{x}\g\vect{w}}$ with respect to $\prob{\vect{w}\g D}$.

In most practical cases, the computation described by the above equations is intractable, even for tractable prior distribution and likelihood. The posterior $\prob{\vect{w}\g D}$ is typically unnormalizable in practice, that is, it can only be evaluated up to a multiplicative constant. Moreover, the expectation over the posterior that defines the predictive distribution $\prob{\vect{x}\g D}$ is typically high-dimensional and intractable to compute, either analytically or by numerical integration.

One way of approximating the above computation is via Markov Chain Monte Carlo\index{MCMC} (or MCMC for short). MCMC works by setting up a Markov chain whose equilibrium distribution is the posterior $\prob{\vect{w}\g D}$. Then, by simulating this chain, a set of samples $\set{\vect{w}_s}$ from $\prob{\vect{w}\g D}$ can be obtained. The advantage of using MCMC over other sampling methods is that (a) MCMC only needs to know $\prob{\vect{w}\g D}$ up to a multiplicative constant, which as we already discussed is usually the case, and (b) MCMC scales better than other sampling methods to high-dimensional spaces, which the $\vect{w}$ space often is. There are numerous different methods in the MCMC literature, which is extensively reviewed by \citet{Neal:1993:mcmc}
and \citet{Murray:2007:mcmc}.

Having sampled the posterior using MCMC and having obtained a set of parameter samples $\set{\vect{w}_s}$, the predictive distribution can be approximated by the following Monte Carlo estimate
\begin{equation}
\prob{\vect{x}\g D} \approx p_\mathrm{MC}\br{\vect{x}} = \frac{1}{S}\sum_s{\prob{\vect{x}\g \vect{w}_s}}.
\end{equation}
The above however is an expensive and inefficient representation of the predictive distribution, especially if highly accurate estimates are needed and therefore a large number of samples has to be collected. Indeed, every time a prediction needs to be made for some particular value of $\vect{x}$, the full collection of samples needs to be scanned. Storing the samples can also be expensive, especially if $\vect{w}$ is high-dimensional.

To alleviate this problem and make MCMC inference more efficient, we can distil the knowledge contained in $\set{\vect{w}_s}$ into a compact representation of the predictive distribution. Let $\prob{\vect{x}\g \bm{\theta}}$ be a family of distributions, parameterized by a set of parameters $\bm{\theta}$, that we have chosen to be compact but at the same time flexible enough to approximate the true predictive distribution. The goal of distillation is to find the particular value of $\bm{\theta}$ for which $\prob{\vect{x}\g \bm{\theta}}$ is as good an approximation to $p_\mathrm{MC}\br{\vect{x}}$ as possible. That is, distillation trains $\prob{\vect{x}\g \bm{\theta}}$ in order to make
\begin{equation}
p_\mathrm{MC}\br{\vect{x}} \approx \prob{\vect{x}\g \bm{\theta}}
\end{equation}
become as close to an equality as possible.

In the rest of this chapter, we will describe a set of techniques for performing knowledge distillation of bags of MCMC samples in the context Bayesian inference.
Following \citet{Snelson:2005:compact_approximations}, we consider two different problems, Bayesian density estimation and Bayesian binary classification. 

\section{Bayesian density estimation}

In this section, we study the problem of Bayesian density estimation, where $\prob{\vect{x}\g \vect{w}}$ is a density model of $\vect{x}$, such as a mixture of Gaussians. We wish to learn a compact model $\prob{\vect{x}\g \bm{\theta}}$ to approximate the predictive distribution $\prob{\vect{x}\g D}$, given a Markov chain that explores the posterior $\prob{\vect{w}\g D}$.

We will start by defining a loss function that measures the discrepancy between $\prob{\vect{x}\g D}$ and $\prob{\vect{x}\g \bm{\theta}}$. Learning the compact model amounts to minimizing this loss function with respect to $\bm{\theta}$. A natural, information-theoretic measure of discrepancy between distributions is the KL divergence\index{KL divergence} from $\prob{\vect{x}\g D}$ to $\prob{\vect{x}\g \bm{\theta}}$, defined as follows
\begin{equation}
\kl{\prob{\vect{x}\g D}}{\prob{\vect{x}\g \bm{\theta}}} = \avg{\log{\prob{\vect{x}\g D}}}{\prob{\vect{x}\g D}}
- \avg{\log{\prob{\vect{x}\g  \bm{\theta}}}}{\prob{\vect{x}\g D}}.
\end{equation}
The term $\avg{\log{\prob{\vect{x}\g D}}}{\prob{\vect{x}\g D}}$ is the negative entropy of $\prob{\vect{x}\g D}$, and it is a constant with respect to $\bm{\theta}$. Hence, minimizing the above KL divergence is equivalent to maximizing $\avg{\log{\prob{\vect{x}\g  \bm{\theta}}}}{\prob{\vect{x}\g D}}$. This can be interpreted as fitting $\prob{\vect{x}\g  \bm{\theta}}$ using maximum likelihood to an infinite amount of data generated by $\prob{\vect{x}\g  D}$.

Our basic assumption is that all our knowledge about $\prob{\vect{x}\g  D}$ comes from a Markov chain that generates MCMC samples from it. In the following two sections, we describe two ways of maximizing $\avg{\log{\prob{\vect{x}\g  \bm{\theta}}}}{\prob{\vect{x}\g D}}$ using this Markov chain: a ``batch'' way, first proposed by \citet{Snelson:2005:compact_approximations}, and a new ``online'' way proposed herein.

\subsection{Batch distillation}

By simulating the Markov chain, we can generate a set of MCMC samples $\set{\vect{w}_s}$ from the posterior $\prob{\vect{w}\g D}$. The Monte Carlo estimate
\begin{equation}
p_\mathrm{MC}\br{\vect{x}} = \frac{1}{S}\sum_s{\prob{\vect{x}\g \vect{w}_s}}
\end{equation}
can then be used as an approximation to the true predictive distribution $\prob{\vect{x}\g D}$. Hence, maximizing $\avg{\log{\prob{\vect{x}\g  \bm{\theta}}}}{\prob{\vect{x}\g D}}$ can be approximately done by (a) generating a set of samples $\set{\vect{x}_m}$ from $p_\mathrm{MC}\br{\vect{x}}$ and (b) fitting $\prob{\vect{x}\g  \bm{\theta}}$ to $\set{\vect{x}_m}$ using maximum likelihood. This procedure is outlined in the algorithm below.
\begin{framed}
\begin{enumerate}[label=(\roman*)]
\item Generate a bag of samples $\set{\vect{w}_s}$ of size $S$ from $\prob{\vect{w}\g D}$ by simulating the Markov chain.
\item Generate a bag of samples $\set{\vect{x}_m}$ of size $M$ from $p_\mathrm{MC}\br{\vect{x}}$.
\item Train $\prob{\vect{x}\g  \bm{\theta}}$ by maximizing its average log likelihood on $\set{\vect{x}_m}$.
\end{enumerate}
\end{framed}

\subsection{Online distillation}

The batch algorithm works well in practice (as we shall see later in the experiments), but requires the full sets of samples $\set{\vect{w}_s}$ and $\set{\vect{x}_m}$ to be generated and stored before $\bm{\theta}$ is optimized. This is problematic when a lot of samples are required and $\vect{w}$ or $\vect{x}$ (or both) are high-dimensional.

We will now describe an alternative online method for learning $\prob{\vect{x}\g  \bm{\theta}}$ that iteratively updates $\bm{\theta}$ on the fly as the Markov chain is simulated. Firstly, notice that
\begin{equation}
\avg{\log{\prob{\vect{x}\g  \bm{\theta}}}}{\prob{\vect{x}\g D}}
 = \avg{\log{\prob{\vect{x}\g  \bm{\theta}}}}{\prob{\vect{x}, \vect{w}\g D}},
\end{equation}
where $\prob{\vect{x}, \vect{w}\g D} = \prob{\vect{x}\g \vect{w}}\prob{\vect{w}\g D}$. 
Hence, a pair of samples $\pair{\vect{x}_s}{\vect{w}_s}$ can be generated by first simulating the 
Markov chain for a single step to obtain $\vect{w}_s$ and then sampling $\vect{x}_s$ from $\prob{\vect{x}\g \vect{w}_s}$. Having generated a minibatch  $\set{\pair{\vect{x}_s}{\vect{w}_s}}$ of size $S$, a stochastic update to $\bm{\theta}$ can be made. Hence, optimizing $\bm{\theta}$ can be done by iterating the above procedure, as outlined in the algorithm below.
\begin{framed}
\begin{enumerate}[label=(\roman*)]
\item Generate a minibatch of pairs of samples $\set{\pair{\vect{x}_s}{\vect{w}_s}}$ of size $S$ from $\prob{\vect{w}\g D}$ by simulating the Markov chain.
\item Make a stochastic update to $\bm{\theta}$ using only the minibatch.
\item Repeat until $\bm{\theta}$ converges.
\end{enumerate}
\end{framed}

In the above algorithm, the size of the minibatch $S$ can be made as small as desired. The big advantage of this algorithm is that its storage requirements are $\bigo{S}$, since every minibatch can be thrown away after having been used. Compare this to the batch algorithm, where all the samples have to be stored in advance.

The downside of the online algorithm is that it requires care in arranging the stochastic updates, so that the algorithm remains stable. This is especially important when the minibatch is chosen to be small, since in this case the updates to $\bm{\theta}$ have higher variance. A stochastic update can be, for instance, a single gradient update, such as
\begin{equation}
\bm{\theta} \leftarrow \bm{\theta} - \frac{\alpha}{S}\,\sum_s{\pderivnull{\bm{\theta}}\log{\prob{\vect{x}_s\g  \bm{\theta}}}},
\end{equation}
where $\alpha > 0$ is an appropriately chosen, possibly adaptive, learning rate.
Alternatively, $\prob{\vect{x}\g \bm{\theta}}$ can be fitted on the minibatch, and the $\bm{\theta}$ can be interpolated between its new and its old value. That is
\begin{equation}
\bm{\theta} \leftarrow \br{1-\alpha}\,\bm{\theta} + \alpha\,\bm{\theta}',
\end{equation}
where
\begin{equation}
\bm{\theta}' = \arg\max_{\bm{\theta}}{\frac{1}{S}\sum_s{\log\prob{\vect{x}_s\g  \bm{\theta}}}}
\end{equation}
and $0<\alpha<1$ an appropriately chosen, possibly adaptive, learning rate. In our case study in section~\ref{sec:compact_predictive:mogs} we will provide more details on our particular choice of update strategy.

\subsection{Case study: Bayesian Mixture of Gaussians}
\label{sec:compact_predictive:mogs}
\index{MoG}

We will now showcase our framework for Bayesian density estimation on a toy inference task involving Mixtures of Gaussians (or MoG for short). The goal of the task is, given a dataset of observations that have been generated by a MoG with unknown means, to infer the predictive density.

\subsubsection{The setup}

The true model that generated the observations is a one-dimensional MoG with $3$ components, given by
\begin{equation}
\prob{x\g \vect{w}} = \frac{1}{3}\,\gaussian{x}{m_1}{s_1^2}
+\frac{1}{3}\,\gaussian{x}{m_2}{s_2^2}+\frac{1}{3}\,\gaussian{x}{m_3}{s_3^2},
\end{equation}
where
\begin{equation}
\renewcommand{\arraystretch}{\arrstretchvalue}
\begin{array}{ccc}
m_1 = -3 & \qquad m_2 =  0 &\qquad m_3 =  2 \\
s_1^2 = 2 &\qquad s_2^2 = 5  &\qquad s_3^2 = 1 
\end{array}.
\end{equation}
We assume that the variances $s_1^2, s_3^2, s_3^2$ are known, but the means are not. We therefore identify the unknown model parameters to be $\vect{w} = \set{m_1, m_2, m_3}$, over which we place the following broad Gaussian prior
\begin{equation}
\prob{\vect{w}} = \gaussian{\vect{w}}{\vect{0}}{100\mat{I}}.
\end{equation}

In our experiments, we generated a dataset of observations $D = \set{x_1, x_2, \ldots, x_N}$ from the true model. We varied $N$ across experiments. Pretending not to know the true means of the MoG, the task is to calculate $\prob{x\g D}$. Note that even in this toy setting, the true posterior $\prob{\vect{w}\g D}$ is a sum of $3^N$ Gaussians. Therefore, exact inference becomes hard even for a small set of observations, and an approximation like MCMC is typically needed.

\subsubsection{Distillation details}

To generate samples from the posterior $\prob{\vect{w}\g D}$, we used slice sampling\index{Sampling!Slice sampling} \citep{Neal:2000:slice_sampling}. Slice sampling is a general MCMC method, whose main advantage is that it requires minimal tuning. We used linear stepping out, and we performed univariate updates to each parameter in turn. The chain was initialized at $\vect{0}$, and burned in for $1000$ iterations. No thinning was used.

The compact model into which we chose to distil the Markov chain is a MoG with $3$ components, given by
\begin{equation}
\prob{x\g \bm{\theta}} = \pi_1\,\gaussian{x}{\mu_1}{\sigma_1^2}
+\pi_2\,\gaussian{x}{\mu_2}{\sigma_2^2}+\pi_3\,\gaussian{x}{\mu_3}{\sigma_3^2},
\end{equation}
whose free parameters are $\bm{\theta} = \set{\pi_1, \pi_2, \pi_3, \mu_1, \mu_2, \mu_3, \sigma_1^2, \sigma_2^2, \sigma_3^2}$. For batch distillation, we first generated 
$10{,}000$ MCMC samples $\set{\vect{w}_s}$ in order to form the approximate posterior 
$p_{\mathrm{MC}}\br{x}$. Note that in this case $p_{\mathrm{MC}}\br{x}$ 
is a MoG with $30{,}000$ components, so we can exactly sample from it. From $p_{\mathrm{MC}}\br{x}$, we generated $1000$ samples $\set{x_m}$. We then used the EM algorithm \citep{Dempster:1977:em} to fit $\prob{x\g \bm{\theta}}$ on $\set{x_m}$ using maximum likelihood. We ran EM until convergence of the likelihood value. More details on the EM algorithm are given in the appendix, section~\ref{sec:em:batch}.

For online distillation, we used minibatches of size $100$. That is, in each iteration, $100$ MCMC samples $\set{\vect{w}_s}$ and $100$ corresponding data samples $\set{x_s}$ were generated. In each iteration, $\bm{\theta}$ was updated using an online version of the EM algorithm \citep{Cappe:2009:online_em}. This version of EM makes a soft update on $\bm{\theta}$ in each iteration, by interpolating between the sufficient statistics calculated from the minibatch $\set{x_s}$ and the sufficient statistics calculated so far. More details on how this particular version of online EM works are given in the appendix, section~\ref{sec:em:online}. 
Note that online EM needs a learning rate to be specified for each iteration. In our experiments, we used a learning rate that was initialized to $1$, and then was decayed linearly until it reached $0$ in the last iteration. We performed $100$ iterations in total.

\subsubsection{Results and discussion}

\def\imwidth{0.45\textwidth}

\begin{figure}[p]
\centering
\subfloat[Histograms of parameter posterior samples]{
\includegraphics[width=\imwidth]{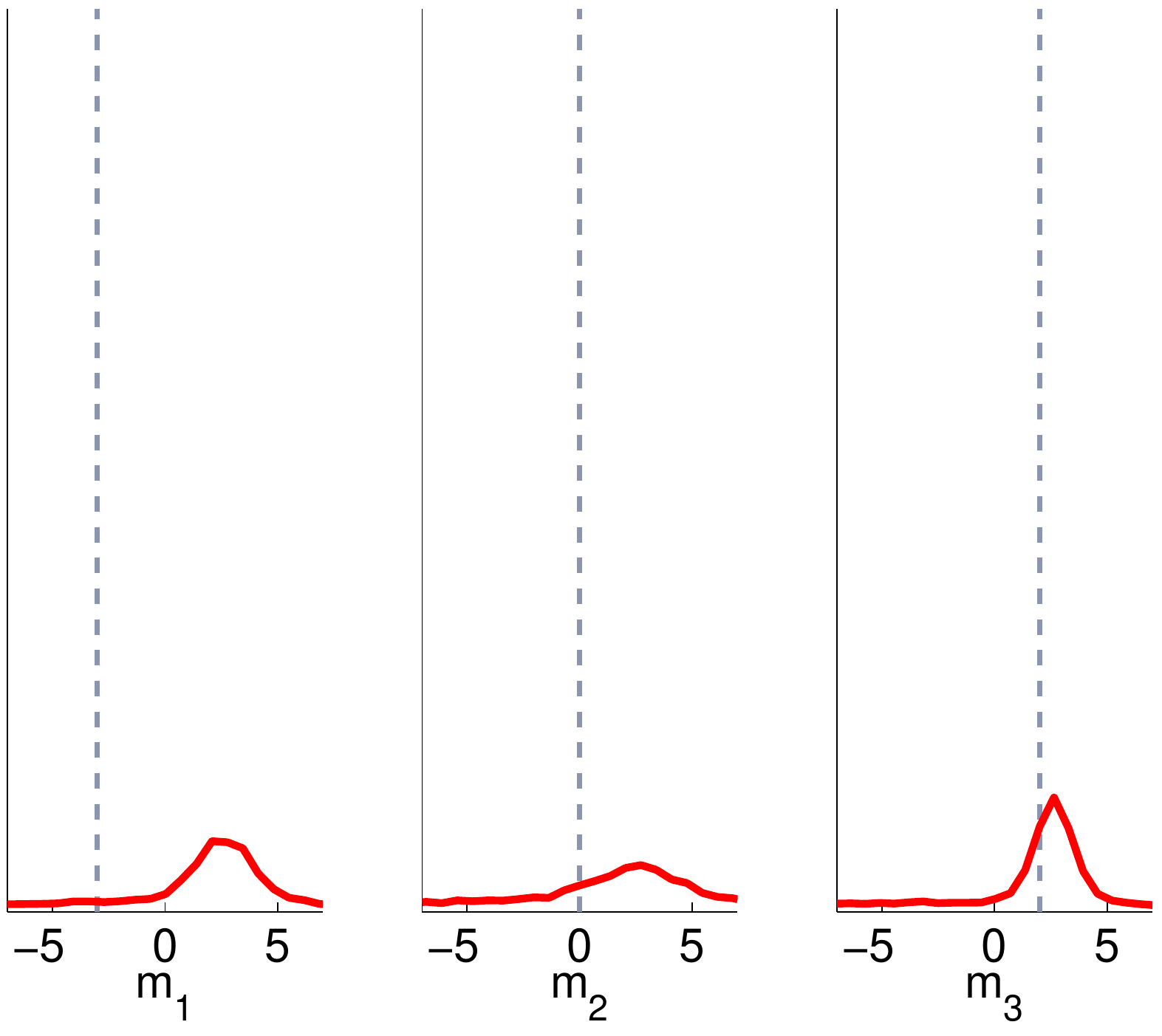}}
\hfill
\subfloat[Predictive densities]{
\includegraphics[width=\imwidth]{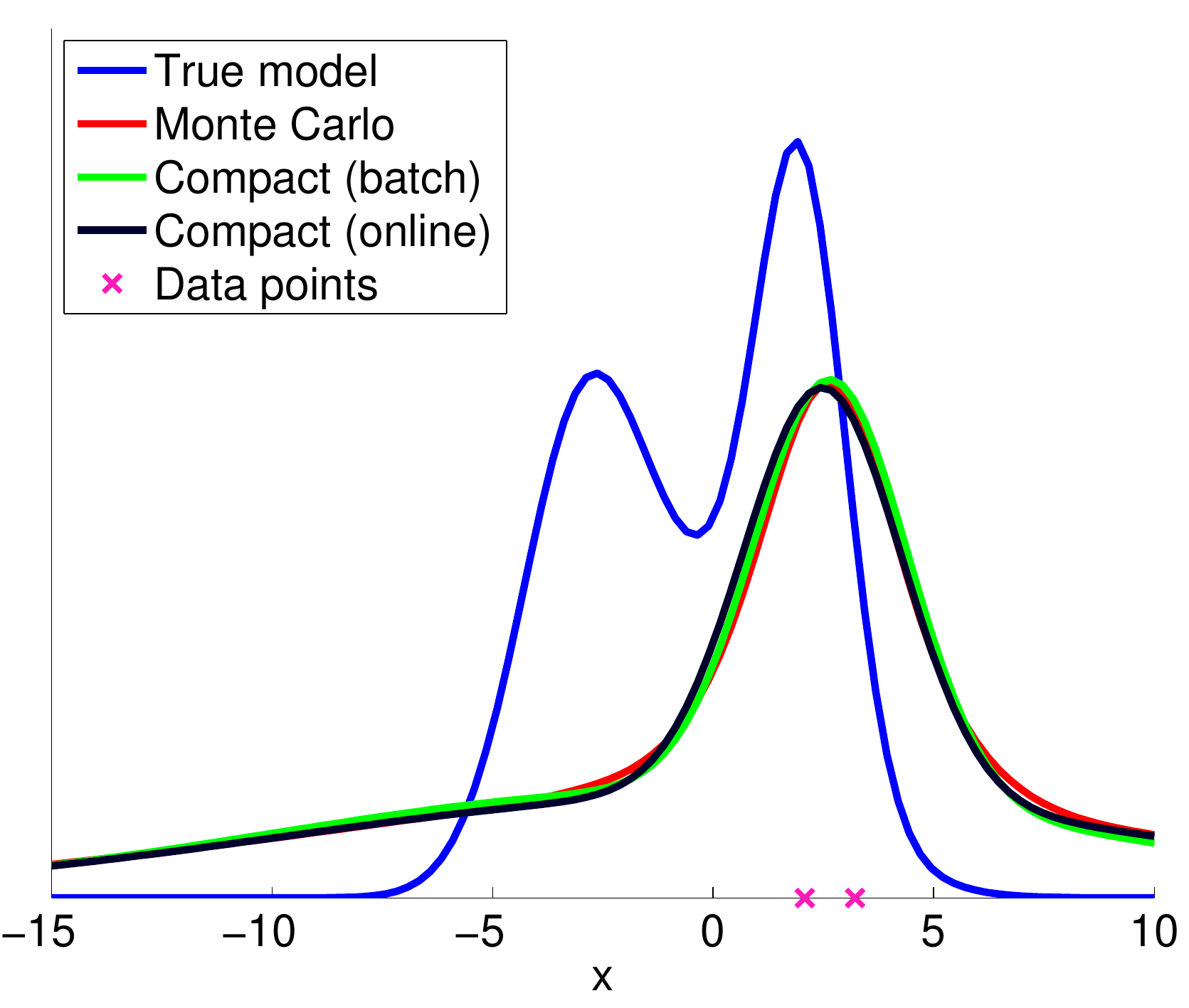}}
\caption{Bayesian MoG with $2$ datapoints.}
\label{fig:compact_predictive:density_results_2}
\end{figure}

\begin{figure}[p]
\centering
\subfloat[Histograms of parameter posterior samples]{
\includegraphics[width=\imwidth]{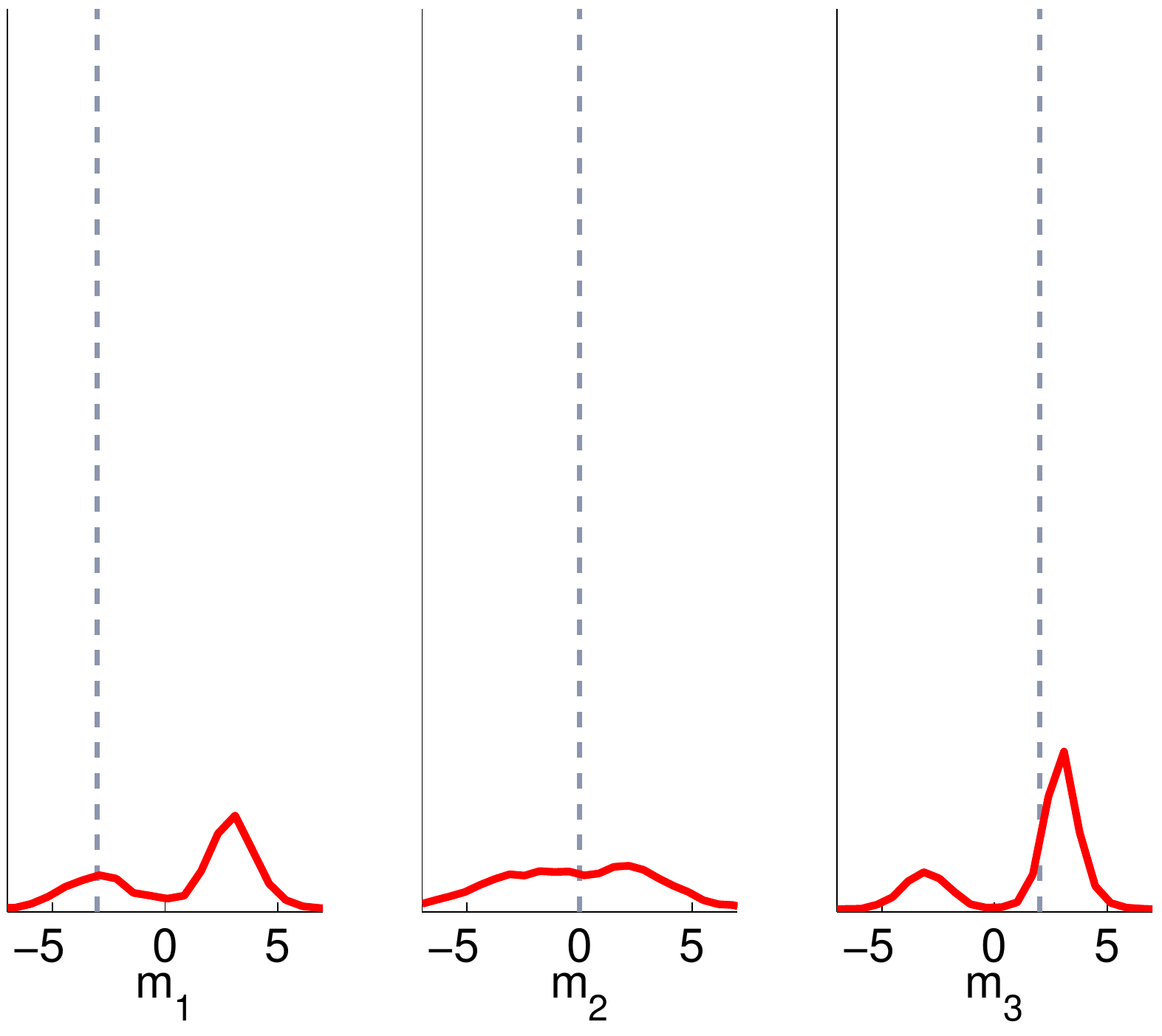}}
\hfill
\subfloat[Predictive densities]{
\includegraphics[width=\imwidth]{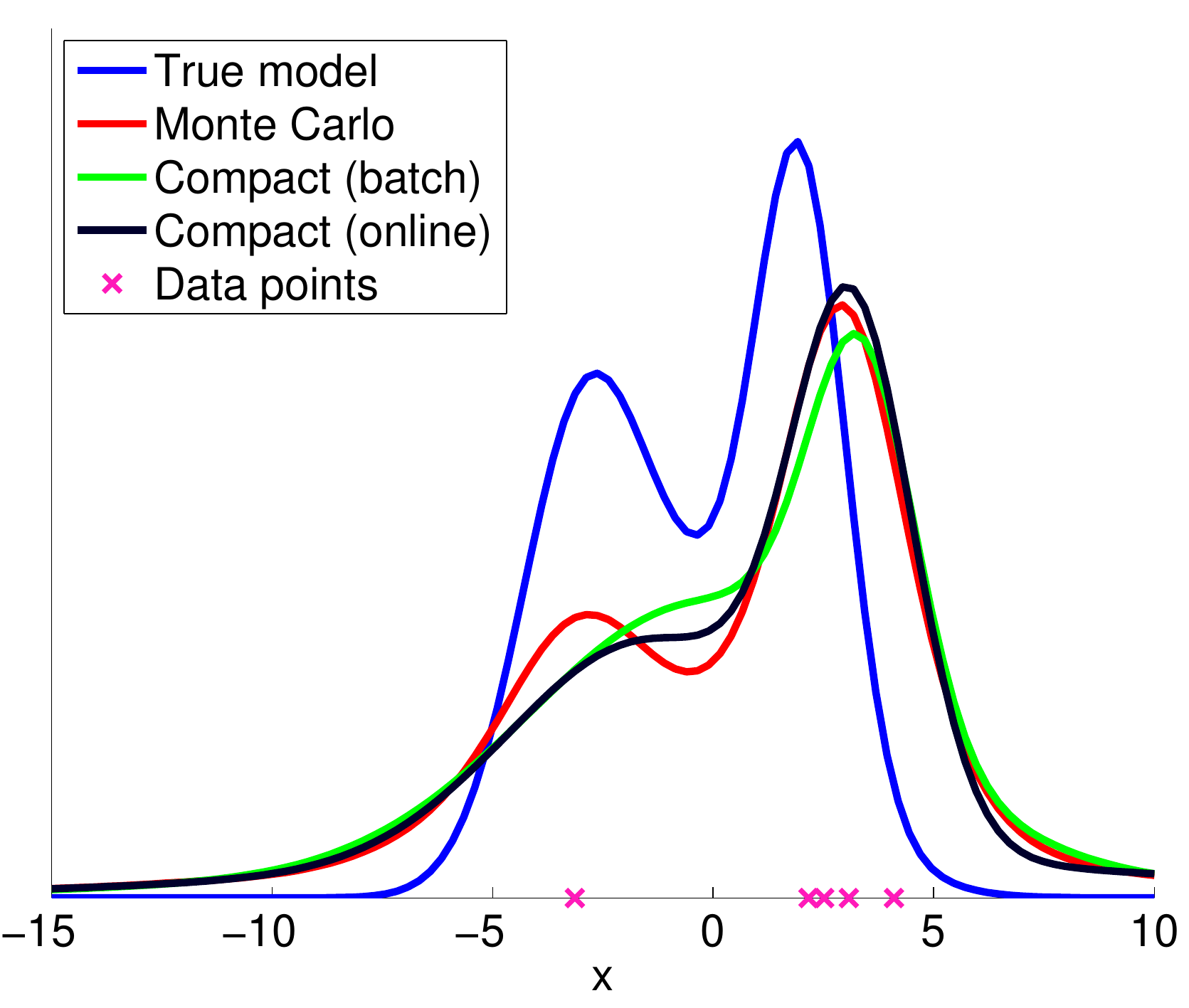}}
\caption{Bayesian MoG with $5$ datapoints.}
\label{fig:compact_predictive:density_results_5}
\end{figure}

\begin{figure}[p]
\centering
\subfloat[Histograms of parameter posterior samples]{
\includegraphics[width=\imwidth]{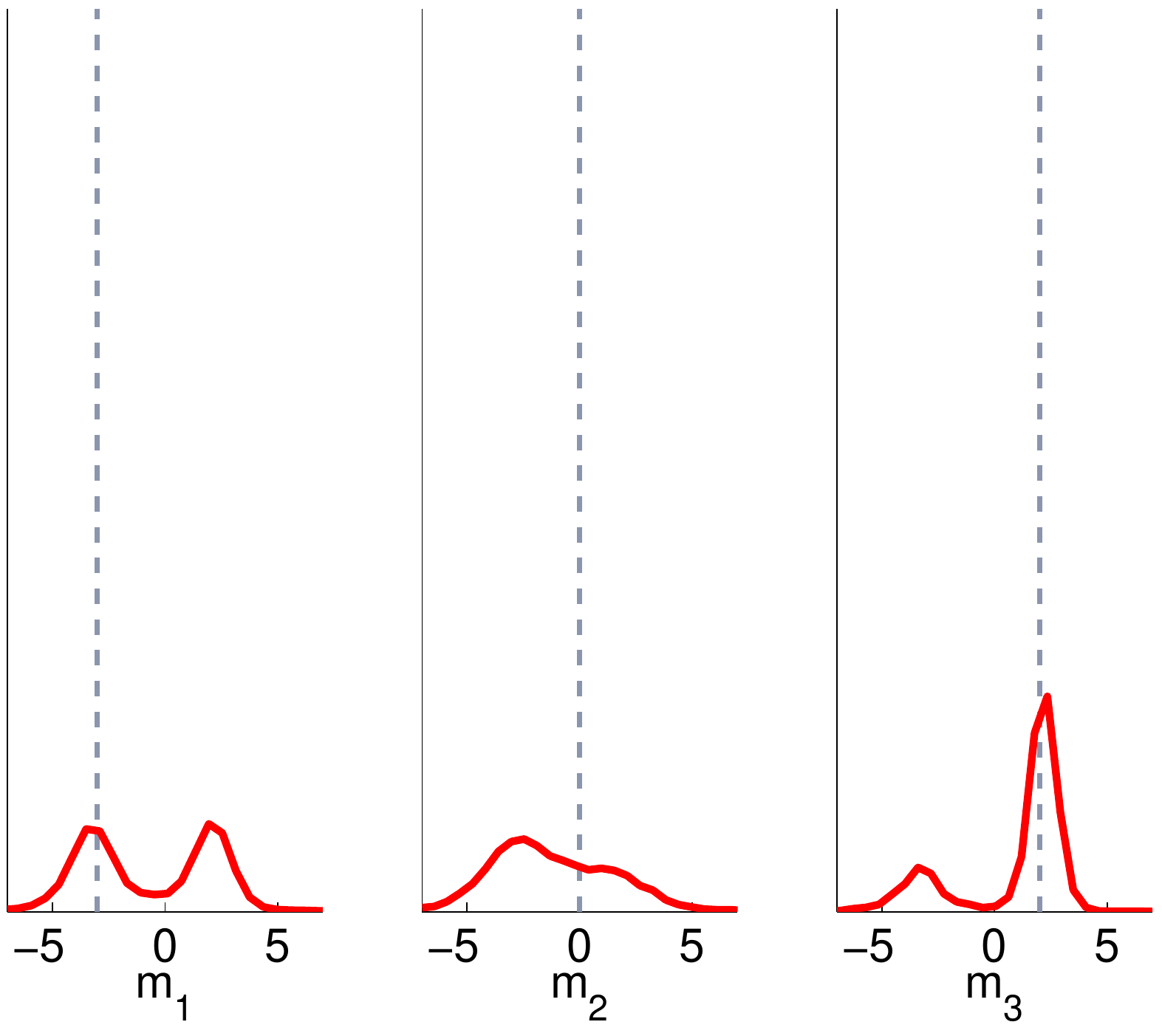}}
\hfill
\subfloat[Predictive densities]{
\includegraphics[width=\imwidth]{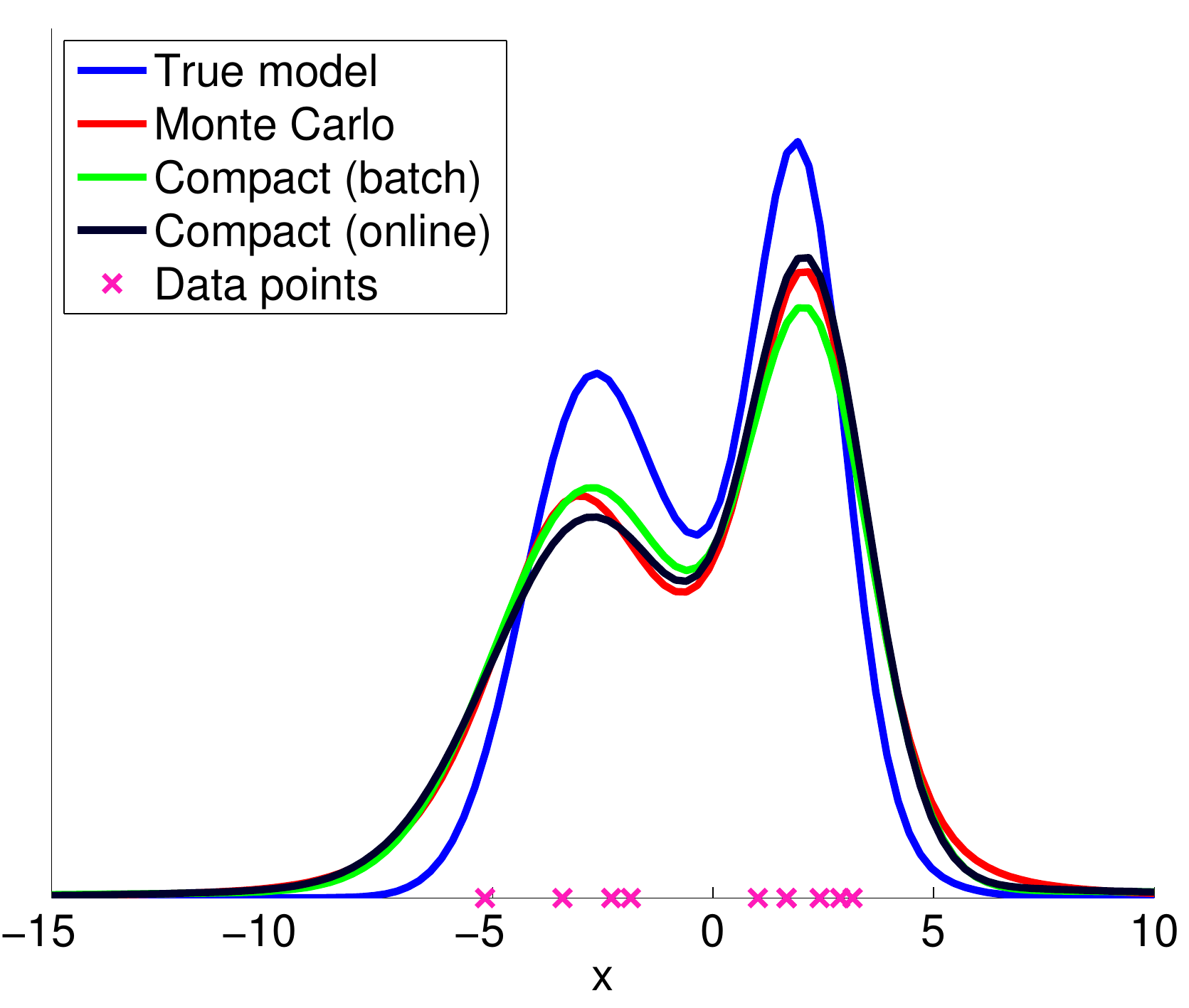}}
\caption{Bayesian MoG with $10$ datapoints.}
\label{fig:compact_predictive:density_results_10}
\end{figure}

\begin{figure}[p]
\centering
\subfloat[Histograms of parameter posterior samples]{
\includegraphics[width=\imwidth]{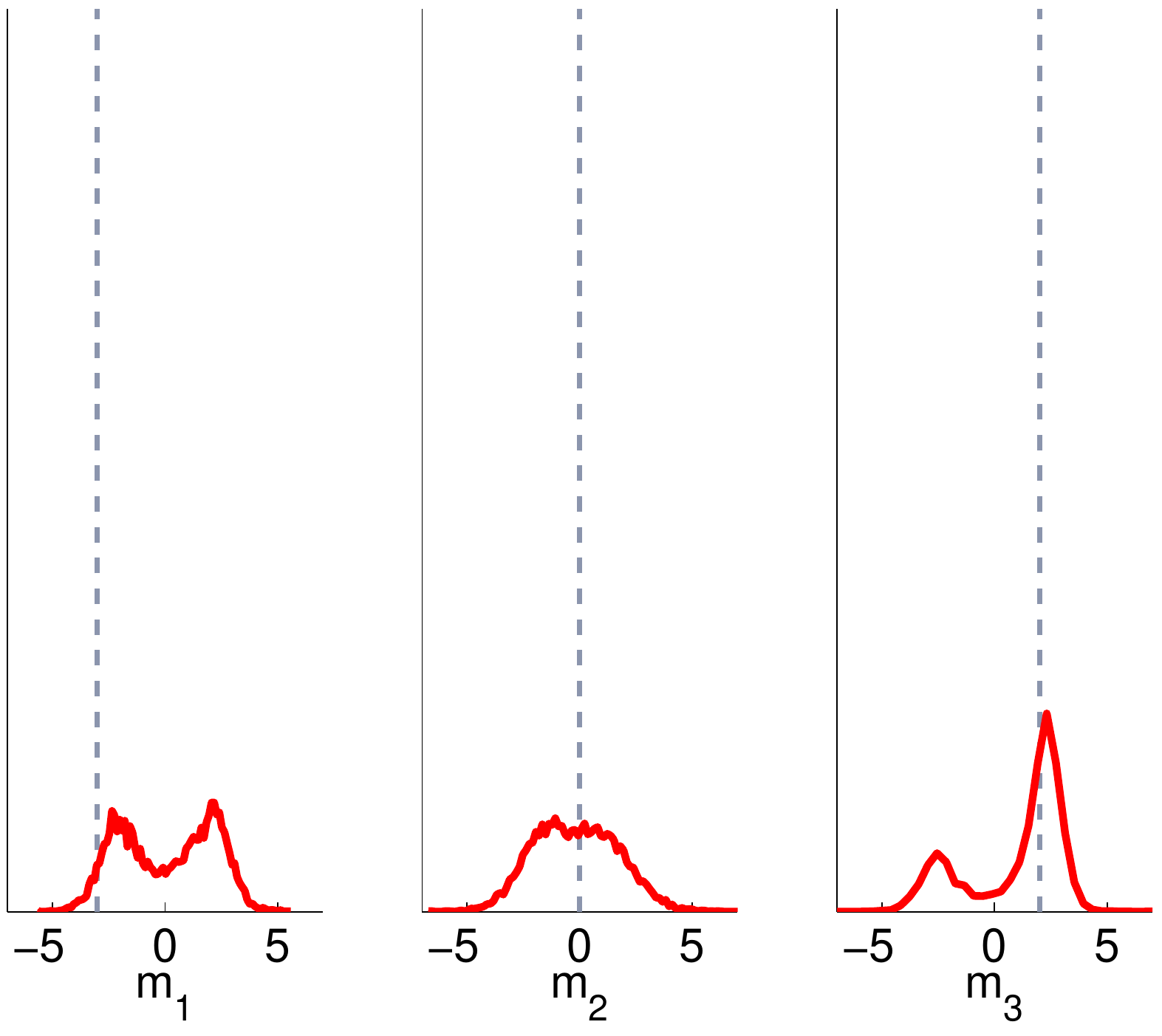}}
\hfill
\subfloat[Predictive densities]{
\includegraphics[width=\imwidth]{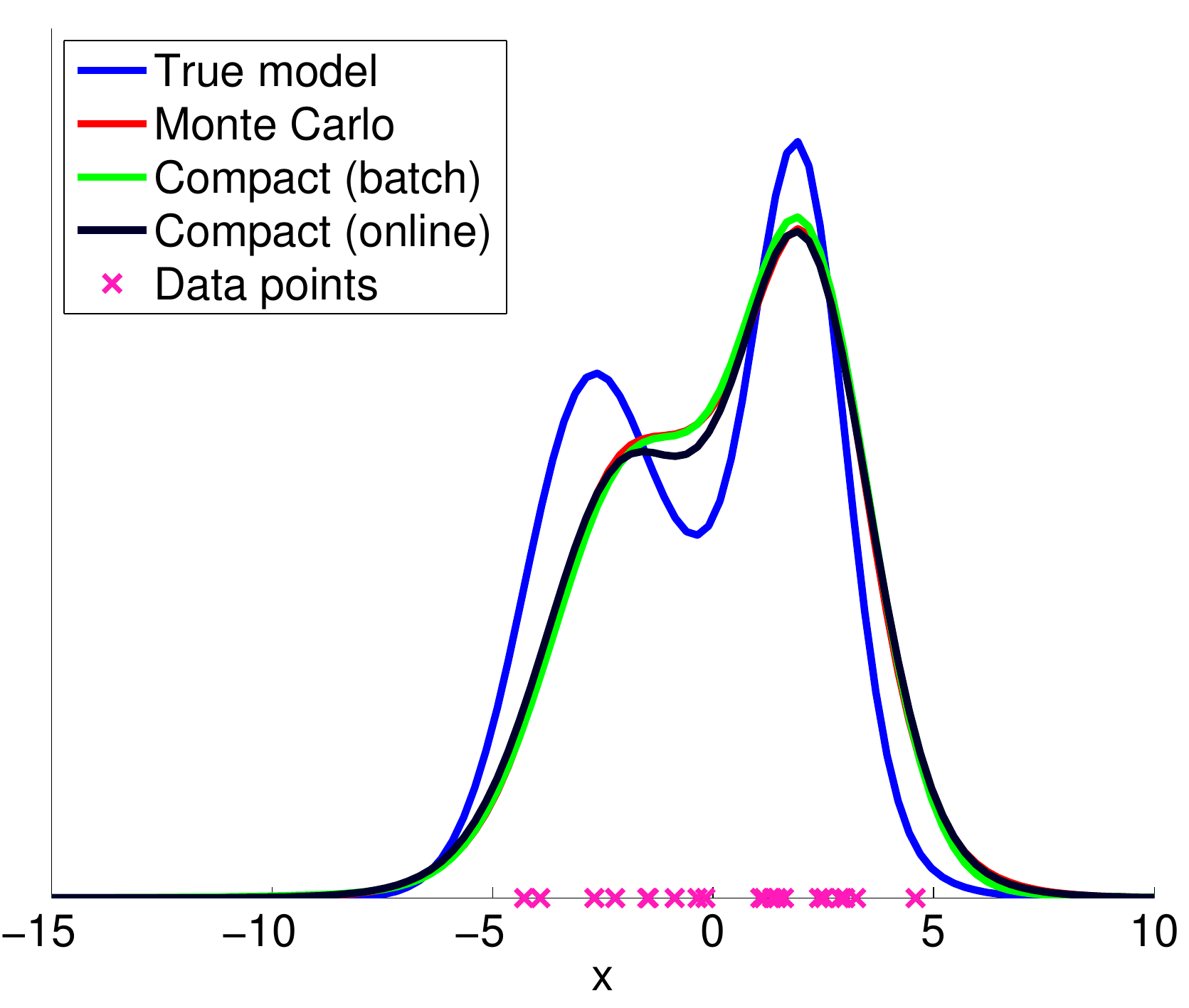}}
\caption{Bayesian MoG with $20$ datapoints.}
\label{fig:compact_predictive:density_results_20}
\end{figure}

\begin{figure}[p]
\centering
\subfloat[Histograms of parameter posterior samples]{
\includegraphics[width=\imwidth]{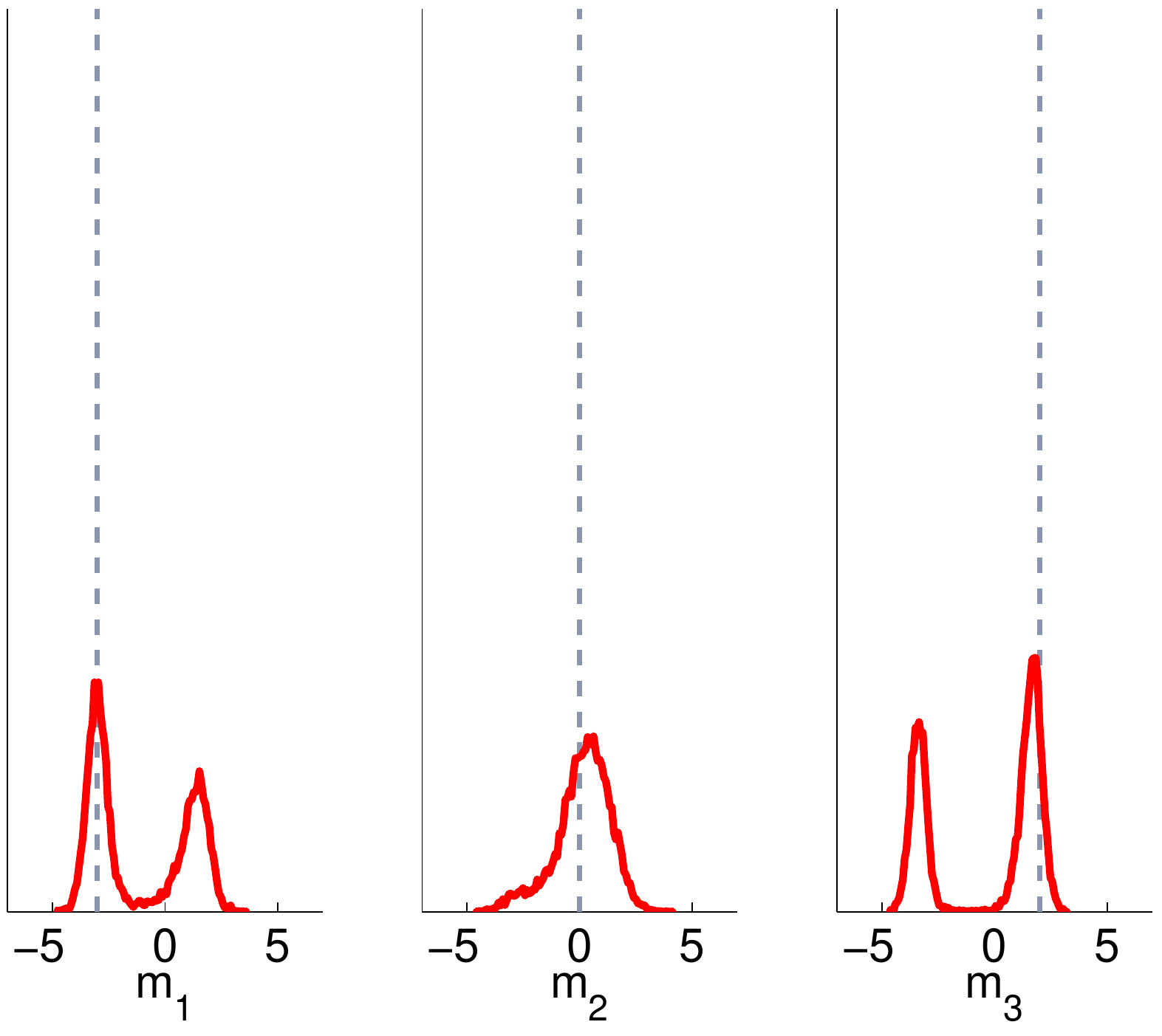}}
\hfill
\subfloat[Predictive densities]{
\includegraphics[width=\imwidth]{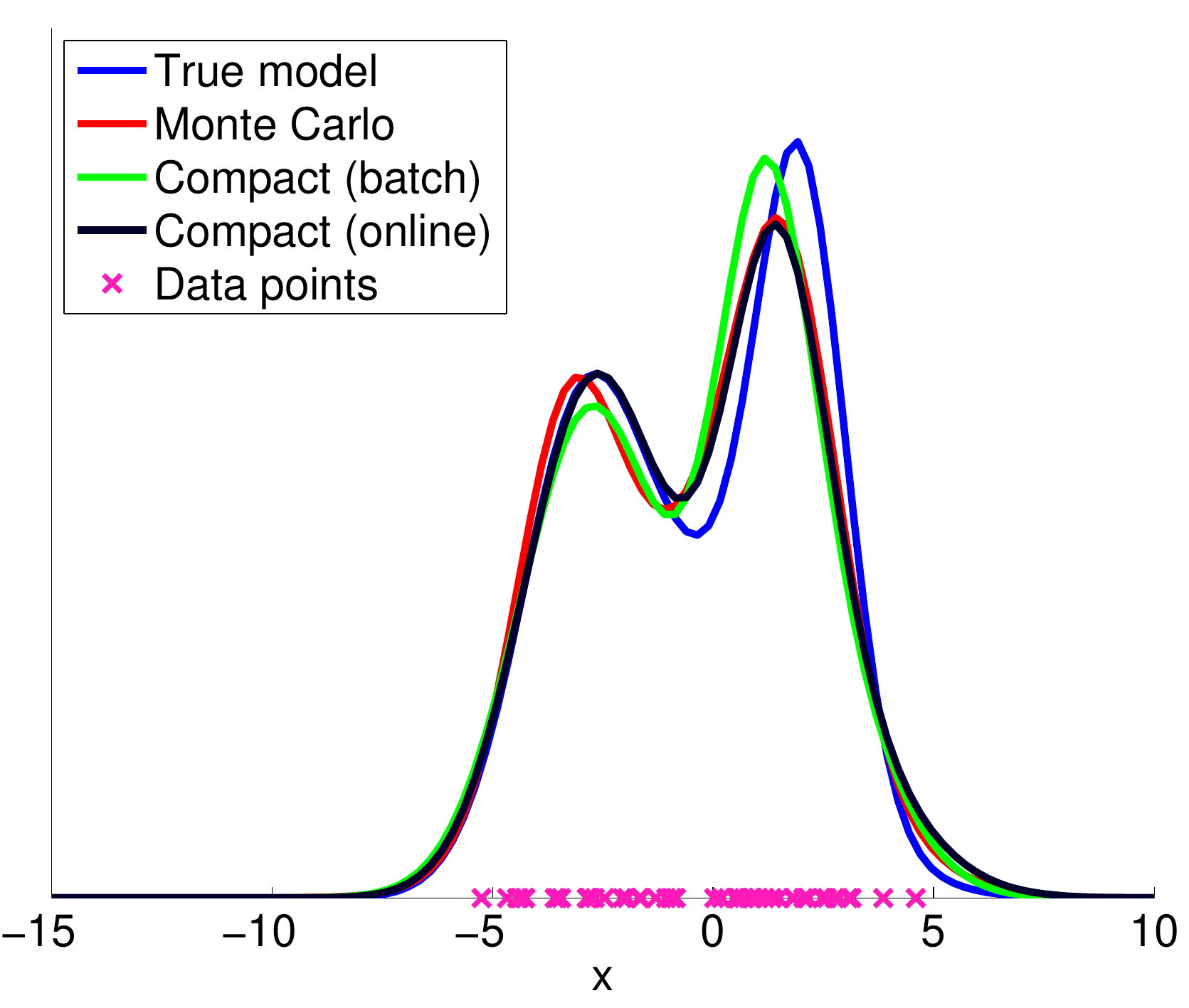}}
\caption{Bayesian MoG with $50$ datapoints.}
\label{fig:compact_predictive:density_results_50}
\end{figure}

\begin{figure}[p]
\centering
\subfloat[Histograms of parameter posterior samples]{
\includegraphics[width=\imwidth]{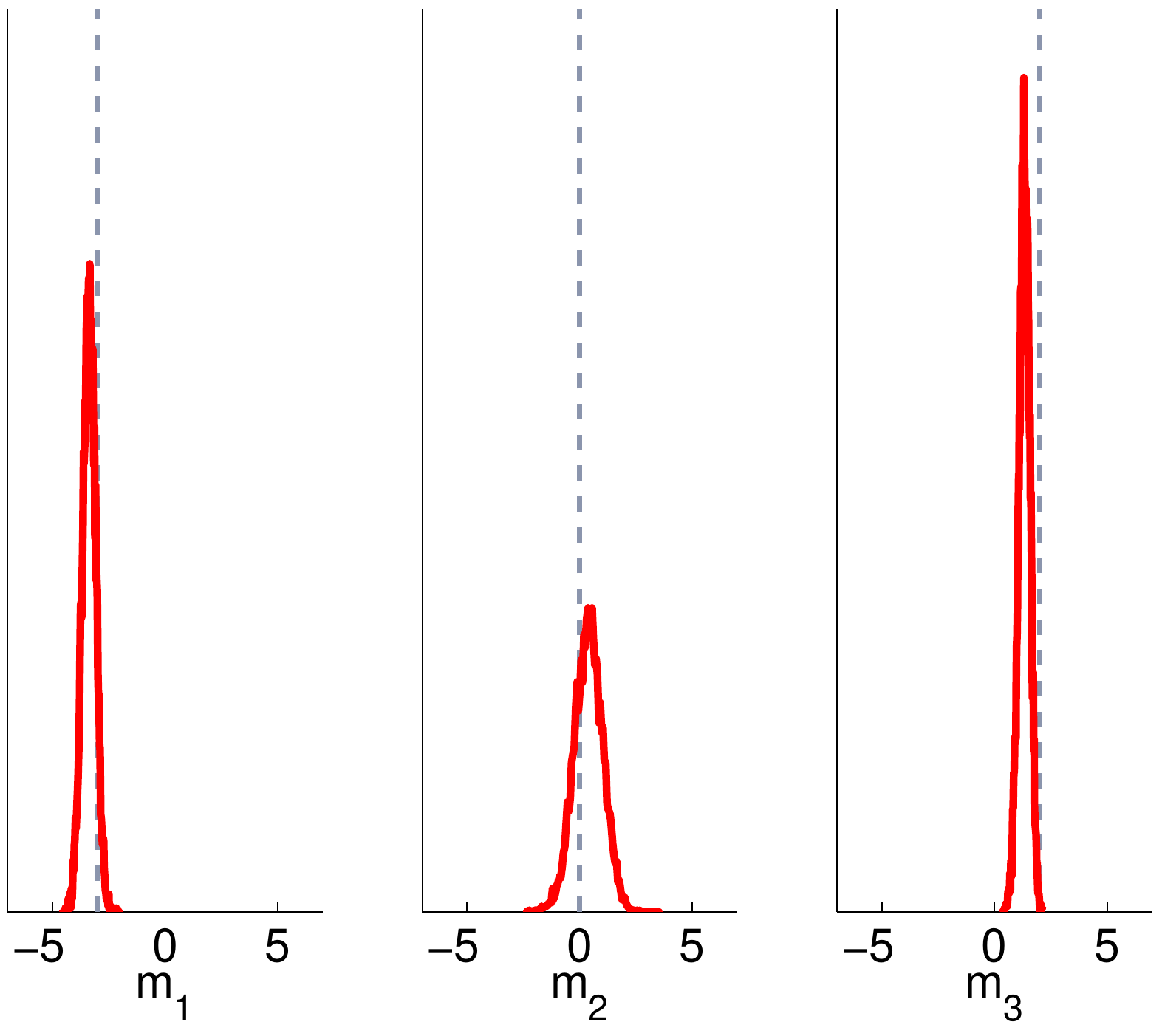}}
\hfill
\subfloat[Predictive densities]{
\includegraphics[width=\imwidth]{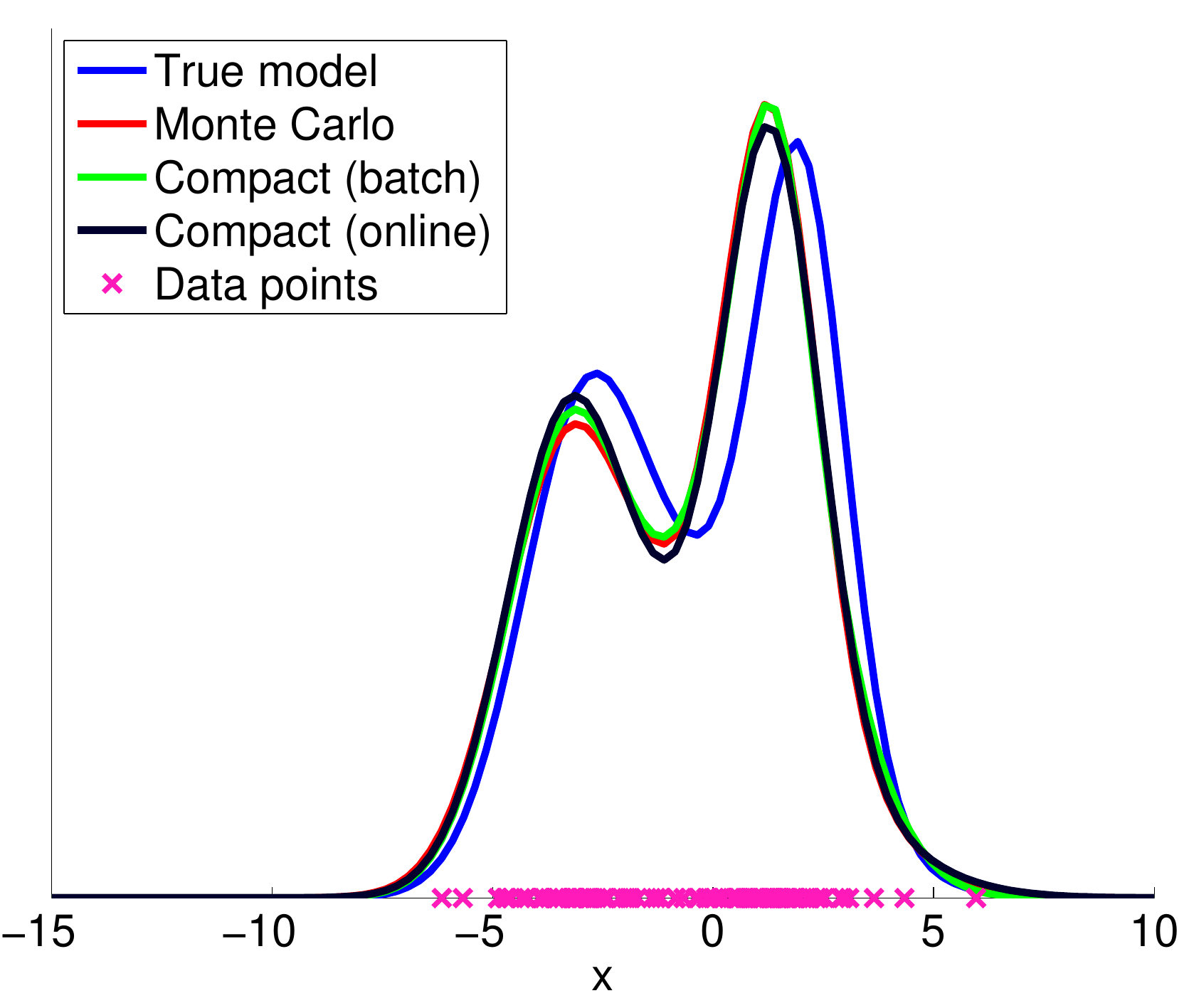}}
\caption{Bayesian MoG with $100$ datapoints.}
\label{fig:compact_predictive:density_results_100}
\end{figure}

\begin{figure}[p]
\centering
\subfloat[Histograms of parameter posterior samples]{
\includegraphics[width=\imwidth]{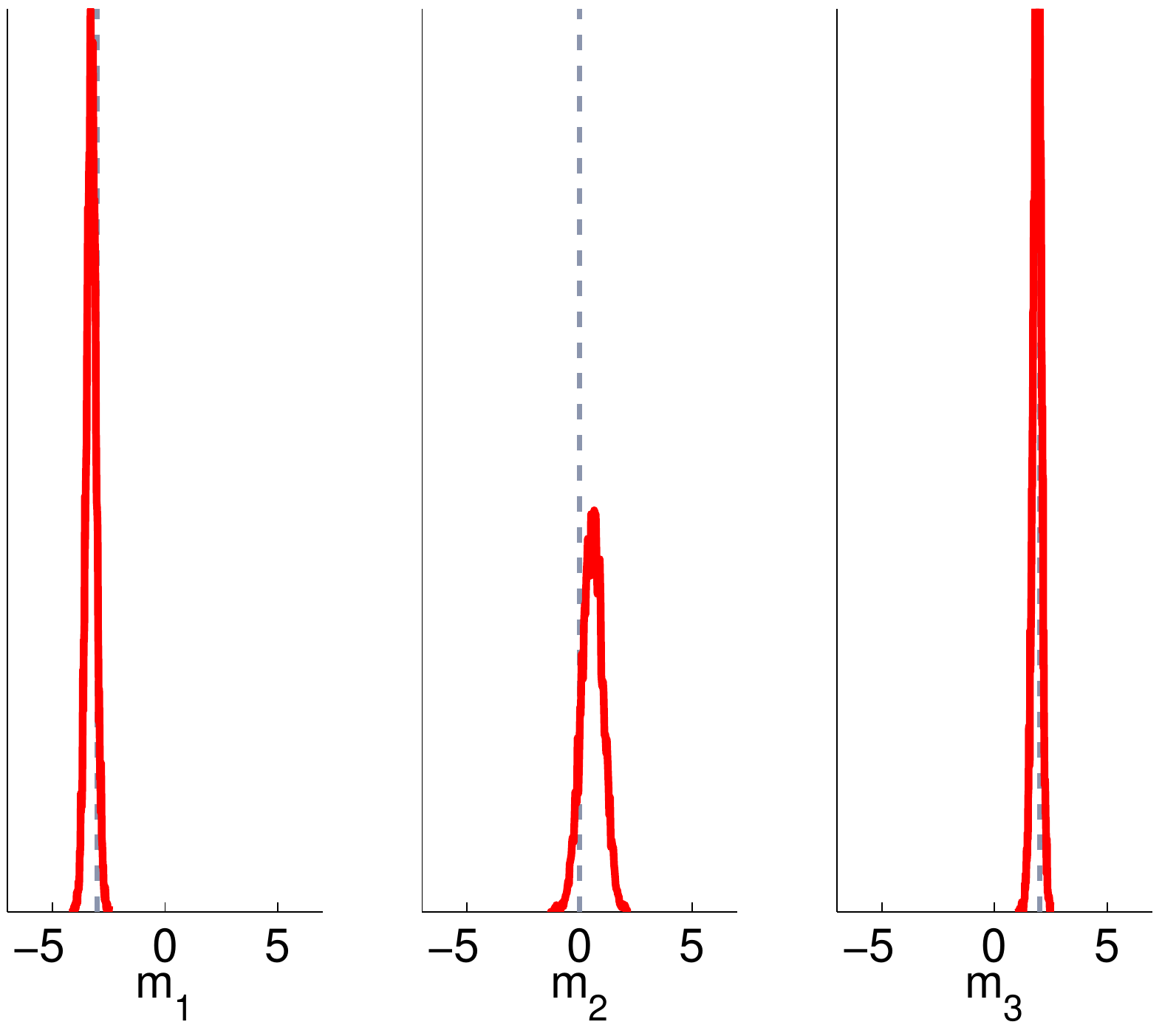}}
\hfill
\subfloat[Predictive densities]{
\includegraphics[width=\imwidth]{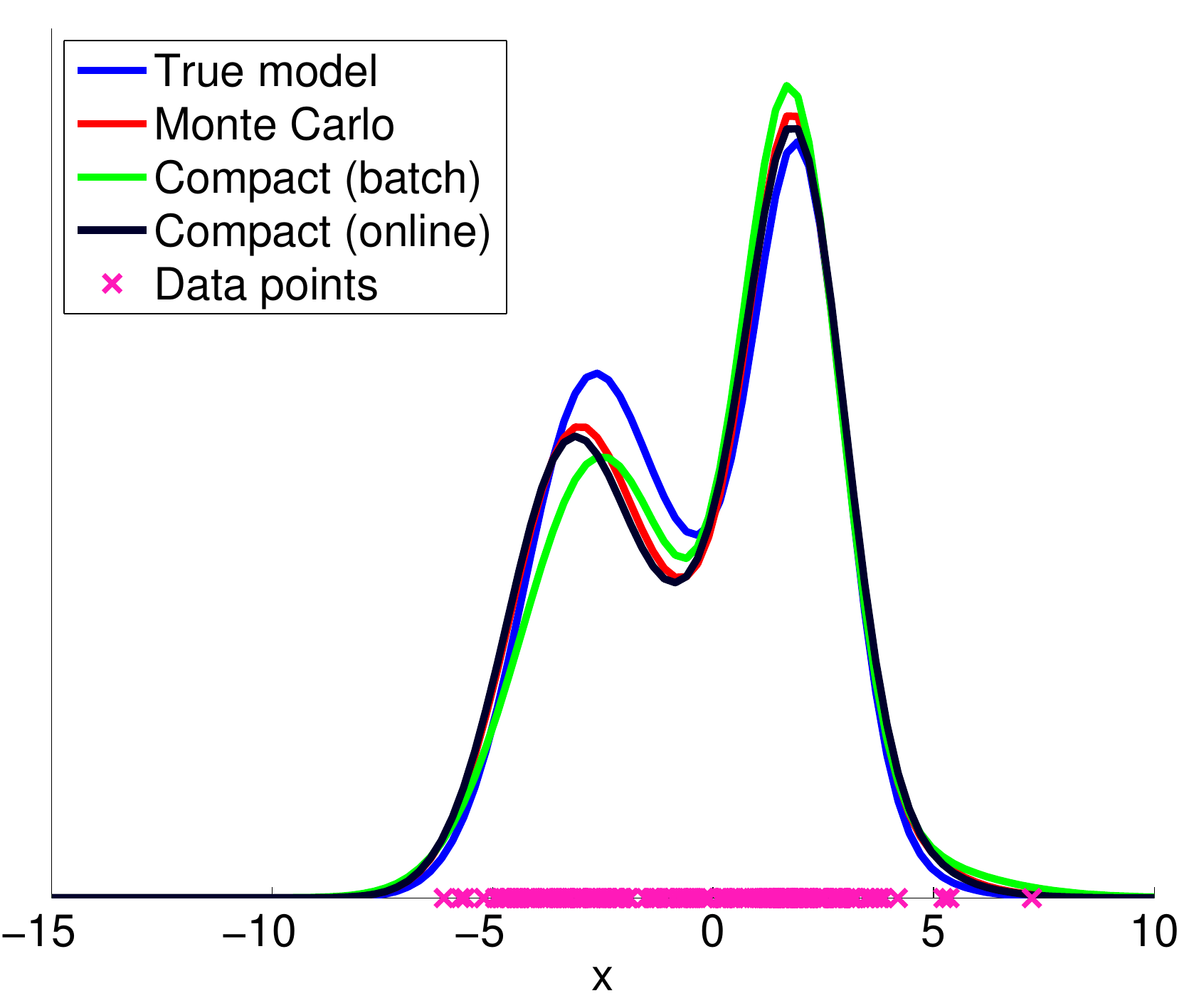}}
\caption{Bayesian MoG with $200$ datapoints.}
\label{fig:compact_predictive:density_results_200}
\end{figure}

\begin{figure}[p]
\centering
\subfloat[Histograms of parameter posterior samples]{
\includegraphics[width=\imwidth]{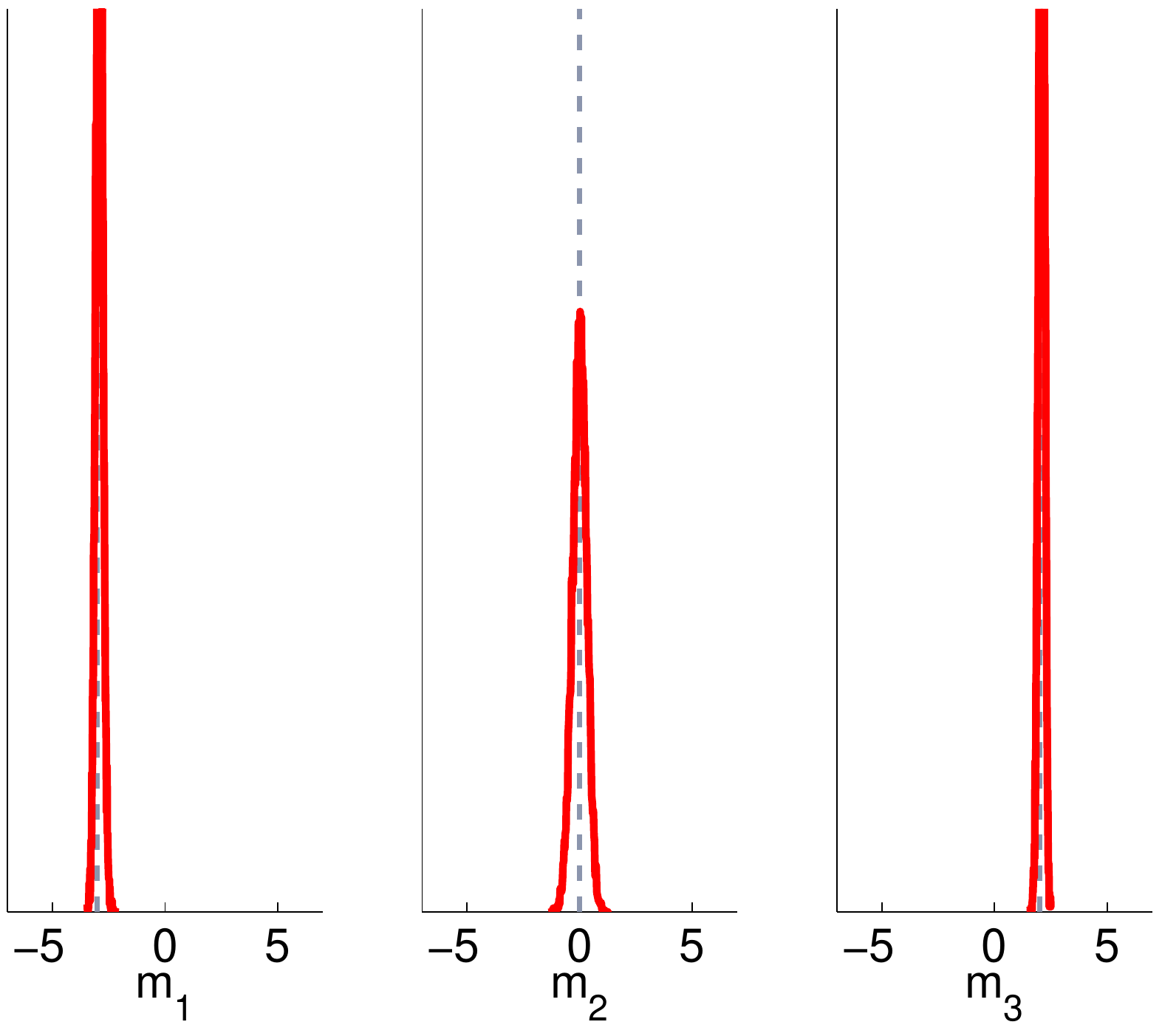}}
\hfill
\subfloat[Predictive densities]{
\includegraphics[width=\imwidth]{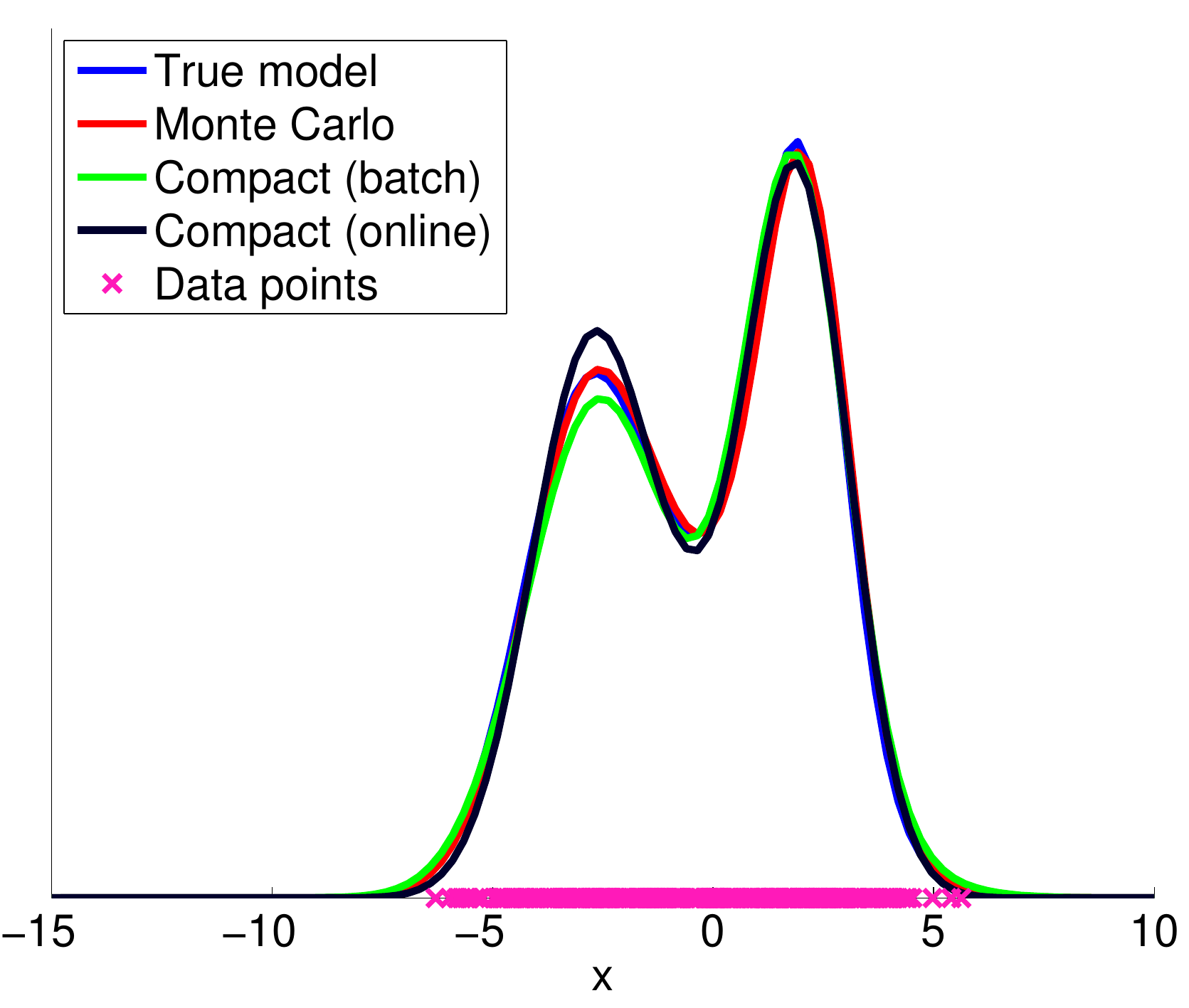}}
\caption{Bayesian MoG with $500$ datapoints.}
\label{fig:compact_predictive:density_results_500}
\end{figure}

\begin{figure}[p]
\centering
\subfloat[Histograms of parameter posterior samples]{
\includegraphics[width=\imwidth]{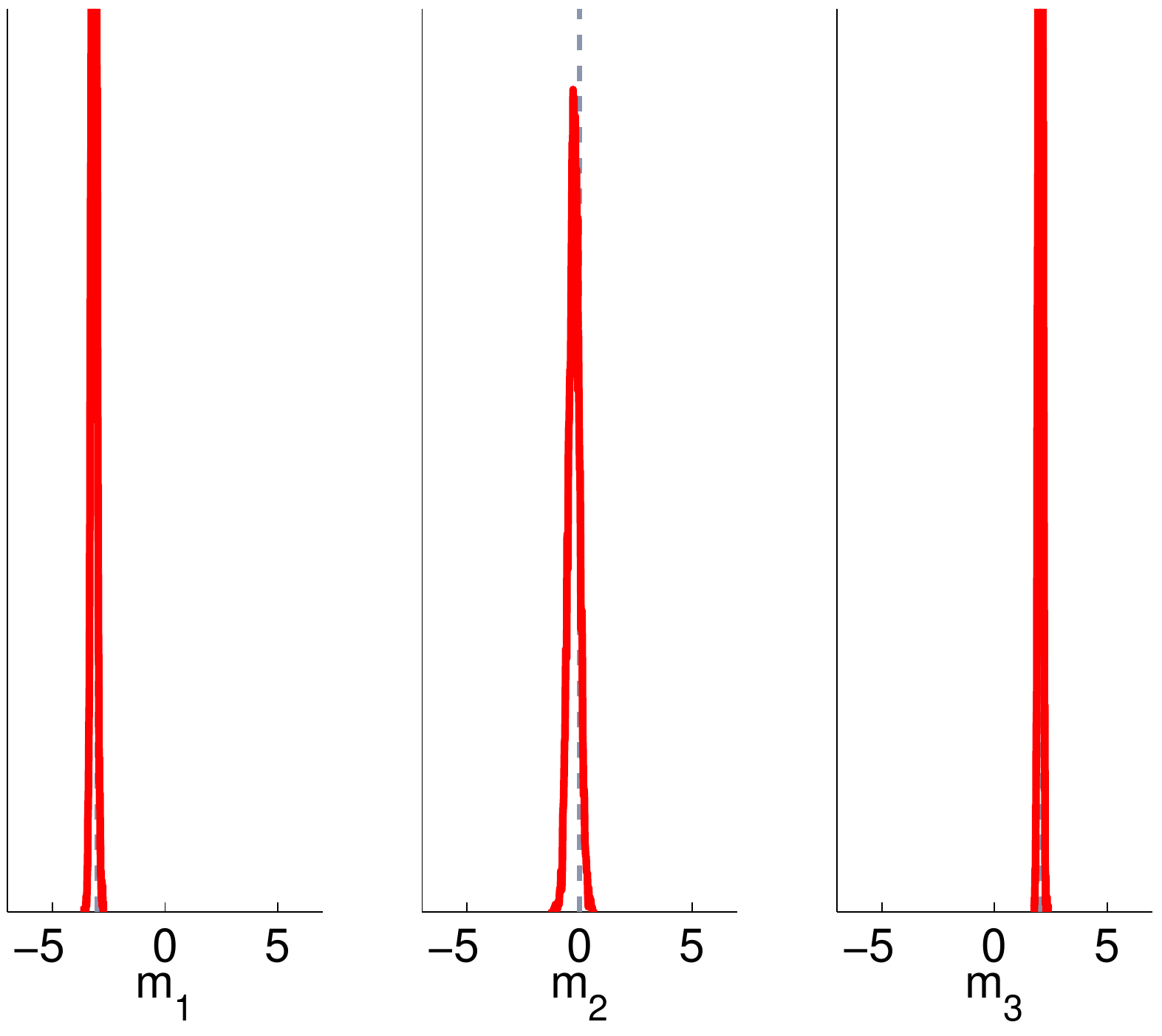}}
\hfill
\subfloat[Predictive densities]{
\includegraphics[width=\imwidth]{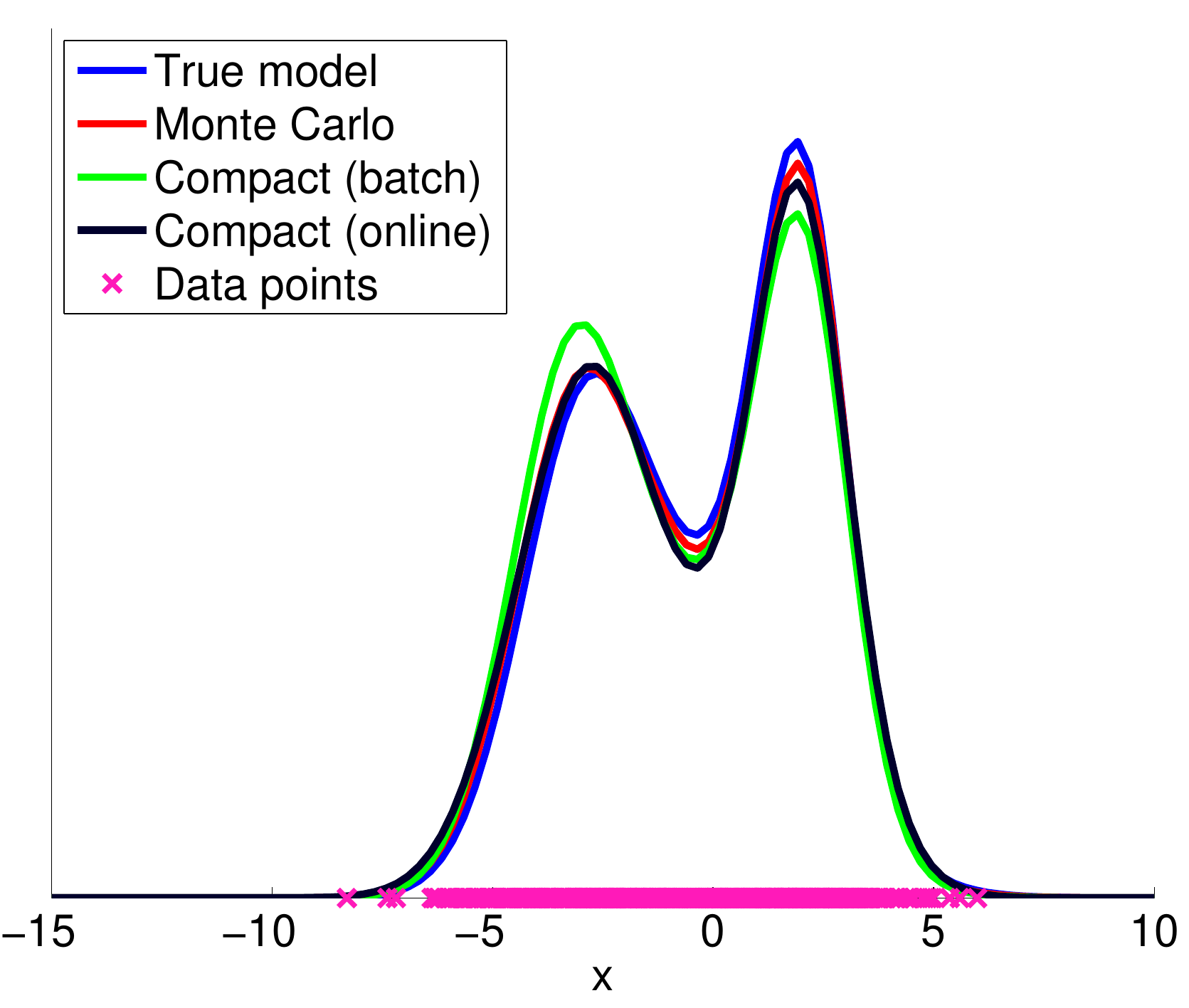}}
\caption{Bayesian MoG with $1000$ datapoints.}
\label{fig:compact_predictive:density_results_1000}
\end{figure}

We ran the experiment described above for various dataset sizes $N$. Figures~\ref{fig:compact_predictive:density_results_2}--\ref{fig:compact_predictive:density_results_1000}
show the results for $N \in\set{2, 5, 10, 20, 50, 100, 200, 500, 1000}$. On the left, the histograms of $m_1$, $m_2$ and $m_3$ are shown, calculated from all $10{,}000$ MCMC samples obtained during batch distillation, with vertical lines indicating the true parameter values. These histograms are a good approximation of the true marginal posteriors $\prob{m_1\g D}$, $\prob{m_2\g D}$ and $\prob{m_3\g D}$ respectively. On the right, we show the pdf of the true model $\prob{x\g \vect{w}}$, from which the dataset was generated and whose means we pretend not to know, superpositioned onto the three approximations of the predictive distribution $\prob{x\g D}$: the Monte Carlo approximation $p_\mathrm{MC}\br{x}$ given by averaging over all $10{,}000$ MCMC samples collected during batch distillation, the compact model $\prob{x\g \bm{\theta}}$ trained using batch distillation, and the compact model $\prob{x\g \bm{\theta}}$ trained using online distillation.

We can see that the larger a dataset we use, the closer the predictive distribution $\prob{x\g D}$ is to the true model $\prob{x\g \vect{w}}$. At the same time, the posterior $\prob{\vect{w}\g D}$ becomes more and more concentrated on the true parameter values. With only few datapoints, the posterior is broad, and as a consequence the predictive distribution is broader than the true model. This is because with only a few datapoints, the predictive distribution does not have enough evidence to rule out the possibility of future datapoints being far away from those that it has already seen. This is the big advantage of Bayesian modelling: the model is aware of its own uncertainty.

The plots show that both compact predictive distributions are close to the Monte Carlo estimate $p_\mathrm{MC}\br{x}$. Note that the latter was given by averaging over $10{,}000$ MCMC samples, therefore we assume that it is a good approximation to the true predictive distribution $\prob{x\g D}$. However, even though $p_\mathrm{MC}\br{x}$ contains $30{,}000$ parameters in total ($3$ means per sample), its functional form is fairly simple. This is why it can be faithfully represented by the compact model $\prob{x\g \bm{\theta}}$, which has only $9$ parameters.

We can see that online distillation has produced compact predictive distributions that are close to those produced by batch distillation. This shows that the distillation process can be done successfully with little storage requirements. In fact, batch distillation needs to store $10{,}000$ MCMC samples and $1000$ data samples, whereas online distillation only needs to store $100$ MCMC samples and $100$ data samples at any time. This can be particularly useful when the dimensionality of either the weights or the datapoints is large, in which case it would be prohibitively expensive to store a lot of samples. In terms of computation cost, we found that batch and online distillation performed comparably.

For this experiment we used minibatches of size $100$, but online distillation can work with any minibatch size, in order to adapt to any storage limitations. Empirically, we found that relatively large minibatches work better, since small minibatches can suffer from high variance. When using small minibatches, we found that thinning the Markov chain is beneficial, since this reduces the correlation of weight samples within the same minibatch. Another approach for reducing correlation within minibatches would be to have independent Markov chains running in parallel, each contributing with a single weight sample in the minibatch. In any case, the learning rate strategy is highly important for good performance. The need to specify and tune a learning rate strategy is the main downside of online distillation compared to batch distillation.

\section{Bayesian binary classification}

In this section, we will extend the framework described for density estimation to handle conditional distributions as well. This will allow us to apply the distillation framework in tasks such as regression and classification. 

We assume that the data takes the form of input-output pairs $\pair{\vect{x}}{\vect{y}}$, where $\vect{x}$ is the input and $\vect{y}$ is the output. We assume that the input is generated by some known distribution $\prob{\vect{x}}$, and that the corresponding output is generated by a conditional distribution $\prob{\vect{y}\g\vect{x}, \vect{w}}$, parameterized by an unknown set of parameters $\vect{w}$. We further assume that we have a prior idea $\prob{\vect{w}}$ of what the parameters are and that we observe a dataset $D = \set{\pair{\vect{x}_1}{\vect{y}_1}, \pair{\vect{x}_2}{\vect{y}_2}, \ldots, \pair{\vect{x}_N}{\vect{y}_N}}$ of independently generated pairs. The task is to calculate the conditional predictive distribution $\prob{\vect{y}\g\vect{x}, D}$.

Similarly to the density estimation case, by applying the rules of probability, we obtain the predictive distribution as follows
\begin{equation}
\prob{\vect{y}\g\vect{x}, D} = \avg{\prob{\vect{y}\g\vect{x}, \vect{w}}}{\prob{\vect{w}\g D}},
\end{equation}
where
\begin{equation}
\prob{\vect{w}\g D} \propto \prob{\vect{w}}\prod_n{\prob{\vect{y}_n\g\vect{x}_n, \vect{w}}}.
\end{equation}
By sampling weights $\set{\vect{w}_s}$ from the posterior $\prob{\vect{w}\g D}$ using MCMC, we can construct a Monte Carlo estimate of the predictive distribution as follows
\begin{equation}
\prob{\vect{y}\g\vect{x}, D} \approx p_\mathrm{MC}\br{\vect{y}\g\vect{x}} = \frac{1}{S}\sum_s{\prob{\vect{y}\g\vect{x}, \vect{w}_s}}.
\end{equation}
Our goal is to distil a compact model $\prob{\vect{y}\g\vect{x}, \bm{\theta}}$ from the MCMC samples, that is as close to the true predictive distribution as possible.

For simplicity, in the rest of this section we will focus on binary classification. In other words, we will assume that output $\vect{y}$ is a binary variable that can be either $0$ or $1$ (and we will denote it as $y$). This will allow expectations over $y$ to be calculated analytically. However, in principle there is no need to restrict ourselves to binary classification, and the framework can be easily extended to handle more general types of output $\vect{y}$.

\subsection{Loss functions}

The first thing we need to do is define an appropriate loss function to measure the discrepancy between the true predictive distribution $\prob{y\g\vect{x}, D}$ and the compact predictive distribution $\prob{y\g\vect{x}, \bm{\theta}}$. Minimizing this loss function with respect to $\bm{\theta}$ will train $\prob{y\g\vect{x}, \bm{\theta}}$ to be as close an approximation to $\prob{y\g\vect{x}, D}$ as possible. In the following, we describe two loss functions that achieve this goal: \emph{cross entropy} and \emph{derivative square error}.

\subsubsection{Cross entropy}
\index{Cross entropy}

In the density estimation case, we used KL divergence as a loss function. To use KL divergence as a loss function for binary classification too, we need to take care of the input variable $\vect{x}$. Following \citet{Snelson:2005:compact_approximations}, we  average the KL divergence over all possible values of $\vect{x}$, which gives the following loss function
\begin{equation}
E_\mathrm{KL}\br{\bm{\theta}} = \avg{\kl{\prob{y\g\vect{x}, D}}{\prob{y\g\vect{x}, \bm{\theta}}}}{\prob{\vect{x}}}.
\end{equation}
We can rewrite the above as follows
\begin{equation}
E_\mathrm{KL}\br{\bm{\theta}} = 
\avg{\log{\prob{y\g\vect{x}, D}}}{\prob{y\g\vect{x}, D}\prob{\vect{x}}}
 - \avg{\log{\prob{y\g\vect{x}, \bm{\theta}}}}{\prob{y\g\vect{x}, D}\prob{\vect{x}}}.
\end{equation}
The first term in the above is a constant with respect to $\bm{\theta}$ so we can simply ignore it. The second term is the cross entropy between $\prob{y\g\vect{x}, D}$ and $\prob{y\g\vect{x}, \bm{\theta}}$ averaged over $\prob{\vect{x}}$, which we shall denote as $E_\mathrm{CE}\br{\bm{\theta}}$. Minimizing $E_\mathrm{CE}\br{\bm{\theta}}$ is then equivalent to minimizing $E_\mathrm{KL}\br{\bm{\theta}}$.

Since $y$ is binary, we can do the expectations with respect to it analytically. For convenience, we define the following two functions 
\begin{align}
t\br{\vect{x}} &= \prob{y = 1\g\vect{x}, D} \\
f\br{\vect{x}\g \bm{\theta}} &= \prob{y = 1\g\vect{x}, \bm{\theta}}.
\end{align}
Using these, we can rewrite cross entropy as
\begin{equation}
E_\mathrm{CE}\br{\bm{\theta}} = - 
\avg{\vphantom{\Big|}t\br{\vect{x}}\log{f\br{\vect{x}\g\bm{\theta}}} + \br{\vphantom{\big|}1-t\br{\vect{x}}}\log\br{\vphantom{\big|}1-f\br{\vect{x}\g\bm{\theta}}}}{\prob{\vect{x}}}.
\end{equation}

It is easy to see that cross entropy is a valid loss function for knowledge distillation. It is well known that the KL divergence is always non-negative, and that it becomes zero if and only if $\prob{y\g\vect{x}, D} = \prob{y\g\vect{x}, \bm{\theta}}$ \citep[section~2.6]{MacKay:2002:IT}. Hence, $E_\mathrm{KL}\br{\bm{\theta}}$---and equivalently $E_\mathrm{CE}\br{\bm{\theta}}$---is minimized if and only if $t\br{\vect{x}} = f\br{\vect{x}\g \bm{\theta}}$ everywhere in the support of $\prob{\vect{x}}$.

\subsubsection{Derivative square error}
\index{Derivative square error}

Functions $t\br{\vect{x}}$ and $f\br{\vect{x}\g \bm{\theta}}$, as defined above, correspond to the probability of $y$ being turned on at input location $\vect{x}$, as given by the true predictive distribution and the compact predictive distribution respectively. Geometrically, $t\br{\vect{x}}$ and $f\br{\vect{x}\g \bm{\theta}}$ are surfaces over the input space. The process of distillation can be viewed as trying to make the surface represented by $f\br{\vect{x}\g \bm{\theta}}$ match as closely as possible the surface represented by $t\br{\vect{x}}$, especially at input locations $\vect{x}$ for which $\prob{\vect{x}}$ is high.

Minimizing the cross entropy $E_\mathrm{CE}\br{\bm{\theta}}$ is one way of making the surfaces match, by matching function values $t\br{\vect{x}}$ and $f\br{\vect{x}\g \bm{\theta}}$ directly. However, assuming the surfaces are smooth, we can take this one step further and try to match function derivatives with respect to the input $\vect{x}$ as well. By matching derivatives as well as function values, we encourage the two surfaces to not only have the same height, but also have the same slope. Since matching derivatives at input location $\vect{x}$ is equivalent to matching a whole tangent hyperplane to the surface at $\vect{x}$, a loss function based on derivatives can constrain the surfaces with fewer evaluations. This can become a great advantage especially in high-dimensional input spaces.

In fact, for the task of binary classification, we can match the two surfaces only by using derivatives, without the need to use function values at all. We propose herein a loss function that involves derivatives with respect to $\vect{x}$ and achieves exactly that. This loss function, which we call \emph{derivative square error}, is defined as follows
\begin{equation}
E_\mathrm{DSE}\br{\bm{\theta}} = 
\avg{\frac{1}{2}\norm{\frac{\partial}{\partial\vect{x}}\log{\prob{y\g\vect{x}, \bm{\theta}}} - \frac{\partial}{\partial\vect{x}}\log{\prob{y\g\vect{x}, D}}}^2
}{\prob{y\g\vect{x}, D}\prob{\vect{x}}}.
\end{equation}
Notice that we first transform the probabilities into the log domain, and then take their derivatives. Expanding the expectation over $y$, and writing the above in terms of $t\br{\vect{x}}$ and $f\br{\vect{x}\g \bm{\theta}}$, we get the following
\begin{equation}
E_\mathrm{DSE}\br{\bm{\theta}} = 
\avg{
t\br{\vect{x}}E_{1}\br{\vect{x}, \bm{\theta}} +
\br{1-t\br{\vect{x}}}E_{0}\br{\vect{x}, \bm{\theta}}
}{\prob{\vect{x}}},
\end{equation}
where
\begin{align}
E_{1}\br{\vect{x}, \bm{\theta}} &= \frac{1}{2}\norm{\frac{\partial}{\partial\vect{x}}\log{f\br{\vect{x}\g \bm{\theta}}} - \frac{\partial}{\partial\vect{x}}\log{t\br{\vect{x}}}}^2 \\
E_{0}\br{\vect{x}, \bm{\theta}} &= \frac{1}{2}\norm{\frac{\partial}{\partial\vect{x}}\log{\br{1-f\br{\vect{x}\g \bm{\theta}}}} - \frac{\partial}{\partial\vect{x}}\log{\br{1-t\br{\vect{x}}}}}^2.
\end{align}
The validity of derivative square error as a loss function is established by the following proposition, which states that, unless $t\br{\vect{x}}$ is the constant function, minimizing $E_\mathrm{DSE}\br{\bm{\theta}}$ is sufficient for making the surfaces match.
\begin{proposition}
Assuming that $f\br{\vect{x}\g\bm{\theta}}$ has enough capacity and that $t\br{\vect{x}}$ is not constant in the support of $\prob{\vect{x}}$, $E_\mathrm{DSE}\br{\bm{\theta}}$ is minimized if and only if $f\br{\vect{x}\g\bm{\theta}} = t\br{\vect{x}}$ for all $\vect{x}$ in the support of $\prob{\vect{x}}$.
\end{proposition}
\begin{proof}
Let $S_p$ be the support of $\prob{\vect{x}}$, i.e.~the set of points $\vect{x}$ for which $\prob{\vect{x}}>0$.
Notice that $E_\mathrm{DSE}\br{\bm{\theta}}\ge 0$, as the expectation of a norm. Obviously, if $f\br{\vect{x}\g\bm{\theta}} = t\br{\vect{x}}$ for all $\vect{x}\in S_p$, we have that $E_\mathrm{DSE}\br{\bm{\theta}}=0$. Conversely, if $E_\mathrm{DSE}\br{\bm{\theta}}=0$, it has to be the case that for all $\vect{x}\in S_p$ the following hold
\begin{align}
\frac{\partial}{\partial\vect{x}}\log{f\br{\vect{x}\g\bm{\theta}}} &= \frac{\partial}{\partial\vect{x}}\log{t\br{\vect{x}}}\\
\frac{\partial}{\partial\vect{x}}\log{\br{1-f\br{\vect{x}\g\bm{\theta}}}} &= \frac{\partial}{\partial\vect{x}}\log{\br{1-t\br{\vect{x}}}}.
\end{align}
The above two constraints imply the existence of positive constants $c_1$ and $c_2$ such that
\begin{align}
f\br{\vect{x}\g\bm{\theta}} &= c_1 \,t\br{\vect{x}}\\
1-f\br{\vect{x}\g\bm{\theta}} &= c_2\,\br {1-t\br{\vect{x}}},
\end{align}
which combine to give
\begin{equation}
\br{c_2-c_1}\, t\br{\vect{x}}= c_2 - 1.
\end{equation}
The above implies that either $t\br{\vect{x}}$ is constant in $S_p$ or $c_1=c_2=1$. Since we have assumed that $t\br{\vect{x}}$ is not constant in $S_p$, it follows that $f\br{\vect{x}\g\bm{\theta}}=t\br{\vect{x}}$ everywhere in $S_p$.
\end{proof}

\subsection{Batch distillation}

In order to be in a position to minimize either $E_\mathrm{CE}\br{\bm{\theta}}$ or $E_\mathrm{DSE}\br{\bm{\theta}}$, we need to have access to the true predictive distribution $t\br{\vect{x}}$. In our framework, we have approximate access to $t\br{\vect{x}}$ via MCMC sampling. By collecting a large number of MCMC samples $\set{\vect{w}_s}$, we can approximate $t\br{\vect{x}}$ using the following Monte Carlo estimate
\begin{equation}
t\br{\vect{x}} \approx p_\mathrm{MC}\br{y=1\g \vect{x}} = \frac{1}{S}\sum_s{\prob{y=1\g\vect{x}, \vect{w}_s}}.
\end{equation}
Provided $S$ is large enough, the above estimate can be used as a proxy for $t\br{\vect{x}}$.

We will now describe a stochastic gradient optimization method for minimizing $E_\mathrm{CE}\br{\bm{\theta}}$ and $E_\mathrm{DSE}\br{\bm{\theta}}$. First, note that both loss functions can be written in the form
\begin{equation}
E\br{\bm{\theta}} = \avg{E\br{\vect{x}, \bm{\theta}}}{\prob{\vect{x}}}
\end{equation}
for an appropriate instantiation of $E\br{\vect{x}, \bm{\theta}}$. The derivative of the above with respect to $\bm{\theta}$ can be written as
\begin{equation}
\vect{g}\br{\bm{\theta}} = \avg{\pderivnull{\bm{\theta}}E\br{\vect{x}, \bm{\theta}}}{\prob{\vect{x}}}.
\end{equation}
If we have a minibatch $\set{\vect{x}_m}$ of samples from $\prob{\vect{x}}$, then we can stochastically approximate the above derivative by
\begin{equation}
\hat{\vect{g}}\br{\bm{\theta}} = \frac{1}{M}\sum_m\pderivnull{\bm{\theta}}E\br{\vect{x}_m, \bm{\theta}}.
\end{equation}
It is easy to see that  $\hat{\vect{g}}\br{\bm{\theta}}$ is an unbiased estimate of $\vect{g}\br{\bm{\theta}}$. By repeatedly sampling minibatches from $\prob{\vect{x}}$, calculating $\hat{\vect{g}}\br{\bm{\theta}}$ and then using it to update $\bm{\theta}$, we can effectively minimize $E\br{\bm{\theta}}$. This is outlined in the procedure below.
\begin{framed}
\begin{enumerate}[label=(\roman*)]
\item Generate a bag of samples $\set{\vect{w}_s}$ of size $S$ from $\prob{\vect{w}\g D}$ by simulating the Markov chain. Use $p_\mathrm{MC}\br{y=1\g \vect{x}}$ as an approximation to $t\br{\vect{x}}$.
\item\label{alg:compact_predictive:binary_class:batch:step2} Generate a minibatch $\set{\vect{x}_m}$ of size $M$ from $\prob{\vect{x}}$.
\item Calculate the stochastic gradient $\hat{\vect{g}}\br{\bm{\theta}} = \frac{1}{M}\sum_{m}{\pderivnull{\bm{\theta}} E\br{\vect{x}_m,\bm{\theta}}}$.
\item Make an update on $\bm{\theta}$ using $\hat{\vect{g}}\br{\bm{\theta}}$.
\item If $\bm{\theta}$ has not converged yet, go back to step~\ref{alg:compact_predictive:binary_class:batch:step2}.
\end{enumerate}
\end{framed}

The above algorithm relies on three conditions. Firstly, the number of MCMC samples $S$ has to be large enough so that $p_\mathrm{MC}\br{y=1\g \vect{x}}$ can  be effectively used in place of $t\br{\vect{x}}$. Secondly, we need to be able to calculate the derivative $\pderivnull{\bm{\theta}}E\br{\vect{x}, \bm{\theta}}$. In our case study in section~\ref{sec:compact_predictive:bayesian_logreg}, we discuss how to calculate $\pderivnull{\bm{\theta}}E\br{\vect{x}, \bm{\theta}}$ in the particular setting of logistic regression. Thirdly, the update strategy of $\bm{\theta}$ using $\hat{\vect{g}}\br{\bm{\theta}}$ has to be suitably chosen, so that the algorithm converges. \citet{Bottou:1999:online_learning} provides sufficient conditions for stochastic gradient algorithms of the above type that guarantee almost sure convergence to a stationary point of $E\br{\bm{\theta}}$ (provided samples from $\prob{\vect{x}}$ are drawn independently).

\subsection{Online distillation}

The downside of batch distillation as previously considered by \citet{Snelson:2005:compact_approximations} is that a large number of MCMC samples $\set{\vect{w}_s}$ has to be generated and stored in advance. This can be problematic, especially if $\vect{w}$ is high-dimensional. To alleviate this problem, in this section we present an online version of the distillation procedure. In online distillation, the updates to $\bm{\theta}$ take place on the fly as the Markov chain is simulated, and hence there is no need to store the previously generated MCMC samples.

Online distillation relies on using a single MCMC sample as a stochastic estimate of $t\br{\vect{x}}$. To see how this is done, define the following
\begin{equation}
t\br{\vect{x}, \vect{w}} = \prob{y = 1\g \vect{x},\vect{w}}.
\end{equation}
It is easy to see that the above is an unbiased estimate of $t\br{\vect{x}}$ since
\begin{equation}
t\br{\vect{x}} = \avg{t\br{\vect{x}, \vect{w}}}{\prob{\vect{w}\g D}}.
\end{equation}
The idea of online distillation is to use $t\br{\vect{x}, \vect{w}}$ instead of $t\br{\vect{x}}$ in the loss function and then average the whole loss function with respect to  $\prob{\vect{w}\g D}$. This way, the loss function takes the following form
\begin{equation}
E\br{\bm{\theta}} = \avg{E\br{\vect{x},\vect{w},\bm{\theta}}}{\prob{\vect{w}\g D}\prob{\vect{x}}}.
\end{equation}

Having the loss function in the above form, we can now minimize it using stochastic gradients not only with respect to $\vect{x}$, but also with respect to $\vect{w}$. That is, by simulating the Markov chain $M$ times to get a minibatch of samples $\set{\vect{w}_m}$, and at the same time drawing $M$ samples $\set{\vect{x}_m}$ from $\prob{\vect{x}}$, we obtain an unbiased stochastic gradient as follows
\begin{equation}
\hat{\vect{g}}\br{\bm{\theta}} = \frac{1}{M}\sum_m{\pderivnull{\bm{\theta}}E\br{\vect{x}_m,\vect{w}_m,\bm{\theta}}}.
\end{equation}
Hence, by repeatedly sampling minibatches from $\prob{\vect{w}\g D}$ and $\prob{\vect{x}}$, calculating $\hat{\vect{g}}\br{\bm{\theta}}$ and using it to update $\bm{\theta}$, we can  minimize $E\br{\bm{\theta}}$. The above procedure is outlined in the algorithm below.
\begin{framed}
\begin{enumerate}[label=(\roman*)]
\item\label{alg:compact_predictive:binary_class:online:step1} Generate a minibatch $\set{\vect{w}_m}$ of size $M$ from $\prob{\vect{w}\g D}$ by simulating the Markov chain.
\item Generate a minibatch $\set{\vect{x}_m}$ of size $M$ from $\prob{\vect{x}}$.
\item Calculate the stochastic gradient $\hat{\vect{g}}\br{\bm{\theta}} = \frac{1}{M}\sum_{m}{\pderivnull{\bm{\theta}} E\br{\vect{x}_m,\vect{w}_m,\bm{\theta}}}$.
\item Make an update on $\bm{\theta}$ using $\hat{\vect{g}}\br{\bm{\theta}}$.
\item If $\bm{\theta}$ has not converged yet, go back to step~\ref{alg:compact_predictive:binary_class:online:step1}.
\end{enumerate}
\end{framed}
\noindent Note that, in the above algorithm, MCMC samples are not independent. Convergence of stochastic approximation algorithms of the above type, where samples are not independent, has been studied by \citet{Younes:1999:convergence_markov}. If desired, dependence between samples can be reduced by thinning or using parallel Markov chains.

To summarize, in order to make the loss function suitable for online distillation, $t\br{\vect{x}}$ was replaced by $t\br{\vect{x},\vect{w}}$ and then the whole loss function was averaged with respect to $\prob{\vect{w}\g D}$. It is natural to ask what effect this change to the loss function has. For cross entropy, it turns out that the two versions of the loss function are actually equivalent. To see why, rewrite cross entropy as follows
\begin{align}
\avg{E_\mathrm{CE}\br{\vect{x},\bm{\theta}}}{\prob{\vect{x}}} &= - 
\avg{\vphantom{\Big|}t\br{\vect{x}}\log{f\br{\vect{x}\g\bm{\theta}}} + \br{\vphantom{\big|}1-t\br{\vect{x}}}\log\br{\vphantom{\big|}1-f\br{\vect{x}\g\bm{\theta}}}}{\prob{\vect{x}}} \notag\\
&= - \avg{\vphantom{\Big|}\avg{t\br{\vect{x},\vect{w}}}{\prob{\vect{w}\g D}}\log{f\br{\vect{x}\g\bm{\theta}}} + \br{\vphantom{\big|}1-\avg{t\br{\vect{x},\vect{w}}}{\prob{\vect{w}\g D}}}\log\br{\vphantom{\big|}1-f\br{\vect{x}\g\bm{\theta}}}}{\prob{\vect{x}}}\notag\\
&= - \avg{\vphantom{\Big|}t\br{\vect{x},\vect{w}}\log{f\br{\vect{x}\g\bm{\theta}}} + \br{\vphantom{\big|}1-t\br{\vect{x},\vect{w}}}\log\br{\vphantom{\big|}1-f\br{\vect{x}\g\bm{\theta}}}}{\prob{\vect{w}\g D}\prob{\vect{x}}}\notag\\
&= \avg{E_\mathrm{CE}\br{\vect{x},\vect{w},\bm{\theta}}}{\prob{\vect{w}\g D}\prob{\vect{x}}}.
\end{align}
Hence, when using cross entropy, both batch and online distillation minimize the same loss function. The same is not true when using derivative square error though. In derivative square error, the term $t\br{\vect{x}}$ appears inside a log and a square, both of which are non-linear functions. Therefore, expectations with respect to $\prob{\vect{w}\g D}$ cannot be moved outside the loss function without changing it. In our case study in section~\ref{sec:compact_predictive:bayesian_logreg}, we will assess experimentally whether this change in derivative square error affects distillation performance.

\subsection{Case study: Bayesian logistic regression}
\label{sec:compact_predictive:bayesian_logreg}

In this section, we apply the distillation framework to the problem of Bayesian logistic regression. Given a dataset with known binary labels, the task in Bayesian logistic regression is to infer the class probability for every possible input location.

\subsubsection{The setup}

Logistic regression is a probabilistic linear model for binary classification. Under a logistic regression model with weights $\vect{w}$, the probability of a datapoint $\vect{x}$ having $y=1$ is given by
\begin{equation}
\prob{y = 1\g\vect{x}, \vect{w}} = \sigm{\vect{w}^T\vect{x}},
\end{equation}
where $\sigm{\cdot}$ is the logistic sigmoid function. In our experiments, we used a broad Gaussian prior for the weights (shown in Figure~\ref{fig:compact_predictive:logreg_setup:prior}) that is given by
\begin{equation}
\prob{\vect{w}} = \gaussian{\vect{w}}{\vect{0}}{100\mat{I}}.
\end{equation}

We created a toy dataset $D$ of $24$ two-dimensional input points $\set{\vect{x}_n}$, and we labelled $12$ of them with $y=1$ and the other $12$ with $y=0$. The datapoints are shown in Figure~\ref{fig:compact_predictive:logreg_setup:dataset}, colour-coded according to their labels. Given the dataset, the task is to calculate the predictive probability $\prob{y=1\g\vect{x},D}$ for every input location $\vect{x}$. Despite the simplicity of the logistic regression model, calculating the predictive distribution is analytically intractable due to the sigmoid likelihood factor, and typically an approximation like MCMC is needed.

\def\imwidth{0.49\textwidth}

\begin{figure}[tbp]
\centering
\subfloat[Prior (unnormalized)\label{fig:compact_predictive:logreg_setup:prior}]{
\includegraphics[width=\imwidth]{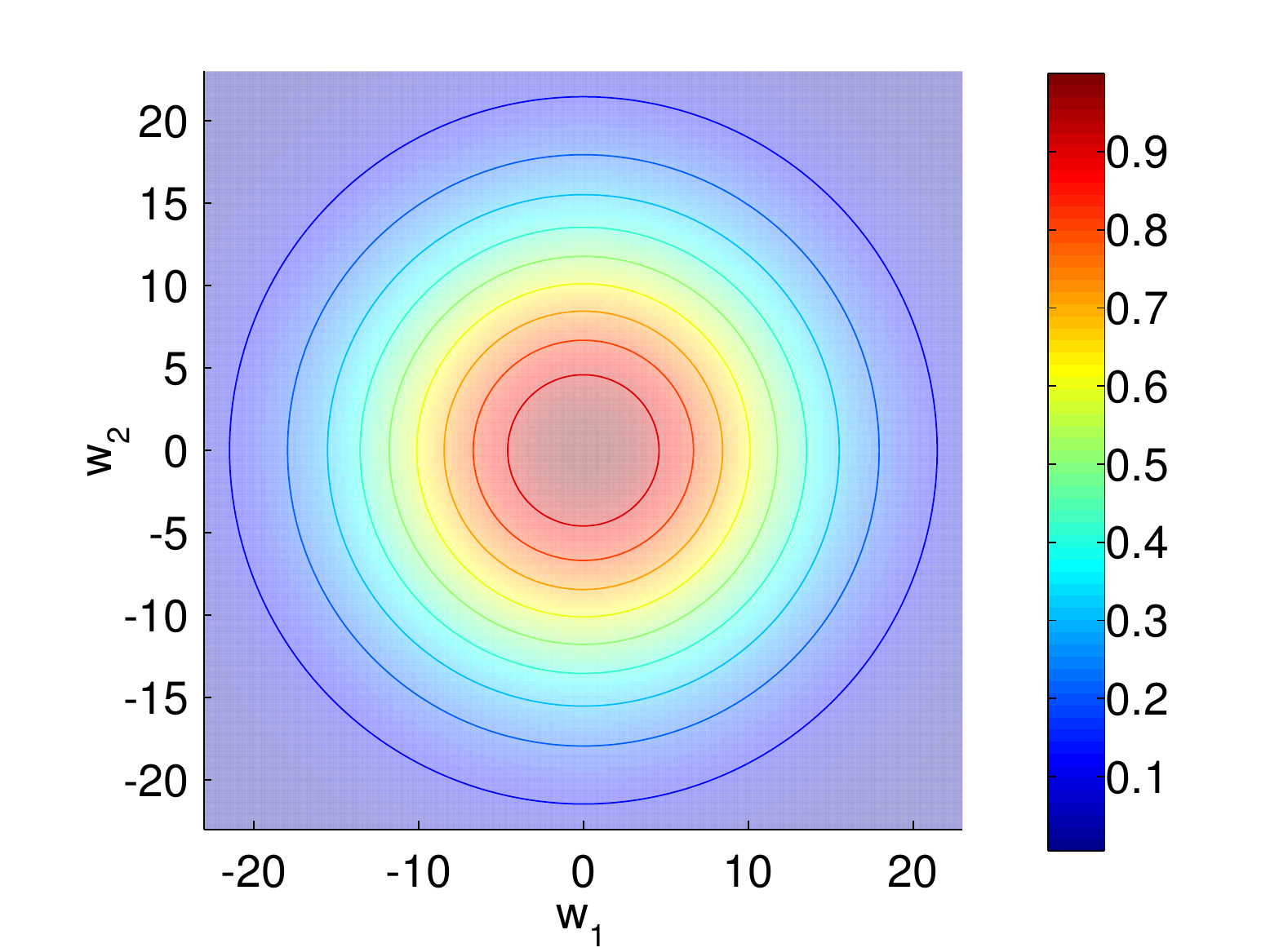}}
\hfill
\subfloat[Dataset\label{fig:compact_predictive:logreg_setup:dataset}]{
\includegraphics[width=\imwidth]{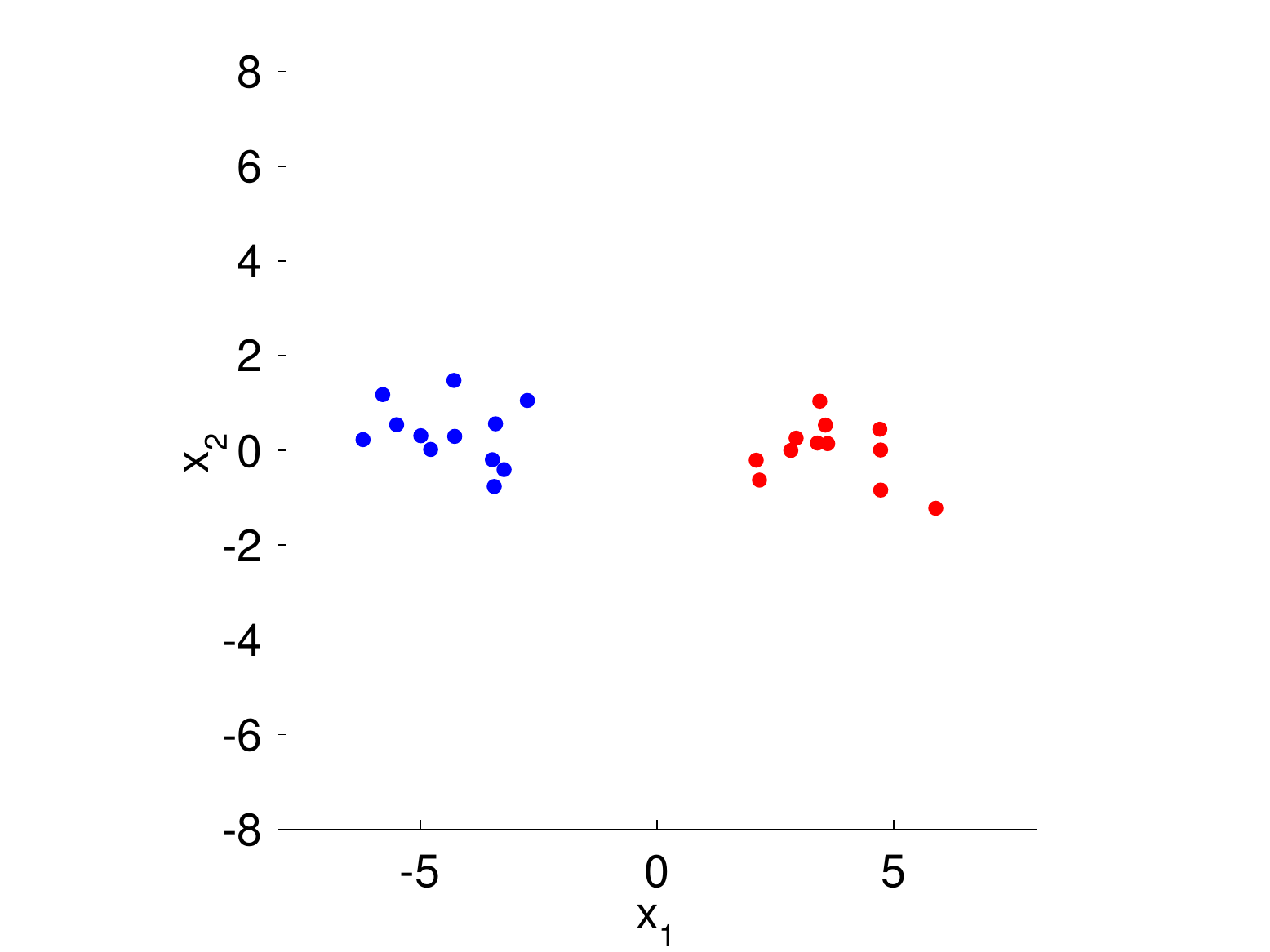}}
\caption{Setup for Bayesian logistic regression. The Gaussian prior over the weights is shown in (a). The dataset is shown in (b), with red corresponding to $y=1$ and blue corresponding to $y = 0$.}
\label{fig:compact_predictive:logreg_setup}
\end{figure}

\subsubsection{Distillation details}

In our experiments, we used slice sampling \citep{Neal:2000:slice_sampling} to generate MCMC samples from the posterior $\prob{\vect{w}\g D}$. Slice sampling has the advantage over other MCMC methods that it is broadly applicable and requires minimal tuning. We used linear stepping out, and we performed univariate updates to each parameter in turn. In all experiments, we initialized the Markov chain at $\vect{0}$, burned it in for $1000$ iterations and used no thinning.

Having collected a set of MCMC samples $\set{\vect{w}_s}$, the predictive distribution is given by the following Monte Carlo estimate
\begin{equation}
t\br{\vect{x}} = \prob{y = 1\g\vect{x},D}  \approx \frac{1}{S}\sum_s{\sigm{\vect{w}_s^T\vect{x}}}.
\end{equation}
For batch distillation, we used $S=10{,}000$ MCMC samples to form the above estimate.
Following \citet{Snelson:2005:compact_approximations}, the compact model we used in our experiments is of the same form as the above Monte Carlo estimate, but with only few components. That is, the compact model is defined to be
\begin{equation}
f\br{\vect{x}\g\bm{\theta}} = \prob{y = 1\g\vect{x},\bm{\theta}} =  \frac{1}{S'}\sum_{s'}{\sigm{\vect{w}_{s'}^T\vect{x}}},
\label{eq:compact_predictive:binclass:compact_model}
\end{equation}
with $S'\ll S$ and $\bm{\theta} = \set{\vect{w}_{s'}}$. In all our experiments, use used $S'=10$. Note that the parameters $\set{\vect{w}_{s'}}$ of the compact model can be viewed as a small, compact set of $10$ weight samples. Hence, distillation in this case can be interpreted as trying to compress the knowledge of the whole Markov chain into only $10$ samples. The locations of those $10$ samples are optimized during training, in order to yield a predictive distribution that is as close as possible to the predictive distribution obtained by the whole chain.

As discussed in the previous sections, in order to train the compact model $f\br{\vect{x}\g\bm{\theta}}$, we need to be in a position to calculate the gradient of the loss function with respect to $\bm{\theta}$. Both the MCMC predictor $t\br{\vect{x}}$ and the compact predictor $f\br{\vect{x}\g\bm{\theta}}$ are given by a linear combination of sigmoid terms. Hence, they can be viewed as feedforward neural networks\index{Neural network} with a single hidden layer of logistic sigmoid units, and a linear output layer. Therefore, the gradients of the loss functions can be easily calculated using standard backprop and R\{backprop\}, which are described in detail in appendix~\ref{chapter:derivatives_in_neural_networks}.

For optimizing the loss functions, in all our experiments we performed $5000$ iterations, using minibatches of size $10$. We used learning rates that were initially set at $1$ and linearly decayed to reach $0$ in the last iteration. Finally, minibatches of input locations were independently drawn from the following input distribution
\begin{equation}
\prob{\vect{x}} = \gaussian{\vect{x}}{\vect{0}}{100\mat{I}}.
\end{equation}

\subsubsection{Results and discussion}

\begin{figure}[p]
\centering
\subfloat[Posterior (unnormalized)]{
\includegraphics[width=\imwidth]{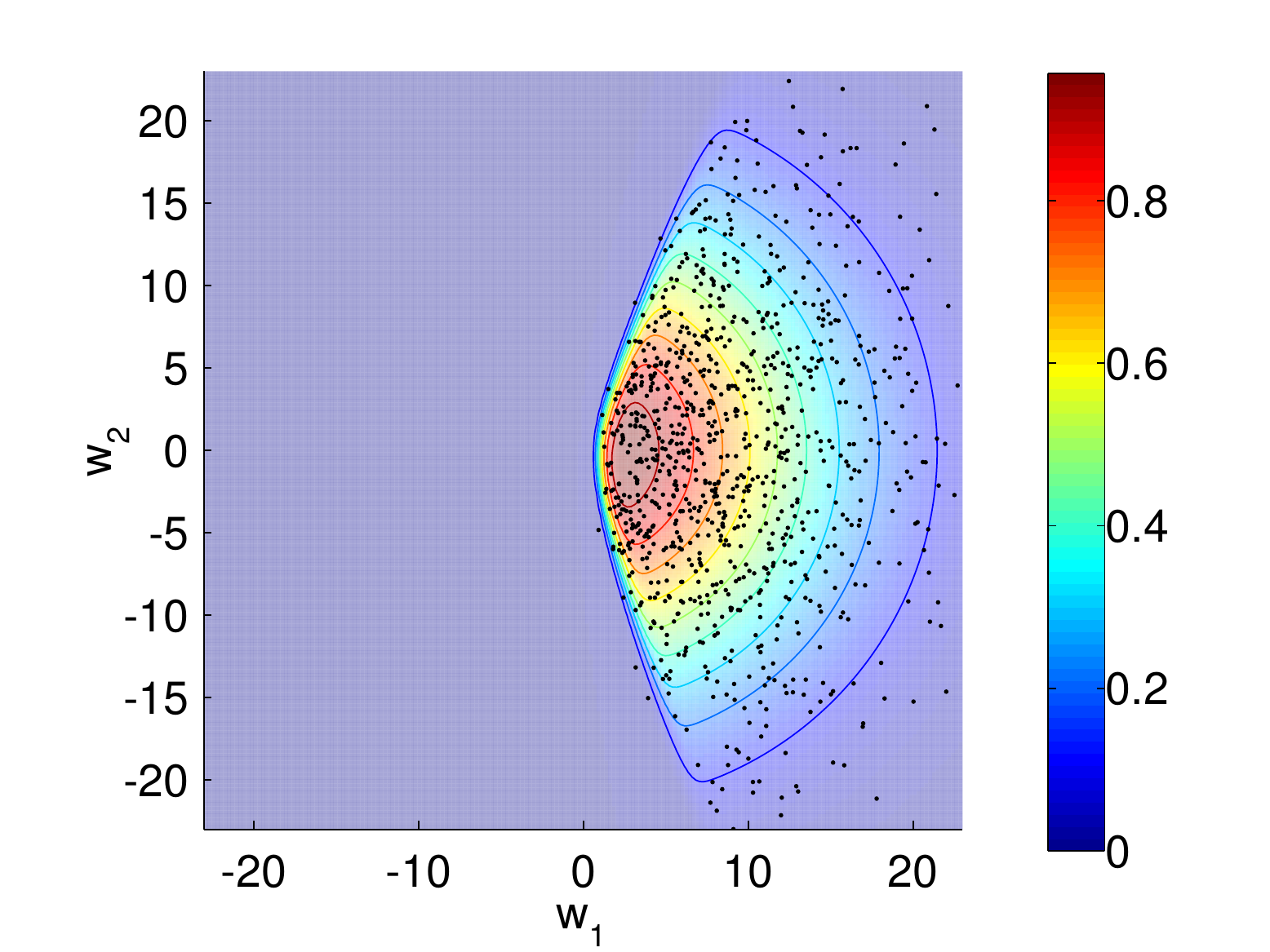}}
\hfill
\subfloat[Predictive distribution]{
\includegraphics[width=\imwidth]{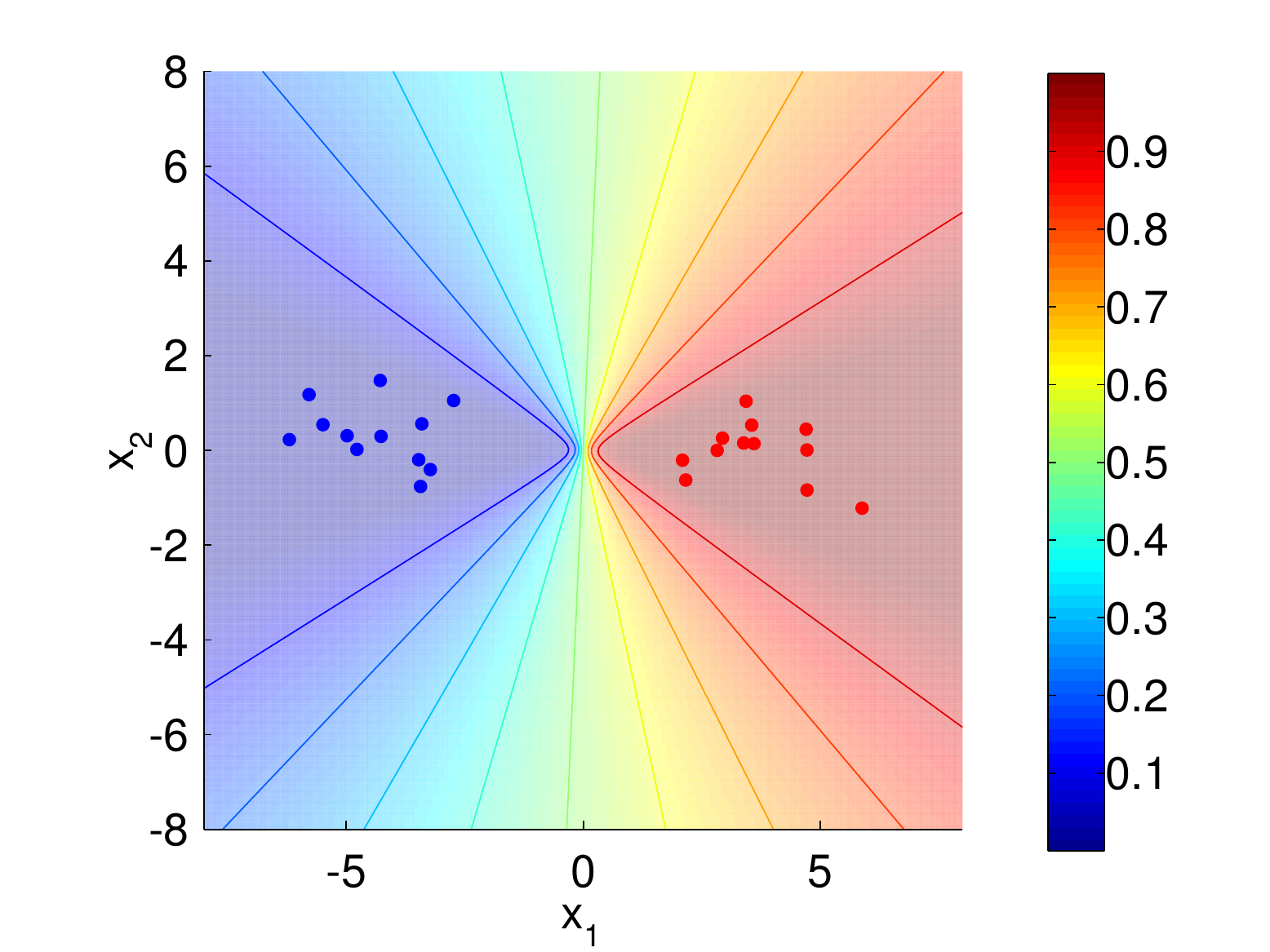}}
\caption{Bayesian logistic regression using MCMC\@. Black dots in (a) correspond to $10\%$ of all MCMC samples.}
\label{fig:compact_predictive:logreg_results_mcmc}
\end{figure}

\begin{figure}[p]
\centering
\subfloat[Posterior (unnormalized)]{
\includegraphics[width=\imwidth]{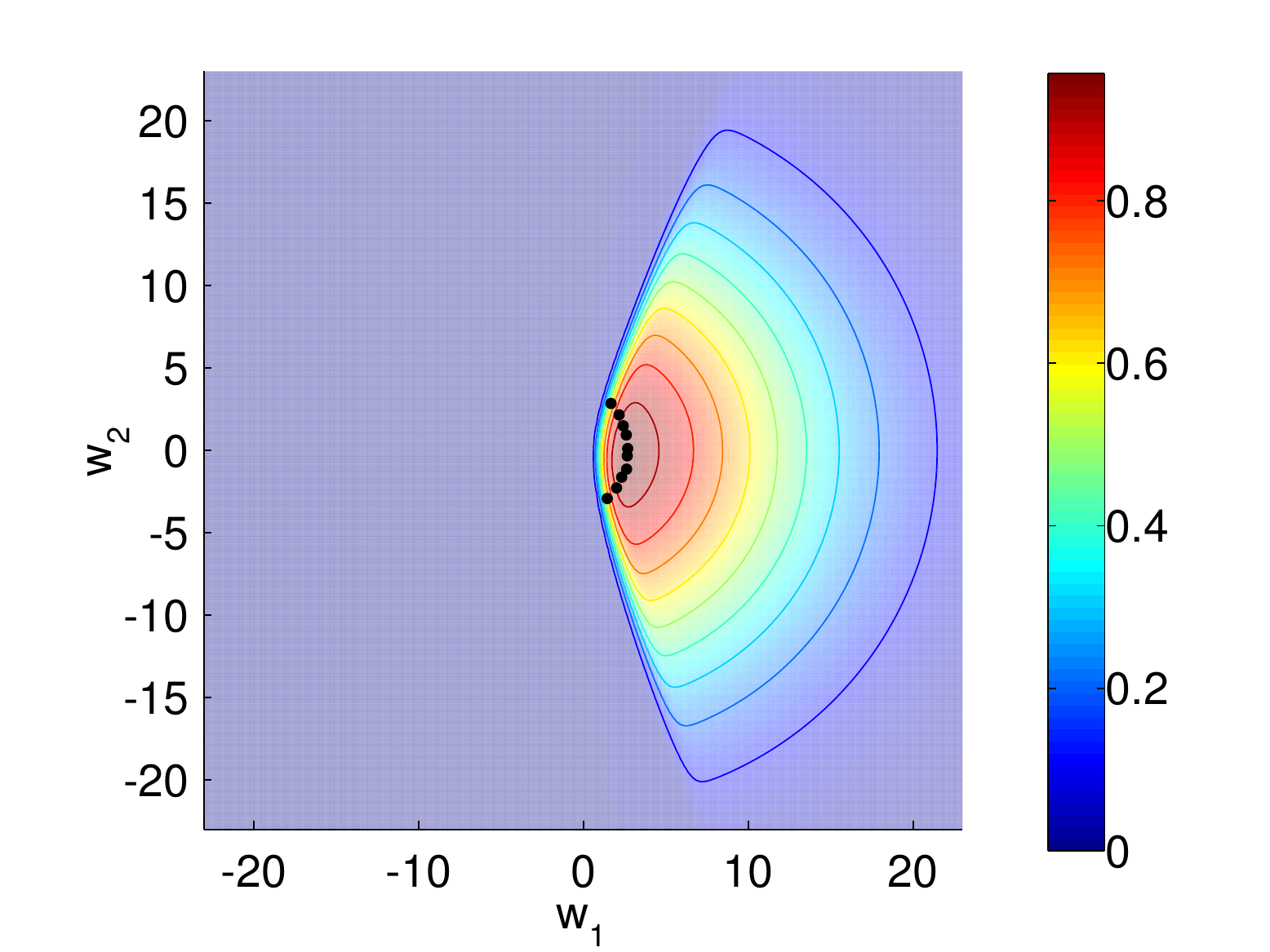}}
\hfill
\subfloat[Predictive distribution]{
\includegraphics[width=\imwidth]{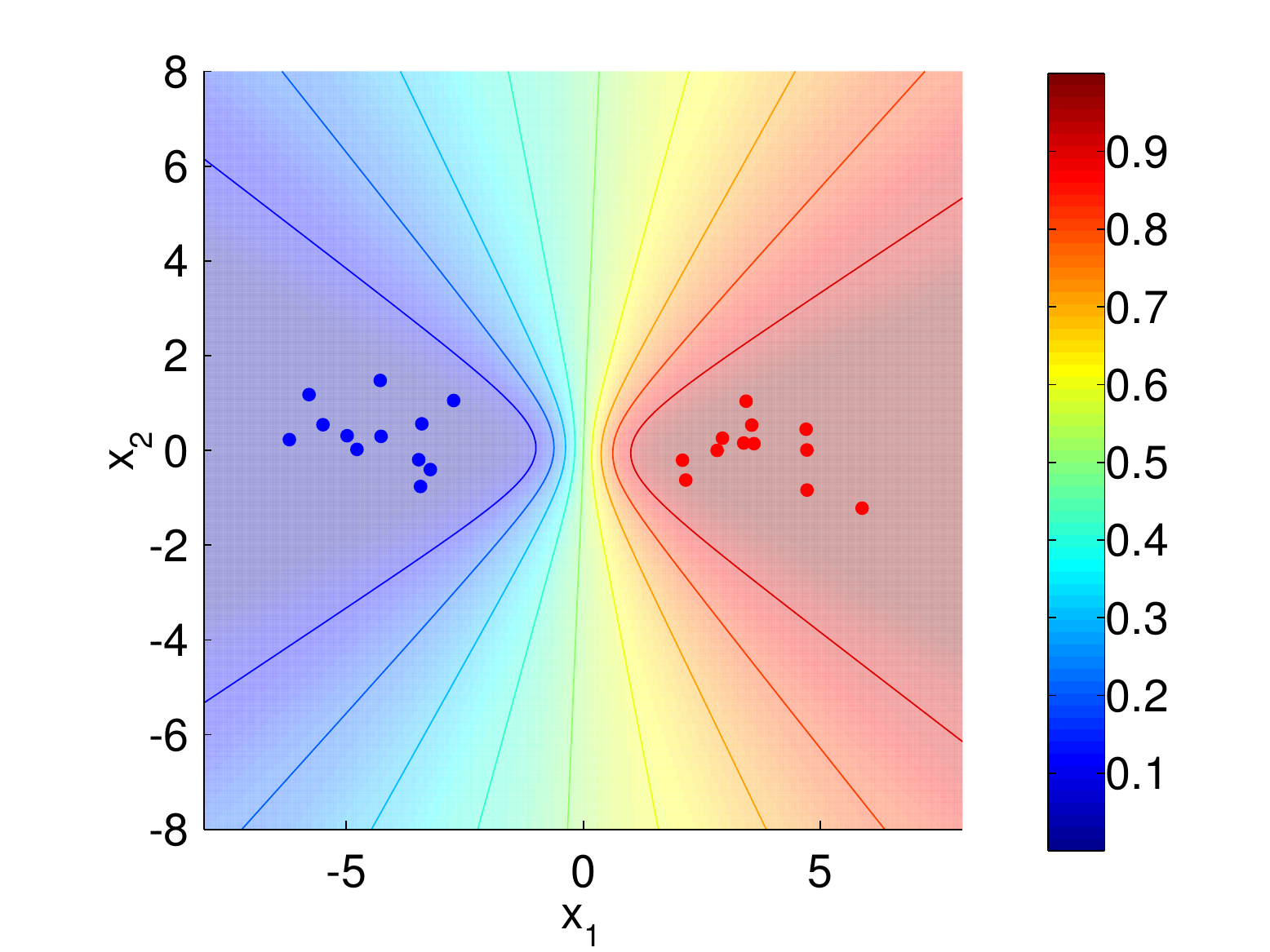}}
\caption{Batch distillation using cross entropy. Black dots in (a) correspond to the compact model's weights.}
\label{fig:compact_predictive:logreg_results_ce_batch}
\end{figure}

\begin{figure}[p]
\centering
\subfloat[Posterior (unnormalized)]{
\includegraphics[width=\imwidth]{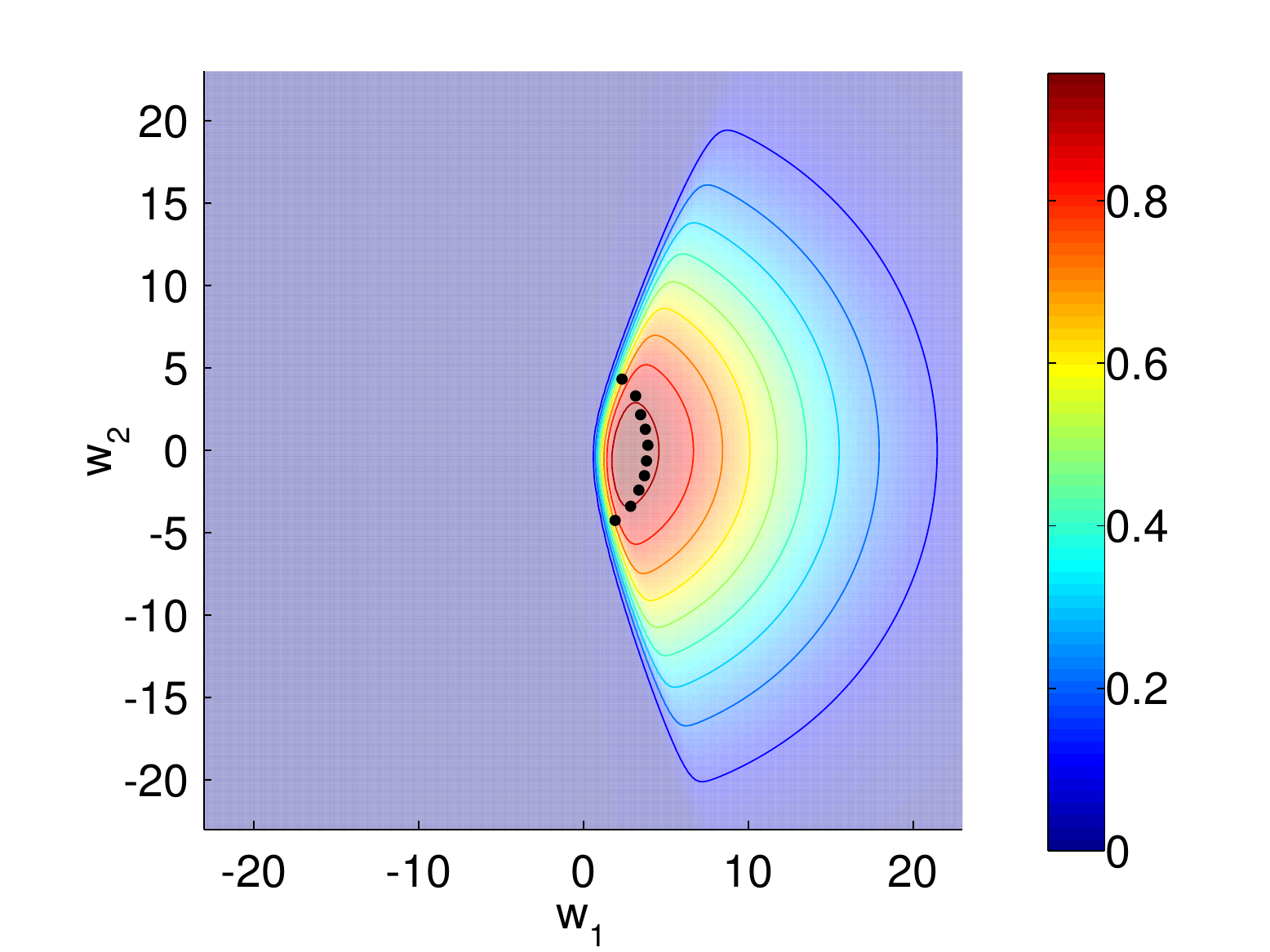}}
\hfill
\subfloat[Predictive distribution]{
\includegraphics[width=\imwidth]{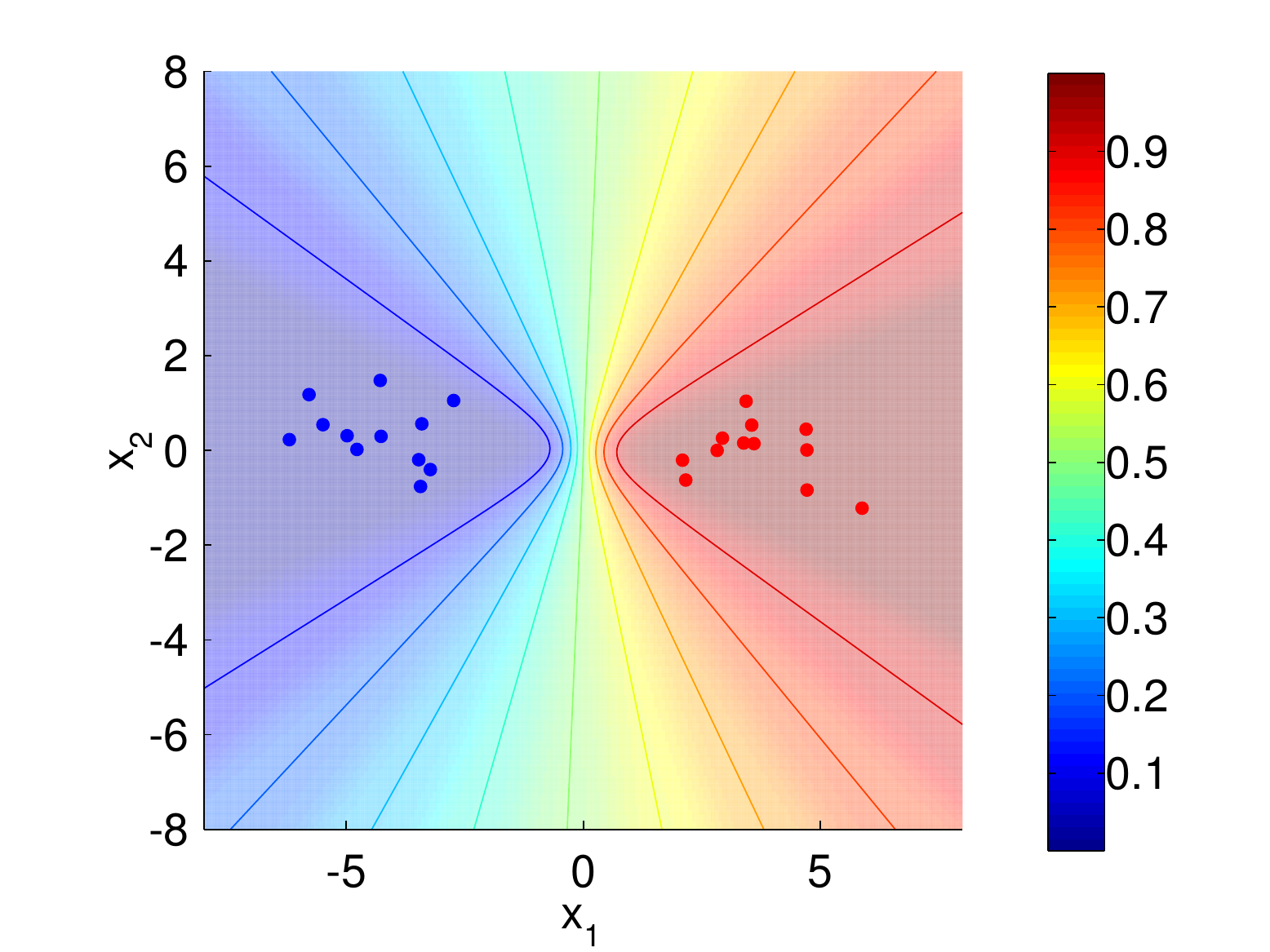}}
\caption{Batch distillation using derivative square error. Black dots in (a) correspond to the compact model's weights.}
\label{fig:compact_predictive:logreg_results_dse_batch}
\end{figure}

\begin{figure}[p]
\centering
\subfloat[Posterior (unnormalized)]{
\includegraphics[width=\imwidth]{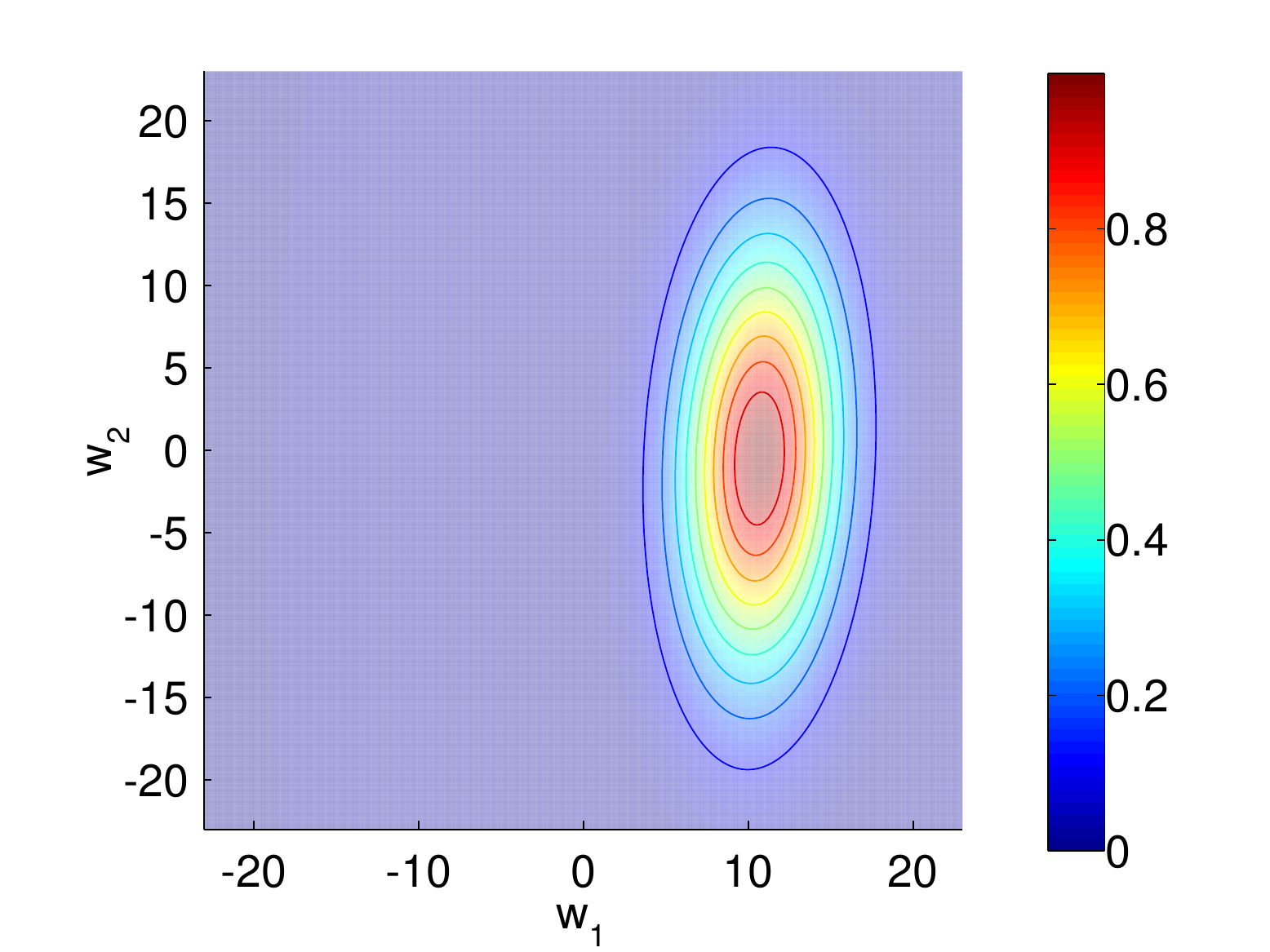}}
\hfill
\subfloat[Predictive distribution]{
\includegraphics[width=\imwidth]{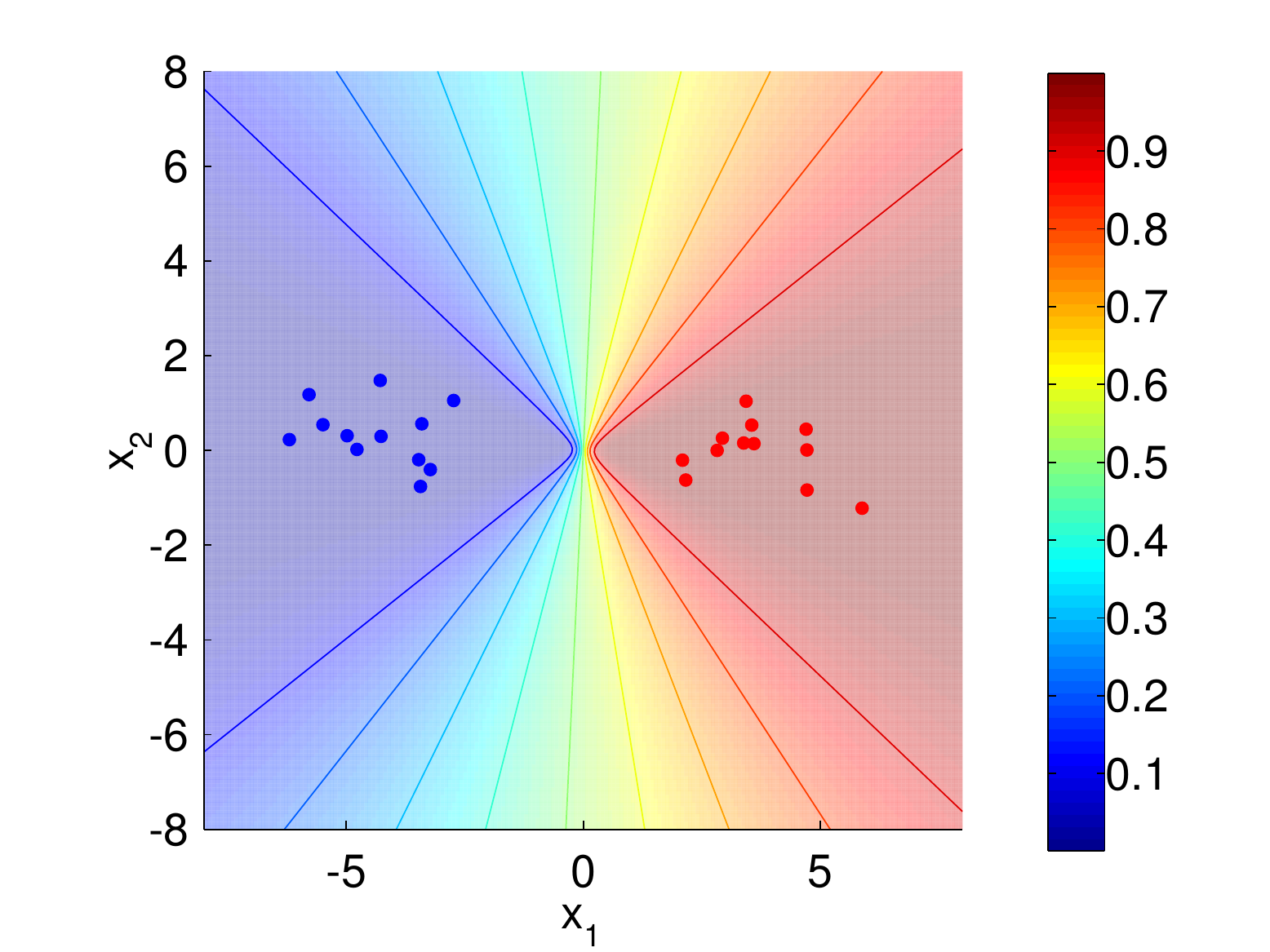}}
\caption{Bayesian logistic regression using EP\@. The posterior in (a) is the Gaussian approximation to the true posterior.}
\label{fig:compact_predictive:logreg_results_ep}
\end{figure}

\begin{figure}[p]
\centering
\subfloat[Posterior (unnormalized)]{
\includegraphics[width=\imwidth]{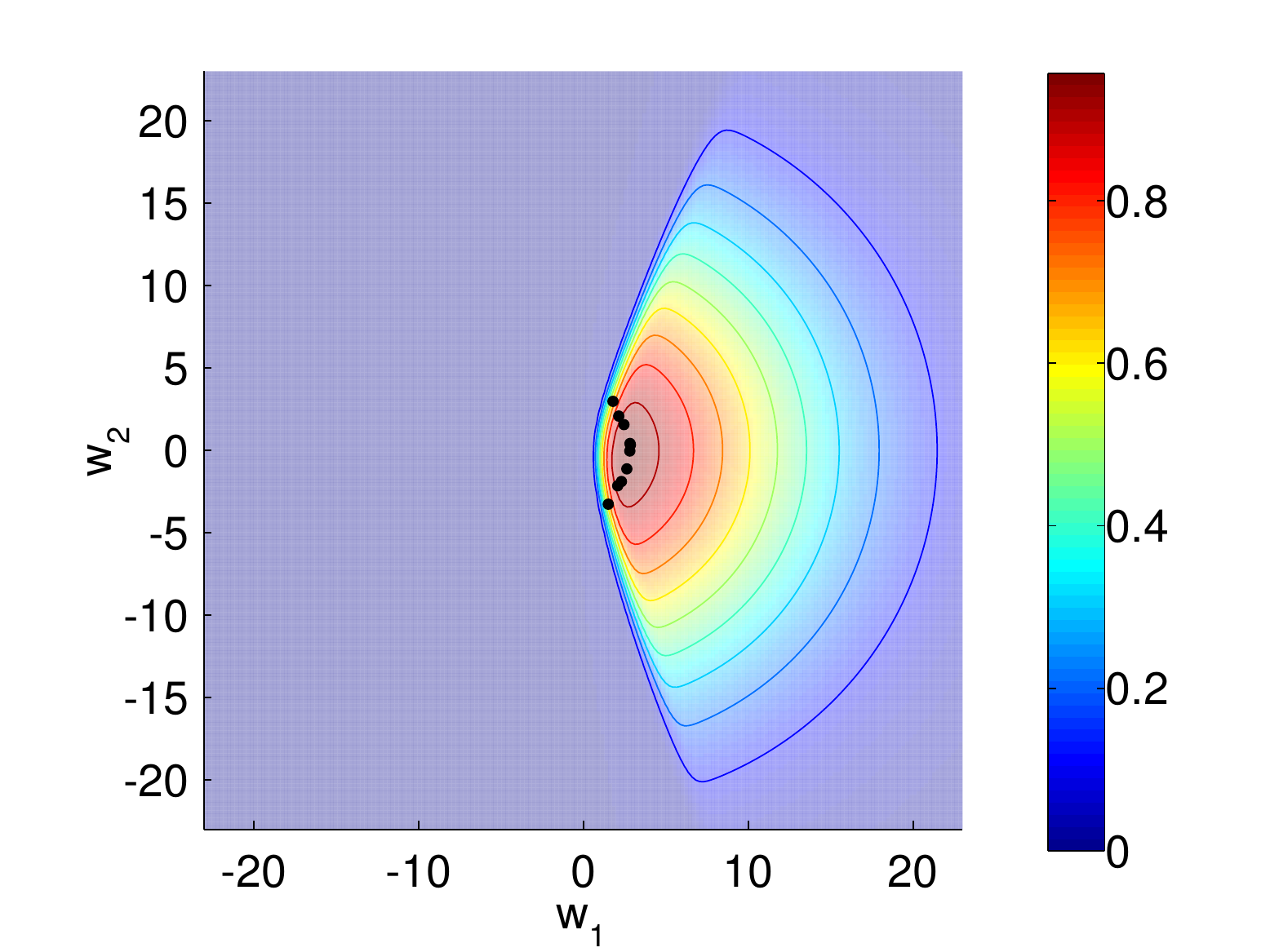}}
\hfill
\subfloat[Predictive distribution]{
\includegraphics[width=\imwidth]{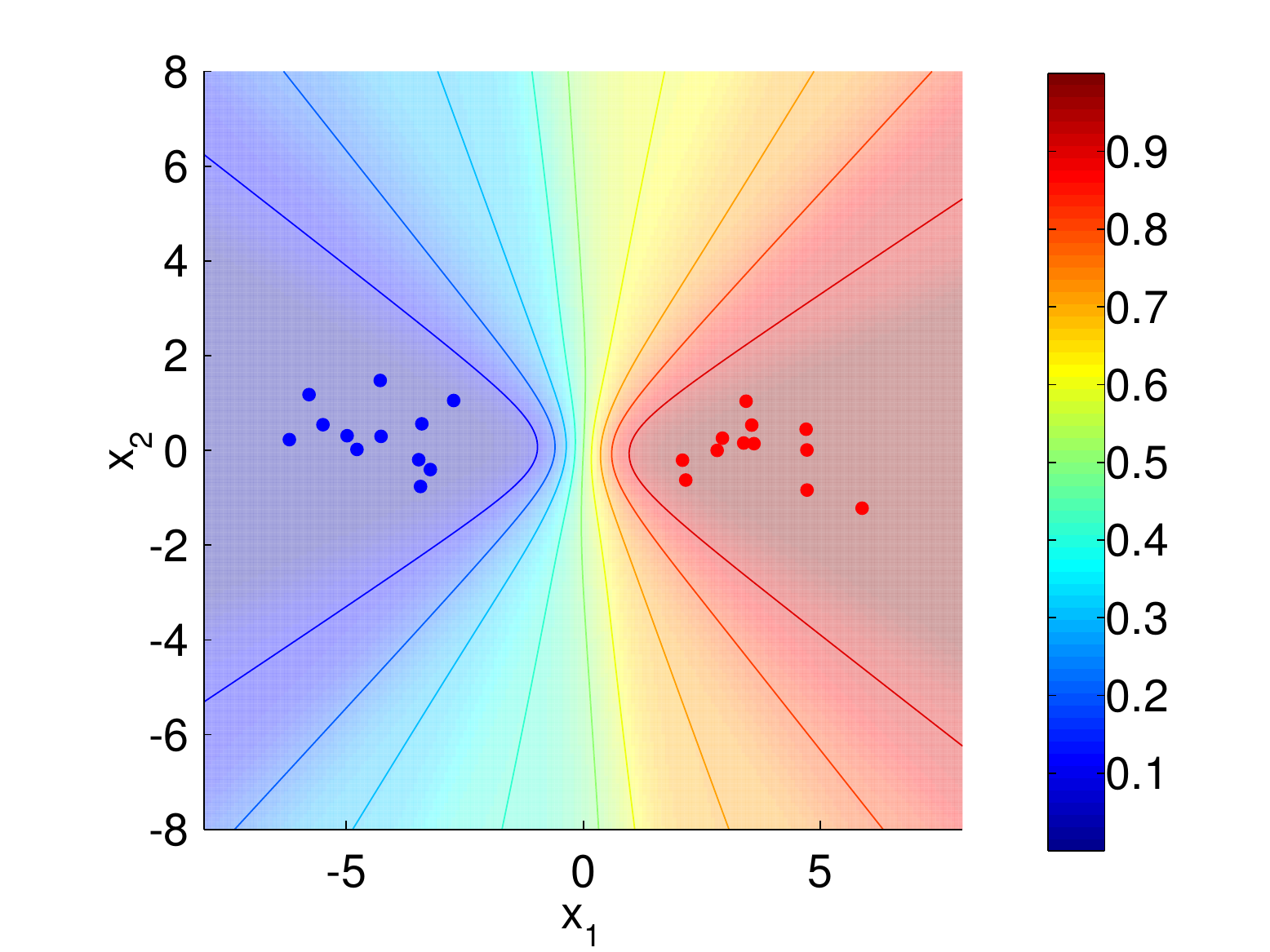}}
\caption{Online distillation using cross entropy. Black dots in (a) correspond to the compact model's weights.}
\label{fig:compact_predictive:logreg_results_ce_online}
\end{figure}

\begin{figure}[p]
\centering
\subfloat[Posterior (unnormalized)]{
\includegraphics[width=\imwidth]{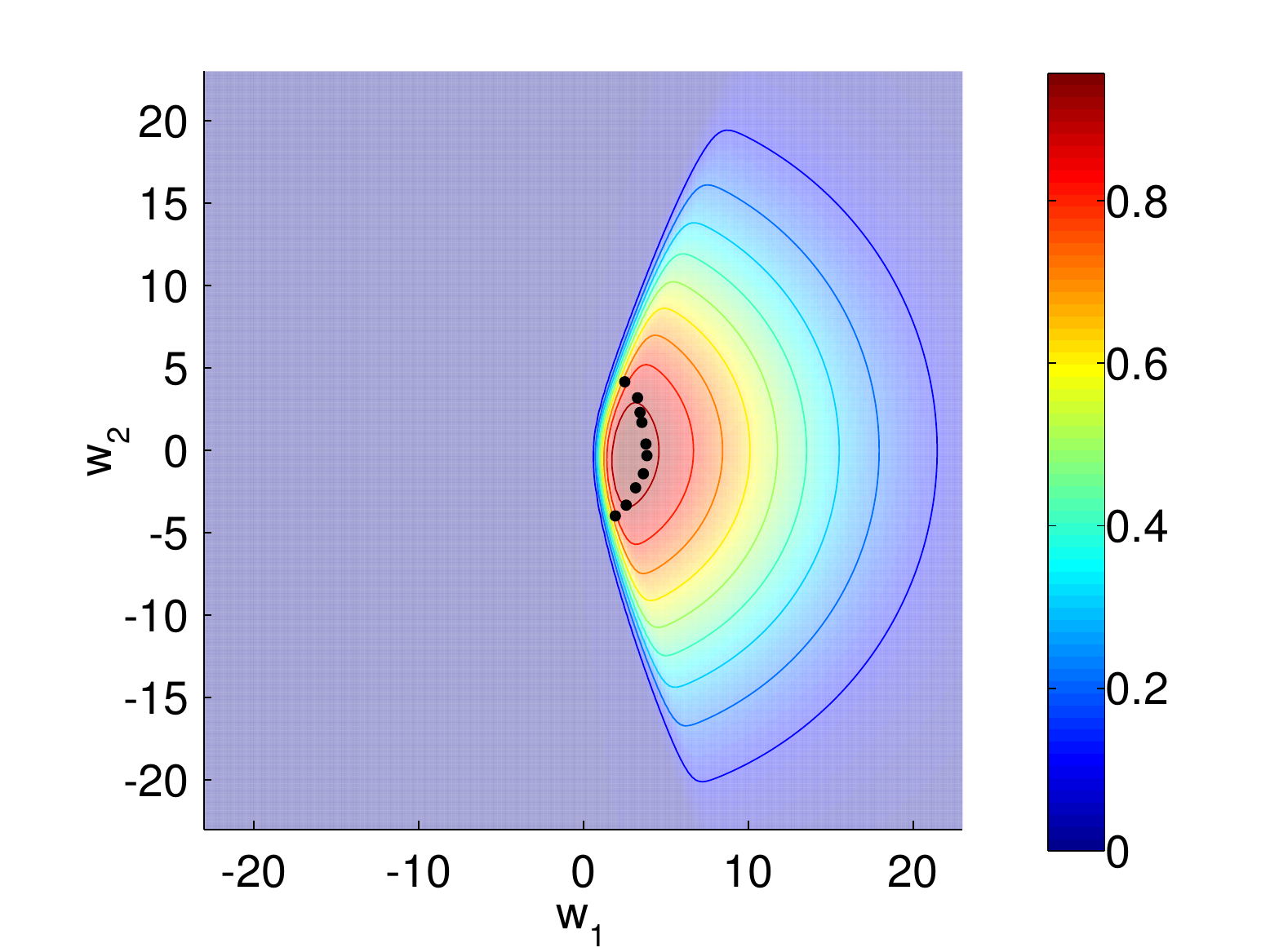}}
\hfill
\subfloat[Predictive distribution]{
\includegraphics[width=\imwidth]{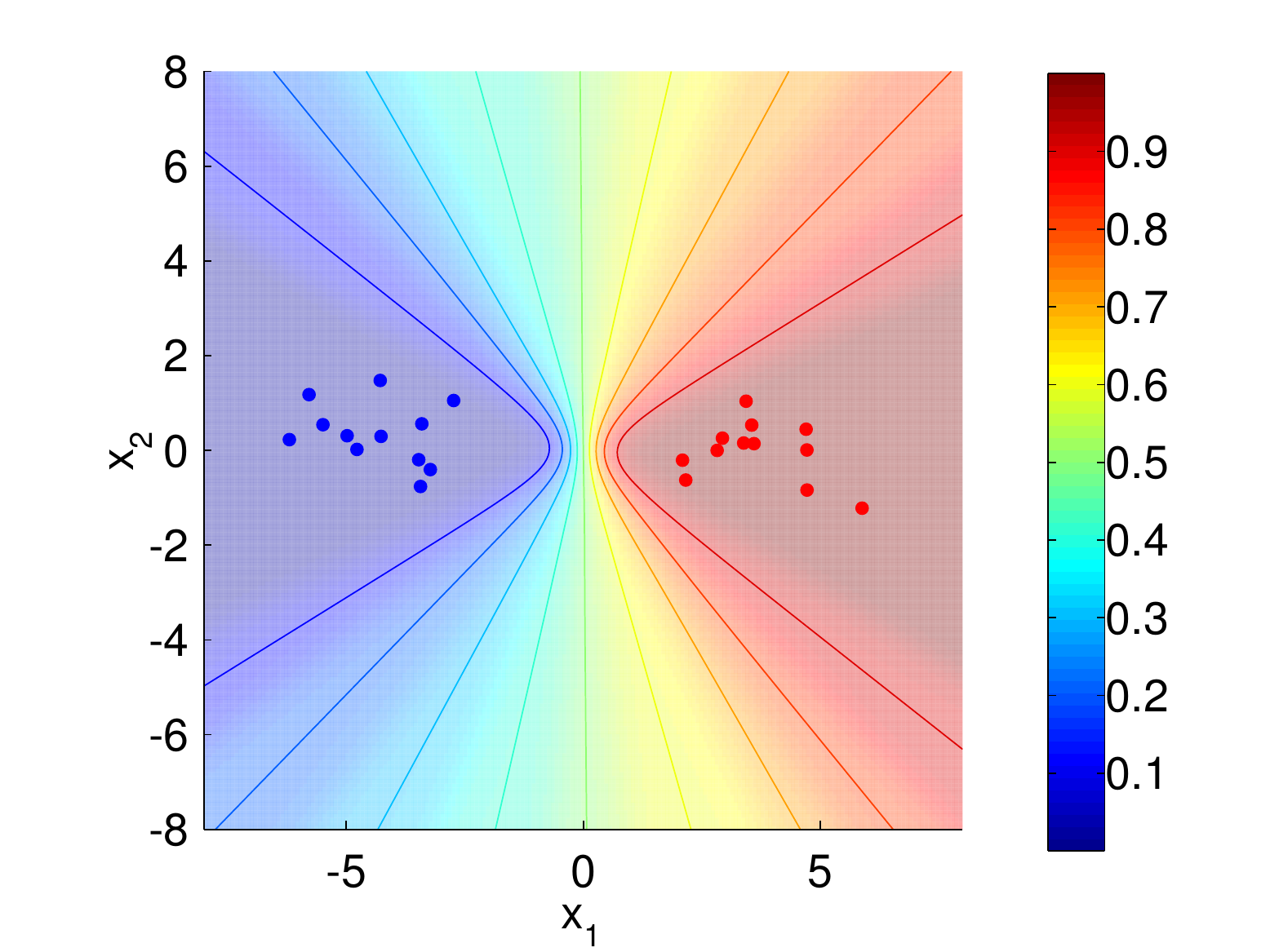}}
\caption{Online distillation using derivative square error. Black dots in (a) correspond to the compact model's weights.}
\label{fig:compact_predictive:logreg_results_dse_online}
\end{figure}

Figure~\ref{fig:compact_predictive:logreg_results_mcmc} shows graphically the MCMC predictor that was used in batch distillation, which was formed by $10{,}000$ MCMC samples drawn from $\prob{\vect{w}\g D}$. The left hand side plot shows the posterior $\prob{\vect{w}\g D}$ and $10\%$ of all MCMC samples drawn from it. The right hand side plot shows the predictive distribution $p_\mathrm{MC}\br{y=1\g\vect{x}}$, formed by averaging over all MCMC samples. Since the number of MCMC samples is large, we assume that this predictive distribution is a close approximation to the true one.

The shape of the MCMC predictive distribution shows how Bayesian logistic regression works. In the regions of input space that are densely populated by datapoints, Bayesian logistic regression is fairly confident about the value of $y$, which agrees with the true data labels. However, as we move away from the datapoints, the model becomes more and more uncertain about the value of $y$. This uncertainty is due to the small size of the dataset and the fact that it can be linearly separated by a wide range of linear separators. This is also reflected by the shape of the posterior, which shows that the model considers quite a large region of the weight space to be plausible. This region corresponds to the linear separators which Bayesian logistic regression has not yet received enough evidence to rule out.

Despite the fact that the MCMC predictive distribution is given by a sum of $10{,}000$ logistic sigmoid terms, it corresponds to a rather simple and smooth surface. Hence, it is reasonable to expect that this surface can be accurately captured by a more compact model. Using the compact model described in Equation~\eqref{eq:compact_predictive:binclass:compact_model}, we performed both batch and online distillation, using both cross entropy and derivative square error, resulting in $4$ compact models in total.
Figures~\ref{fig:compact_predictive:logreg_results_ce_batch}, \ref{fig:compact_predictive:logreg_results_dse_batch}, \ref{fig:compact_predictive:logreg_results_ce_online} and \ref{fig:compact_predictive:logreg_results_dse_online} show graphically the $4$ compact models. The left hand side plot shows the $10$ weights $\set{\vect{w}_{s'}}$ of each compact model, overlaid on top of the posterior $\prob{\vect{w}\g D}$. The right hand side plot shows the predictive distribution $\prob{y=1\g\vect{x},\bm{\theta}}$.

It can be seen that all $4$ compact models yield a predictive distribution that is similar to that of the MCMC model. The region where the compact models and the MCMC model appear to be most dissimilar is around the origin, between the red and blue datapoints. In this region, the predictive surface of the compact models appears to be smoother, whereas the MCMC model has a steeper slope. Still, all compact models correctly capture the uncertainty Bayesian logistic regression should have in regions away from the datapoints. As can be seen by the plots on the left hand side, the compact model chose an elegant way to approximate the MCMC model by placing its weights in a semicircle, spanning the region of plausible weights.

Compared to cross entropy, derivative square error appears to yield a better approximation of the MCMC model. From the left hand side plots, it can be seen that the weights found by derivative square error have larger magnitude than those found by cross entropy. As a result, the predictive distribution of derivative square error better captures the steep slope of the MCMC model around the origin. This suggests that using derivatives in the loss function can convey more information about the shape of the predictive surface, and hence lead to more accurate compact models.

We can see that online distillation produced similar results to batch distillation. Note that batch distillation needs to store the full set of $10{,}000$ MCMC samples, whereas online distillation only needs to store $10$ MCMC samples at any time (i.e.~as many as the minibatch size). This shows that compact models of the same quality can be learnt using only a fraction of the memory. This can be particularly useful in situations where the weights $\vect{w}$ are high-dimensional, and storing more than few of them becomes prohibitively expensive.

For comparison, we also used Expectation Propagation\index{EP} to calculate an approximation to the predictive distribution. Expectation Propagation \citep{Minka:2001:EP} works by first finding a Gaussian approximation to the posterior
\begin{equation}
\prob{\vect{w}\g D} \approx \gaussian{\vect{w}}{\vect{m}}{\mat{S}},
\end{equation}
and then using this Gaussian approximation instead of the true posterior in calculating the predictive distribution. For Bayesian logistic regression, it is easy to show that the EP predictive distribution is given by
\begin{equation}
\prob{y=1\g\vect{x},D} \approx \int{\sigm{z}\gaussian{z}{\vect{m}^T\vect{x}}{\vect{x}^T\mat{S}\vect{x}}\,\mathrm{d}z}.
\end{equation}
The above integral is always one-dimensional, regardless of the dimensionality of $\vect{x}$, but cannot be solved analytically. There are two ways to approximate it: (a) using quadrature, and (b) using the following approximation, due to \citet{MacKay:1992:evidence_framework}
\begin{equation}
\int{\sigm{z}\gaussian{z}{\vect{m}^T\vect{x}}{\vect{x}^T\mat{S}\vect{x}}\,\mathrm{d}z}
\approx \sigm{z^*}
\quad\text{where}\quad
z^* = \frac{\vect{m}^T\vect{x}}{\sqrt{1 + \frac{\pi}{8}\,\vect{x}^T\mat{S}\vect{x}}}.
\end{equation}
In our implementation of EP, we used the MacKay approximation for integrals of the above form in calculating the Gaussian approximate posterior, and we used global adaptive quadrature \citep{Shampine:2008:matlab_quadrature} for calculating predictions. \citet{Minka:2008:EP_cheatsheet} provides all the details of how to implement EP for Bayesian logistic regression.

Figure~\ref{fig:compact_predictive:logreg_results_ep} shows the approximate Gaussian posterior (left plot) and the predictive distribution (right plot) as computed by EP\@. Compared to the compact models, we can see that the EP predictive distribution manages to better capture the steepness of the MCMC prediction surface around the origin, but also appears to slightly underestimate the correct uncertainty away from the datapoints. In terms of compactness, EP only needs to store the mean $\vect{m}$ and the covariance $\mat{S}$ of the Gaussian approximate posterior, so it is competitive to distillation. However, distillation has a significant advantage over EP in terms of prediction speed. Indeed, for every location $\vect{x}$, EP needs to numerically calculate a one-dimensional integral, which is significantly more expensive compared to the compact model, whose predictive distribution is given in closed form.

\section{Related work}

In this section, we review work related to compact Bayesian inference, and compare it to our knowledge distillation framework presented in this chapter. We first review approaches that are based on knowledge distillation and are the most similar to our work. Then, we briefly review deterministic approximate inference, which provides an alternative to knowledge distillation for compact Bayesian inference.

\subsection{Distillation-based approaches}

\citet{Snelson:2005:compact_approximations} were the first who proposed training compact models to approximate Bayesian predictive distributions. In their framework, the true predictive distribution is approximated either by a large bag of MCMC samples or by deterministic methods such as EP\@. Training is done by minimizing the KL divergence from the approximate predictive distribution to the one given by the compact model.

Our knowledge distillation framework described in this chapter includes and extends that of \citet{Snelson:2005:compact_approximations}. Their methods have been incorporated into our framework under the name ``batch distillation with KL divergence''. Also, we followed their experiments in our case studies and confirmed their results. Moreover, our work extends theirs by (a) showing how distillation can be performed in an online fashion and (b) incorporating derivative information into the distillation procedure.

Recently, \citet{Korattikara:2015:bayesian_dark_knowledge} proposed what they call \emph{Bayesian dark knowledge}\index{Bayesian dark knowledge}, which is also a framework for distilling compact predictive distributions from MCMC samples. Their work focuses on Bayesian neural networks\index{Neural network!Bayesian neural network} \citep{Neal:1996:bayesian_nets}, which are neural networks that make predictions by taking into account the uncertainty in their parameters. Bayesian neural networks pose a hard inference problem; the typical MCMC-based solution is to sample a large collection of networks using MCMC, and then make predictions by averaging the predictions of all networks in the collection.

Bayesian dark knowledge aims at distilling the whole collection of MCMC-sampled networks into a single compact network, by minimizing the KL divergence from the true predictive distribution to the one given by the compact network. Distillation takes place on the fly as the networks are generated by the Markov chain, thanks to which there is no need for storing more than one MCMC-sampled network at any time. Every time a network is generated by the Markov chain, the parameters of the compact network are updated using a stochastic gradient step.

The distillation procedure used in Bayesian dark knowledge corresponds to ``online distillation with cross entropy'' in our framework. Note that Bayesian dark knowledge was developed at the same time as us but independently. It was published only after we had finished with our own work, so we were not aware of it when developing our framework.

\subsection{Deterministic approximate inference}
\index{Deterministic approximate inference}

Knowledge distillation attempts to construct a compact approximation directly to the predictive distribution. However, this is not the only approach to making Bayesian inference compact. A different approach is to construct a compact approximation not to the predictive distribution but to the posterior. Then, every time a prediction needs to be made, the compact approximate posterior can be used to calculate it.

A general framework for approximating complex posteriors with simpler distributions is \emph{variational inference}\index{Variational inference} \citep{Jordan:1999:variational_inference, Wainwright:2008:variational_inference}. In variational inference, the approximate posterior is chosen such that expectations over it are easy to compute. Nevertheless, there is nothing in the framework stopping us to also choose it to be compact.

Two general purpose implementations of variational inference that typically lead to compact approximate posteriors are \emph{variational message passing}\index{Variational message passing} \citep{Winn:2005:VMP} and \emph{Expectation Propagation}\index{EP} \citep{Minka:2001:EP}. Both of them work by locally passing messages in a graphical model, and their main difference lies in the variational objective they choose to minimize. The approximate posterior they compute is an exponential family distribution, hence it can be compactly represented by the vector of its natural parameters. EP was showcased in our experiments in section~\ref{sec:compact_predictive:bayesian_logreg}, where a Gaussian approximation to the true posterior was computed.

Having a posterior in a compact form does provide savings in storage, but it does not necessarily provide savings in computation at prediction time. This is because every time a prediction needs to be made, an expectation over the approximate posterior has to be computed. If this expectation can be computed analytically, then typically making predictions is fast. However, if numerical integration has to be used (such as in our EP example in section~\ref{sec:compact_predictive:bayesian_logreg}), then making predictions can be slow. In contrast, knowledge distillation provides a closed form approximation directly to the predictive distribution, which leads to fast predictions by design.

\section{Summary and conclusions}

In this chapter we presented a framework for constructing compact approximations to Bayesian predictive distributions. Our approach is based on distilling the knowledge contained in a Markov chain designed to generate samples from the posterior. We successfully applied our framework in the problems of Bayesian density estimation and Bayesian binary classification.

Our framework extends the state of the art in two different ways. Firstly, we showed how knowledge distillation can be performed on the fly as the Markov chain generates samples. This technique, which we called online distillation, can significantly reduce memory requirements. Secondly, in the case of Bayesian binary classification, we showed how to use derivatives to convey more information about the true predictive surface to the compact model, which can lead to more effective distillation.

Our framework relies on the Markov chain being capable of effectively exploring the space of plausible weights, and hence the compact predictive distribution it computes will be only as good as this Markov chain. For our experiments slice sampling was sufficiently good, however more sophisticated MCMC methods can be used to improve the performance in harder inference tasks. For instance, Hamiltonian dynamics \citep{Neal:2010:hmc} can be used to reduce dependence between samples, and stochastic gradient Langevin dynamics \citep{Welling:2011:SGLD} can be used to make MCMC scale to large datasets.

\chapter{Intractable Generative Models}
\label{chapter:generative_models}

A generative model\index{Generative model} is a probability distribution $\prob{\vect{x}}$ over some variable $\vect{x}$. In general, $\vect{x}$ can be any sort of data, such as text, sound, images, or other models. Being able to represent, store and compute with $\prob{\vect{x}}$ is perhaps the most general problem in machine learning; if we know the probability of everything, then we can condition on the observations, integrate out the unknowns, and thus infer any quantity of interest. 

Nevertheless, in practice all but the simplest generative models are intractable.
In this chapter, we focus our attention on intractable generative models of the form
\begin{equation}
\prob{\vect{x}} = \frac{1}{Z}\uprob{\vect{x}},
\end{equation}
where $\uprob{\vect{x}}$ is some tractable non-negative function (sometimes called a potential) and $Z$ is an intractable partition function\index{Partition function} (also called the normalizing constant) given by
\begin{equation}
Z = \int\uprob{\vect{x}}\,\mathrm{d}\vect{x}.
\end{equation}
We will assume that we can easily evaluate $\uprob{\vect{x}}$ but not $Z$, and therefore not $\prob{\vect{x}}$. 

Distributions of the above type are particularly common in Bayesian inference\index{Bayesian inference}. If $\vect{x}$ is some quantity we wish to infer given some data $D$, the posterior\index{Posterior distribution} over $\vect{x}$ is given by Bayes rule\index{Bayes rule}
\begin{equation}
\prob{\vect{x}\g D} = \frac{\prob{D\g \vect{x}}\,\prob{\vect{x}}}{\prob{D}}.
\end{equation}
Even if the likelihood\index{Likelihood} $\prob{D\g \vect{x}}$ and the prior\index{Prior distribution} $\prob{\vect{x}}$ are tractable, the marginal likelihood $\prob{D}$ and hence the posterior $\prob{\vect{x}\g D}$ can be intractable. Apart from the posterior, the marginal likelihood $\prob{D}$ is also a quantity of interest, since it measures how well the model fits the data and so it can be used for model comparison\index{Model comparison}.

Another example where distributions of the above type are commonplace is undirected graphical models\index{Graphical model}. An undirected graphical model \citep{Koller:2009:PGM} defines a distribution as
\begin{equation}
\prob{\vect{x}} = \frac{1}{Z}\exp\br{-U\br{\vect{x}}},
\end{equation}
where $U\br{\vect{x}}$ is some tractable energy function. Even if the energy is tractable, $Z$ and therefore $\prob{\vect{x}}$ can still be intractable. The partition function $Z$ is also a quantity of interest, since useful properties of the model can be obtained from the derivatives of $\log{Z}$ with respect to the model parameters---for  models in the exponential family, for instance, these derivatives correspond exactly to the expected sufficient statistics under the model.

In this chapter, we present a framework for distilling intractable generative models such as the above into tractable generative models of our choosing. After successful distillation, we can use the tractable model in place of the intractable one in order to perform computations that were previously infeasible. We showcase our framework by distilling an intractable RBM into a tractable NADE, and we demonstrate that the distilled NADE can be successfully used in combination with simple sampling schemes for robustly estimating the RBM's intractable partition function.

\section{Distilling intractable into tractable generative models}

Let $\prob{\vect{x}}$ be some generative model we are interested in but can only evaluate up to a multiplicative constant; that is, $\prob{\vect{x}}$ is given by
\begin{equation}
\prob{\vect{x}} = \frac{1}{Z}\uprob{\vect{x}},
\end{equation}
where $\uprob{\vect{x}}$ is easy to calculate but $Z$ is not. Assume we can construct a tractable generative model $\prob{\vect{x}\g \bm{\theta}}$, parameterized by a set of parameters $\bm{\theta}$, which is flexible enough to represent complex distributions. Our goal is to distil the knowledge contained in  $\prob{\vect{x}}$ into $\prob{\vect{x}\g \bm{\theta}}$; that is, we want to train $\prob{\vect{x}\g \bm{\theta}}$ to do the same job as $\prob{\vect{x}}$.

Our approach is based on minimizing an appropriate loss function $E\br{\bm{\theta}}$ that measures the discrepancy between $\prob{\vect{x}}$ and $\prob{\vect{x}\g \bm{\theta}}$ as a function of $\bm{\theta}$. The challenge is to specify a loss function that only involves tractable quantities and find a way to minimize it efficiently. In the following two sections, we describe and analyze our choice of appropriate loss functions, together with a stochastic way of minimizing them. As we shall see, our information about $\prob{\vect{x}}$ will come from MCMC samples and the values of $\uprob{\vect{x}}$ at sample locations.

\subsection{Loss functions}
\label{sec:generative_models:loss_functions}

In measuring the discrepancy between $\prob{\vect{x}}$ and $\prob{\vect{x}\g \bm{\theta}}$, the challenge is to come up with a loss function that avoids intractable quantities such as $\prob{\vect{x}}$ and $Z$, but at the same time captures as much information about $\prob{\vect{x}}$ as possible. In the following, we describe three loss functions that achieve this goal; \emph{KL divergence}, \emph{square error} and \emph{score matching}.

\subsubsection{KL divergence}
\index{KL divergence}

A natural measure of discrepancy is the KL divergence from $\prob{\vect{x}}$ to $\prob{\vect{x}\g \bm{\theta}}$, which is defined as
\begin{equation}
E_\mathrm{KL}\br{\bm{\theta}} = \kl{\prob{\vect{x}}}{\prob{\vect{x}\g \bm{\theta}}} = 
\avg{\log{\frac{\prob{\vect{x}}}{\prob{\vect{x}\g \bm{\theta}}}}}{\prob{\vect{x}}}.
\end{equation}
It is well known that the KL divergence is non-negative, and is equal to zero if and only if $\prob{\vect{x}} = \prob{\vect{x}\g \bm{\theta}}$ \citep[section~2.6]{MacKay:2002:IT}. Thus, if $\prob{\vect{x}\g \bm{\theta}}$ has sufficient capacity, globally minimizing $E_\mathrm{KL}\br{\bm{\theta}}$ will result in $\prob{\vect{x}\g \bm{\theta}}$ exactly matching $\prob{\vect{x}}$.

In order to efficiently minimize $E_\mathrm{KL}\br{\bm{\theta}}$ with respect to $\bm{\theta}$ while avoiding intractable quantities, note that
\begin{equation}
E_\mathrm{KL}\br{\bm{\theta}} = \avg{\log{\prob{\vect{x}}}}{\prob{\vect{x}}} - \avg{\log{\prob{\vect{x}\g \bm{\theta}}}}{\prob{\vect{x}}} = 
- \avg{\log{\prob{\vect{x}\g \bm{\theta}}}}{\prob{\vect{x}}} + \mathrm{const}.
\end{equation}
Therefore, minimizing $E_\mathrm{KL}\br{\bm{\theta}}$ is equivalent to maximizing $\avg{\log{\prob{\vect{x}\g \bm{\theta}}}}{\prob{\vect{x}}}$. This can be interpreted as fitting $\prob{\vect{x}\g \bm{\theta}}$ using maximum likelihood\index{Maximum likelihood} to an infinite amount of data generated from $\prob{\vect{x}}$. Later on, we will use MCMC to generate such data.

The KL divergence has a nice information-theoretic interpretation; it measures the extra number of nats (or bits, if the logs are base $2$) that are needed in order to encode samples $\vect{x}$ from $\prob{\vect{x}}$ using $\prob{\vect{x}\g \bm{\theta}}$ \citep[section~5.4]{MacKay:2002:IT}. In other words, it is the encoding cost we must pay for using the wrong distribution.

\subsubsection{Square error}
\index{Square error}

Another approach to measuring discrepancy, is to directly measure how much the values of $\prob{\vect{x}\g \bm{\theta}}$ differ from the values of $\prob{\vect{x}}$. A natural choice is the average square error of the logarithms, given by\footnote{Using a norm here is not necessary, since the quantities are scalar, but we find it aesthetically pleasing.}
\begin{equation}
E_\mathrm{SE}^0\br{\bm{\theta}} = \avg{\frac{1}{2}\norm{\log{\prob{\vect{x}\g \bm{\theta}}} - \log{\prob{\vect{x}}}}^2}{\prob{\vect{x}}}.
\end{equation}
Obviously, directly working with the above is not possible, since we have assumed that directly evaluating $\log{\prob{\vect{x}}}$ is intractable. However, it is possible to circumvent this problem by using the following tractable loss function instead
\begin{equation}
E_\mathrm{SE}\br{\bm{\theta}} = \avg{\frac{1}{2}\norm{\log{\prob{\vect{x}\g \bm{\theta}}} - \log{\uprob{\vect{x}}} + c}^2}{\prob{\vect{x}}},
\end{equation}
where $c$ is an appropriately chosen constant. The following proposition establishes the range of $c$ values for which minimizing $E_\mathrm{SE}\br{\bm{\theta}}$ correctly trains $\prob{\vect{x}\g \bm{\theta}}$ to match $\prob{\vect{x}}$.
\begin{proposition}
Assuming $\prob{\vect{x}\g \bm{\theta}}$ has sufficient capacity and $c\le \log{Z}$, $E_\mathrm{SE}\br{\bm{\theta}}$ is globally minimized if and only if $\prob{\vect{x}\g \bm{\theta}} = \prob{\vect{x}}$.
\end{proposition}
\begin{proof}
Using the fact that $\uprob{\vect{x}} = Z\prob{\vect{x}}$ and the definition of the KL divergence, $E_\mathrm{SE}\br{\bm{\theta}}$ can be decomposed as follows
\begin{equation}
E_\mathrm{SE}\br{\bm{\theta}} = E_\mathrm{SE}^0\br{\bm{\theta}}
+ \br{\log{Z}-c}E_\mathrm{KL}\br{\bm{\theta}}
+ \frac{1}{2}\norm{\log{Z}-c}^2,
\end{equation}
where
\begin{equation}
E_\mathrm{SE}^0\br{\bm{\theta}} = \avg{\frac{1}{2}\norm{\log{\prob{\vect{x}\g \bm{\theta}}} - \log{\prob{\vect{x}}}}^2}{\prob{\vect{x}}}
\end{equation}
and
\begin{equation}
E_\mathrm{KL}\br{\bm{\theta}} = \kl{\prob{\vect{x}}}{\prob{\vect{x}\g \bm{\theta}}}.
\end{equation}
Note that both $E_\mathrm{SE}^0\br{\bm{\theta}}$ and $E_\mathrm{KL}\br{\bm{\theta}}$ are non-negative. Therefore, assuming $c\le \log{Z}$, the global minimum of $E_\mathrm{SE}\br{\bm{\theta}}$ is achieved if and only if $E_\mathrm{SE}^0\br{\bm{\theta}} = E_\mathrm{KL}\br{\bm{\theta}} = 0$, which in turn is achieved if and only if $\prob{\vect{x}\g \bm{\theta}} = \prob{\vect{x}}$.
\end{proof}

The above proposition tells us that minimizing $E_\mathrm{SE}\br{\bm{\theta}}$ is equivalent to minimizing $E_\mathrm{SE}^0\br{\bm{\theta}}$ regularized by $E_\mathrm{KL}\br{\bm{\theta}}$, with regularizer $\br{\log{Z}-c}$. Therefore, as long as the regularizer is non-negative, $E_\mathrm{SE}\br{\bm{\theta}}$ is a sensible loss. 

In order to set $c$, we do not necessarily need to know $\log{Z}$, but we only need to lower bound it. If $\vect{x}$ is discrete, this becomes trivial since for any $\vect{x}$
\begin{equation}
\prob{\vect{x}} \le 1 \quad\Rightarrow\quad \log{\uprob{\vect{x}}} \le \log{Z}.
\end{equation}
Thus, setting $c$ less or equal to $\max_{\vect{x}}{\log{\uprob{\vect{x}}}}$ is guaranteed to work. More generally, lower bounding $\log{Z}$ can be done using the following inequality based on the variational free energy\index{Variational free energy}
\begin{equation}
\avg{\log{q\br{\vect{x}}}}{q\br{\vect{x}}} - \avg{\log{\uprob{\vect{x}\g \bm{\theta}}}}{q\br{\vect{x}}} \le \log{Z}.
\end{equation}
The above inequality holds for any distribution $q\br{\vect{x}}$, so $q\br{\vect{x}}$ can be freely chosen to be convenient. Finally, note that in the limit case where $c\rightarrow -\infty$, minimizing the square error loss becomes equivalent to minimizing the KL divergence loss.

We conclude our analysis with a final remark about the relationship between the square error loss and the KL divergence loss. Using Jensen's inequality and the convexity of norms, we can easily show that
\begin{equation}
E_\mathrm{SE}\br{\bm{\theta}} \ge \frac{1}{2}\norm{E_\mathrm{KL}\br{\bm{\theta}} + \br{\log{Z}-c}}^2.
\end{equation}
Hence, as long as $c\le \log{Z}$, minimizing the square error loss corresponds to minimizing an upper bound of the squared KL divergence loss plus a positive constant.

\subsubsection{Score matching}
\index{Score matching}

Score matching was first proposed by \citet{Hyvarinen:2005:score_matching}, where it was used to learn unnormalizable generative models. It is based on the clever trick that the derivative of the log probability does not depend on $Z$, that is
\begin{equation}
\pderivnull{\vect{x}}\log{\prob{\vect{x}}} = \pderivnull{\vect{x}}\log{\uprob{\vect{x}}}.
\end{equation}
Therefore, even though $\log{\prob{\vect{x}}}$ can be intractable, its derivative with respect to $\vect{x}$ is not. 

In the context of knowledge distillation, the intuition behind score matching is to train $\prob{\vect{x}\g \bm{\theta}}$ to mimic $\prob{\vect{x}}$ by matching their derivatives with respect to $\vect{x}$. Since they are both distributions, and therefore have to integrate to $1$, matching derivatives provides sufficient information for training. The score matching loss function is defined as
\begin{equation}
E_\mathrm{SM}\br{\bm{\theta}} = \avg{\frac{1}{2}\norm{\pderivnull{\vect{x}}\log{\prob{\vect{x}\g \bm{\theta}}} - \pderivnull{\vect{x}}\log{\prob{\vect{x}}}}^2}{\prob{\vect{x}}}.
\end{equation}
The following proposition, adapted from \citet[theorem 2]{Hyvarinen:2005:score_matching}, establishes the validity of the score matching loss for the continuous case.
\begin{proposition}
Assuming $\prob{\vect{x}\g \bm{\theta}}$ has sufficient capacity, $\vect{x}$ is continuous and $\prob{\vect{x}}>0$ everywhere, $E_\mathrm{SM}\br{\bm{\theta}}$ is globally minimized if and only if $\prob{\vect{x}\g \bm{\theta}} = \prob{\vect{x}}$.
\end{proposition}
\begin{proof}
First notice that $E_\mathrm{SM}\br{\bm{\theta}}\ge 0$ as the expectation of a norm. If if $\prob{\vect{x}\g \bm{\theta}} = \prob{\vect{x}}$ then trivially $E_\mathrm{SM}\br{\bm{\theta}}= 0$ and thus it is globally minimized. Conversely, if $E_\mathrm{SM}\br{\bm{\theta}}= 0$ then we have that for all $\vect{x}$
\begin{equation}
\pderivnull{\vect{x}}\log{\prob{\vect{x}\g \bm{\theta}}} = \pderivnull{\vect{x}}\log{\prob{\vect{x}}}\quad\Rightarrow\quad
\prob{\vect{x}\g \bm{\theta}} = c\,\prob{\vect{x}},
\end{equation}
for some constant $c>0$. But since both $\prob{\vect{x}\g \bm{\theta}}$ and $\prob{\vect{x}}$ must integrate to $1$, we have
\begin{equation}
c = \int c\,\prob{\vect{x}} \,\mathrm{d}\vect{x}= \int \prob{\vect{x}\g \bm{\theta}} \,\mathrm{d}\vect{x} =1,
\end{equation}
hence $\prob{\vect{x}\g \bm{\theta}} = \prob{\vect{x}}$.
\end{proof}

It is important to note that score matching works only in the continuous case and therefore it cannot be used for discrete distributions. Indeed, even though we can in principle calculate the derivatives of $\log{\prob{\vect{x}\g \bm{\theta}}}$ and $\log{\prob{\vect{x}}}$ at every location in the discrete domain of $\vect{x}$, making them equal does not necessarily make the distributions equal. To handle this, an extension to score matching termed \emph{ratio matching}\index{Ratio matching} was proposed by 
\citet{Hyvarinen:2007:score_matching_extensions} for learning unnormalizable discrete generative models. Ratio matching is based on matching conditional distributions, which do not depend on $Z$. However, it is not obvious how to adapt it within the context of knowledge distillation.

Another caveat of score matching is that it requires the support of $\prob{\vect{x}}$ to be the whole space. To see why, let the support of $\prob{\vect{x}}$ be $S_p$ and assume it is not the whole space. Then, a distribution satisfying $\prob{\vect{x}\g \bm{\theta}} = c\,\prob{\vect{x}}$ with $0<c<1$ for all $\vect{x}$ in $S_p$ will globally minimize $E_\mathrm{SM}\br{\bm{\theta}}$ without perfectly matching $\prob{\vect{x}}$. In other words, score matching does not constrain the behaviour of $\prob{\vect{x}\g \bm{\theta}}$ outside $S_p$ and it only guarantees proportionality within $S_p$, hence if $S_p$ is not the whole space, $\prob{\vect{x}\g \bm{\theta}}$ is allowed to ``leak'' mass outside $S_p$.

\subsection{Stochastic gradient training}
\label{sec:generative_models:stochastic_gradient_training}

In the previous section, we described and analyzed the following three loss functions for measuring the discrepancy between $\prob{\vect{x}\g \bm{\theta}}$ and $\prob{\vect{x}}$.
\begin{enumerate}[label=(\roman*)]
\item \textbf{KL divergence}
\begin{equation}
E_\mathrm{KL}\br{\bm{\theta}} =
- \avg{\log{\prob{\vect{x}\g \bm{\theta}}}}{\prob{\vect{x}}} + \mathrm{const}
\end{equation}
\item \textbf{Square error}
\begin{equation}
E_\mathrm{SE}\br{\bm{\theta}} = \avg{\frac{1}{2}\norm{\log{\prob{\vect{x}\g \bm{\theta}}} - \log{\uprob{\vect{x}}} + c}^2}{\prob{\vect{x}}}
\end{equation}
\item \textbf{Score matching}
\begin{equation}
E_\mathrm{SM}\br{\bm{\theta}} = \avg{\frac{1}{2}\norm{\pderivnull{\vect{x}}\log{\prob{\vect{x}\g \bm{\theta}}} - \pderivnull{\vect{x}}\log{\prob{\vect{x}}}}^2}{\prob{\vect{x}}}
\end{equation}
\end{enumerate}
All three loss functions exhibit a common pattern; they contain a tractable quantity within an intractable expectation over $\prob{\vect{x}}$. That is, they all have the following form
\begin{equation}
E\br{\bm{\theta}} = \avg{E\br{\vect{x},\bm{\theta}}}{\prob{\vect{x}}},
\end{equation}
for an appropriate instantiation of $E\br{\vect{x},\bm{\theta}}$.

We will describe a stochastic gradient\index{Stochastic gradient} learning procedure for optimizing loss functions of the above form. The gradient of $E\br{\bm{\theta}}$ with respect to $\bm{\theta}$ is
\begin{equation}
\vect{g}\br{\bm{\theta}} = \avg{\pderivnull{\bm{\theta}} E\br{\vect{x},\bm{\theta}}}{\prob{\vect{x}}}
\end{equation}
Even though $\vect{g}\br{\bm{\theta}}$ is intractable, we can stochastically approximate it by
\begin{equation}
\hat{\vect{g}}\br{\bm{\theta}} = \frac{1}{S}\sum_{s}{\pderivnull{\bm{\theta}} E\br{\vect{x}_s,\bm{\theta}}},
\end{equation}
where $\set{\vect{x}_s}$ are samples generated from $\prob{\vect{x}}$. Typically, we can generate $\set{\vect{x}_s}$ using MCMC\@. It is easy to see that, in expectation, the stochastic gradient is equal to the original gradient. Using the above stochastic gradient, we can minimize $E\br{\bm{\theta}}$ using the following procedure.
\begin{framed}
\begin{enumerate}[label=(\roman*)]
\item Generate a minibatch $\set{\vect{x}_s}$ of size $S$ from $\prob{\vect{x}}$ using MCMC\@.
\item Calculate the stochastic gradient $\hat{\vect{g}}\br{\bm{\theta}} = \frac{1}{S}\sum_{s}{\pderivnull{\bm{\theta}} E\br{\vect{x}_s,\bm{\theta}}}$.
\item Make an update on $\bm{\theta}$ using $\hat{\vect{g}}\br{\bm{\theta}}$.
\item Repeat until convergence.
\end{enumerate}
\end{framed}

Stochastic optimization algorithms of the above type are analyzed by \citet{Bottou:1999:online_learning}, where it is proved that---under certain regularity conditions---they converge almost surely to a stationary point of the objective function. However, their analysis assumes that samples generated from $\prob{\vect{x}}$ are exact and independent, therefore it does not apply when samples are generated by MCMC\@. Sufficient conditions for convergence even when samples are generated by MCMC are provided by \citet{Younes:1999:convergence_markov}.

To use the above algorithm in practice, the user has to specify (a) what MCMC algorithm to use to generate from $\prob{\vect{x}}$ and how to tune it, (b) what step size strategy to use in the updates on $\bm{\theta}$ and (c) how to determine convergence. We will address these questions in the specific context of our case study in section~\ref{sec:rbm_vs_nade}.

To complete our discussion, we will provide some details on how to calculate $\pderivnull{\bm{\theta}} E\br{\vect{x},\bm{\theta}}$ for each of the three loss functions under consideration.
\begin{enumerate}[label=(\roman*)]
\item \textbf{KL divergence}
\begin{equation}
\pderivnull{\bm{\theta}} E_\mathrm{KL}\br{\vect{x},\bm{\theta}} = -\pderivnull{\bm{\theta}}\log{\prob{\vect{x}\g \bm{\theta}}}
\end{equation}
\item \textbf{Square error}
\begin{equation}
\pderivnull{\bm{\theta}} E_\mathrm{SE}\br{\vect{x},\bm{\theta}} = \pderivnull{\bm{\theta}}\log{\prob{\vect{x}\g \bm{\theta}}}\,\br{\log{\prob{\vect{x}\g \bm{\theta}}} - \log{\uprob{\vect{x}}} + c}
\end{equation}
\item \textbf{Score matching}
\begin{equation}
\pderivnull{\bm{\theta}} E_\mathrm{SM}\br{\vect{x},\bm{\theta}} 
= \spderivb{}{\bm{\theta}}{\vect{x}}\log{\prob{\vect{x}\g \bm{\theta}}}\,\br{\pderivnull{\vect{x}}\log{\prob{\vect{x}\g \bm{\theta}}} - \pderivnull{\vect{x}}\log{\prob{\vect{x}}}} 
= \Rop{\pderivnull{\bm{\theta}}\log{\prob{\vect{x}\g \bm{\theta}}}}
\end{equation}
\end{enumerate}
As can be seen by the above equations, the only quantities we need to be in a position to calculate are
\begin{equation}
\pderivnull{\bm{\theta}}\log{\prob{\vect{x}\g \bm{\theta}}}
\quad\quad\text{and}\quad\quad
\Rop{\pderivnull{\bm{\theta}}\log{\prob{\vect{x}\g \bm{\theta}}}}.
\end{equation}
Note that in the case of score matching, $\pderivnull{\bm{\theta}} E_\mathrm{SM}\br{\vect{x},\bm{\theta}}$ takes the form of a Hessian-vector product, where the Hessian is that of $\log{\prob{\vect{x}\g \bm{\theta}}}$ and the vector is $\br{\pderivnull{\vect{x}}\log{\prob{\vect{x}\g \bm{\theta}}} - \pderivnull{\vect{x}}\log{\prob{\vect{x}}}}$. Hence, we can efficiently calculate it in linear time using the R technique\index{R technique} (as described in appendix~\ref{chapter:R_technique}).

\section{Case study: distilling RBM into NADE}
\label{sec:rbm_vs_nade}

In the previous section we described a general framework for distilling an intractable generative model $\prob{\vect{x}}$ into a tractable generative model $\prob{\vect{x}\g \bm{\theta}}$. In this section, we put the framework into practice in order to distil an intractable RBM into a tractable NADE\@.

\subsection{The Restricted Boltzmann Machine}
\index{RBM}

The Restricted Boltzmann Machine (RBM for short) is a widely-used generative model that is known to be a good distribution estimator \citep{Salakhutdinov:2008:quantitative_analysis_of_DBNs}. It was introduced by \citet{Smolensky:1986:RBM} and in the last decade it has enjoyed considerable success as a building block of deep generative models \citep{Hinton:2006:DBN, Bengio:2009:deep_AI}. 
It owes its success to its flexibility in capturing high order dependencies between variables and to the existence of effective algorithms for training it \citep{Hinton:2002:TPE}.

The RBM is a bipartite undirected graphical model (as shown in Figure~\ref{fig:rbm}), consisting of a layer of visible variables $\vect{x}$ and a layer of hidden variables $\vect{h}$. It is ``restricted'' in the sense that there are no connections within layers, only between them. In this work, we focus on the binary version of the RBM, where each variable can either be $0$ or $1$, however real-valued versions exist \citep{Nair:2010:RBM_relus}. From here on, whenever we say ``RBM'' we mean ``binary RBM''.

The probability distribution defined by the RBM over the visible variables $\vect{x}$ takes the following form
\begin{equation}
\prob{\vect{x}} = \frac{1}{Z}\sum_{\vect{h}}{\exp{\br{\vect{x}^T\mat{W}\vect{h} + \vect{a}^T\vect{x} + \vect{b}^T\vect{h}}}},
\end{equation}
where $\set{\mat{W},\vect{a},\vect{b}}$ are the parameters of the RBMs. Evaluating the unnormalized probability $\uprob{\vect{x}} = Z\prob{\vect{x}}$ seems to be intractable, since it is a sum of $2^J$ terms, where $J$ is the number of hidden variables. However, due to the special structure of the RBM, it is actually tractable. To see why, define the following quantity
\begin{equation}
\vect{z} = \mat{W}^T\vect{x} + \vect{b}
\end{equation}
and rewrite $\uprob{\vect{x}}$ as follows
\begin{align}
\uprob{\vect{x}} &= \exp\br{\vect{a}^T\vect{x}}\sum_{\vect{h}}{\exp{\br{\vect{z}^T\vect{h}}}}  \notag\\
&= \exp\br{\vect{a}^T\vect{x}}\sum_{\vect{h}}\prod_j{\exp{\br{z_jh_j}}}  \notag\\
&= \exp\br{\vect{a}^T\vect{x}}\prod_j\sum_{h_j}{\exp{\br{z_jh_j}}}  \notag\\
&= \exp\br{\vect{a}^T\vect{x}}\prod_j\br{1 + \exp{z_j}}.
\end{align}
Hence, $\uprob{\vect{x}}$ can be computed in time which is linear in $J$.

Even though $\uprob{\vect{x}}$ is tractable, several quantities of interest associated with the RBM are intractable. The normalizing constant $Z$ is intractable, as it is a sum of $2^I$ terms where $I$ is the number of visible variables. As a consequence, $\prob{\vect{x}}$ is intractable as well, and so are marginals $\prob{x_i}$ and conditionals $\prob{x_i\g \vect{x}_{<i}}$. Finally, drawing exact samples from the RBM is also hard.

Even though exact sampling is hard, obtaining approximate samples from the RBM can be done using MCMC\@. A practical way to do it is using block Gibbs sampling\index{Sampling!Gibbs sampling}. This technique is based on the fact that both $\prob{\vect{h}\g \vect{x}}$ and $\prob{\vect{x}\g \vect{h}}$ are easy to evaluate. Indeed, it is easy to show that
\begin{equation}
\prob{\vect{h}\g \vect{x}} = \sigm{\mat{W}^T\vect{x} + \vect{b}}
\end{equation}
and by symmetry
\begin{equation}
\prob{\vect{x}\g \vect{h}} = \sigm{\mat{W}\vect{h} + \vect{a}},
\end{equation}
where $\sigm{\cdot}$ is the logistic sigmoid function applied elementwise. Given a sample $\vect{x}_n$, block Gibbs sampling draws a sample $\vect{h}_n$ from $\prob{\vect{h}\g \vect{x}_n}$ and then draws a sample $\vect{x}_{n+1}$ from $\prob{\vect{x}\g \vect{h}_n}$. Starting from some seed value $\vect{x}_0$ and iterating the above procedure $N$ times, block Gibbs sampling generates a sequence $\vect{x}_1,\vect{x}_2,\ldots,\vect{x}_N$ of approximate correlated samples from $\prob{\vect{x}}$. Of course, as with any MCMC method, letting the chain burn in, thinning samples or using parallel chains can all be employed to improve the quality of samples.

\begin{figure}[tbp]
\centering
\begin{tikzpicture}

	\pgfmathsetmacro{\a}{-1.5}

    \tikzstyle{neuron} = [draw,semithick,circle, minimum size=20pt, inner sep=0pt]

    \foreach \y in {1,...,4}
        \node[neuron] (I-\y) at (-\a*\y,0) {$x_\y$};

    \foreach \y in {1,...,3}
        \path[xshift=-\a*0.5cm]
            node[neuron] (H-\y) at (-\a*\y,1.8) {$h_\y$};

    \foreach \source in {1,...,4}
        \foreach \dest in {1,...,3}
            \path[draw=black!75] (I-\source) edge (H-\dest);

\end{tikzpicture}
\caption{Graphical model of an RBM with $4$ visible and $3$ hidden variables. Its visible variables are $x_1$, $x_2$, $x_3$ and $x_4$ (bottom layer) and its hidden variables are $h_1$, $h_2$ and $h_3$ (top layer).}
\label{fig:rbm}
\end{figure}
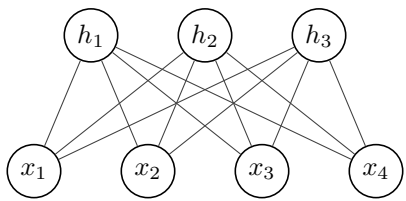

\subsection{The Neural Autoregressive Distribution Estimator}
\label{sec:generative_models:NADE}
\index{NADE}

The Neural Autoregressive Distribution Estimator (NADE for short) is a generative model that is fully tractable, but at the same time flexible enough to model complex distributions. It was introduced by \citet{Larochelle:2011:NADE}, where it was used for distribution estimation of binary data. 

NADE is a feedforward neural network with a special autoregressive structure. The basic version of NADE has a single hidden layer and models binary data, however there are extensions that can handle real-valued data \citep{Uria:2013:RNADE} and have multiple hidden layers \citep{Uria:2014:deep_NADE}. In this work we will make use of the binary, single hidden layer version, and from here on whenever we say ``NADE'' we shall refer to this.

NADE takes as input $I$ binary variables $\set{x_i}$ that can be either $0$ or $1$, and returns as output $I$ real-valued variables $\set{\hat{x}_i}$ that take on values between $0$ and $1$. The $\nth{i}$ output $\hat{x}_i$ is the conditional probability of the $\nth{i}$ variable being equal to $1$ given all variables thus far, that is
\begin{equation}
\hat{x}_i = \prob{x_i = 1\g \vect{x}_{<i},\bm{\theta}}.
\end{equation}
Hence, given an input vector $\vect{x}$ and using the chain rule, its probability according to NADE can be calculated as
\begin{equation}
\prob{\vect{x}\g \bm{\theta}} = \prod_i{\hat{x}_i^{x_i}\br{1-\hat{x}_i}^{1-x_i}}.
\end{equation}
The outputs of NADE are computed by forward-propagating $\vect{x}$ through the network, one variable at a time. The first output $\hat{x}_1$ is computed as follows
\begin{align}
\vect{a}_1 &= \vect{c} \\
\vect{h}_1 &= \sigm{\vect{a}_1} \\
\hat{x}_1 &= \sigm{\vect{u}_1^T\vect{h}_1 + b_1}.
\end{align}
where $\sigm{\cdot}$ is the logistic sigmoid function applied elementwise. 
From there on, each output $\hat{x}_i$ (for $i>1$) is computed in turn by forward-propagating the $\nth{\br{i-1}}$ input $x_{i-1}$ as follows
\begin{align}
\vect{a}_i &= \vect{a}_{i-1} + \vect{w}_{i-1}x_{i-1} \\
\vect{h}_i &= \sigm{\vect{a}_i} \\
\hat{x}_i &= \sigm{\vect{u}_i^T\vect{h}_i + b_i}.
\end{align}
In the above equations, $\set{\vect{u}_i}$, $\set{b_i}$, $\set{\vect{w}_i}$ and $\vect{c}$ are the adjustable parameters of NADE, which are collectively denoted as $\bm{\theta}$. Also, $\vect{a}_i$ and $\vect{h}_i$ are the hidden layer activations and hidden layer states in the $\nth{i}$ step.

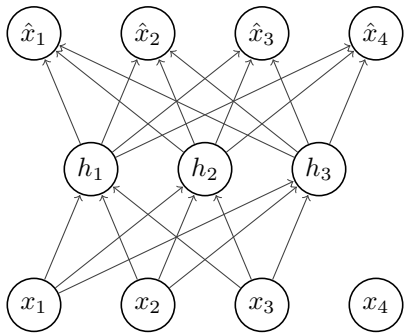
\begin{figure}[tbp]
\centering
\begin{tikzpicture}[shorten >=0.3pt, ->]

	\pgfmathsetmacro{\a}{-1.5}
	\pgfmathsetmacro{\layersep}{1.8}

    \tikzstyle{neuron}=[draw, semithick, circle, minimum size=20pt, inner sep=0pt]

    \foreach \y in {1,...,4}
        \node[neuron] (I-\y) at (-\a*\y,0) {$x_\y$};

    \foreach \y in {1,...,3}
        \path[xshift=-\a*0.5cm]
            node[neuron] (H-\y) at (-\a*\y,\layersep) {$h_\y$};

    \foreach \y in {1,...,4}
	    \node[neuron] (O-\y) at (-\a*\y,2*\layersep) {$\hat{x}_\y$};

    \foreach \source in {1,...,3}
        \foreach \dest in {1,...,3}
            \path[draw=black!75] (I-\source) edge (H-\dest);

    \foreach \source in {1,...,3}
        \foreach \dest in {1,...,4}
	        \path[draw=black!75] (H-\source) edge (O-\dest);

\end{tikzpicture}
\caption{Architecture of a NADE with $4$ input variables and $3$ hidden units. Its input variables are $x_1$, $x_2$, $x_3$ and $x_4$ (bottom layer), its hidden units are $h_1$, $h_2$ and $h_3$ (middle layer) and its outputs are $\hat{x}_1$, $\hat{x}_2$, $\hat{x}_3$ and $\hat{x}_4$ (top layer). Note that input $x_4$ is not actually forward-propagated, thus there are no connections from it to the hidden units.}
\label{fig:nade}
\end{figure}

Notice the special autoregressive structure of NADE; the activation $\vect{a}_i$ of the hidden units in the $\nth{i}$ step depends on the activation $\vect{a}_{i-1}$ in the step before. Also notice that the last input $x_I$ is never forward-propagated through NADE, but is only used in computing $\prob{\vect{x}\g \bm{\theta}}$. Thanks to this autoregressive structure, NADE can compute $\prob{\vect{x}\g \bm{\theta}}$ in time $\bigo{IJ}$; without it, this computation would require time $\bigo{IJ^2}$. The downside of it is that NADE is sensitive to the ordering of the input variables; that is, NADEs that are trained on the same dataset with different orderings will not necessarily be consistent with each other.

As can be seen from the above description, NADE is fully tractable. That is, $\prob{\vect{x}\g \bm{\theta}}$ can be calculated exactly, in time linear in the size of the input (compare this to the exponential time needed by an RBM). Also, conditioning in NADE becomes trivial, since it directly outputs conditionals.\footnote{Note however that if a different conditional from those output by NADE is needed, then the only obvious way to obtain it would be to retrain NADE from scratch with a convenient ordering.} Generating exact samples from NADE is also straightforward. This can be done using ancestral sampling\index{Sampling!Ancestral sampling}; generate $x_1$ from $\hat{x}_1$, then forward-propagate it to get $\hat{x}_2$, generate $x_2$ from $\hat{x}_2$, then forward-propagate it to get $\hat{x}_3$, and so on until the whole vector is generated. Last but not least, due to its neural network structure and the universal approximation theorem for neural networks, NADE can approximate with arbitrary accuracy any desired distribution, as long as it is given enough capacity.\footnote{Note however that this might mean that the number of hidden units has to become prohibitively large.}

Finally, as discussed in section~\ref{sec:generative_models:stochastic_gradient_training}, in order to train NADE to mimic other generative models, we need to be in a position to calculate
\begin{equation}
\pderivnull{\bm{\theta}}\log{\prob{\vect{x}\g \bm{\theta}}}
\quad\quad\text{and}\quad\quad
\Rop{\pderivnull{\bm{\theta}}\log{\prob{\vect{x}\g \bm{\theta}}}}.
\end{equation}
This can be done efficiently (in time linear in the size of the input) by backward propagation\index{Backward propagation} and R\{backprop\}\index{R\{backprop\}} respectively. More details on how to do this are given in appendix~\ref{chapter:derivatives_in_nade}.

\subsection{Distilling an RBM into NADE}
\label{sec:rbm_vs_nade:distillation}

In this section, we present the experimental results of distilling an intractable RBM into a tractable NADE\@. Our goal is to assess how successful the distillation is, and to what extent NADE manages to capture the knowledge of the RBM\@.

\subsubsection{The RBM}

\begin{figure}[tbp]
\def\imwidth{0.48\textwidth}
\centering
\subfloat[Samples from binarized MNIST]{
\includegraphics[width=\imwidth]{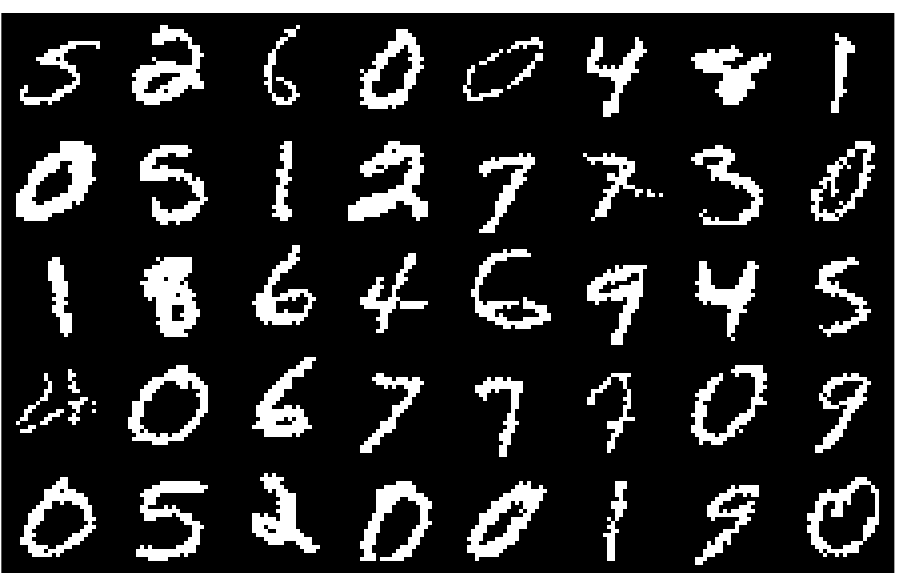}}
\hfill
\subfloat[Samples from RBM]{
\includegraphics[width=\imwidth]{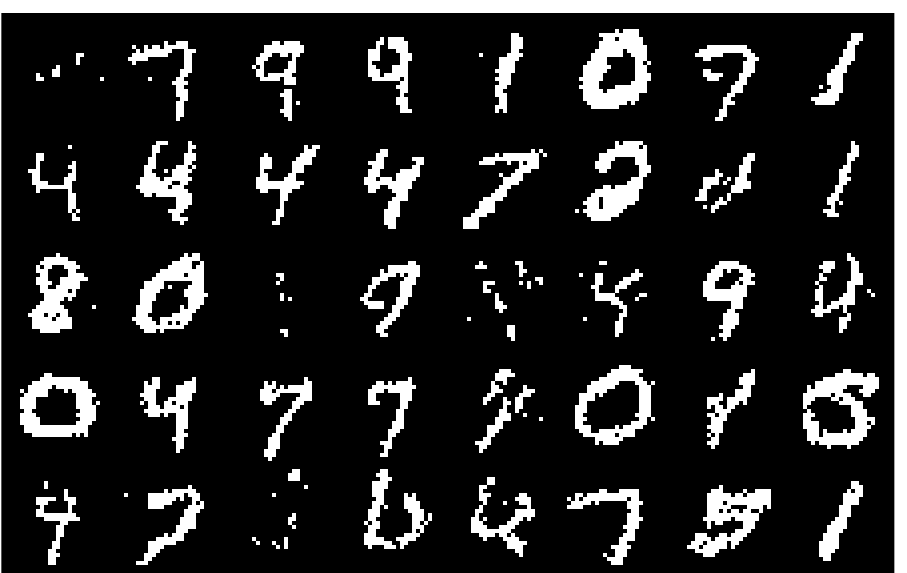}}
\caption{Random samples from (a) the binarized version of MNIST and (b) the RBM\@. The RBM samples were generated by Gibbs sampling; each sample was generated by a separate Markov chain, randomly initialized and burned in until convergence.}
\label{fig:generative_models:samples_from_mnist_and_rbm}
\end{figure}

In all our experiments, we used an RBM that was trained on the MNIST dataset of handwritten digits \citep{LeCun:MNIST_web}. The RBM was provided by \citet{Salakhutdinov:2008:quantitative_analysis_of_DBNs}, who referred to it as CD25(500).\footnote{It was downloaded as is from \url{http://www.utstat.toronto.edu/~rsalakhu/rbm_ais.html} (accessed on 20 July 2015).} This particular RBM has $500$ hidden variables and it was trained using Contrastive Divergence \citep{Hinton:2002:TPE} with $25$ Gibbs steps for each gradient evaluation.

The RBM was trained to be a generative model of MNIST digits. The MNIST dataset consists of greyscale images of handwritten digits ($0$ to $9$), of size $28\times 28$. The dataset is partitioned into a train set of $60{,}000$ images and a test set of $10{,}000$ images. The images were binarized before training; each pixel was independently set to $1$ with probability equal to the original pixel intensity (which varies from $0$ to $1$). For consistency, in all our experiments below, we use the same binarized version of MNIST\@.

In their paper, \citet{Salakhutdinov:2008:quantitative_analysis_of_DBNs} argue that this particular RBM is indeed a good distribution estimator for MNIST\@. They also use annealed importance sampling\index{Sampling!Annealed importance sampling} \citep{Neal:2001:AIS} to estimate the partition function of the RBM\@. Figure~\ref{fig:generative_models:samples_from_mnist_and_rbm} shows random samples from both the binarized MNIST dataset and the RBM; their similarity indicates that the RBM is indeed a good distribution estimator of MNIST\@. Figure~\ref{fig:generative_models:features_from_rbm} shows some of the features the RBM has learnt.

\begin{figure}[tbp]
\def\imwidth{0.48\textwidth}
\centering
\includegraphics[width=\imwidth]{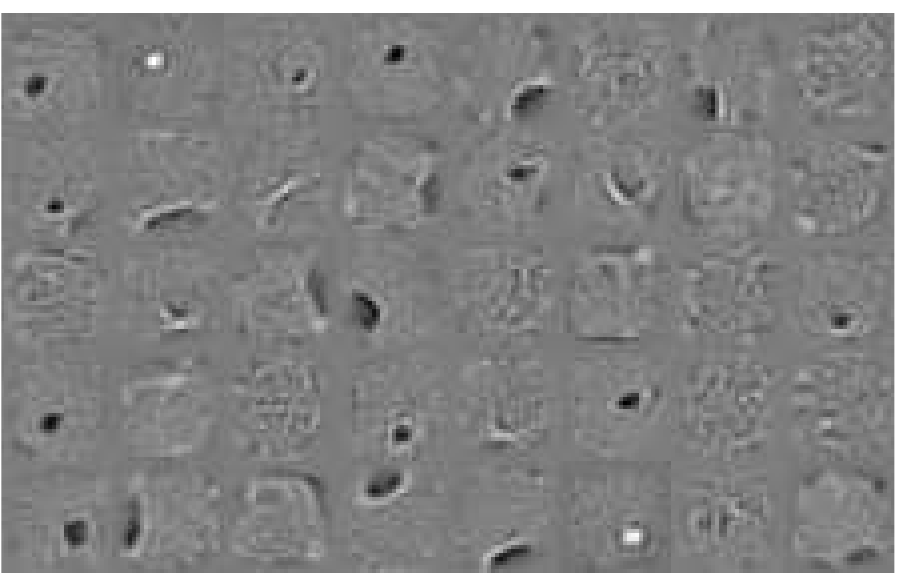}
\caption{Features learnt by the RBM trained on MNIST\@. The features are the columns of the $\mat{W}$ matrix of the RBM, reshaped as $28\times 28$ images.}
\label{fig:generative_models:features_from_rbm}
\end{figure}

\subsubsection{Distillation details}

In order to assess how much capacity NADE should have, we used NADEs with $250$, $500$, $750$ and $1000$ hidden units. All NADEs used a raster ordering of variables; that is, the visible variables (image pixels) were ordered columnwise. Empirically, we found that using a raster ordering as opposed to a random ordering is more suitable for MNIST\@.

The loss functions we used are \textit{KL divergence} and \textit{square error} (see section~\ref{sec:generative_models:loss_functions}). We did not use score matching, since---as we argued in the relevant section---score matching is only suitable for continuous distributions. The square error loss function has a free parameter $c$ that needs to be set. In our relevant analysis, we argued that any value of $c$ that is less or equal to $\max_{\vect{x}}{\log{\uprob{\vect{x}}}}$ should do. In the experiments below, we found a large value of $\log{\uprob{\vect{x}}}$ and set $c$ to it. We used random sampling followed by hill climbing to determine a large value for $\log{\uprob{\vect{x}}}$. The largest value we found was $436.49$ which, incidentally, corresponds to the all-black image.\footnote{Note that this is not guaranteed to be the maximum of $\log{\uprob{\vect{x}}}$, as the maximization method we used is only approximate. For a general RBM, finding the exact maximum is intractable. However, any value less than the maximum will do for our purposes.} In all our experiments, we used this value for $c$.

During stochastic gradient training, we used minibatches of $20$ samples from the RBM\@. The samples were generated using block Gibbs sampling with parallel chains. In particular, we maintained $2000$ parallel Markov chains and, in each iteration, we simulated each chain once and selected $20$ of them (in sequence) to form the minibatch. This way, all samples within a minibatch are independent (since they come from independently simulated chains) and each chain is thinned $2000/20 = 100$ times before contributing a sample. Empirically, we found that both having independent samples within the minibatch and thinning across minibatches improved learning. Note that this sampling scheme is also efficient to run, since all chains can be updated in parallel. Also, all chains were randomly initialized and burned in for $2000$ iterations before they actually started contributing samples.

In order to determine the learning rate in each stochastic update, we used ADADELTA\index{ADADELTA} \citep{Zeiler:2012:adadelta}, which is an adaptive learning rate method that uses a different learning rate for each parameter. ADADELTA is controlled by two hyperparameters, for which we used the default values suggested in the original paper. It comes with no convergence guarantees, but we have found that it works well in practice, while it requires minimal tuning. Finally, we chose to run stochastic gradient training for $30{,}000$ iterations, producing in total $30{,}000\times 20 = 600{,}000$ RBM samples.

\subsubsection{Results and discussion}

We trained $4$ NADEs with $250$, $500$, $750$ and $1000$ hidden units, using KL divergence and square error; this resulted in $8$ models in total. Figure~\ref{fig:generative_models:training_progress} shows the progress of each NADE during training. The plots track the average log probability each NADE assigns to the first $500$ MNIST test images, as this evolves during training. Table~\ref{table:generative_models:mnist_log_prob} shows the average log probability each NADE assigns to the full MNIST test set at the end of training. We can see that, in general, more hidden units lead to a higher log probability, with the difference being significant between the NADEs with $250$ and $500$ hidden units. This suggests that at least $500$ hidden units are needed for NADE to have enough flexibility to capture the knowledge in the RBM (it is no coincidence that the RBM has $500$ hidden variables as well).

\begin{figure}[tbp]
\def\imwidth{0.48\textwidth}
\centering
\subfloat[Trained with KL divergence]{
\includegraphics[width=\imwidth]{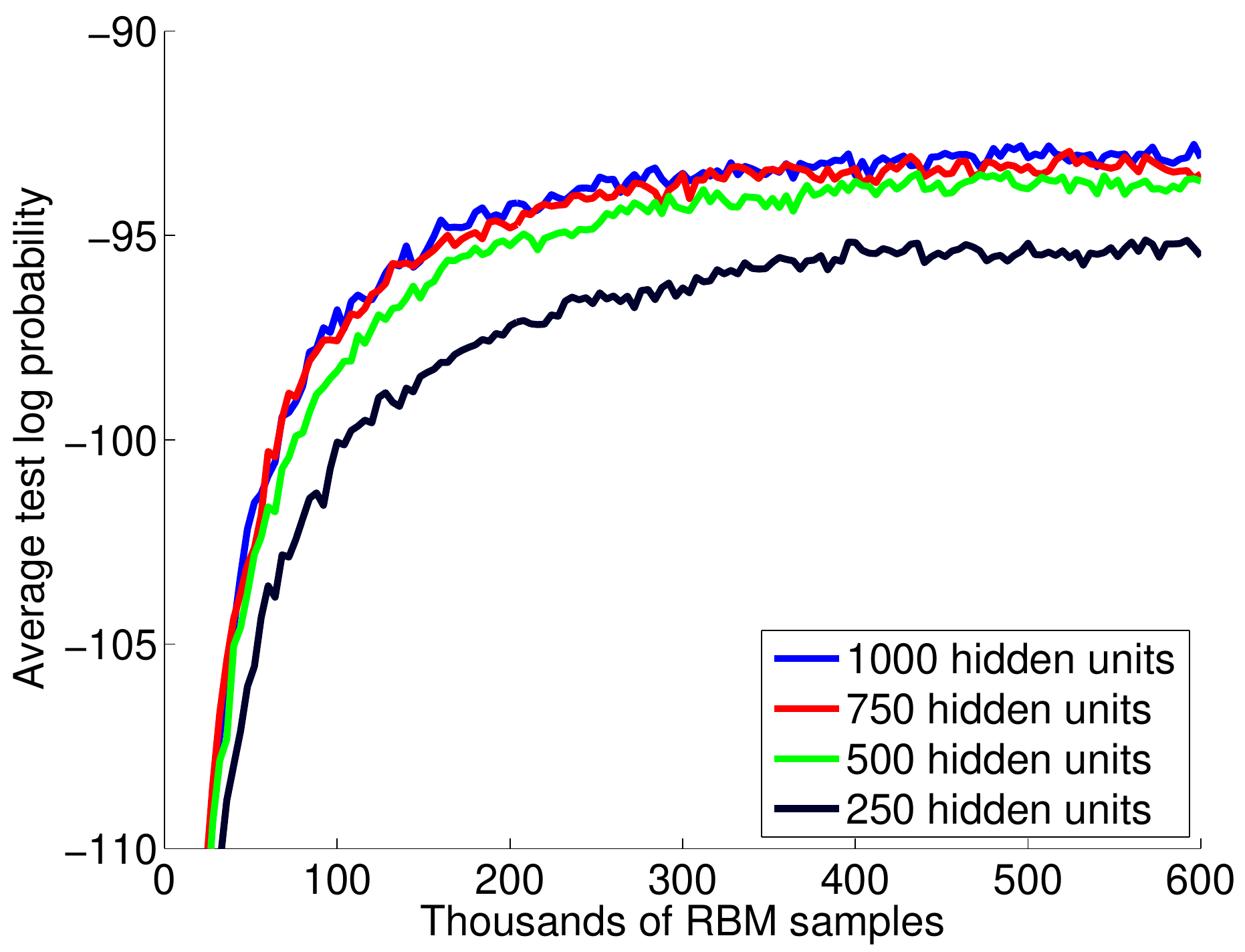}}
\hfill
\subfloat[Trained with square error]{
\includegraphics[width=\imwidth]{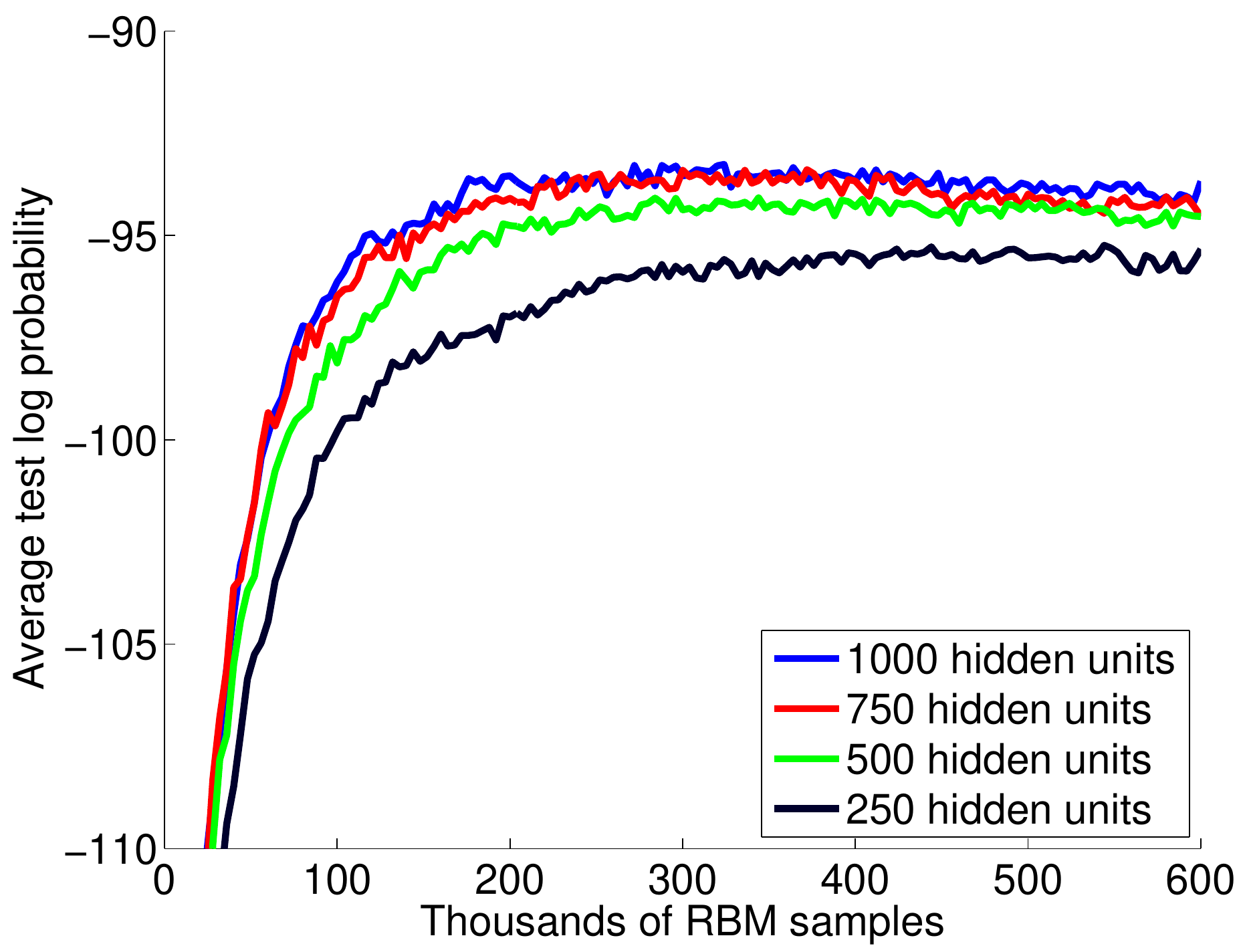}}
\caption{Training progress for each NADE, as measured by average log probability of the first $500$ images of the MNIST test set. The training progress was measured every $200$ iterations, that is, every $200\times 20 = 4000$ RBM samples.}
\label{fig:generative_models:training_progress}
\end{figure}

\begin{table}[tbp]
\renewcommand{\tabcolsep}{0.6cm}
\renewcommand{\arraystretch}{\arrstretchvalue}
\centering
\begin{tabular}{ccc}
\toprule
\textbf{Hidden units} & \textbf{KL divergence} & \textbf{Square error} \\
\midrule
$\bm{250}$  & $-94.75 \pm 0.78$ & $-94.62 \pm 0.78$ \\
$\bm{500}$  & $-93.19 \pm 0.76$ & $-93.83 \pm 0.77$ \\
$\bm{750}$  & $-92.93 \pm 0.76$ & $-93.61 \pm 0.77$ \\ 
$\bm{1000}$ & $-92.49 \pm 0.76$ & $-93.01 \pm 0.76$ \\
\bottomrule
\end{tabular}
\caption{Average log probability each NADE assigns to the full MNIST test set, at the end of training. Error bars correspond to $3$ standard deviations.}
\label{table:generative_models:mnist_log_prob}
\end{table}

\begin{figure}[p]
\def\imwidth{0.395\textwidth}
\centering
\subfloat[KL divergence, $250$ hidden units]{
\includegraphics[width=\imwidth]{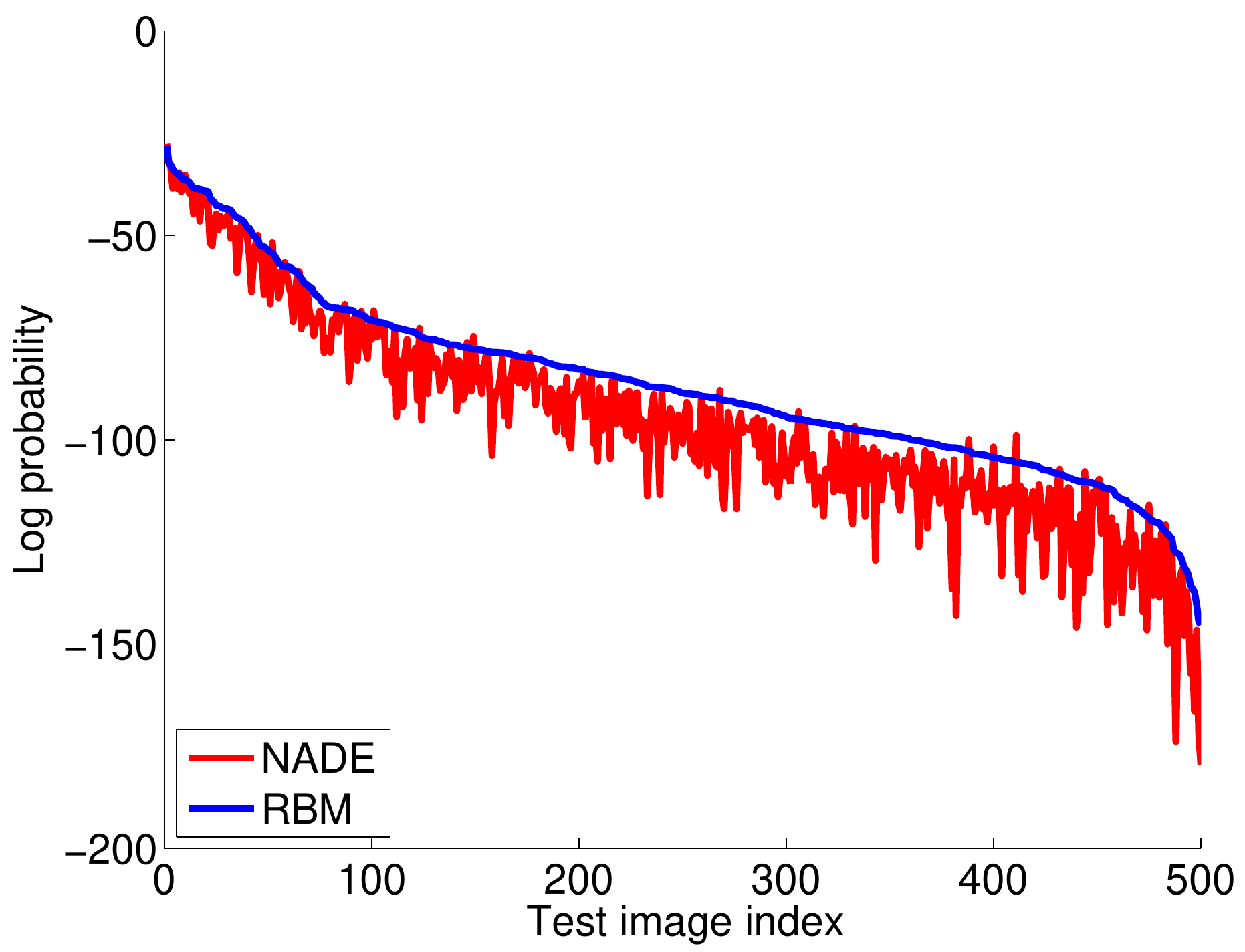}}
\hspace{1.5cm}
\subfloat[Square error, $250$ hidden units]{
\includegraphics[width=\imwidth]{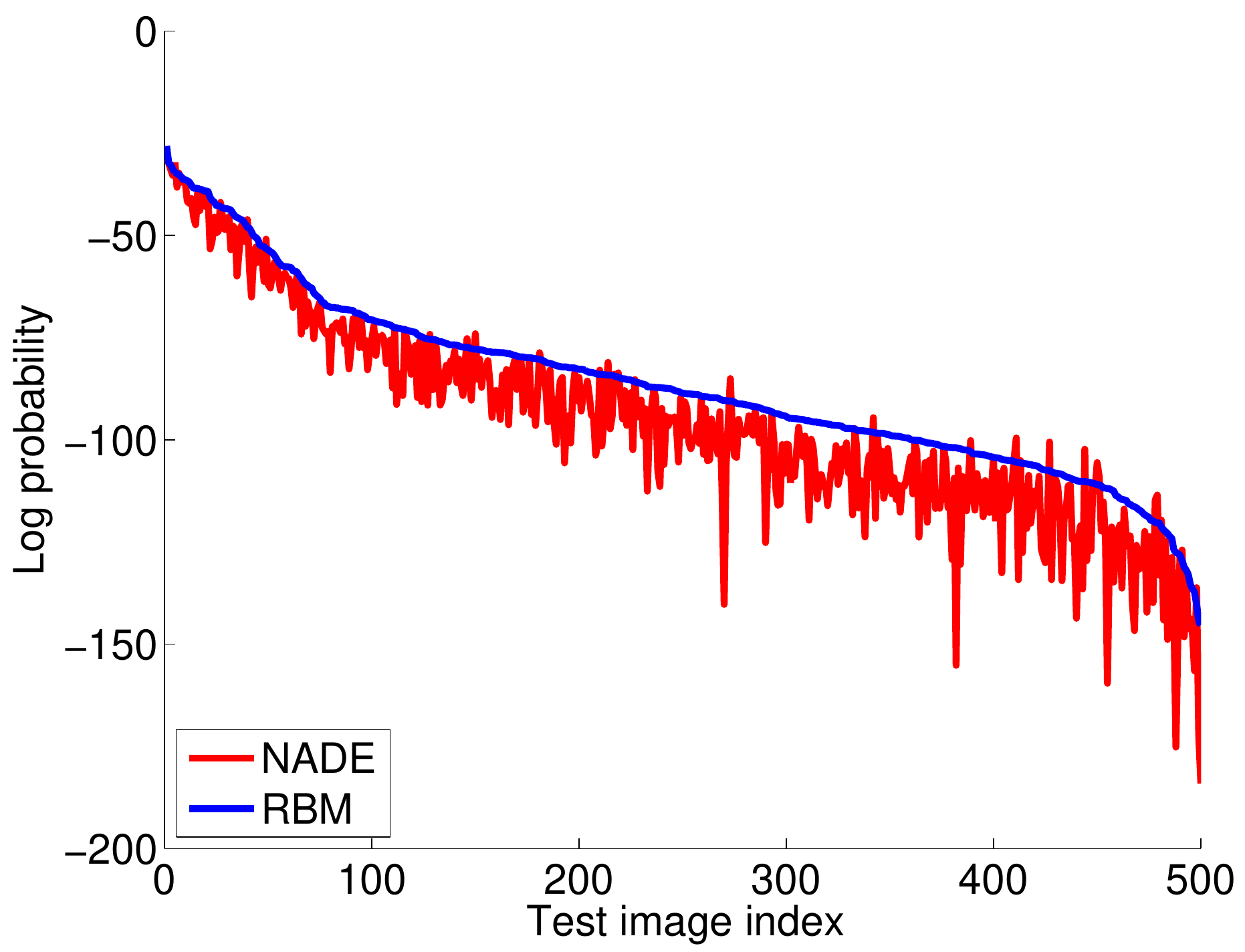}}\\
\subfloat[KL divergence, $500$ hidden units]{
\includegraphics[width=\imwidth]{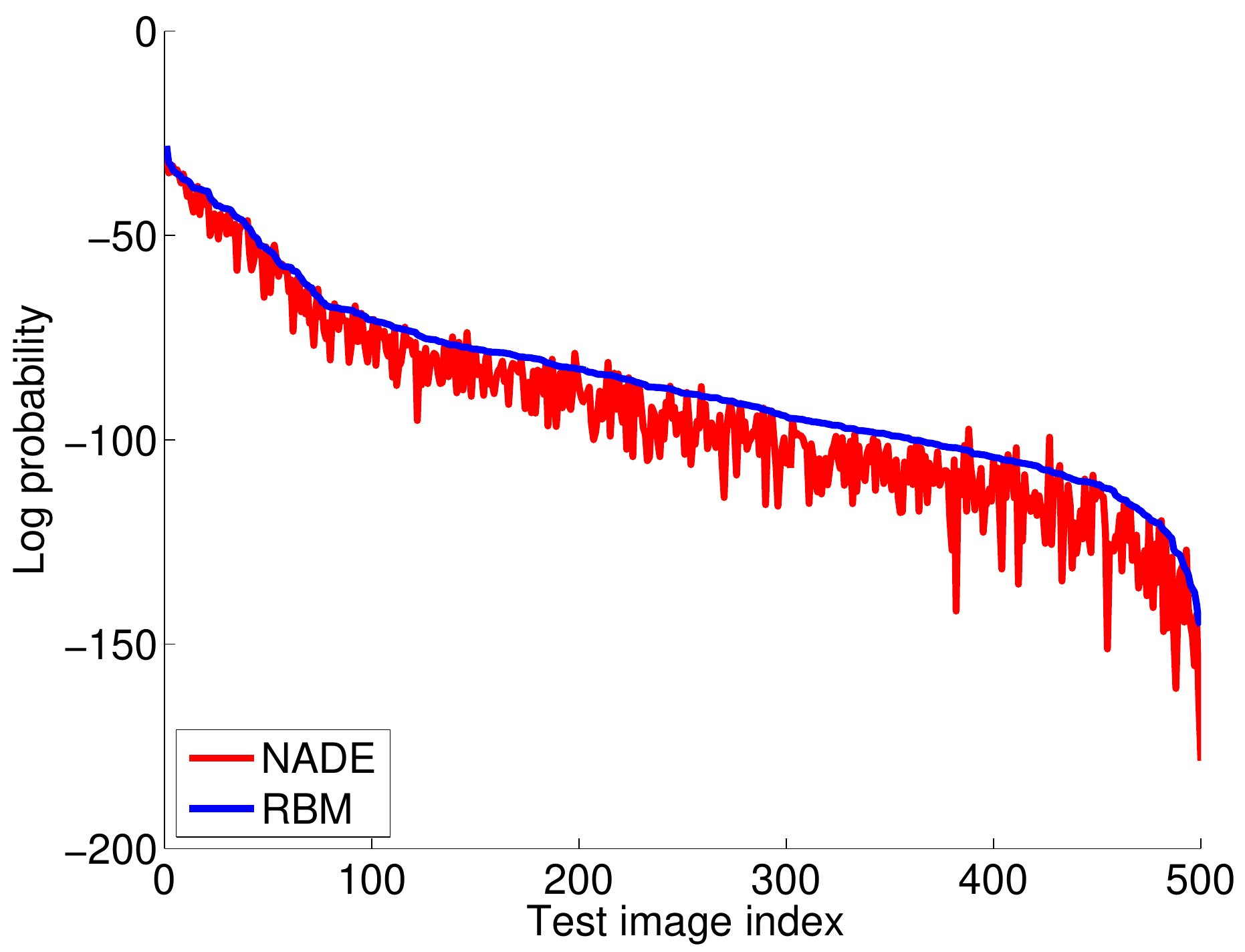}}
\hspace{1.5cm}
\subfloat[Square error, $500$ hidden units]{
\includegraphics[width=\imwidth]{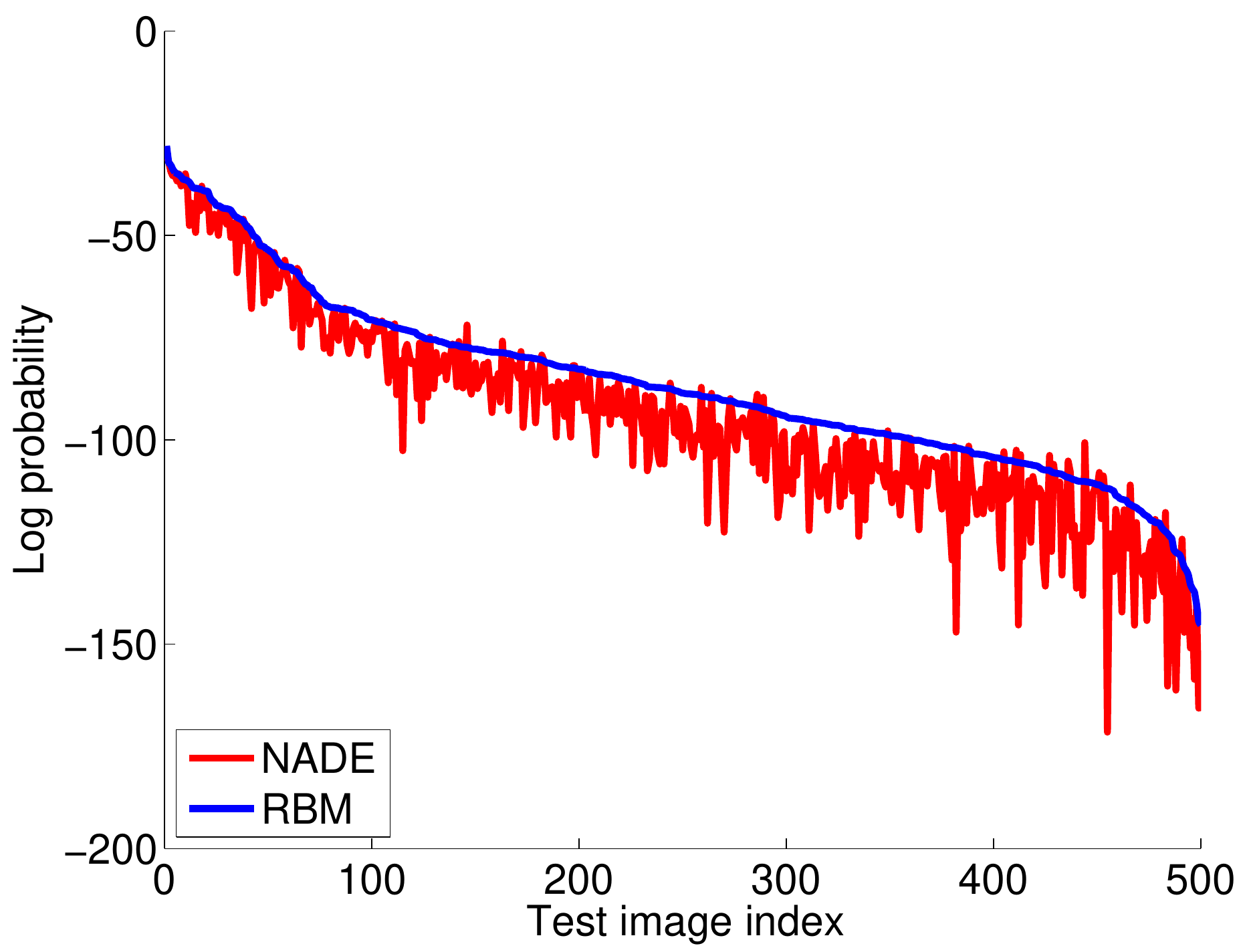}}\\
\subfloat[KL divergence, $750$ hidden units]{
\includegraphics[width=\imwidth]{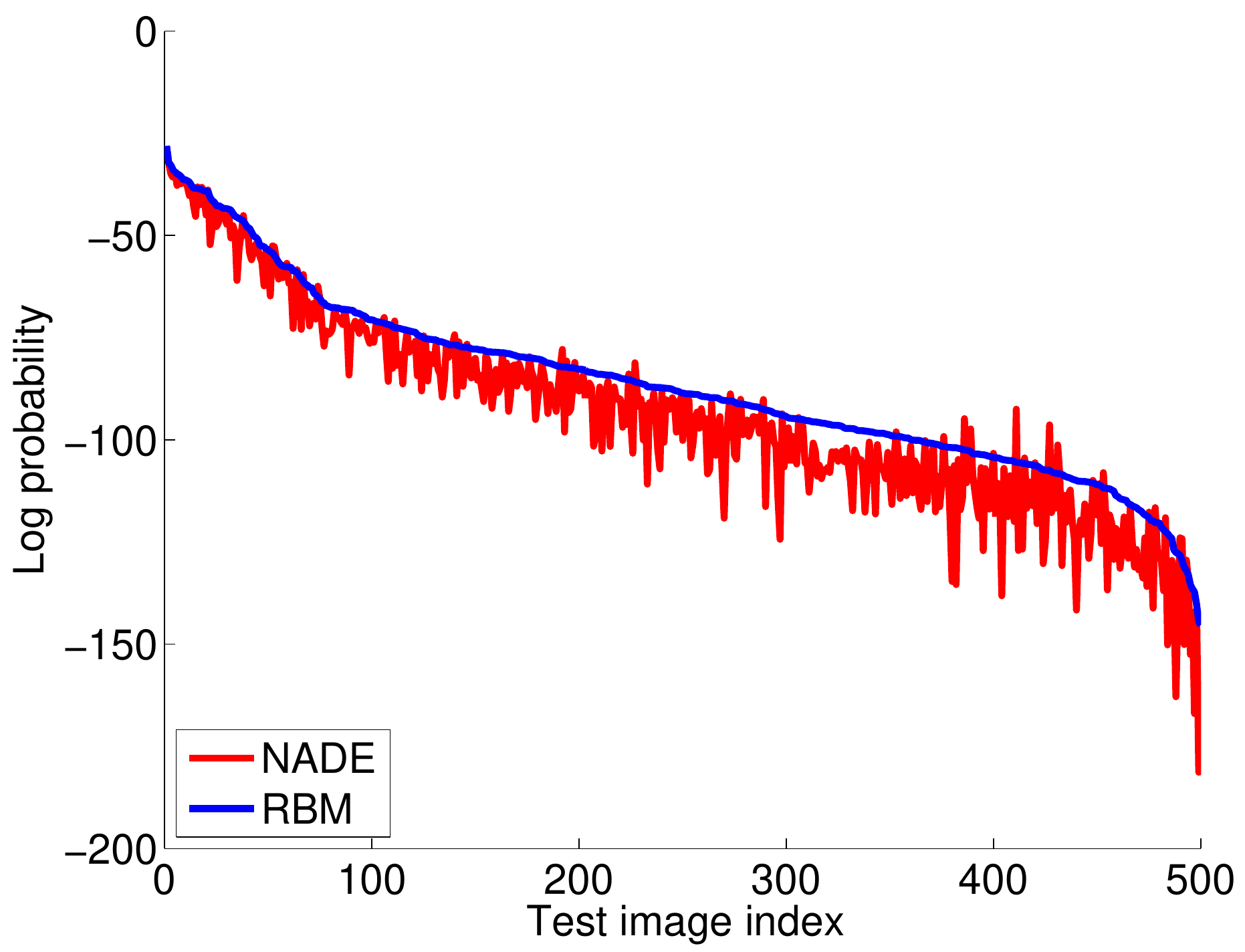}}
\hspace{1.5cm}
\subfloat[Square error, $750$ hidden units]{
\includegraphics[width=\imwidth]{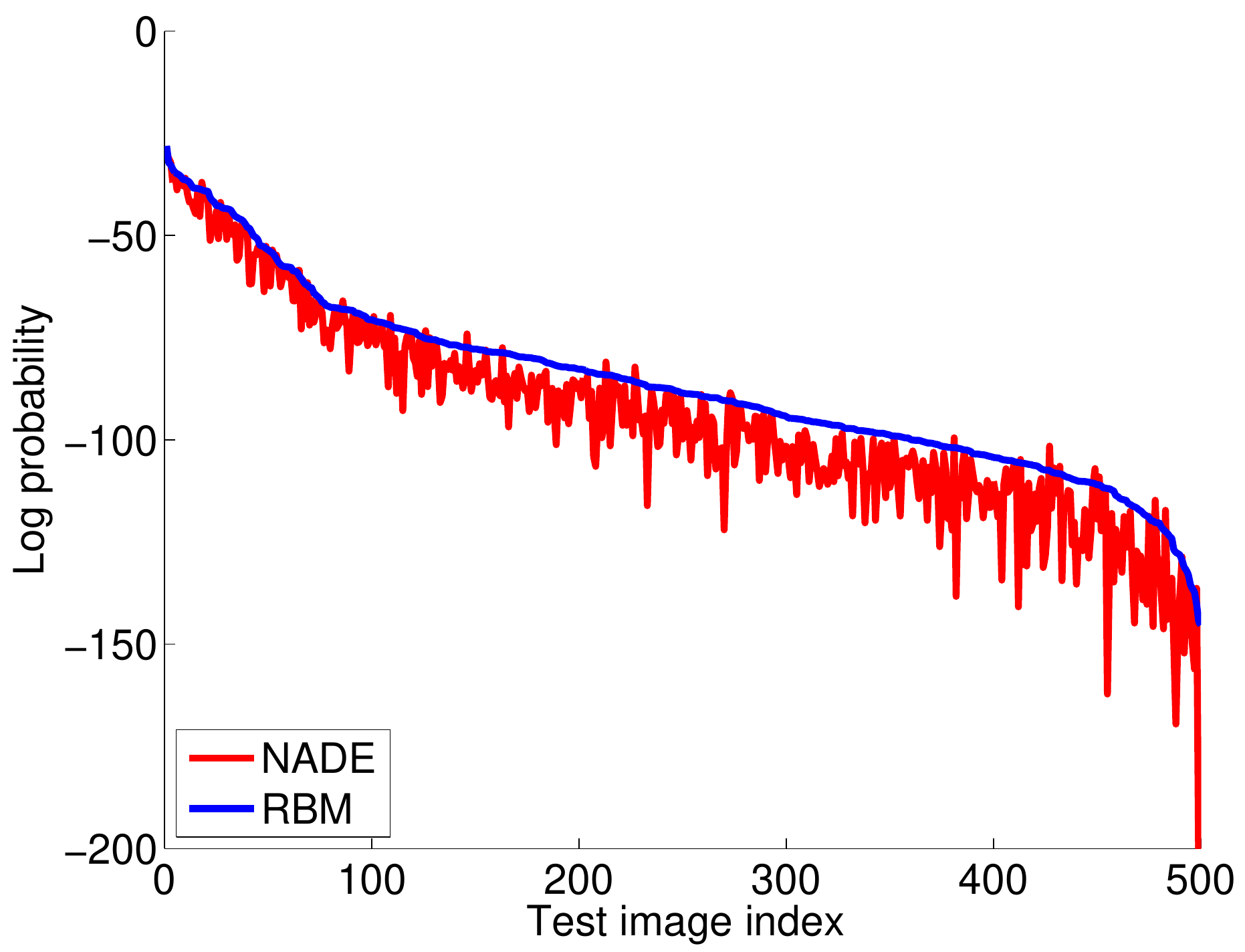}}\\
\subfloat[KL divergence, $1000$ hidden units]{
\includegraphics[width=\imwidth]{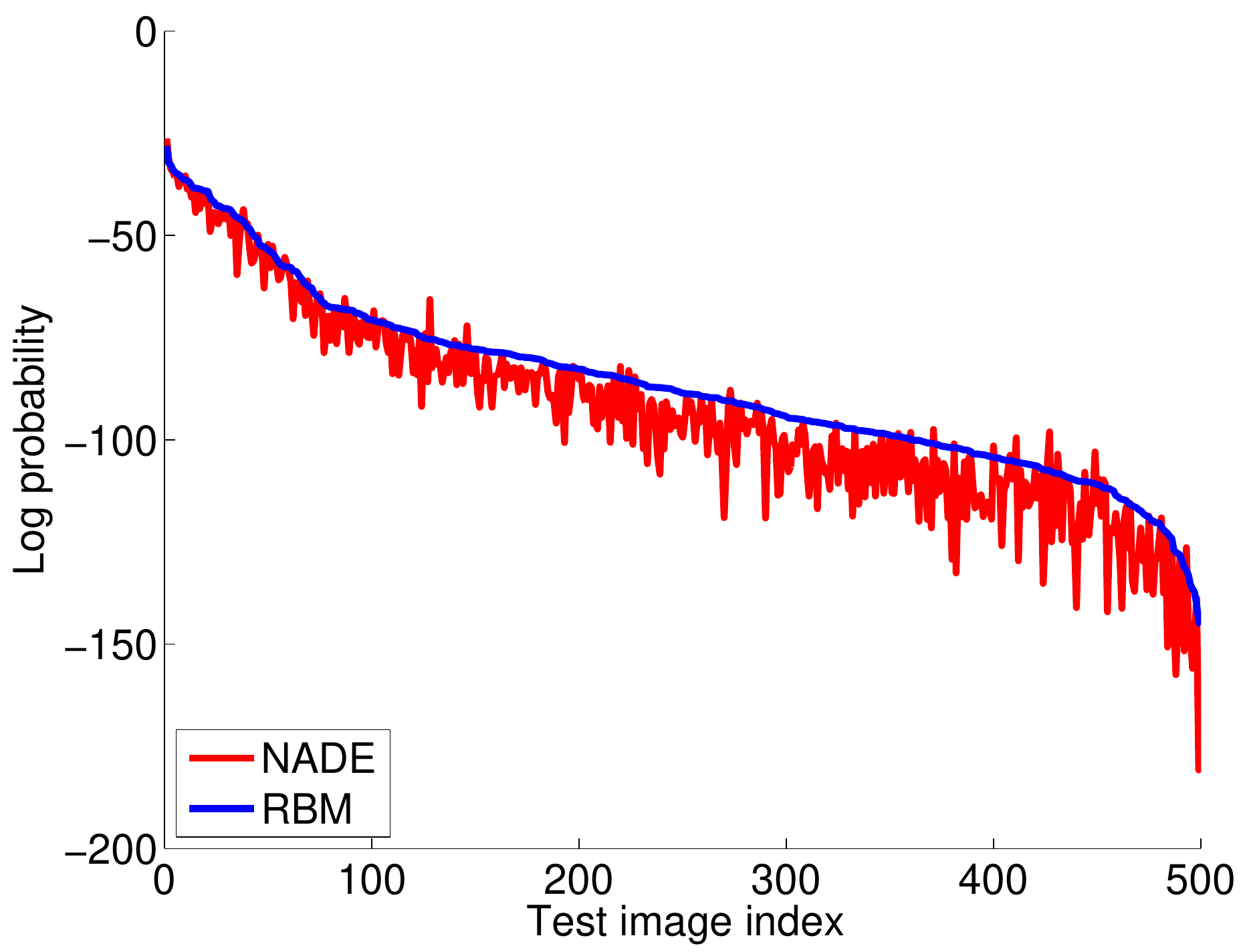}}
\hspace{1.5cm}
\subfloat[Square error, $1000$ hidden units]{
\includegraphics[width=\imwidth]{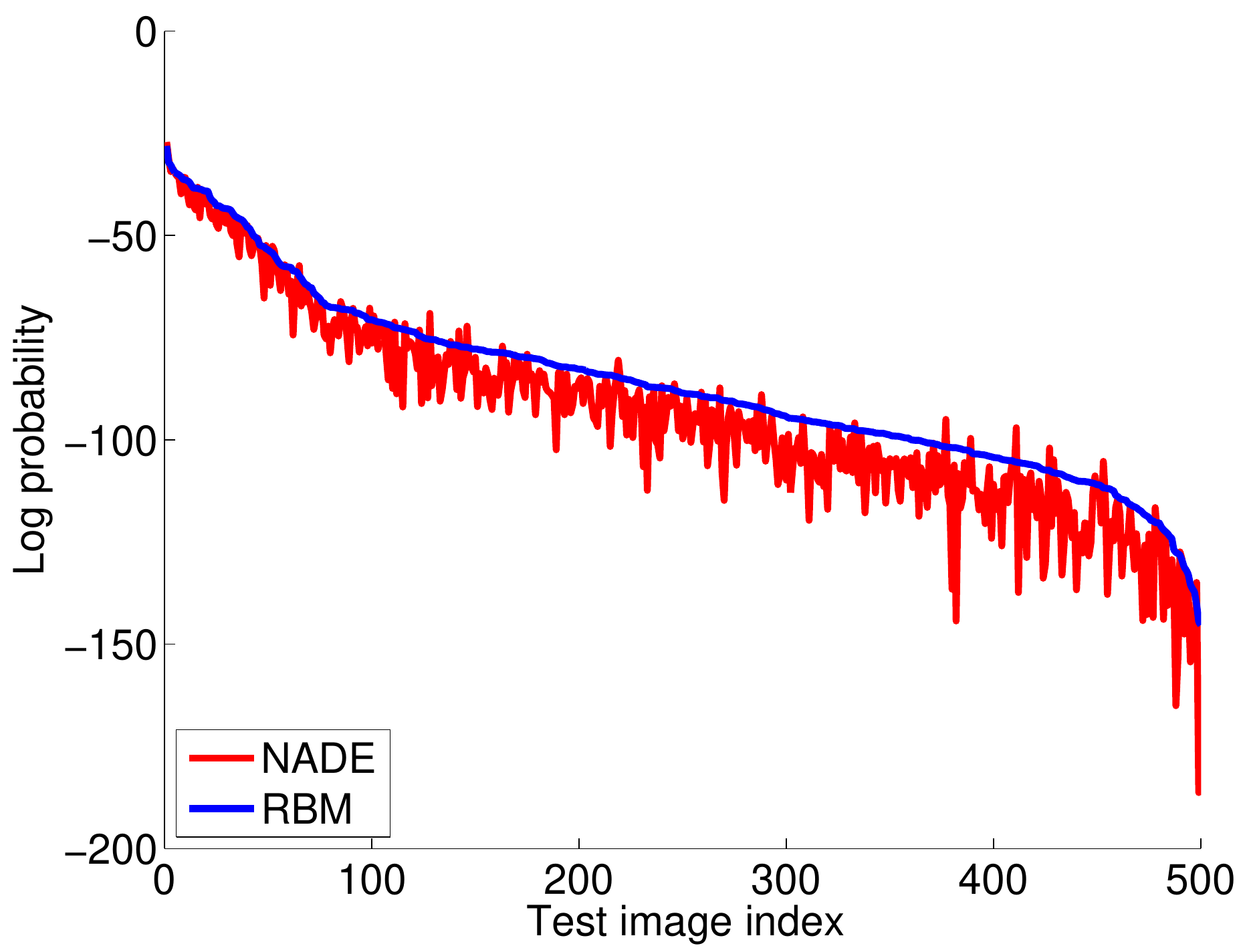}}
\caption{Log probability each trained NADE assigns to the first $500$ images of the MNIST test set, compared to the RBM\@. Note that the log probability of the RBM was calculated by setting $\log{Z} = 451.28$, as estimated by \citet{Salakhutdinov:2008:quantitative_analysis_of_DBNs}.}
\label{fig:generative_models:image_logprob}
\end{figure}

\begin{figure}[p]
\def\imwidth{0.48\textwidth}
\centering
\subfloat[KL divergence, $250$ hidden units]{
\includegraphics[width=\imwidth]{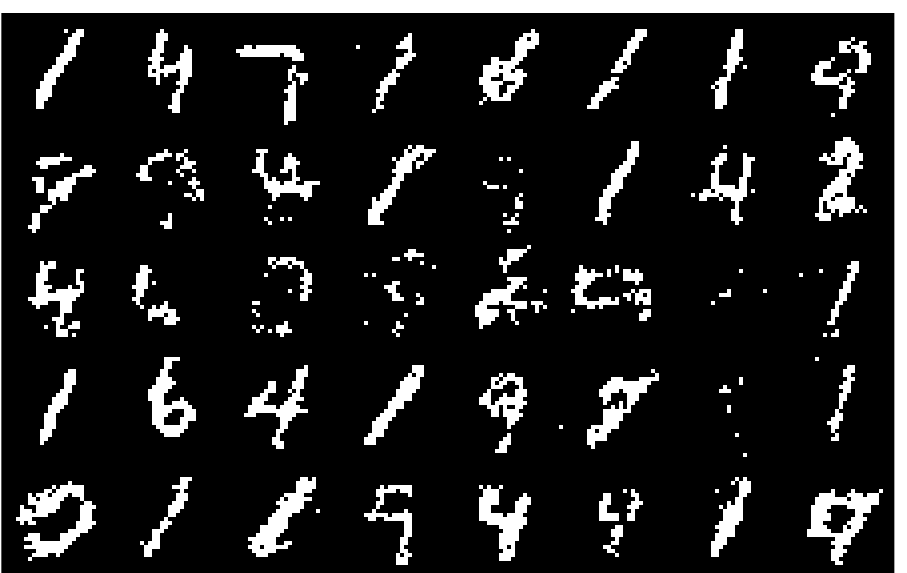}}
\hfill
\subfloat[Square error, $250$ hidden units]{
\includegraphics[width=\imwidth]{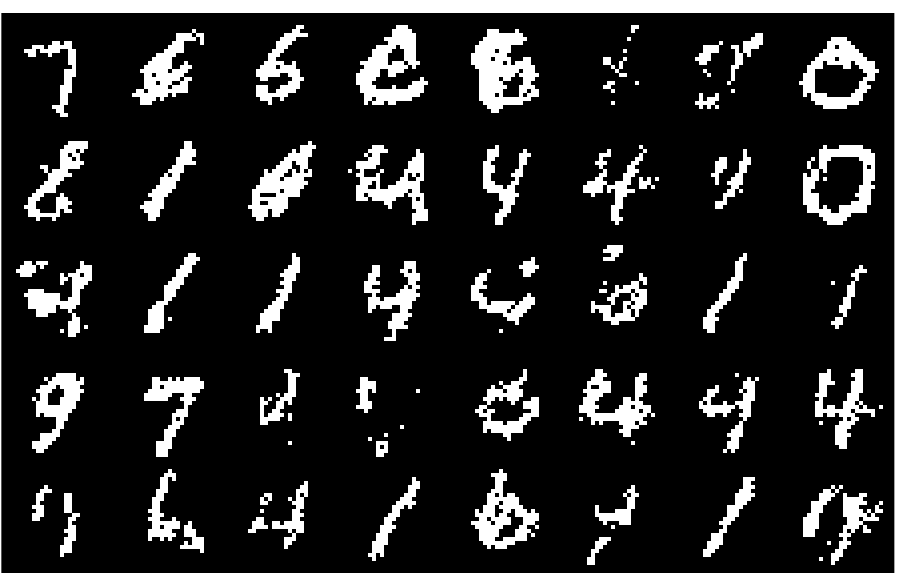}}\\
\subfloat[KL divergence, $500$ hidden units]{
\includegraphics[width=\imwidth]{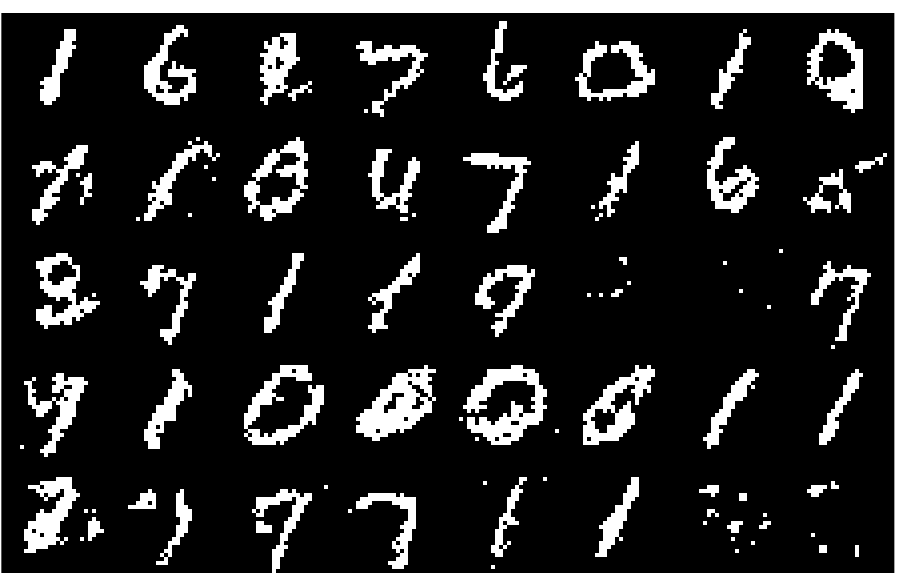}}
\hfill
\subfloat[Square error, $500$ hidden units]{
\includegraphics[width=\imwidth]{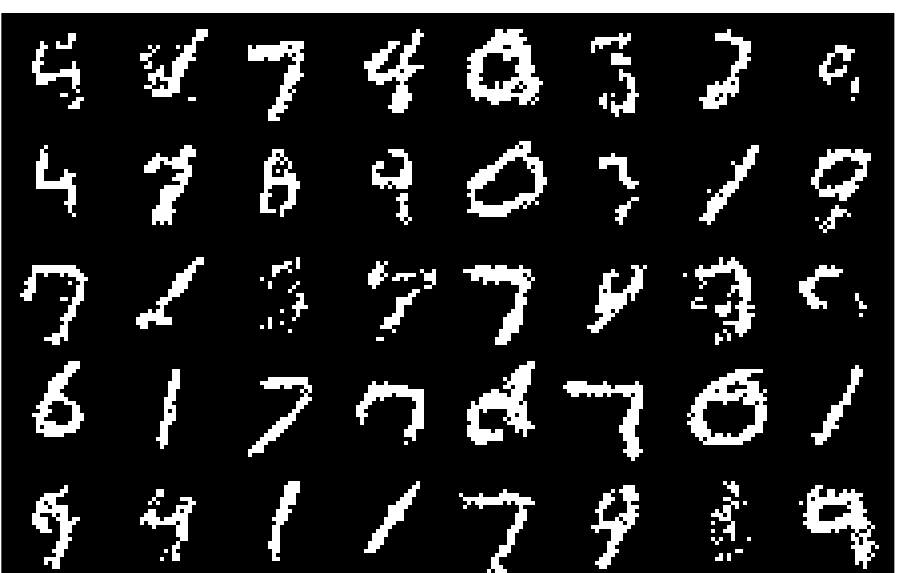}}\\
\subfloat[KL divergence, $750$ hidden units]{
\includegraphics[width=\imwidth]{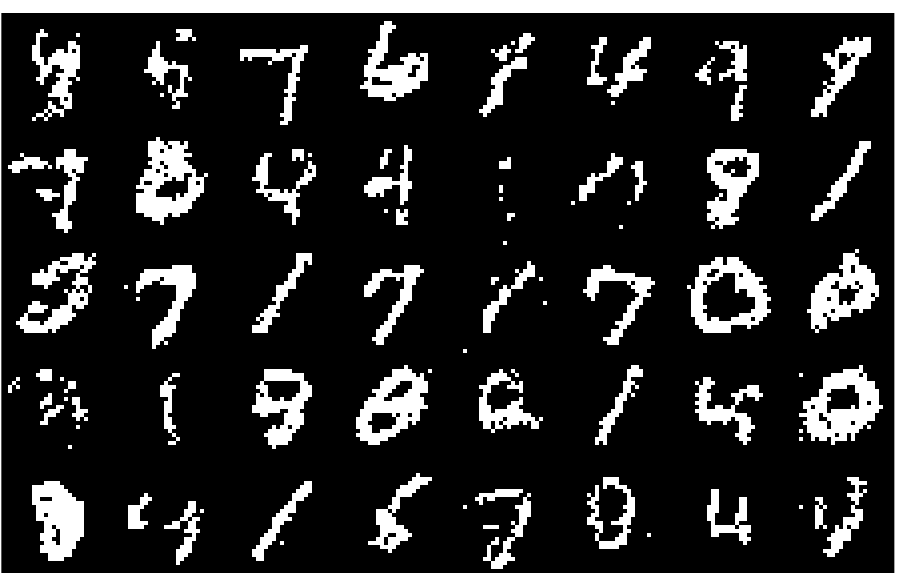}}
\hfill
\subfloat[Square error, $750$ hidden units]{
\includegraphics[width=\imwidth]{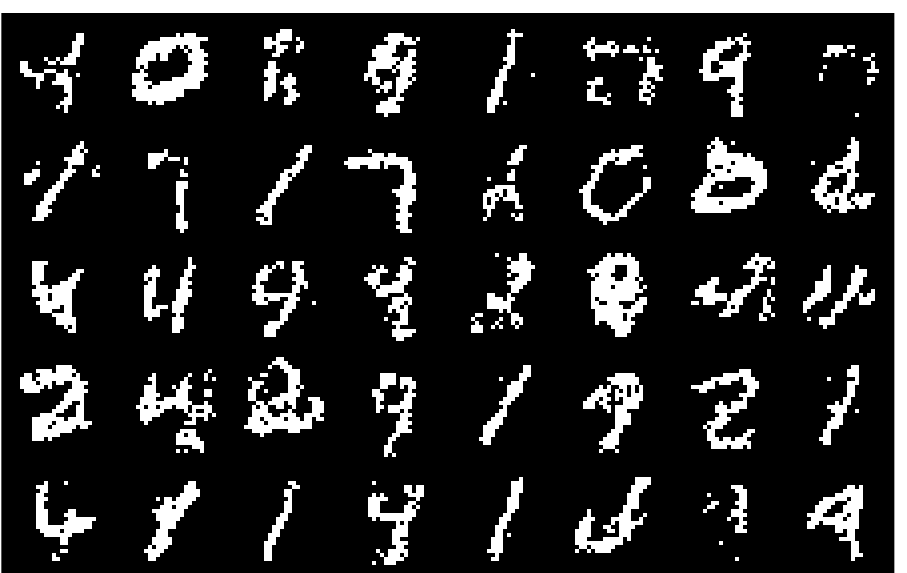}}\\
\subfloat[KL divergence, $1000$ hidden units]{
\includegraphics[width=\imwidth]{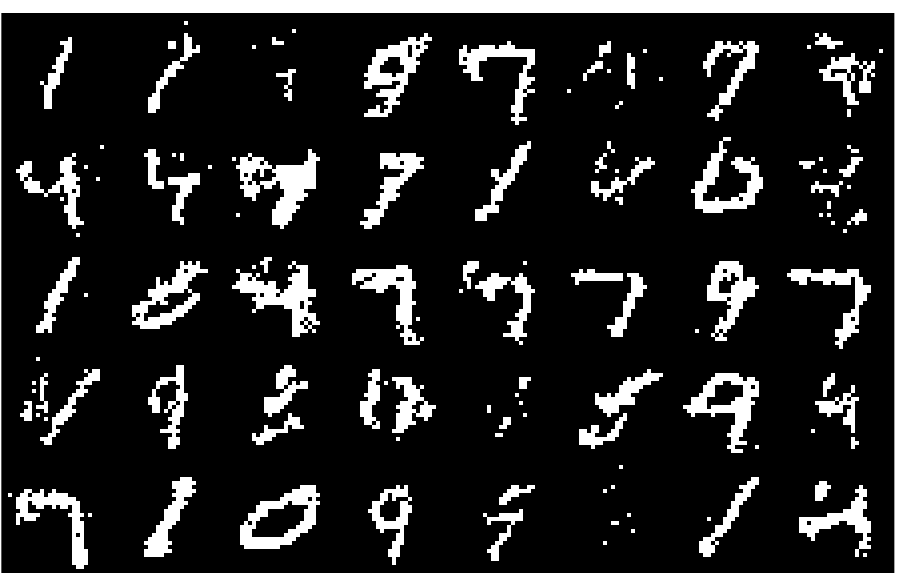}}
\hfill
\subfloat[Square error, $1000$ hidden units]{
\includegraphics[width=\imwidth]{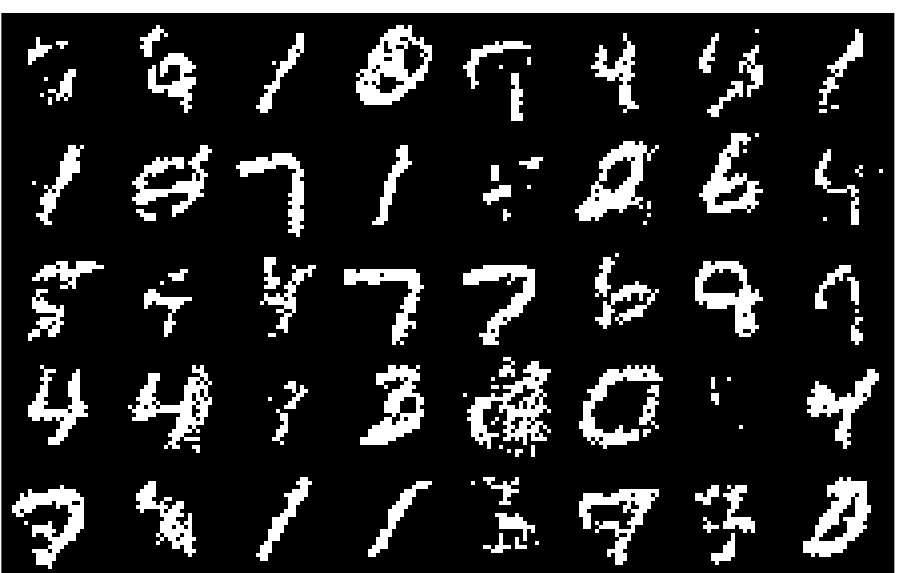}}
\caption{Random samples from each trained NADE\@. All samples are exact and were generated using ancestral sampling.}
\label{fig:generative_models:nade_samples}
\end{figure}

\begin{figure}[p]
\def\imwidth{0.48\textwidth}
\centering
\subfloat[KL divergence, $250$ hidden units]{
\includegraphics[width=\imwidth]{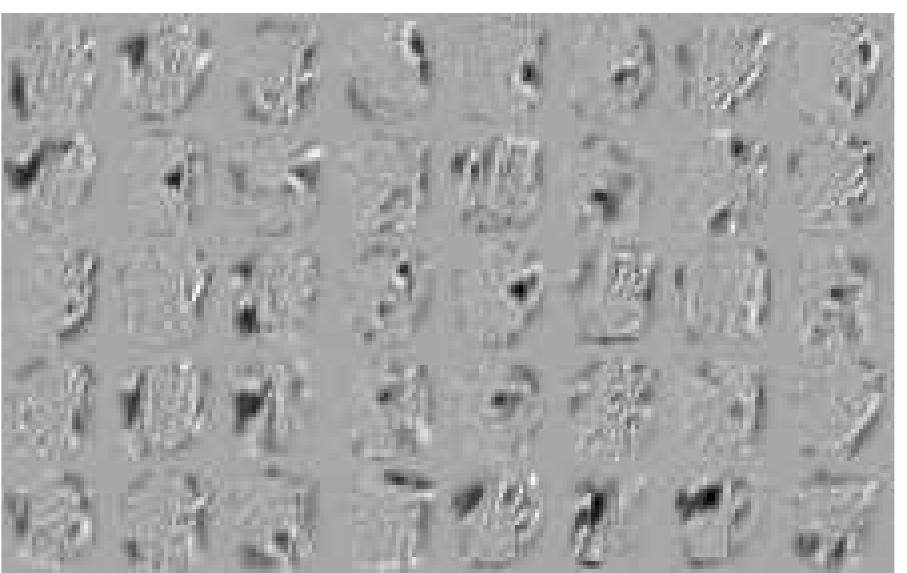}}
\hfill
\subfloat[Square error, $250$ hidden units]{
\includegraphics[width=\imwidth]{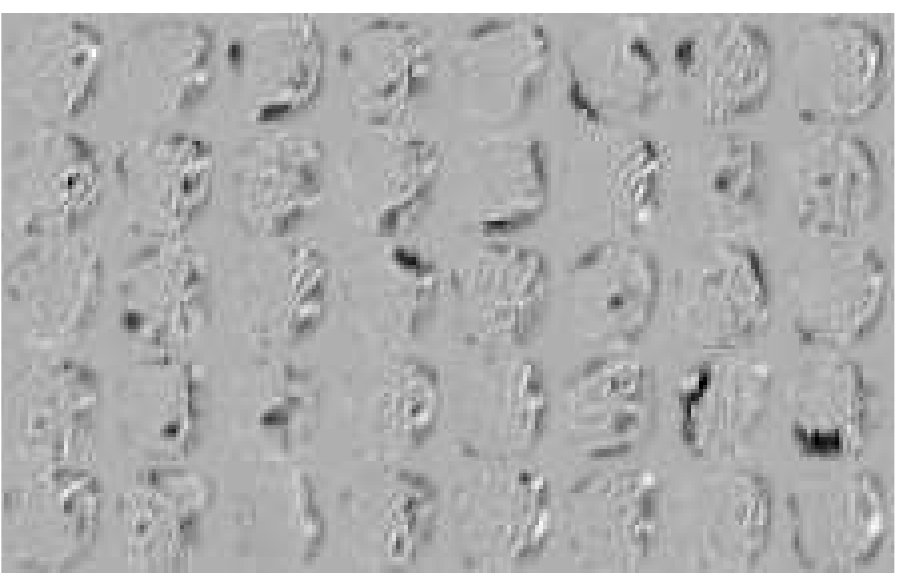}}\\
\subfloat[KL divergence, $500$ hidden units]{
\includegraphics[width=\imwidth]{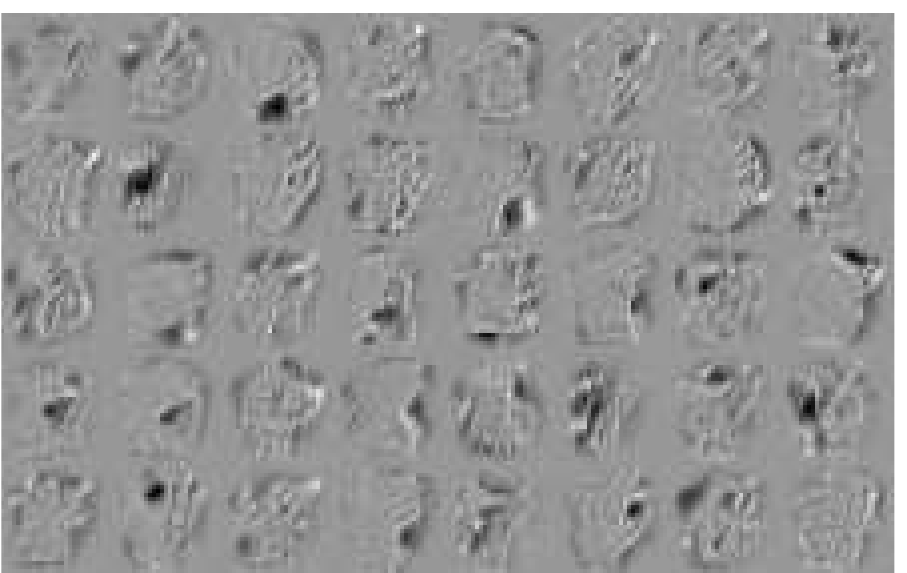}}
\hfill
\subfloat[Square error, $500$ hidden units]{
\includegraphics[width=\imwidth]{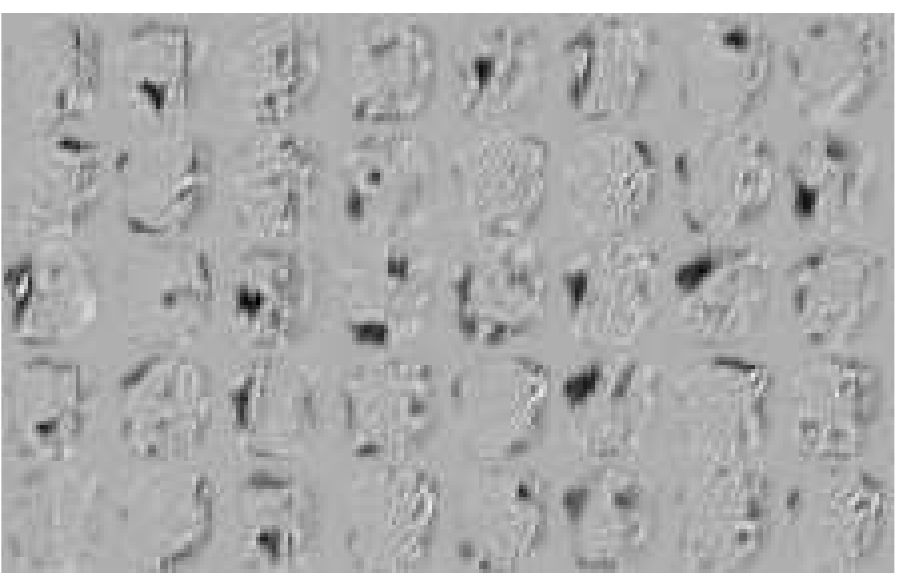}}\\
\subfloat[KL divergence, $750$ hidden units]{
\includegraphics[width=\imwidth]{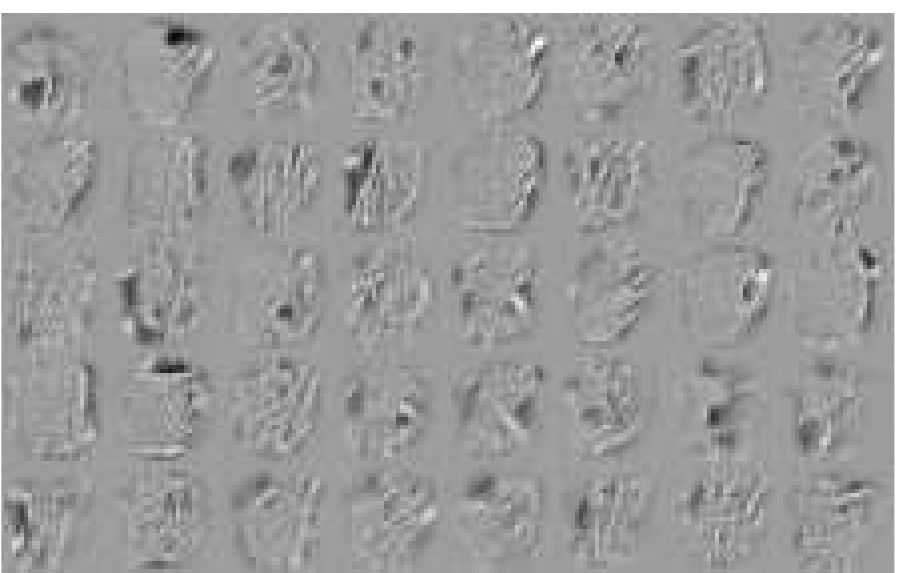}}
\hfill
\subfloat[Square error, $750$ hidden units]{
\includegraphics[width=\imwidth]{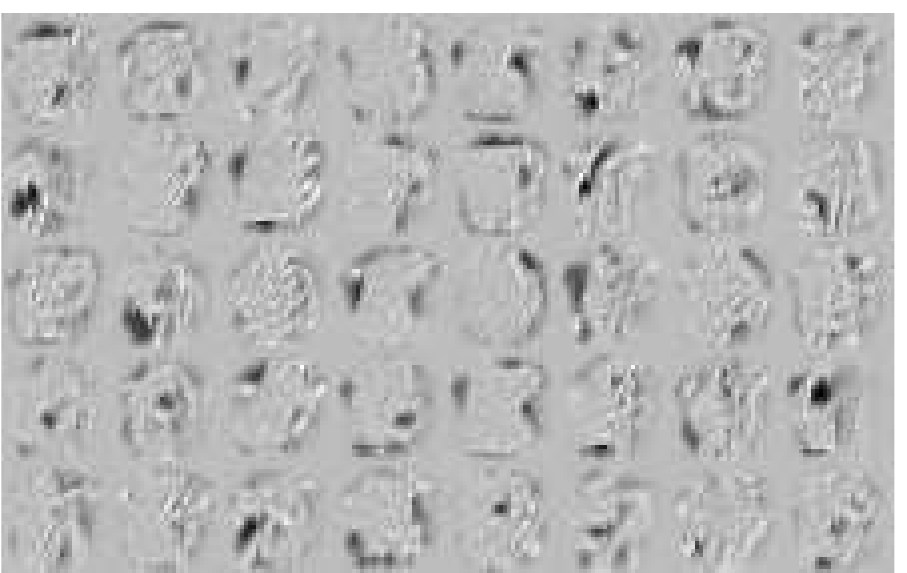}}\\
\subfloat[KL divergence, $1000$ hidden units]{
\includegraphics[width=\imwidth]{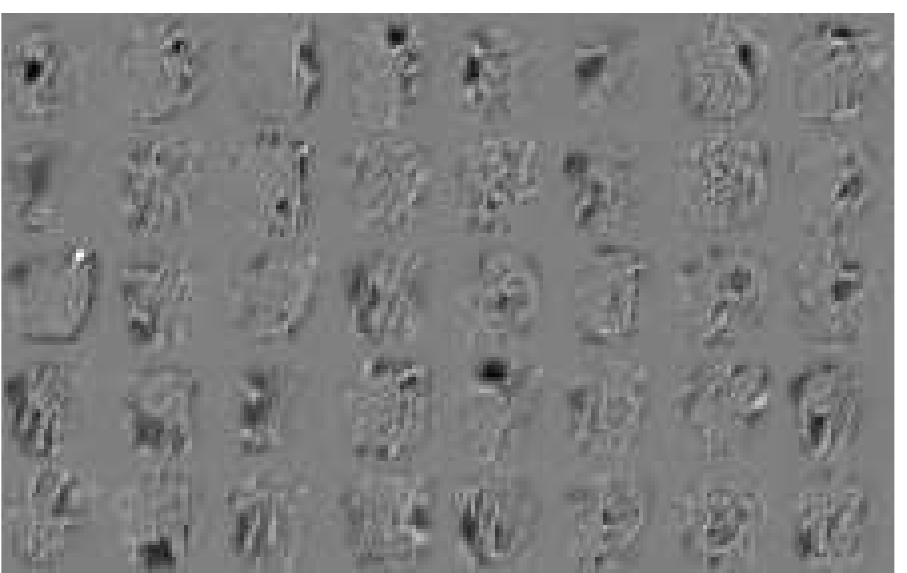}}
\hfill
\subfloat[Square error, $1000$ hidden units]{
\includegraphics[width=\imwidth]{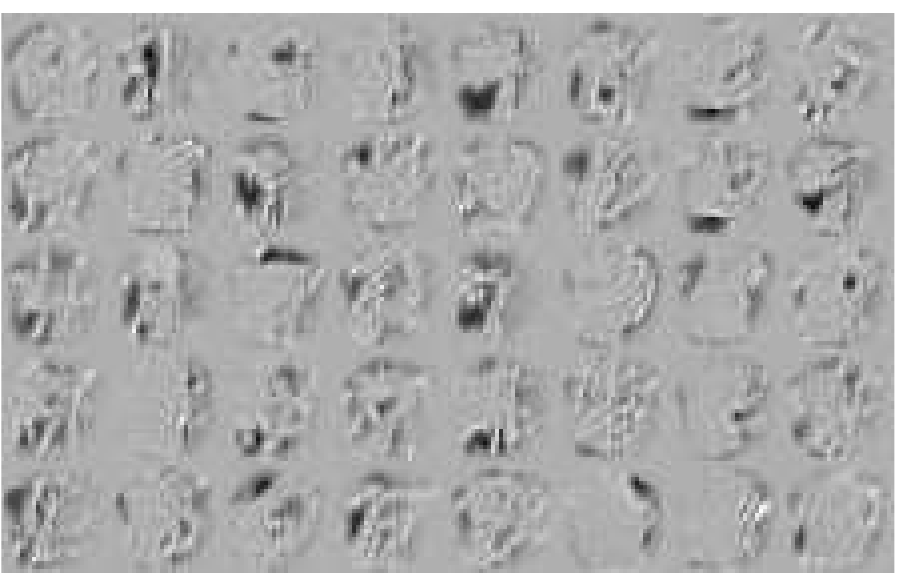}}
\caption{Features learnt by the trained NADEs. The features are the columns of the  matrix that has $\set{\vect{w}_i}$ as rows, reshaped as $28\times 28$ images.}
\label{fig:generative_models:nade_features}
\end{figure}

Figure~\ref{fig:generative_models:image_logprob} visualizes the agreement between each NADE and the RBM\@. To get the plots, the log probability of each of the first $500$ images of the MNIST test set was evaluated under each NADE and the RBM\@. Then, the images were sorted in descending order, according to the log probability the RBM assigns to them. The plots show the log probabilities of each NADE superimposed on the log probabilities of the RBM (the RBM curve is the same for all plots). Note that the true log probability of the RBM is actually intractable, as the partition function is needed to evaluate it. For the purposes of these plots, we set $\log{Z} = 451.28$, which is the value estimated by \citet{Salakhutdinov:2008:quantitative_analysis_of_DBNs}. We can see that in general the shape of the NADE curves match that of the RBM curve, which suggests that a fair amount of the RBM knowledge was distilled into the NADEs. However, the NADEs seem to underestimate the log probability assigned by the RBM\@. This is actually expected, since minimizing $\kl{\prob{\vect{x}}}{\prob{\vect{x}\g \bm{\theta}}}$ encourages $\prob{\vect{x}\g \bm{\theta}}$ to be broader that $\prob{\vect{x}}$, so as to prevent NADE from not having mass in regions where the RBM has mass. Also, for certain images, the log probabilities from the RBM and the NADEs differ by several nats, which shows that the distillation was not perfect. However, in interpreting these results, one should bear in mind that the log probability of the RBM in these plots is only an estimate (albeit a good one), so the conclusions drawn from them are only as good as this estimate.

Figure~\ref{fig:generative_models:nade_samples} shows samples generated from each trained NADE\@. These samples are exact, and they were generated using ancestral sampling\index{Sampling!Ancestral sampling}. It is evident that there is a strong resemblance between samples from NADE and actual images of digits, which shows that the NADEs have managed to learn a lot from the RBM\@. To see if the knowledge distillation was successful, this figure should be contrasted with Figure~\ref{fig:generative_models:samples_from_mnist_and_rbm}, which shows samples from the binarized MNIST dataset and the RBM\@. Obviously, it is easy to discriminate between real digits and samples from the RBM and the NADEs, which indicates that not even the RBM has learnt a perfect distribution estimator for MNIST\@. However, it is not so easy (at least visually) to discriminate between samples from the RBM and samples from the NADEs, suggesting that knowledge distillation has done a fairly good job. Finally, Figure~\ref{fig:generative_models:nade_features} shows some features learnt by each NADE\@.

From the above results, we can see that both KL divergence and square error were successful in distilling the RBM into NADE\@. However, it appears that KL divergence performs slightly better than square error. In Figure~\ref{fig:generative_models:training_progress}, we can see that the performance of KL divergence continues to increase with the number of samples, whereas square error seems to reach a plateau fairly quickly, after which the performance shows signs of declining. This is also reflected in the results of Table~\ref{table:generative_models:mnist_log_prob}, where KL divergence is shown to achieve a better performance at the end of training---although not significantly better. On the other hand, in the beginning of training, NADE appears to learn faster with square error. This is due to the fact that square error uses the unnormalized probabilities of the RBM in addition to sampling from it, therefore it transfers more information about the RBM to NADE\@. In contrast, KL divergence relies only on samples to transfer information from the RBM to NADE\@. This suggests that a potentially good strategy is to start training with square error to benefit from fast training, and after a few iterations switch to KL divergence to take advantage of the continuing improvement in performance.

Note that we also experimented with score matching. However, we found that score matching fails to train NADE satisfactorily. This to be expected, since score matching does not work with discrete probability distributions, and it is only meant to be used with continuous probability distributions, as we have discussed in section~\ref{sec:generative_models:loss_functions}.

\subsection{Estimating the partition function of the RBM}\index{Partition function}

Having access to a tractable generative model that closely mimics an intractable generative model provides a significant opportunity for estimating intractable quantities of the latter. To demonstrate this, in this section we will show how to use the trained NADEs of the previous section in order to obtain good estimates of the partition function of the RBM\@.

Estimating the partition function of an intractable RBM was also done by \citet{Salakhutdinov:2008:quantitative_analysis_of_DBNs}, who used annealed importance sampling\index{Sampling!Annealed importance sampling} for this purpose. Their method involved setting up a chain of $10{,}000$ intermediate distributions between the RBM and a tractable base distribution. Here, we will show that, having access to a mimicking NADE, we can achieve estimates which are as good as theirs with simpler sampling methods.

\subsubsection{Bounds based on KL divergence}
\index{KL divergence}

Training a NADE to mimic the RBM was done by minimizing either the KL divergence from RBM to NADE, or the square error between their log probabilities, which, as we argued, achieves a similar effect. Directly estimating the KL divergence between the trained NADE and the RBM provides a good upper bound to the log partition function. To see why, write the KL as follows
\begin{align}
\kl{\prob{\vect{x}}}{\prob{\vect{x}\g \bm{\theta}}} &= \avg{\log{\prob{\vect{x}}} - \log{\prob{\vect{x}\g \bm{\theta}}}}{\prob{\vect{x}}} \notag\\
&= \avg{\log{\uprob{\vect{x}}} - \log{\prob{\vect{x}\g \bm{\theta}}}}{\prob{\vect{x}}} - \log{Z}.
\end{align}
Due to the non-negativity of the KL divergence, we can upper bound $\log{Z}$ as follows
\begin{equation}
\log{Z} \le \avg{\log{\uprob{\vect{x}}} - \log{\prob{\vect{x}\g \bm{\theta}}}}{\prob{\vect{x}}}.
\end{equation}
The upper bound on the right hand side can be easily estimated by Monte Carlo as follows
\begin{equation}
\avg{\log{\uprob{\vect{x}}} - \log{\prob{\vect{x}\g \bm{\theta}}}}{\prob{\vect{x}}}\approx
\frac{1}{S}\sum_s\br{\log{\uprob{\vect{x}_s}} - \log{\prob{\vect{x}_s\g \bm{\theta}}}}
\quad\text{where}\quad
\vect{x}_s \sim \prob{\vect{x}}.
\end{equation}
Since the objective during training was to minimize $\kl{\prob{\vect{x}}}{\prob{\vect{x}\g \bm{\theta}}}$, this upper bound is optimized to be as tight as possible. Furthermore, its tightness indicates how faithfully NADE mimics the RBM\@.

\begin{table}[tbp]
\renewcommand{\tabcolsep}{0.6cm}
\renewcommand{\arraystretch}{\arrstretchvalue}
\centering
\begin{tabular}{ccc}
\toprule
\textbf{Hidden units} & \textbf{KL divergence} & \textbf{Square error} \\
\midrule
$\bm{250}$  & $459.01 \pm 0.20$ & $459.71 \pm 0.21$ \\
$\bm{500}$  & $458.20 \pm 0.18$ & $458.87 \pm 0.19$ \\
$\bm{750}$  & $457.66 \pm 0.17$ & $458.72 \pm 0.19$ \\
$\bm{1000}$ & $457.70 \pm 0.17$ & $458.52 \pm 0.19$ \\
\bottomrule
\end{tabular}
\caption{Approximate upper bounds of $\log{Z}$ given by estimating $\kl{\prob{\vect{x}}}{\prob{\vect{x}\g \bm{\theta}}}$ for each NADE\@. The lower this value is, the better NADE has fitted the RBM\@. Error bars correspond to $3$ standard deviations.}
\label{table:generative_models:logZ_kl_upper_bound}
\end{table}

\begin{table}[tbp]
\renewcommand{\tabcolsep}{0.6cm}
\renewcommand{\arraystretch}{\arrstretchvalue}
\centering
\begin{tabular}{ccc}
\toprule
\textbf{Hidden units} & \textbf{KL divergence} & \textbf{Square error} \\
\midrule
$\bm{250}$  & $442.80 \pm 0.26$ & $439.21 \pm 0.39$ \\ 
$\bm{500}$  & $442.98 \pm 0.27$ & $440.85 \pm 0.34$ \\
$\bm{750}$  & $443.41 \pm 0.25$ & $439.19 \pm 0.45$ \\
$\bm{1000}$ & $443.74 \pm 0.24$ & $438.11 \pm 0.51$ \\
\bottomrule
\end{tabular}
\caption{Approximate lower bounds of $\log{Z}$ given by estimating $\kl{\prob{\vect{x}\g \bm{\theta}}}{\prob{\vect{x}}}$ for each NADE\@. The higher this value is, the better NADE has fitted the RBM\@. Error bars correspond to $3$ standard deviations.}
\label{table:generative_models:logZ_kl_lower_bound}
\end{table}

We can use the same technique to calculate approximate lower bounds for $\log{Z}$ by estimating the KL divergence the other way round, from NADE to RBM\@. We rewrite this KL divergence as follows
\begin{align}
\kl{\prob{\vect{x}\g \bm{\theta}}}{\prob{\vect{x}}} &= \avg{\log{\prob{\vect{x}\g \bm{\theta}}} - \log{\prob{\vect{x}}}}{\prob{\vect{x}\g \bm{\theta}}} \notag\\
&= \avg{\log{\prob{\vect{x}\g \bm{\theta}}} - \log{\uprob{\vect{x}}}}{\prob{\vect{x}\g \bm{\theta}}} + \log{Z},
\end{align}
and by non-negativity we get the following lower bound
\begin{equation}
\avg{\log{\uprob{\vect{x}}} - \log{\prob{\vect{x}\g \bm{\theta}}}}{\prob{\vect{x}\g \bm{\theta}}} \le \log{Z},
\end{equation}
which can be easily estimated by Monte Carlo as follows
\begin{equation}
\avg{\log{\uprob{\vect{x}}} - \log{\prob{\vect{x}\g \bm{\theta}}}}{\prob{\vect{x}\g \bm{\theta}}}\approx
\frac{1}{S}\sum_s\br{\log{\uprob{\vect{x}_s}} - \log{\prob{\vect{x}_s\g \bm{\theta}}}}
\quad\text{where}\quad
\vect{x}_s \sim \prob{\vect{x}\g \bm{\theta}}.
\end{equation}
The higher this lower bound is, the better NADE has fitted the RBM\@. However, unlike the upper bound we calculated earlier, we expect this lower bound to be somewhat loose, since it was not optimized directly during training.

Tables~\ref{table:generative_models:logZ_kl_upper_bound} and \ref{table:generative_models:logZ_kl_lower_bound} show the estimated upper and lower bounds respectively given by each NADE\@. We used $10{,}000$ samples for each estimate. Samples from the RBM were obtained by running $10{,}000$ parallel Gibbs chains, randomly initialized and burned in for $2000$ iterations. Judging from the tightness of the bounds, we can see that NADE fits the RBM better with KL divergence, a result we also observed in the experiments of section~\ref{sec:rbm_vs_nade:distillation}. Also, the fit improves with the number of hidden units, which is to be expected since more hidden units add extra flexibility. We can thus conclude that the best models are the NADEs with $750$ and $1000$ hidden units, trained with KL divergence.

For comparison, $\log{Z}$ was estimated by \citet{Salakhutdinov:2008:quantitative_analysis_of_DBNs} to be $451.28$ nats, which, as expected, is closer to the estimated upper bounds than the estimated lower bounds. Note that we can always trivially lower bound $\log{Z}$, since
\begin{equation}
\max_{\vect{x}}{\log{\uprob{\vect{x}}}} \le \log{Z}.
\end{equation}
By approximately maximizing $\log{\uprob{\vect{x}}}$ using random sampling followed by hill climbing, we obtained a lower bound equal to $436.49$ nats. The lower bounds in Table~\ref{table:generative_models:logZ_kl_lower_bound} are all better than this.

\subsubsection{Importance sampling}
\index{Sampling!Importance sampling}

Importance sampling \citep[section 29.2]{MacKay:2002:IT} can be used in order to estimate the partition function. Given a tractable proposal distribution $q\br{\vect{x}}$ that is non-zero wherever $\prob{\vect{x}}$ is non-zero, importance sampling is based on rewriting the partition function as an expectation over $q\br{\vect{x}}$ as follows
\begin{equation}
Z = \sum_\vect{x}\uprob{\vect{x}} = \sum_\vect{x}\frac{\uprob{\vect{x}}}{q\br{\vect{x}}}q\br{\vect{x}} = \avg{\frac{\uprob{\vect{x}}}{q\br{\vect{x}}}}{q\br{\vect{x}}}.
\end{equation}
Hence, $Z$ can be estimated by Monte Carlo as follows
\begin{equation}
Z \approx \frac{1}{S}\sum_s \frac{\uprob{\vect{x}_s}}{q\br{\vect{x}_s}}
\quad\text{where}\quad
\vect{x}_s\sim q\br{\vect{x}_s}.
\end{equation}
The above estimate is only as good as the proposal distribution. To work well, $q\br{\vect{x}}$ must be as similar to $\prob{\vect{x}}$ as possible. If there is significant discrepancy between the two distributions, the above estimator has high variance.

In our framework, NADEs that were trained by minimizing $\kl{\prob{\vect{x}}}{\prob{\vect{x}\g \bm{\theta}}}$ will tend to mimic closely the RBM and at the same time avoid at all costs not putting mass where the RBM has mass. Thanks to these properties, such NADEs can serve as effective proposal distributions for importance sampling. In this section, we demonstrate that proposing from a trained NADE can make a method as simple as importance sampling work decently.

Table~\ref{table:generative_models:logZ_importance_sampling} shows the estimates of $\log{Z}$ obtained by using each trained NADE as proposal distribution. We used $10{,}000$ exact samples from NADE for each estimate. In general, the estimates agree across NADEs. \citet{Salakhutdinov:2008:quantitative_analysis_of_DBNs} report an estimate of $451.28$, with a $3$ standard deviation confidence interval of $\br{450.97, 451.52}$. Compared to this, it appears that importance sampling with NADE slightly underestimates $\log{Z}$. Still, it is interesting that a not too dissimilar result can be obtained just by using simple importance sampling.

\begin{table}[tbp]
\renewcommand{\tabcolsep}{0.6cm}
\renewcommand{\arraystretch}{\arrstretchvalue}
\centering
\begin{tabular}{ccc}
\toprule
\textbf{Hidden units} & \textbf{KL divergence} & \textbf{Square error} \\
\midrule
$\bm{250}$  & $450.86\;\br{450.06, 451.30}$ & $450.81\;\br{449.77, 451.31}$ \\ 
$\bm{500}$  & $450.96\;\br{449.94, 451.46}$ & $450.96\;\br{450.33, 451.34}$ \\
$\bm{750}$  & $450.79\;\br{450.47, 451.04}$ & $450.95\;\br{450.40, 451.30}$ \\
$\bm{1000}$ & $450.79\;\br{450.40, 451.07}$ & $451.51\;\br{450.09, 452.08}$ \\
\bottomrule
\end{tabular}
\caption{Estimates of RBM's log partition function based on importance sampling, using each trained NADE as proposal distribution. Brackets show confidence intervals that correspond to $3$ standard deviations.}
\label{table:generative_models:logZ_importance_sampling}
\end{table}

\subsubsection{Bridge sampling}
\index{Sampling!Bridge sampling}

Previously, we argued that the trained NADEs are excellent proposal distributions for estimating the partition function of the RBM\@. However, importance sampling is still a poor method. A better method, with the same proposal distribution, should give better results.

Importance sampling tends to fail when the proposal distribution has little mass in some regions where the target distribution has significant mass. Bridge sampling \citep[section 1]{Neal:2005:LIS} is an improvement over importance sampling that tries to alleviate this problem. Bridge sampling uses a ``bridge distribution'' $p_b\br{\vect{x}} = \frac{1}{Z_b}\bar{p}_b\br{\vect{x}}$, which is constructed to be the overlap between $\prob{\vect{x}}$ and $q\br{\vect{x}}$. Then, using standard importance sampling, the following two quantities are estimated
\begin{equation}
Z_b = \avg{\frac{\bar{p}_b\br{\vect{x}}}{q\br{\vect{x}}}}{q\br{\vect{x}}}
\quad\quad\text{and}\quad\quad
\frac{Z_b}{Z} = \avg{\frac{\bar{p}_b\br{\vect{x}}}{\uprob{\vect{x}}}}{\prob{\vect{x}}}.
\end{equation}
Notice that, by construction, $p_b\br{\vect{x}}$ does not have regions with significant mass that do not also have enough mass under $\prob{\vect{x}}$ and $q\br{\vect{x}}$, hence importance sampling works well in this case. Having estimated the above quantities, $Z$ can be trivially estimated by their ratio.

When the same number of importance samples has been used from both $\prob{\vect{x}}$ and $q\br{\vect{x}}$, the asymptotically optimal choice of bridge distribution \citep{Meng:1996:simulatingratios} is the following
\begin{equation}
\bar{p}_b\br{\vect{x}} = \frac{q\br{\vect{x}}\uprob{\vect{x}}}{Zq\br{\vect{x}} + \uprob{\vect{x}}}.
\end{equation}
In the above, the value of $Z$ is needed, which is what the method is trying to estimate. Thus, the method can be set up as a fixed point system; initialize $Z$ (e.g.~to $1$), calculate $\bar{p}_b\br{\vect{x}}$, use it to estimate $Z$, and repeat until convergence. In our experiments, we found that in practice only few iterations are necessary.

Following the above description of bridge sampling, we used it together with the trained NADE as proposal distribution. Our algorithm is described below.
\begin{framed}
\begin{enumerate}[label=(\roman*)]
\item\label{generative_models:bridge_sampling:step1} Generate $S$ exact samples from NADE\@.
\item\label{generative_models:bridge_sampling:step2} Run $S$ parallel Gibbs chains, initialized with the NADE samples, to generate $S$ independent RBM samples.
\item Initialize $Z = 1$.
\item\label{generative_models:bridge_sampling:step4} Iterate between calculating $\bar{p}_b\br{\vect{x}}$ and calculating $Z$. Reuse the samples collected in step \ref{generative_models:bridge_sampling:step1} and step \ref{generative_models:bridge_sampling:step2} throughout.
\end{enumerate}
\end{framed}
\noindent Note that in the above algorithm we take advantage of the fact that NADE samples are similar to RBM samples, and we initialize the RBM Gibbs chains with them. This way we ensure we obtain good quality samples from the RBM\@. In our experiments, we used $S=10{,}000$ samples and we burned in the RBM Gibbs chains for $2000$ iterations. Also, we iterated step \ref{generative_models:bridge_sampling:step4} $10$ times.

\begin{table}[tbp]
\renewcommand{\tabcolsep}{0.6cm}
\renewcommand{\arraystretch}{\arrstretchvalue}
\centering
\begin{tabular}{ccc}
\toprule
\textbf{Hidden units} & \textbf{KL divergence} & \textbf{Square error} \\
\midrule
$\bm{250}$  & $451.27\;\br{451.14, 451.39}$ & $451.24\;\br{451.10, 451.38}$ \\ 
$\bm{500}$  & $451.28\;\br{451.16, 451.39}$ & $451.30\;\br{451.17, 451.43}$ \\
$\bm{750}$  & $451.22\;\br{451.11, 451.33}$ & $451.21\;\br{451.08, 451.33}$ \\
$\bm{1000}$ & $451.24\;\br{451.13, 451.35}$ & $451.36\;\br{451.23, 451.49}$ \\
\bottomrule
\end{tabular}
\caption{Estimates of RBM's log partition function based on bridge sampling, using each trained NADE as proposal distribution. Brackets show confidence intervals that correspond to $3$ standard deviations.}
\label{table:generative_models:logZ_bridge_sampling}
\end{table}

Table~\ref{table:generative_models:logZ_bridge_sampling} shows the estimates of $\log{Z}$ using the above bridge sampling procedure, for each trained NADE\@. We can see that all estimates are consistent with each other and have small variance. This suggests that bridge sampling provides more accurate and more reliable estimates of $\log{Z}$ compared to importance sampling. For comparison, \citet{Salakhutdinov:2008:quantitative_analysis_of_DBNs}, using annealed importance sampling, report an estimate of $451.28$, with a $3$ standard deviation confidence interval of $\br{450.97, 451.52}$. The fact that our results are in close agreement with theirs, even though they were obtained by different means, is strong evidence for their correctness.\footnote{Note however that consistency does not necessarily imply correctness. It could still be the case that both our estimates and those of \citet{Salakhutdinov:2008:quantitative_analysis_of_DBNs} are not correct.}

\section{Related work and extensions}

In this section, we review two general techniques that are related to knowledge distillation in generative models. These are \emph{variational inference} and \emph{generative adversarial training}. Incorporating these techniques would be an interesting extension to the knowledge distillation framework presented in this chapter.

\subsection{Variational inference}
\index{Variational inference}

Variational inference originated as a method for approximate Bayesian inference and has its roots in statistical physics. In Bayesian inference, predictions are made by taking expectations over posterior distributions that are often intractable. The idea behind variational inference is to replace the intractable posterior distribution with a simpler one, over which expectations become tractable.

Let $\prob{\vect{x}} = \frac{1}{Z}\uprob{\vect{x}}$ be the intractable distribution of interest and  $\prob{\vect{x}\g \bm{\theta}}$ be a tractable distribution, parameterized by $\bm{\theta}$. Variational inference approximates $\prob{\vect{x}}$ with the particular $\prob{\vect{x}\g \bm{\theta}}$ that minimizes $\kl{\prob{\vect{x}\g \bm{\theta}}}{\prob{\vect{x}}}$. If $\prob{\vect{x}\g \bm{\theta}}$ is chosen such that expectations over it are analytically tractable, then minimizing the above KL divergence is particularly convenient. To see why, rewrite the KL divergence as follows
\begin{equation}
\kl{\prob{\vect{x}\g \bm{\theta}}}{\prob{\vect{x}}} = -F\br{\bm{\theta}} + \log{Z},
\end{equation}
where
\begin{equation}
F\br{\bm{\theta}} = \avg{\log{\uprob{\vect{x}}}}{\prob{\vect{x}\g \bm{\theta}}} - \avg{\log{\prob{\vect{x}\g \bm{\theta}}}}{\prob{\vect{x}\g \bm{\theta}}}.
\end{equation}
The above quantity $F\br{\bm{\theta}}$ is tractable, because it involves expectations of tractable quantities over $\prob{\vect{x}\g \bm{\theta}}$, and is known as \emph{variational free energy}\index{Variational free energy}. Minimizing the KL divergence is equivalent to maximizing the variational free energy. Due to the non-negativity of the KL divergence, it is easy to see that $F\br{\bm{\theta}}$ is a lower bound on $\log{Z}$.

Variational inference is closely related to knowledge distillation with KL divergence (as described in section~\ref{sec:generative_models:loss_functions}). Their difference is that the former minimizes $\kl{\prob{\vect{x}\g \bm{\theta}}}{\prob{\vect{x}}}$ whereas the later minimizes $\kl{\prob{\vect{x}}}{\prob{\vect{x}\g \bm{\theta}}}$. Nevertheless, despite their conceptual similarity, these two techniques are computationally different, and have different requirements regarding the choice of $\prob{\vect{x}\g \bm{\theta}}$. Variational inference typically requires the expectations over $\prob{\vect{x}\g \bm{\theta}}$ to be analytically tractable, which limits the choice of $\prob{\vect{x}\g \bm{\theta}}$ severely. For instance, expectations under NADE are not analytically tractable, and therefore the variational framework is not directly applicable to distilling RBMs into NADEs. On the other hand, knowledge distillation only requires $\prob{\vect{x}\g \bm{\theta}}$ to be differentiable, which is a much weaker requirement. Making variational inference work well with models like NADE is a future research direction.

For more information and details, variational inference and its applications are discussed in length by \citet[chapter 10]{Bishop:2006:PRML} and \citet[chapter 21]{Murphy:2012:MLPP}.

\subsection{Generative adversarial training}
\index{Generative adversarial training}

If two generative models are equivalent, then, given a sample from one of them, it would be hard to tell which model it came from. Generative adversarial training \citep{Goodfellow:2014:GAN} takes this idea and turns it into a training objective, in order to train a student generative model to mimic a teacher generative model.

Figure~\ref{fig:generative_models:generative_adversarial_training} provides a schematic representation of generative adversarial training. The setup includes a teacher generative model $\prob{\vect{x}}$, a student generative model $\prob{\vect{x}\g \bm{\theta}}$, parameterized by $\bm{\theta}$, and an adversary $f\br{\vect{x}\g \bm{\phi}}$, parameterized by $\bm{\phi}$. Given a sample $\vect{x}$, the job of the adversary is to output how likely it thinks it is that $\vect{x}$ came from the student or the teacher. During training, both the teacher and the student generate samples, and the adversary tries to predict their model of origin. The student and the adversary are trained simultaneously; the training objective for the student is to fool the adversary, and the training objective for the adversary is to improve so as not to be fooled. As this min-max game between the student and the adversary progresses, the student adapts in order to generate samples that look like those of the teacher, following the training signal from the adversary.

\begin{figure}[tbp]
\centering
\begin{tikzpicture}[>=stealth']

\tikzstyle{block} = [draw, rectangle, thick, inner sep=8pt, color=black]

\node[block] (loss) {\textbf{adversary}};
\node[block, below left=1.2cm  and 0.1cm of loss] (teacher) {\textbf{teacher}}; 
\node[block, below right=1.2cm and 0.1cm of loss] (student) {\textbf{student}}; 
	
\draw [->, semithick] (teacher) |- node [left,  pos=0.3] {\textit{sample}} (loss);
\draw [->, semithick] (student) |- node [right, pos=0.3] {\textit{sample}} (loss);
\draw [->, semithick] (loss)    |- node [right, pos=0.2] {\textit{train}} (student);

\end{tikzpicture}
\caption{Schematic representation of generative adversarial training. The student is trained to fool the adversary, whereas the adversary is trained to discriminate between the teacher and the student.}
\label{fig:generative_models:generative_adversarial_training}
\end{figure}
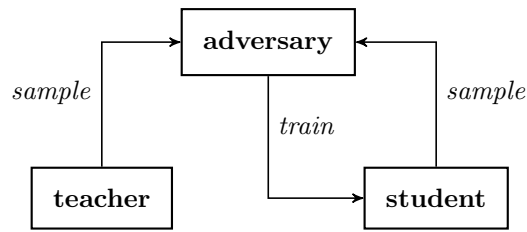

The main advantage of generative adversarial training is that it requires very little from $\prob{\vect{x}}$ and $\prob{\vect{x}\g \bm{\theta}}$ and thus exhibits broad applicability. There are only two main requirements; (a) it must be possible to sample from both $\prob{\vect{x}}$ and $\prob{\vect{x}\g \bm{\theta}}$, and (b) $\prob{\vect{x}\g \bm{\theta}}$ must be differentiable. The downside is that it can be tricky to get it to work well; there are critical decisions to be made regarding how much capacity to allow to the adversary and how long the adversary should be trained for before an update to the student is made. Also, in case we can evaluate $\prob{\vect{x}}$ and $\prob{\vect{x}\g \bm{\theta}}$, it is not clear whether the adversarial framework is any better than directly trying to minimize a discrepancy measure between $\prob{\vect{x}}$ and $\prob{\vect{x}\g \bm{\theta}}$. In any case, incorporating generative adversarial training is an interesting future direction, as it would add flexibility and generality to the knowledge distillation framework.

\section{Summary and conclusions}

In this chapter we presented a framework for distilling intractable generative models to tractable ones. We managed to achieve this by using loss functions that circumvent the intractable partition function, while at the same time capture as much information about the intractable model as possible.

We validated our framework by successfully distilling an intractable RBM to a tractable NADE\@. As an application of the distillation, we showed that the distilled NADE can be successfully used to estimate the intractable partition function of the RBM, using simple sampling methods that would not work well otherwise. This constitutes a novel method of estimating the partition function, and, to the best of our knowledge, this thesis is the first in the literature to describe it.

We concluded the chapter by examining two approaches relevant to knowledge distillation, namely variational inference and generative adversarial training. In future work, these two approaches can be incorporated into our framework, extending its flexibility and applicability.

Even though in principle the framework can be used as much for continuous as for discrete probability distributions, our evaluation focused on the latter kind. Due to this, score matching, even though it was included and implemented in our framework, was left out of the evaluation, since it only works with continuous probability distributions. Future work is needed in order to validate the framework and score matching in particular in the case of continuous distributions as well.

\chapter{Conclusions}
\label{chapter:conclusions}

Our goal in this thesis was to show how we could take a cumbersome but accurate model, and replace it with a convenient model that does the same job just as well. Imagining we could extract the knowledge hidden inside the cumbersome model, distil it, and inject it into the convenient model, we referred to the process of constructing the convenient model from the cumbersome one as knowledge distillation\index{Knowledge distillation}.

Our approach to knowledge distillation was to train the convenient model to mimic the cumbersome one, by making it observe the behaviour of the latter and penalizing it whenever it failed to reproduce it accurately. We applied this approach to three distinct and fairly general fields of machine learning: discriminative models, Bayesian inference and generative models. For each field, we presented a knowledge distillation framework, and successfully applied it to an example problem within the field.

Our work brought together ideas often referred to in the literature as knowledge distillation, model compression, compact approximations, mimicking models and teacher-student models. Our view in this thesis was that all these different terms refer to similar concepts and ideas. By viewing them as different manifestations of knowledge distillation, we were able to treat them in a unifying manner.

We contributed to the literature of knowledge distillation with three main novel ideas. Firstly, we introduced the concept of online distillation in the context of compact predictive distributions (chapter~\ref{chapter:compact_predictive_distributions}). With online distillation, we were able to significantly reduce the memory requirements of existing methods. Secondly, we proposed matching derivatives as part of the distillation process. We showed that matching derivatives can be done for a constant extra cost, and provided an efficient implementation of it. We argued that derivatives can convey more information about the model we wish to distil, justifying the extra cost when data is scarce. Thirdly, in the context of generative models (chapter~\ref{chapter:generative_models}), we demonstrated that the distilled model can be used for estimating intractable quantities of the original model, and we successfully applied this idea in estimating the partition function of an intractable RBM\@.

\section{Knowledge distillation: is it worth it?}

Knowledge distillation relies on the premise that the convenient model is capable of doing the job of the cumbersome model just as well. However, in the introduction (chapter~\ref{chapter:introduction}) we argued that the cumbersome model often owes its success to those characteristics that make it cumbersome, such as its size and its complexity. Then, how can we expect a simple model to be able to perform just as well? On the other hand, if a simple model does indeed perform just as well, then why take the trouble to train the cumbersome model at all? Why not simply train and use the convenient model in the first place?

It is true indeed that a simple model will typically have less capacity than a larger, more complicated model. For instance, in our case study in chapter~\ref{chapter:model_compression}, where we distilled an ensemble of neural networks into a single small neural network, we observed that the single network performed less well than the ensemble. Our goal there was to perform the distillation in such a way as to make this decrease in performance as small as possible. Moreover, in our case study in chapter~\ref{chapter:generative_models}, where an RBM was distilled into a NADE, we found that a NADE with at least the same capacity as the RBM was needed for accurately mimicking it. In cases such as the above, distillation makes sense as long as the decrease in performance is counterbalanced by the benefits offered by the convenient model.

On the other hand, the fact that a larger model typically has more capacity does not necessarily mean that the particular function it represents could not in principle  be learnt by a simpler model. It could be the case that the larger model wastes capacity. For instance, in our two case studies in chapter~\ref{chapter:compact_predictive_distributions}, we observed that the prediction surfaces represented by large bags of MCMC samples were actually rather simple and smooth, and therefore could be easily represented by more compact models. After all, this is exactly what knowledge distillation as a metaphor is all about: the fact that the knowledge of the large model can be contained in a condensed, distilled form in the compact model, not wasting capacity.

But if the knowledge can fit in a compact model, why then not just train and use the compact model in the first place? Certainly, in several cases this is the right thing to do. However, there are situations where training a larger model and then distilling it into a smaller one might actually be preferable. In the following paragraphs, we describe four such situations.

Firstly, it might not be at all obvious how the compact model could be trained directly. In chapter~\ref{chapter:compact_predictive_distributions}, where a compact model was trained to approximate the true predictive distribution in the context of Bayesian inference, we needed to explore the posterior with MCMC in order to extract information from it. It is not obvious how we could have fitted a compact model to the predictive distribution directly, since the true predictive distribution is typically intractable.\footnote{\citet{Snelson:2005:vb_pred} tried fitting a compact model directly to the predictive distribution using variational inference, but with little success.}

Secondly, the fact that a compact model can in principle represent the knowledge of a large model does not necessarily mean that it can also extract it directly from data. Perhaps our learning algorithms are simply not good enough yet. Existing learning algorithms might favour larger models, and more capacity might be needed for the successful extraction of knowledge from data. Hence, it might be more effective to first extract the knowledge from the data by a large capacity-wasting model, and then successfully distil it to a more compact model.

Thirdly, we might not have enough labelled data for training the compact model to begin with. A large model, together with proper regularization, might be better at learning from a small dataset. On the other hand, as we have seen, knowledge distillation does not suffer from shortage of labelled data, as it does not need labels at all. For instance, in the context of model compression (chapter~\ref{chapter:model_compression}) all labelling is done by the teacher model. As long as we have access to a suitable data generator (and in this work we demonstrated ways for building one) then we can have as much input data as we need.

Fourthly, it could be the case that we actually want the convenient model to be consistent with the cumbersome model. For example, in our case study in chapter~\ref{chapter:generative_models}, we saw that we could obtain better estimates of the partition function of the RBM by having a NADE that matched it as closely as possible. Moreover, it could be the case that the cumbersome model we try to emulate is actually the one and only true model. For instance, in chapter~\ref{chapter:compact_predictive_distributions}, the predictive distribution the compact model was trying to match was the one and only true distribution (according the rules of probability). The same would be true if we were trying to distil an intractable true posterior to a tractable model, as discussed in chapter~\ref{chapter:generative_models}.

\section{What is there next?}

Model compression, Bayesian inference and generative models do not exhaust the applications of knowledge distillation. Another potential application is in learning efficient models of otherwise expensive to evaluate functions. We can imagine having a function we could evaluate, but doing so routinely would be too expensive. Such a function, for instance, could be the result of a long physics simulation. Using knowledge distillation techniques, we could construct an efficient model, such as a neural network, that would compute the same function in only a fraction of the time. Matching derivatives could prove useful in a situation like this, as it could reduce the number of times the original function would have to be called.

Distilling an expensive to evaluate function into a neural network was also done by \citet{Heess:2013:learn_ep_msg}. Their work focuses on speeding up the calculation of messages that need to be passed within a graphical model, when using Expectation Propagation for doing inference. In EP, outgoing messages are deterministic functions of incoming messages that are often expensive to evaluate. By distilling the functions that compute messages into neural networks, they were able to make EP run in only a fraction of the time.

Another potential extension of knowledge distillation could be in the framework itself. As we have seen in previous chapters, knowledge distillation often requires the solution of an optimization problem of the form
\begin{equation}
\min_{\bm{\theta}}{\avg{E\br{\bm{\vect{x},\theta}}}{\prob{\vect{x}}}},
\end{equation}
where $E\br{\bm{\vect{x},\theta}}$ is some loss and $\prob{\vect{x}}$ is some input distribution. In our framework, we proposed optimizing objectives of the above form using a stochastic approximation, by first generating samples from $\prob{\vect{x}}$ and then minimizing the empirical average over the samples. This approach prioritizes the minimization of $E\br{\bm{\vect{x},\theta}}$ at locations $\vect{x}$ with high probability. However, we could imagine prioritizing locations $\vect{x}$ for which the loss is high as well. In other words, the minimization procedure could actively seek input locations for which minimizing the loss would have the strongest impact. Such a scheme may need fewer loss evaluations, and would be particularly beneficial when such evaluations are expensive.

In this thesis we validated our knowledge distillation framework on example problems and standard machine learning benchmarks, such as the MNIST dataset of handwritten digits. It remains to be seen how useful knowledge distillation will prove to be in large-scale industrial applications. Our hope is that the techniques and ideas presented in this thesis will contribute to extending the applicability of machine learning to yet more interesting and useful practical problems.

\appendix

\chapter{The R Technique}
\label{chapter:R_technique}

Let $f$ be a scalar-valued function of a multidimensional vector $\vect{x}$. Assume we are interested in calculating $\mat{H}\vect{v}$, where $\mat{H}$ is the Hessian of $f$, i.e.~the matrix of all second derivatives of $f$, and $\vect{v}$ is an arbitrary vector of the same size as $\vect{x}$. Calculating and storing $\mat{H}$ takes quadratic time and space in the size of $\vect{x}$, making it inefficient for a large number of inputs. Nevertheless, storing $\mat{H}\vect{v}$ takes linear space in the size of $\vect{x}$ and, as we will describe, computing it can be done in linear time as well.

Calculating $\mat{H}\vect{v}$  in linear time is made possible by the \emph{R technique}\index{R technique}, first described by \citet{Pearlmutter:1994:fast_mult_hess}. Given some vector-valued function $\vect{g}$ of $\vect{x}$ and a vector $\vect{v}$ of the same size as $\vect{x}$, we define the operator $\Rop{\cdot}$ as
\begin{equation}
\Rop{\vect{g}} = \br{\pderiv{\vect{g}}{\vect{x}}}^T\vect{v}.
\end{equation}
Since $\Rop{\cdot}$ is a differential operator, its properties are similar to the familiar properties of derivatives. The following properties can be easily proved using the definition.
\begin{enumerate}[label=(\roman*)]
\item $\Rop{\vect{c}} = \vect{0}$ for any constant $\vect{c}$
\item $\Rop{\vect{x}} = \vect{v}$
\item $\Rop{\vect{g}_1+\vect{g}_2} = \Rop{\vect{g}_1} + \Rop{\vect{g}_2}$
\item $\Rop{\vect{g}_1\vect{g}_2} = \Rop{\vect{g}_1}\vect{g}_2 + \vect{g}_1\Rop{\vect{g}_2}$
\item $\Rop{\vect{g}_1\br{\vect{g}_2}} = \br{\pderiv{\vect{g}_1}{\vect{g}_2}}^T\Rop{\vect{g}_2}$
\end{enumerate}
Using the $\Rop{\cdot}$ operator, it is easy to see that $\mat{H}\vect{v}$ can be written as $\Rop{\pderiv{f}{\vect{x}}}$. The latter can then be calculated in linear time by the same equations that are used to calculate $\pderiv{f}{\vect{x}}$, after transforming them by $\Rop{\cdot}$ making use of the above list of properties.

\chapter{Derivatives in Neural Networks}
\label{chapter:derivatives_in_neural_networks}
\index{Neural network}

A feedforward neural network is a function $f$, parameterized by a set of parameters $\bm{\theta}$, that maps some input $\vect{x}$ to some output $\vect{y}$. Mathematically, we write
\begin{equation}
\vect{y} = f\br{\vect{x}\g\bm{\theta}}.
\end{equation}
In this appendix, we discuss forward and backward propagation in neural networks, as well as common nonlinearities and output functions, with a focus on how to efficiently calculate first and second derivatives.

\section{Computing the output: forward propagation}
\index{Forward propagation}

The network $f$ is made of $L$ cascaded layers, each consisting of a linear and a nonlinear phase, together with an input layer. Let $\vect{x}^{\br{\ell}}$ be the output of the $\nth{\ell}$ layer, with $\ell=0$ corresponding to the input layer. Then, $f$ is defined by the following recursive procedure, known as \emph{forward propagation}.
\begin{align}
\vect{x}^{\br{0}} &= \vect{x} \\
\vect{z}^{\br{\ell}} &= \mat{W}^{\br{\ell}}\vect{x}^{\br{\ell-1}} + \vect{b}^{\br{\ell}} \\
\vect{x}^{\br{\ell}} &= \bm{\phi}^{\br{\ell}}\br{\vect{z}^{\br{\ell}}} \\
\vect{y} &= \vect{x}^{\br{L}}
\end{align}
The network's adjustable parameters $\bm{\theta}$ are the collection of parameters from all linear phases, that is, $\bm{\theta} = \set{\mat{W}^{\br{\ell}}, \vect{b}^{\br{\ell}}}$. The nonlinear phase of the $\nth{\ell}$ layer is specified by $\bm{\phi}^{\br{\ell}}$, which in practice can be any convenient nonlinear function. The choice of $\bm{\phi}^{\br{\ell}}$ and the number of layers $L$ are part of the network's design  and do not belong to the set of adjustable parameters.

In order to reduce unnecessary clutter and make the calculations more explicit, we will drop the dependence on $\ell$ and expand the matrix equations, and thus we will express forward propagation as
\begin{align}
z_i &= \sum_j{w_{ij}x_j} + b_i \\
x_i &= \phi_i\br{z_1,z_2,\ldots,z_I}.
\end{align}
In a slight abuse of notation, our convention is that variables indexed by $i$ belong to layer $\ell$ and variables indexed by $j$ belong to layer $\ell-1$.

\section{Computing first derivatives: backward propagation}
\label{sec:derivatives_in_neural_nets:backprop}
\index{Backward propagation}

Assume that the network's output $\vect{y}$ is mapped to a real value $E\br{\vect{y}}$ by some differentiable function $E$. We would like to compute the derivatives of $E$ with respect to the network's parameters $\bm{\theta}$ and inputs $\vect{x}$, that is $\pderiv{E}{\bm{\theta}}$ and $\pderiv{E}{\vect{x}}$. This is easily done by applying the chain rule from calculus to the equations describing forward propagation, yielding the following recursive  procedure known as \emph{backward propagation} or \emph{backprop} for short.
\begin{align}
\pderiv{E}{b_i} &= \sum_{i'}{\pderiv{E}{x_{i'}}\pderiv{\phi_{i'}}{z_i}} \label{eq:backprop:dEdb}\\
\pderiv{E}{w_{ij}}&= \pderiv{E}{b_i} x_j 
\label{eq:backprop:dEdw}\\
\pderiv{E}{x_j} &= \sum_{i'}{\pderiv{E}{b_{i'}}w_{i'j}} 
\end{align}
The recursion starts by setting $\pderiv{E}{\vect{x}^{\br{L}}} = \pderiv{E}{\vect{y}}$, which we assume we can calculate, and works from the output layer all the way down to the input layer. In each iteration, it yields $\pderiv{E}{\mat{W}^{\br{\ell}}}$ and $\pderiv{E}{\vect{b}^{\br{\ell}}}$, and at the end of the recursion, it outputs $\pderiv{E}{\vect{x}} = \pderiv{E}{\vect{x}^{\br{0}}}$. Forward propagation is assumed to have been run in advance, so that  quantities $\vect{x}^{\br{\ell}}$ and $\pderiv{\bm{\phi}^{\br{\ell}}}{\vect{z}^{\br{\ell}}}$ are available. Note that if only $\pderiv{E}{\vect{x}}$ is needed, then Equation~\eqref{eq:backprop:dEdw}---which in practice is the most time-consuming---can be skipped.

In the common case where the nonlinearities $\phi_i$ act elementwise, we have $\pderiv{\phi_{i'}}{z_i} = 0$ when $i'\neq i$, so Equation~\eqref{eq:backprop:dEdb} can be simplified to 
\begin{equation}
\pderiv{E}{b_i} = \pderiv{E}{x_{i}}\pderiv{\phi_{i}}{z_i}.
\end{equation}

\section{Computing second derivatives: R\{backprop\}}
\label{sec:derivatives_in_neural_nets:Rbackprop}
\index{R\{backprop\}}

Computing the matrix of second derivatives (i.e.~the Hessian) of $E$ with respect to the network's parameters $\bm{\theta}$ and inputs $\vect{x}$ is in principle possible, by applying the chain rule to the procedure jointly defined by forward and backward propagation. However, calculating or even storing the Hessian takes quadratic time and space in the size of $\bm{\theta}$ and $\vect{x}$, making it inefficient for large networks. Nevertheless, in practice what is usually needed is not the full Hessian $\mat{H}$ but its product with some vector $\vect{v}$, i.e.~$\mat{H}\vect{v}$. Storing $\mat{H}\vect{v}$ takes linear space in the size of $\bm{\theta}$ and $\vect{x}$ and it can be computed in linear time using the R technique\index{R technique} (described in appendix~\ref{chapter:R_technique}).

It is easy to see that $\mat{H}\vect{v}$ can be written in terms of the  $\Rop{\cdot}$ operator as $\Rop{\pderiv{E}{\pair{\bm{\theta}}{\vect{x}}}}$. The latter can be calculated in linear time by a procedure similar to joint forward and backward propagation, where all the equations have been transformed by $\Rop{\cdot}$. From here on, we will call this procedure \emph{R\{backprop\}}. Using the properties of $\Rop{\cdot}$ and applying it to both sides of the equations defining forward propagation, we get the forward phase of R\{backprop\}, which is as follows.
\begin{align}
\Rop{z_i} &= \sum_j{\br{\Rop{w_{ij}}x_j + w_{ij}\Rop{x_j}}} + \Rop{b_i} \label{eq:Rbackprop:Rz}\\
\Rop{x_i} &= \sum_{i'}{\pderiv{\phi_i}{z_{i'}}\Rop{z_{i'}}}\label{eq:Rbackprop:Rx}.
\end{align}
In the above, $\Rop{\mat{W}^{\br{\ell}}}$ and $\Rop{\vect{b}^{\br{\ell}}}$ are simply the elements of $\vect{v}$ that correspond to $\mat{W}^{\br{\ell}}$ and $\vect{b}^{\br{\ell}}$ respectively. Similarly, $\Rop{\vect{x}^{\br{0}}}$ is initialized with the elements of $\vect{v}$ that correspond to $\vect{x}$. Forward propagation is assumed to have been run in advance, so that  quantities $\vect{x}^{\br{\ell}}$ and $\pderiv{\bm{\phi}^{\br{\ell}}}{\vect{z}^{\br{\ell}}}$ are available. Note also that in the common case where the nonlinearities $\phi_i$ act elementwise, Equation~\eqref{eq:Rbackprop:Rx} simplifies to
\begin{equation}
\Rop{x_i} = \pderiv{\phi_i}{z_{i}}\Rop{z_{i}}.
\end{equation}
Next, we apply $\Rop{\cdot}$ to both sides of the backprop equations to get the backward phase of R\{backprop\}, which is as follows.
 \begin{align}
\Rop{\pderiv{E}{b_i}} &= \sum_{i'}{\br{\Rop{\pderiv{E}{x_{i'}}}\pderiv{\phi_{i'}}{z_i} + \pderiv{E}{x_{i'}}\Rop{\pderiv{\phi_{i'}}{z_i}}}} \label{eq:Rbackprop:RdEdb}\\
\Rop{\pderiv{E}{w_{ij}}} &=\Rop{ \pderiv{E}{b_i}} x_j + \pderiv{E}{b_i} \Rop{x_j}\\
\Rop{\pderiv{E}{x_j}} &= \sum_{i'}{\br{\Rop{\pderiv{E}{b_{i'}}}w_{i'j} + \pderiv{E}{b_{i'}}\Rop{w_{i'j}}} }\label{eq:Rbackprop:RdEdx}
\end{align}
In the above, we assume that $\vect{x}^{\br{\ell}}$, $\pderiv{\bm{\phi}^{\br{\ell}}}{\vect{z}^{\br{\ell}}}$ and $\Rop{\pderiv{\bm{\phi}^{\br{\ell}}}{\vect{z}^{\br{\ell}}}}$ are available from a previous run of forward propagation,  $\pderiv{E}{\vect{x}^{\br{\ell}}}$ and $\pderiv{E}{\vect{b}^{\br{\ell}}}$ are available from a previous run of backprop, and $\Rop{\vect{x}^{\br{\ell}}}$ is available from the forward phase of R\{backprop\}. The procedure needs to be initialized with $\Rop{\pderiv{E}{\vect{x}^{\br{L}}}} = \Rop{\pderiv{E}{\vect{y}}}$, whose calculation depends on the choice of function $E$. Finally, in each iteration the procedure yields $\Rop{\pderiv{E}{\mat{W}^{\br{\ell}}}}$ and $\Rop{\pderiv{E}{\vect{b}^{\br{\ell}}}}$, and at the end it outputs $\Rop{\pderiv{E}{\vect{x}}} = \Rop{\pderiv{E}{\vect{x}^{\br{0}}}}$, which collectively form $\mat{H}\vect{v}$. Note that in the common case where the nonlinearities $\phi_i$ act elementwise, Equation~\eqref{eq:Rbackprop:RdEdb} simplifies to
\begin{equation}
\Rop{\pderiv{E}{b_i}} = \Rop{\pderiv{E}{x_{i}}}\pderiv{\phi_{i}}{z_i} + \pderiv{E}{x_{i}}\Rop{\pderiv{\phi_{i}}{z_i}}.
\end{equation}

A couple of simplifications in the R\{backprop\} procedure are possible if only part of the Hessian needs to be considered. Let $\vect{v} = \begin{bmatrix}
\vect{v}_{\bm{\theta}} & \vect{v}_{\vect{x}}
\end{bmatrix}^T$, where $\vect{v}_{\bm{\theta}}$ is the part of $\vect{v}$ corresponding to $\bm{\theta}$ and $\vect{v}_{\vect{x}}$ is the part of $\vect{v}$ corresponding to $\vect{x}$. Then, the Hessian-vector product can be written as
\begin{equation}
\mat{H}\vect{v} = \begin{bmatrix}
\spderiva{E}{\bm{\theta}}\vect{v}_{\bm{\theta}} + \spderivb{E}{\bm{\theta}}{\vect{x}}\vect{v}_{\vect{x}}\\[5pt]
\spderivb{E}{\vect{x}}{\bm{\theta}}\vect{v}_{\bm{\theta}} + \spderiva{E}{\vect{x}}\vect{v}_{\vect{x}}
\end{bmatrix}.
\end{equation}
If only $\spderivb{E}{\bm{\theta}}{\vect{x}}\vect{v}_{\vect{x}}$ and $\spderiva{E}{\vect{x}}\vect{v}_{\vect{x}}$ are needed, we can equivalently set $\vect{v}_{\bm{\theta}} = \vect{0}$. In this case, Equation~\eqref{eq:Rbackprop:Rz} simplifies to 
\begin{equation}
\Rop{z_i} = \sum_j{w_{ij}\Rop{x_j}}
\end{equation}
and Equation~\eqref{eq:Rbackprop:RdEdx} simplifies to
\begin{equation}
\Rop{\pderiv{E}{x_j}} = \sum_{i'}{\Rop{\pderiv{E}{b_{i'}}}w_{i'j}}.
\end{equation}

\section{Common nonlinearities}

In forward propagation, the output  $x_i$ of a layer is given by applying a nonlinearity $\phi_i$, that is
\begin{equation}
x_i = \phi_i\br{z_1,z_2,\ldots,z_I}.
\end{equation}
As part of the network's design, for each layer a nonlinearity $\phi_i$ has to be specified, together with its derivatives $\pderiv{\phi_i}{z_j}$ and $\Rop{\pderiv{\phi_i}{z_j}}$, which are needed for backprop and R\{backprop\} respectively. In this section, we describe some common choices for $\phi_i$, namely linear, rectified linear, logistic, probit and softmax.

\subsection{Linear}
\index{Layer!Linear}

The simplest choice for $\phi_i$ is the identity function, that is
\begin{equation}
\phi_i\br{z_i} = z_i.
\end{equation}
This function acts elementwise and its derivatives are
\begin{equation}
\pderiv{\phi_i}{z_i} = 1\quad\quad\text{and}\quad\quad\Rop{\pderiv{\phi_i}{z_i}} = 0.
\end{equation}
It is easy to see that the expressions in backprop and R\{backprop\} that involve $\pderiv{\phi_i}{z_i}$ or $\Rop{\pderiv{\phi_i}{z_i}}$ can be significantly simplified in this case. 

A linear layer is a good choice for the final layer of a network used for regression.
Note however that a network having only linear layers throughout can only represent linear functions. 

\subsection{Rectified linear}
\label{sec:derivs_neural_nets:relu}
\index{Layer!Rectified linear}

The rectified linear layer (or rectified linear unit---ReLU for short) acts elementwise on its inputs and is defined as 
\begin{equation}
\phi_i\br{z_i} = \max\br{0,z_i}
\end{equation}
that is, it behaves as linear for $z_i>0$ and as constant (zero) otherwise. Its derivatives are
\begin{equation}
\pderiv{\phi_i}{z_i} = \begin{cases}0 & z_i <0 \\ 1 & z_i > 0 \end{cases}
\quad\quad\text{and}\quad\quad\Rop{\pderiv{\phi_i}{z_i}} = 0.
\end{equation}
Notice that $\phi_i$ is not differentiable at one point, namely at  $z_i = 0$. However, since it is differentiable arbitrarily close to it, in practice we can take the derivative at $z_i = 0$ to be either $0$ or $1$. Of course, expressions in R\{backprop\} that involve $\Rop{\pderiv{\phi_i}{z_i}}$ can be simplified due to it being always equal to $0$.

\subsection{Logistic}
\index{Layer!Logistic}

The logistic nonlinearity acts elementwise and is defined as
\begin{equation}
\phi_i\br{z_i} = \frac{1}{1 + \exp\br{-z_i}}.
\end{equation}
Given that $x_i = \phi_i\br{z_i}$, it can easily be shown that its first derivative can be written as
\begin{equation}
\pderiv{\phi_i}{z_i} = x_i\br{1-x_i},
\end{equation}
and therefore
\begin{equation}
\Rop{\pderiv{\phi_i}{z_i}} = \Rop{x_i}\br{1-2x_i}.
\end{equation}
Mathematically it is always the case that $0<\phi_i\br{z_i}<1$. However, due to numerical reasons, it is easy for $\phi_i\br{z_i}$ to underflow to $0$ for large negative $z_i$ or to $1$ for large positive $z_i$, which is undesirable. Therefore care needs to be taken; one possible solution is to limit $z_i$ to be within a safe range.

\subsection{Probit}
\index{Layer!Probit}

The probit function is similar in shape to the logistic and also acts elementwise. It is defined as the cumulative density of the standard normal, that is
\begin{equation}
\phi_i\br{z_i} = \int_{-\infty}^{z_i}{\frac{1}{\sqrt{2\pi}}\exp\br{-\frac{t^2}{2}}\,\mathrm{d}t}.
\end{equation}
Its derivatives can be easily calculated,
\begin{equation}
\pderiv{\phi_i}{z_i} = \frac{1}{\sqrt{2\pi}}\exp\br{-\frac{z_i^2}{2}},
\end{equation}
and, given that $\Rop{x_i} = \pderiv{\phi_i}{z_i}\Rop{z_i}$,
\begin{equation}
\Rop{\pderiv{\phi_i}{z_i}} = -z_i\Rop{x_i}.
\end{equation}
Similarly to the logistic function, it has to be the case that $0<\phi_i\br{z_i}<1$, therefore care has to be taken to avoid underflow to $0$ or $1$ for large negative or positive values of $z_i$ respectively.

\subsection{Softmax and log-softmax}
\label{sec:derivs_neural_nets:softmax}
\index{Layer!Softmax}

Unlike the nonlinearities discussed so far, the softmax function does not act elementwise and it is defined as
\begin{equation}
\phi_i\br{z_1,z_2,\ldots,z_I} = \frac{\exp{z_i}}{\sum_{i'}{\exp{z_{i'}}}}.
\end{equation}
Notice that the outputs $x_i$ of a softmax layer form a categorical probability distribution, since $0<x_i<1$ and $\sum_{i}{x_i}=1$, therefore the softmax function is commonly used in the final layer of networks that perform multi-label classification. Given that  $x_i = \phi_i\br{z_1,z_2,\ldots,z_I}$, its derivatives can be easily calculated to be
\begin{equation}
\pderiv{\phi_i}{z_j} = x_i\br{\delta_{ij} - x_j},
\end{equation}
where $\delta_{ij}$ is the Kronecker delta, and
\begin{equation}
\Rop{\pderiv{\phi_i}{z_j}} = \Rop{x_i}\br{\delta_{ij} - x_j} - x_i\Rop{x_j}.
\end{equation}

When using the softmax, care has to be taken so as to prevent the exponential from overflowing or underflowing. One approach is based on the fact that for any scalar $c$ we have
\begin{equation}
\frac{\exp{\br{z_i - c}}}{\sum_{i'}{\exp{\br{z_{i'} - c}}}} = \frac{\exp{z_i}}{\sum_{i'}{\exp{z_{i'}}}}.
\end{equation}
Taking $c = \max\br{z_1,z_2,\ldots,z_I}$ prevents the exponential from overflowing. However, underflowing is still possible. A better approach is to work entirely with the logarithm of the softmax (or log-softmax for short). In this case, we define
\begin{equation}
\phi_i\br{z_1,z_2,\ldots,z_I} = z_i - \log\sum_{i'}{\exp{z_{i'}}}
\end{equation}
whose derivatives are 
\begin{equation}
\pderiv{\phi_i}{z_j} = \delta_{ij} -\exp{x_j}\quad\quad\text{and}\quad\quad\Rop{\pderiv{\phi_i}{z_j}} = -\br{\exp{x_j}}\Rop{x_j}.
\end{equation}
Note that for any scalars $c$ we also have
\begin{equation}
\br{z_i - c} + \log\sum_{i'}{\exp{\br{z_{i'} - c}}} = z_i + \log\sum_{i'}{\exp{z_{i'}}}
\end{equation}
therefore we can prevent both overflow and underflow by taking $c = \max\br{z_1,z_2,\ldots,z_I}$. 

\section{Useful output functions}

In backprop and R\{backprop\}, we assumed that the network's output $\vect{y}$ is mapped to $E\br{\vect{y}}$ by some differentiable function $E$. Typically, $E$ measures how much the network fits or deviates from a target, and the objective is to maximize or minimize it accordingly with respect to the network's adjustable parameters $\bm{\theta}$. As we have seen, backprop and R\{backprop\} need to be initialized with $\pderiv{E}{\vect{y}}$ and $\Rop{\pderiv{E}{\vect{y}}}$ respectively. Below, we discuss some common choices for $E$ and provide the corresponding derivatives.

\subsection{Square error}
\index{Square error}

The square error is typically used for regression problems, and it is defined as
\begin{equation}
E\br{\vect{y}} = \frac{1}{2}\norm{\vect{y} - \vect{t}}^2,
\end{equation}
for some given target vector $\vect{t}$. Its derivatives are
\begin{equation}
\pderiv{E}{\vect{y}} = \vect{y} - \vect{t}
\quad\quad\text{and}\quad\quad
\Rop{\pderiv{E}{\vect{y}}} = \Rop{\vect{y}}.
\end{equation}

\subsection{Cross entropy}
\index{Cross entropy}

Cross entropy is typically used for classification problems. Given a target vector $\vect{t}$, it is defined as
\begin{equation}
E\br{\vect{y}} = \sum_i{t_i\log{y_i}},
\end{equation}
with derivatives
\begin{equation}
\pderiv{E}{y_i} = \frac{t_i}{y_i}
\quad\quad\text{and}\quad\quad
\Rop{\pderiv{E}{y_i}} = -\frac{t_i}{y_i^2}\Rop{y_i}.
\end{equation}
Cross entropy is suitable when both $\vect{t}$ and $\vect{y}$ are discrete probability distributions (e.g.~$\vect{y}$ could be the output of a softmax layer), in which case maximizing $E$ is the same as minimizing $\kl{\vect{t}}{\vect{y}}$.

For binary classification problems, $t$ and $y$ are typically scalars, and represent the probability of one of the two classes. In this case, the following version of cross entropy can be used
\begin{equation}
E\br{y} = t\log{y} + \br{1-t}\log{\br{1-y}}
\end{equation}
whose derivatives are
\begin{equation}
\pderiv{E}{y} = \frac{t}{y}-\frac{1-t}{1-y}
\quad\quad\text{and}\quad\quad
\Rop{\pderiv{E}{y}} = -\br{\frac{t}{y^2}+\frac{1-t}{\br{1-y}^2}}\Rop{y}.
\end{equation}

\subsection{Dot product}
\index{Dot product}

Given a target vector $\vect{t}$, the dot product is simply
\begin{equation}
E\br{\vect{y}} = \vect{t}^T\vect{y}
\end{equation}
with derivatives
\begin{equation}
\pderiv{E}{\vect{y}} = \vect{t}
\quad\quad\text{and}\quad\quad
\Rop{\pderiv{E}{\vect{y}}} = \vect{0}.
\end{equation}
Note that using the softmax layer in combination with cross entropy is equivalent to using the log-softmax layer in combination with the dot product.

\chapter{Derivatives in NADE}
\label{chapter:derivatives_in_nade}
\index{NADE}

NADE is a generative model, parameterized by a set of parameters $\bm{\theta}$, which takes a binary input vector $\vect{x}$ and outputs its log probability
\begin{equation}
L = \log{\prob{\vect{x}\g\bm{\theta}}}.
\end{equation}
An overview of NADE has been given in section~\ref{sec:generative_models:NADE}. In this appendix, we discuss in more detail forward and backward propagation in NADE, with a focus on how to efficiently calculate the first and second derivatives of $L$.

\section{Computing the output: forward propagation}

NADE is essentially a feedforward neural network with a single hidden layer and a special autoregressive structure. As such, its output $L$ is given by a procedure similar to forward propagation in neural networks. Assuming NADE has $I$ input units, $I$ output units and $J$ hidden units, forward propagation in NADE is described below.
\begin{align}
a_{1j} &= c_j \\
a_{ij} &= a_{i-1,j} + w_{i-1,j}x_{i-1}\quad\text{(for $i>1$)} \\
h_{ij} &= \sigm{a_{ij}} \\
z_i &= \sum_j{u_{ij}h_{ij}} + b_i \\
\hat{x}_i &= \sigm{z_i} \\
L &= \sum_i\br{x_i\log{\hat{x_i}} + \br{1-x_i}\log{\br{1-\hat{x_i}}}}
\end{align}
In the above equations, $u_{ij}$, $b_i$, $w_{ij}$ and $c_j$ are the adjustable parameters of NADE, which are collectively denoted as $\bm{\theta}$. Also, $a_{ij}$ and $h_{ij}$ are the hidden layer activations and hidden layer states respectively, and $z_i$ are the output layer activations. Our notation is that $i$ indexes input and output units, and $j$ indexes hidden units. Note that the number of input/output units $I$ and the number of hidden units $J$ do not belong to the set of adjustable parameters, but are part of NADE's architecture.

\section{Computing first derivatives: backward propagation}
\index{Backward propagation}

Computing the derivatives of NADE with respect to the parameters $\pderiv{L}{\bm{\theta}}$ and with respect to the input $\pderiv{L}{\vect{x}}$ can easily be done by applying the chain rule of calculus to the equations describing forward propagation. This yields a procedure similar to backward propagation in neural networks, which for NADE is as follows.
\begin{align}
\pderiv{L}{b_i} &= x_i - \hat{x}_i \\
\pderiv{L}{u_{ij}} &= \pderiv{L}{b_i} h_{ij} \\
\pderiv{L}{a_{ij}} &= \pderiv{L}{u_{ij}}u_{ij}\br{1-h_{ij}} \\
\pderiv{L}{c_j} &= \sum_i{\pderiv{L}{a_{ij}}} \label{eq:deriv_nade:dLdc}\\
\pderiv{L}{w_{ij}} &= \br{\sum_{i'>i}{\pderiv{L}{a_{i'j}}}}x_i \label{eq:deriv_nade:dLdw}\\
\pderiv{L}{x_i} &= z_i + \sum_j{\br{\sum_{i'>i}{\pderiv{L}{a_{i'j}}}}w_{ij}} \quad\text{(for $i<I$)} \\
\pderiv{L}{x_I} &= z_I
\end{align}
Before running backward propagation, we assume that forward propagation has already been performed, so that quantities $h_{ij}$, $z_i$ and $\hat{x}_i$ are available. Notice that the cumulative sum $\sum_{i'>i}{\pderiv{L}{a_{i'j}}}$ can be computed efficiently, by using the recursion
\begin{equation}
\sum_{i'>i}{\pderiv{L}{a_{i'j}}} = \pderiv{L}{a_{i+1,j}} + \sum_{i'>i+1}{\pderiv{L}{a_{i'j}}}.
\end{equation}
Note also that, despite $\vect{x}$ being actually binary, there is nothing in the maths stopping us from pretending it is continuous and computing derivatives with respect to it. Finally, if only the derivatives with respect to $\vect{x}$ are needed, then Equations~\eqref{eq:deriv_nade:dLdc} and \eqref{eq:deriv_nade:dLdw} can be skipped.

\section{Computing second derivatives: R\{backprop\}}
\index{R\{backprop\}}

Let $\mat{H}$ be the Hessian of $L$ with respect to the parameters and the inputs and $\vect{v}$  be some vector of appropriate size. Even though computing $\mat{H}$ explicitly is in principle possible by applying the chain rule of calculus to the procedure jointly defined by forward and backward propagation, in practice what is usually needed is only the Hessian-vector product $\mat{H}\vect{v}$. The latter can be computed in linear time using the R technique (see appendix~\ref{chapter:R_technique}). Compare this to calculating $\mat{H}$ explicitly, which takes quadratic time and space.

The R technique is based on applying the $\Rop{\cdot}$ operator to the equations describing forward and backward propagation, yielding a procedure that we shall call R\{backprop\}. The forward part of R\{backprop\} is as follows.
\begin{align}
\Rop{a_{1j}} &= \Rop{c_j} \\
\Rop{a_{ij}} &= \Rop{a_{i-1,j}} + \Rop{w_{i-1,j}}x_{i-1}+ w_{i-1,j}\Rop{x_{i-1}}\quad\text{(for $i>1$)} \\
\Rop{h_{ij}} &= h_{ij}\br{1-h_{ij}}\Rop{a_{ij}} \\
\Rop{z_i} &= \sum_j\br{\Rop{u_{ij}}h_{ij} + u_{ij}\Rop{h_{ij}}} + \Rop{b_i} \\
\Rop{\hat{x}_i} &= \hat{x}_i\br{1-\hat{x}_i}\Rop{z_i}
\end{align}
In the above, $\Rop{u_{ij}}$, $\Rop{b_i}$, $\Rop{w_{ij}}$, $\Rop{c_j}$ and  $\Rop{x_i}$ are simply the corresponding elements of $\vect{v}$. We assume that forward propagation has been performed in advance, so that quantities $h_{ij}$ and $\hat{x}_i$ are available.

Applying the $\Rop{\cdot}$ operator to the equations describing backward propagation yields the backward phase of R\{backprop\}, which is as follows.
\begin{align}
\Rop{\pderiv{L}{b_i}} &= \Rop{x_i} - \Rop{\hat{x}_i} \\
\Rop{\pderiv{L}{u_{ij}}} &= \Rop{\pderiv{L}{b_i}} h_{ij} + \pderiv{L}{b_i} \Rop{h_{ij}}\\
\Rop{\pderiv{L}{a_{ij}}} &= \Rop{\pderiv{L}{u_{ij}}}u_{ij}\br{1-h_{ij}} + \pderiv{L}{u_{ij}}\Rop{u_{ij}}\br{1-h_{ij}} - \pderiv{L}{u_{ij}}u_{ij}\Rop{h_{ij}}\\
\Rop{\pderiv{L}{c_j}} &= \sum_i{\Rop{\pderiv{L}{a_{ij}}}} \\
\Rop{\pderiv{L}{w_{ij}}} &= \br{\sum_{i'>i}{\Rop{\pderiv{L}{a_{i'j}}}}}x_i + \br{\sum_{i'>i}{\pderiv{L}{a_{i'j}}}}\Rop{x_i}\\
\Rop{\pderiv{L}{x_i}} &= \Rop{z_i} + \sum_j\br{\sum_{i'>i}{\Rop{\pderiv{L}{a_{i'j}}}}}w_{ij} + \sum_j\br{\sum_{i'>i}{\pderiv{L}{a_{i'j}}}}\Rop{w_{ij}} \quad\text{(for $i<I$)} \\
\Rop{\pderiv{L}{x_I}} &= \Rop{z_I}
\end{align}
In the above, $h_{ij}$ is available from a previous run of forward propagation, $\pderiv{L}{b_i}$, $\pderiv{L}{u_{ij}}$ and $\pderiv{L}{a_{ij}}$ are available from a previous run of backward propagation, and $\Rop{\hat{x}_i}$, $\Rop{h_{ij}}$ and $\Rop{z_i}$ are available from the forward phase of R\{backprop\}.

R\{backprop\} can be significantly simplified if only part of the Hessian needs to be calculated. To see that, let $\vect{v} = \begin{bmatrix}
\vect{v}_{\bm{\theta}} & \vect{v}_{\vect{x}}
\end{bmatrix}^T$, where $\vect{v}_{\bm{\theta}}$ is the part of $\vect{v}$ corresponding to $\bm{\theta}$ and $\vect{v}_{\vect{x}}$ is the part of $\vect{v}$ corresponding to $\vect{x}$. The Hessian-vector product can then be written as
\begin{equation}
\mat{H}\vect{v} = \begin{bmatrix}
\spderiva{L}{\bm{\theta}}\vect{v}_{\bm{\theta}} + \spderivb{L}{\bm{\theta}}{\vect{x}}\vect{v}_{\vect{x}}\\[5pt]
\spderivb{L}{\vect{x}}{\bm{\theta}}\vect{v}_{\bm{\theta}} + \spderiva{L}{\vect{x}}\vect{v}_{\vect{x}}
\end{bmatrix}.
\end{equation}
If, for instance, only $\spderivb{L}{\bm{\theta}}{\vect{x}}\vect{v}_{\vect{x}}$ needs to be calculated, we can simply set $\vect{v}_{\bm{\theta}} = \vect{0}$.  Hence, all $\Rop{u_{ij}}$, $\Rop{b_i}$, $\Rop{w_{ij}}$ and $\Rop{c_j}$ will become $0$, which simplifies the calculations significantly.

\chapter{Batch and Online EM for MoG}
\label{chapter:em}
\index{EM}

Suppose we have a dataset $\set{\vect{x}_n}$, to which we want to fit a MoG using maximum likelihood. The MoG\index{MoG} is given by
\begin{equation}
\prob{\vect{x}\g \bm{\theta}} = \sum_i\pi_i\,\gaussian{\vect{x}}{\vect{m}_i}{\mat{S}_i}.
\end{equation}
The MoG's adjustable parameters are $\bm{\theta} = \set{\pi_i, \vect{m}_i, \mat{S}_i}$. Fitting $\prob{\vect{x}\g \bm{\theta}}$ to the dataset using maximum likelihood amounts to solving the following optimization problem
\begin{equation}
\min_{\bm{\theta}}{\frac{1}{N}\sum_n{\log{\prob{\vect{x}_n\g \bm{\theta}}}}}.
\end{equation}
One way to solve the above problem is via the \emph{EM algorithm}. In this appendix, we describe two versions of the EM algorithm for MoG, the classic batch version, and an online version.

\section{Batch EM}
\label{sec:em:batch}
\index{EM!Batch EM}

Let $z_{in}$ be a binary random variable with the meaning ``datapoint $\vect{x}_n$ was generated by the $\nth{i}$ component''. The batch EM algorithm is an iterative procedure which alternates between two phases: the \emph{expectation} phase, where it computes the expectation of each $z_{in}$ using the current value of $\bm{\theta}$, and the \emph{maximization} phase, where it updates $\bm{\theta}$ using the previously computed expectations of $z_{in}$. 

More precisely, in the expectation phase, the algorithm computes $\avgnull{z_{in}}$ for all $i$ and $n$ as follows
\begin{equation}
\avgnull{z_{in}} = \frac{\pi_i\,\gaussian{\vect{x}_n}{\vect{m}_i}{\mat{S}_i}}{\sum_{i'}\pi_{i'}\,\gaussian{\vect{x}_n}{\vect{m}_{i'}}{\mat{S}_{i'}}}.
\end{equation}
Note that the above expectation is taken with respect to the posterior of $z_{in}$ given $\vect{x}_n$ and $\bm{\theta}$ (not indicated to simplify notation).
In the maximization phase, using the previously computed expectations of $z_{in}$, the algorithm computes for each $i$ the following three statistics
\begin{align}
\Phi_{i1} &= \sum_n{\avgnull{z_{in}}} \\
\Phi_{i2} &= \sum_n{\avgnull{z_{in}}\vect{x}_n} \\
\Phi_{i3} &= \sum_n{\avgnull{z_{in}}\vect{x}_n\vect{x}_n^T}
\end{align}
and then it updates $\pi_i$, $\vect{m}_i$ and $\mat{S}_i$ as follows
\begin{align}
\pi_i &= \frac{1}{N}\Phi_{i1} \label{eq:em:batch:pi} \\
\vect{m}_i &= \frac{\Phi_{i2}}{\Phi_{i1}} \label{eq:em:batch:m} \\
\mat{S}_i &= \frac{\Phi_{i3}}{\Phi_{i1}} -\vect{m}_i\vect{m}_i^T. \label{eq:em:batch:S}
\end{align}
It can be shown \citep[section 9.4]{Bishop:2006:PRML} that each iteration of the above procedure can only increase $\frac{1}{N}\sum_n{\log{\prob{\vect{x}_n\g \bm{\theta}}}}$ or leave it the same. Hence, by repeating the above until convergence, a local maximum of the average log likelihood can be found.

\section{Online EM}
\label{sec:em:online}
\index{EM!Online EM}

The downside of batch EM is that it needs to scan the whole dataset in every iteration. For large or possibly infinite datasets, this becomes problematic. Online EM, on the other hand, uses only a minibatch of points from the dataset in each iteration, which can be made as small as a single point. There are more than one flavours of online EM \citep{Liang:2009:online_em}. In this thesis, we make use of the version by \citet{Cappe:2009:online_em}, which is exemplified below for the MoG case.

Suppose that $\set{\vect{x}_m}$ is a minibatch of $M$ points from the dataset selected at random. Note that $M$ can be as small as $1$, however larger values are typically used for better stability. For each $m$ and $i$, online EM computes the expectation of the associated variable $z_{im}$ as follows
\begin{equation}
\avgnull{z_{im}} = \frac{\pi_i\,\gaussian{\vect{x}_m}{\vect{m}_i}{\mat{S}_i}}{\sum_{i'}\pi_{i'}\,\gaussian{\vect{x}_m}{\vect{m}_{i'}}{\mat{S}_{i'}}}
\end{equation}
and then it proceeds to calculate the values of the sufficient statistics as follows
\begin{align}
\Phi_{i1}' &= \sum_m{\avgnull{z_{im}}} \\
\Phi_{i2}' &= \sum_m{\avgnull{z_{im}}\vect{x}_m} \\
\Phi_{i3}' &= \sum_m{\avgnull{z_{im}}\vect{x}_m\vect{x}_m^T}.
\end{align}
The above two steps are identical to the corresponding steps in batch EM, except that instead of using the whole dataset, only a minibatch is used. 

The main difference between batch an online EM is in the update of the parameters. The previously calculated minibatch statistics cannot be used like this, since they are only a stochastic approximation of the true statistics. Instead, a running average of statistics is maintained, which is updated with the minibatch statistics as follows
\begin{equation}
\begin{pmatrix}
\Phi_{i1} \\
\Phi_{i2} \\
\Phi_{i3} 
\end{pmatrix}\leftarrow\br{1-\alpha}\begin{pmatrix}
\Phi_{i1} \\
\Phi_{i2} \\
\Phi_{i3} 
\end{pmatrix} +  \alpha\begin{pmatrix}
\Phi_{i1}' \\
\Phi_{i2}' \\
\Phi_{i3}' 
\end{pmatrix}.
\end{equation}
Here, $0<\alpha<1$ is the learning rate. The lower the value of $\alpha$ is, the more biased towards the previous value of the average statistics the algorithm is. Using these average statistics, the parameters $\pi_i$, $\vect{m}_i$ and $\mat{S}_i$ can be updated as in Equations~\eqref{eq:em:batch:pi}--\eqref{eq:em:batch:S}.

The update schedule of the learning rate is highly important for the stability and convergence of the algorithm. In general, the learning rate should be high in the beginning, and it should decay towards zero during training. Let $\alpha_t$ be the learning rate in the $\nth{t}$ iteration. Following standard results in the stochastic approximation literature \citep{Robbins:1951:stochastic_approx}, a learning rate schedule satisfying
\begin{equation}
\sum_t{\alpha_t} = \infty
\quad\text{and}\quad
\sum_t{\alpha_t^2} < \infty
\end{equation}
is sufficient to guarantee convergence. For instance, a schedule that satisfies the above is the following
\begin{equation}
\alpha_t = \frac{1}{\br{t_0 + t}^\beta}
\quad\text{where}\quad
t_0 \ge 0
\quad\text{and}\quad
\frac{1}{2} < \beta \le 1.
\end{equation}
In practice however, other schedules are sometimes used, such as linear or logarithmic decay, that do not necessarily guarantee convergence but work well in practice.

\chapter{Notations and Symbols}
\label{chapter:notations}

\renewcommand{\tabcolsep}{0.8cm}
\renewcommand{\arraystretch}{\arrstretchvalue}
\begin{tabular}{@{}ll}
$x,y,z,X,Y,Z$ & Scalars (lowercase or uppercase letters) \\
$\vect{x},\vect{y},\vect{z}$ & Vectors (lowercase bold letters) \\
$\mat{X},\mat{Y},\mat{Z}$ & Matrices (uppercase bold letters) \\
$\vect{x}^T,\mat{X}^T$ & Transpose of vector $\vect{x}$ and matrix $\mat{X}$ \\
$\mat{I}$ & The identity matrix of appropriate size \\
$\set{\ldots}$ & A set, depending on context \\
$\log{\br{\cdot}}$ & The natural logarithm (base $e$) \\
$\norm{\cdot}$ & $\ell_2$ vector norm \\
$\pderiv{f}{\vect{x}}$ & Derivative of function $f$ with respect to $\vect{x}$ \\
$\sigm{\cdot}$ & Logistic sigmoid function \\
$\deltaf{\cdot}$ & Delta function \\
$\Rop{\cdot}$ & R operator for Hessian-vector products \\
$\avg{\cdot}{\prob{\vect{x}}}$ & Expectation with respect to distribution $\prob{\vect{x}}$ \\
$\gaussian{\vect{x}}{\vect{m}}{\mat{S}}$ & Gaussian distribution over $\vect{x}$ with mean $\vect{m}$ and covariance $\mat{S}$ \\
$\kl{p}{q}$ & KL divergence from distribution $p$ to distribution $q$ \\
$\bigo{\cdot}$ & Asymptotic big O notation \\
\end{tabular}

\chapter{List of Abbreviations}
\label{chapter:abbreviations}

\renewcommand{\tabcolsep}{0.8cm}
\renewcommand{\arraystretch}{\arrstretchvalue}
\begin{tabular}{@{}ll}
EM & Expectation-Maximization \\
EP & Expectation Propagation \\
GPU & Graphics Processing Unit \\
KL & Kullback-Leibler \\
MCMC & Markov Chain Monte Carlo \\
MNIST & Mixed National Institute of Standards and Technology \\
MoG & Mixture of Gaussians \\
NADE & Neural Autoregressive Distribution Estimator \\
pdf & probability density function \\
RBM & Restricted Boltzmann Machine \\
ReLU & Rectified Linear Unit \\
\end{tabular}

\bibliography{final_report}
\clearpage
\phantomsection
\printindex

\end{document}